% This is a template for Ph.D. dissertations in the UCI format.
% 
% All fonts, including those for sub- and superscripts, must be 10
% points or larger.  Recommended sizes are 14-point for chapter
% headings, 12-point for the main body of text and figure/table
% titles, and 10-point for footnotes, sub- and super-scripts, and text
% in figures and tables.
%
% Notes: Add short title to figures, sections, via square brackets,
% e.g. \section[short]{long}.
%
\documentclass[12pt,fleqn]{ucithesis}

%% Online Tensor Decomposition Paper
% A few common packages
\usepackage{amsmath}
\usepackage{amsthm}
\usepackage{array}
\usepackage{graphicx,epsf,pgfplots,wrapfig}
\usepackage{relsize}
\usepackage{url}
\usepackage{psfrag}
\usepackage{color}
\usepackage{xcolor}
\usepackage{amsopn,amsfonts,amsmath,amssymb,amsbsy,amsthm,mathtools,dsfont,units}

% Some other useful packages
\usepackage{caption,float,sidecap}
\usepackage{subfig}%[caption=false]
\usepackage{multirow}
\usepackage{tabularx}
\usepackage{booktabs}
\usepackage{lipsum}
\usepackage{algorithm}
\usepackage[noend]{algorithmic}
\usepackage{tablefootnote}
\usepackage{paralist,enumitem,cases}
\usepackage{siunitx}

% plainpages=false fixes the "duplicate ignored" error with page counters
% Set pdfborder to 0 0 0 to disable colored borders around PDF hyperlinks
\PassOptionsToPackage{hyphens}{url}
\usepackage[plainpages=false,pdfborder={0 0 0}]{hyperref}
\usepackage{tikz}
\usetikzlibrary{fit}
\usetikzlibrary{calc,shapes}
\usetikzlibrary{decorations.pathmorphing} % noisy shapes
\usetikzlibrary{fit}					% fitting shapes to coordinates
\usetikzlibrary{backgrounds}
\usetikzlibrary{calc,shapes}
%\usepackage[pass]{geometry}
%\usepackage{subcaption}  % \begin{subfigure}...\end{subfigure} within figure

% Uncomment the following two lines to use the algorithm package,
% which provides an algorithm environment similar to figure and table
% ("\begin{algorithm}...\end{algorithm}"). A list of algorithms will
% automatically be added in the preliminary pages. Note that you
% probably want a package for the actual code to go with this (e.g.,
% algorithmic).
%\usepackage{algorithm}
\renewcommand{\listalgorithmname}{\protect\centering\protect\Large LIST OF ALGORITHMS}

% Uncomment the following line to enable Unicode support. This will allow you
% to enter non-ASCII characters (such as accented characters) directly without
% having to use LaTeX's awkward escape syntax (e.g., \'{e})
% NOTE: You may have to install the ucs.sty package for this to work. See:
% http://www.unruh.de/DniQ/latex/unicode/
%\usepackage[utf8x]{inputenc}

% Uncomment the following to avoid "widowing", where page breaks cause
% single lines of paragraphs to float onto the next page (this is not
% a UCI requirement but more of an aesthetic choice).
\widowpenalty=10000
\clubpenalty=10000

% Define path
%\def\fighomeCommunity{/Users/furongh/Documents/BITBUCKET/community-implementation-final/arxiv/JMLR_niranjan}
%\def\fighomeSaddle{/Users/furongh/Documents/BITBUCKET/firstordersgdwnoise/Ge15/figures}
%\def\fighomeConv{/Users/furongh/Documents/BITBUCKET/tensorcov/figures}
%\def\fighomeLT{/Users/furongh/Documents/BITBUCKET/latent-tree-parallel/Arxiv15}
%\def\fighomeLTfigure{/Users/furongh/Documents/BITBUCKET/latent-tree-parallel/figures}
%\def\fighomeBrain{/Users/furongh/Documents/BITBUCKET/brainproject/writeup/figures}
%\def\fighomeTalk{/Users/furongh/Documents/BITBUCKET/researchtalk/figures}
\def\fighomeCommunity{.}
\def\fighomeSaddle{.}
\def\fighomeConv{.}
\def\fighomeLT{.}

\def\fighomeBrain{.}
\def\fighomeTalk{.}
% Macros for title, author, abstract, etc.

%% Common
% Modify or extend these at will.
\newtheorem{theorem}{\textsc{Theorem}}[chapter]
\newtheorem{definition}{\textsc{Definition}}[chapter]

 \newtheorem{lemma}{\textsc{Lemma}}[chapter]

\newtheorem{property}{\textsc{Property}}[chapter]

\newtheorem*{remark}{Remark}

%algorithm
\floatname{algorithm}{Procedure}

%proof
\newenvironment{myproof}{\noindent{\em Proof:} \hspace*{1em}}{
    \hspace*{\fill} $\Box$ }
\newenvironment{proof_of}[1]{\noindent {\em Proof of #1: }}{\hspace*{\fill} $\Box$ }

\newcommand{\bprf}{\begin{myproof}}
\newcommand{\eprf}{\end{myproof}}
% insert figure, equation, cases, paragraph
\renewcommand{\paragraph}[1]{\textit{#1}}
\newcommand{\bp}{\begin{psfrags}}
\newcommand{\ep}{\end{psfrags}}
\newcommand{\bc}{\begin{center}}
\newcommand{\ec}{\end{center}}

\def\beq{\begin{equation}}
\def\eeq{\end{equation}\noindent}
\def\beqn{\begin{eqnarray}}
\def\eeqn{\end{eqnarray} \noindent}
\def\beqnn{\begin{eqnarray*}}
\def\eeqnn{\end{eqnarray*}  \noindent}
\def\bcase{\begin{numcases}}
\def\ecase{\end{numcases}   \noindent}
% Math operator

\def\nn{\nonumber}
\newcommand{\llvert}{\left\vert\vphantom{\frac{1}{1}}\right.}

\DeclareMathOperator{\diag}{diag}
\DeclareMathOperator\Diag{Diag}
\DeclareMathOperator{\degree}{Degree}

\def\Ebb{\mathbb{E}}
\newcommand\E{\mathbb{E}}

\newcommand*{\hermconj}{^{\mathsf{H}}}
\def\bfI{{\mathbf I}}

\def\Pbb{{\mathbb P}}
\def\bfU{{\mathbf U}}
\def\Rbb{\mathbb{R}}
\newcommand\R{\mathbb{R}}
\def\nx{n_{\mbox{\tiny\itshape \!X}}}
\def\viz{{viz.,\ \/}}
\def\tha{{\mbox{\tiny th}}}

\DeclareMathOperator{\poly}{poly}

\newcommand{\Conv}{\mathop{\scalebox{1.5}{\raisebox{-0.2ex}{$\ast$}}}}%
\newcommand{\Toep}{\mathsf{Cir}}
\newcommand{\CToep}{{\mathcal{F}}}
\newcommand{\CCir}{{\mathcal{F}}}
\newcommand{\DFT}{\mathsf{FFT}}

\newcommand\inner[1]{\ensuremath{\left<#1\right>}}%%%%%%%%!!!!!!!!!!!!!

% comments

\def\nn{\nonumber}

\def\tcr{\textcolor{red}}
\def\tcb{\textcolor{blue}}

\definecolor{dkr}{rgb}{0.6,0.2,0.2}
\definecolor{dkg}{rgb}{0,0.5,0}
\definecolor{dkb}{rgb}{0.0,0.1,0.7}
\def\tcr{\textcolor{red}}
\def\tcb{\textcolor{blue}}
\def\tcdkb{\textcolor{dkb}}

\def\tcdkr{\textcolor{dkr}}

 %% Paper Specific 
 % convolution paper
 \def\ourframework{{$\textsf{ConvDic}$$+$$\textsf{DeconvDec}$ }}

\def\ourdecode{$\textsf{DeconvDec}$}
\newcommand\mysvdeq{\stackrel{\mathclap{\scriptsize\mbox{svd}}}{=}}

\newcommand{\ouralgorithm}{\mathsf{CT}}

\def\bfS{{\mathbf S}}
\def\bfC{{\mathbf Y}}

\newcommand{\block}{blk}
\newcommand{\blkdiag}{Blkdiag}
\newcommand{\flatten}{unfold}

\newcommand{\opt}{opt}

\newcommand{\Fbb}{\mathbb{F}}
\newcommand{\Cum}{C_3}
\newcommand{\newCum}{M}
\newcommand{\modeA}{\mathcal{F}}
\newcommand{\modeB}{\mathcal{G}}
\newcommand{\modeC}{\mathcal{H}}

% online tensor paper
\DeclareMathOperator{\Ber}{Ber}

\DeclareMathOperator{\community}{Com}
\newcommand\Dir{\operatorname{Dir}}
\DeclareMathOperator{\Pairs}{Pairs}
\newcommand\Poi{\operatorname{Poi}}

\DeclareMathOperator{\topic}{Top}

\newcommand{\Pvalue}{\mathsf{P_{\text{val}}}}
\newcommand{\Pvaluem}{\mathsf{\mathbf{P}_{\text{val}}}}

% saddle point paper

\newcommand{\N}{\mathbb{N}}

\newcommand\supp{\operatorname{supp}}

\newcommand{\name}{{strict saddle}}
\newcommand{\bigname}{{Strict saddle}}

\newcommand{\Hess}{\nabla^2}
\DeclareMathOperator{\Tr}{Tr}
\newcommand\bigO{O}
\newcommand\tlO{\tilde{\bigO}}
\newcommand{\tlOmega}{\tilde{\Omega}}
\newcommand{\nameCQ}{ $\alpha_c$-RLICQ }
\newcommand{\BigC}{\mathfrak{S}}

\newtheorem*{claim}{Remark}

% latent tree
\newcommand{\Nb}{\text{nbd}}
\newcommand{\dist}{\text{dist}}
\DeclareMathOperator{\Adj}{\mathcal{N}}
\newcommand{\RG}{\text{RG}}

% brain
\newcommand{\image}{I}
\newcommand{\filter}{F}
\newcommand{\map}{M}
\newcommand{\fro}{\mathsf{F}}

% unknown

\def\tree{\mathcal{T}}

\thesistitle{Discovery of Latent Factors in High-dimensional Data Using Tensor Methods}

\degreename{Doctor of Philosophy}

% Use the wording given in the official list of degrees awarded by UCI:
% http://www.rgs.uci.edu/grad/academic/degrees_offered.htm
\degreefield{Electrical and Computer Engineering}

% Your name as it appears on official UCI records.
\authorname{Furong Huang}

% Use the full name of each committee member.
\committeechair{Assistant Professor Animashree Anandkumar}
\othercommitteemembers
{
  Professor Carter Butts\\
  Associate Professor Athina Markopoulou
}

\degreeyear{2016}

\copyrightdeclaration
{
  {\copyright} {\Degreeyear} \Authorname
}

% If you have previously published parts of your manuscript, you must list the
% copyright holders; see Section 3.2 of the UCI Thesis and Dissertation Manual.
% Otherwise, this section may be omitted.
 \prepublishedcopyrightdeclaration
 {
	All materials {\copyright} {\Degreeyear} \Authorname
 }

% The dedication page is optional.
\dedications
{
  % (Optional dedication page)
  
  To Jinsong Huang and Shaoyun Liu
}

\acknowledgments
{
First and foremost I want to thank my advisor Animashree Anandkumar, who has been my role model as a successful female professor in machine learning.  It has been an honor to be her first Ph.D. student.  I appreciate all the efforts she put to help build my confidence, guide me through my early research career, and make my graduate study experience productive and stimulating. Her endless enthusiasm for research has been contagious and a source of motivation. She has also continually and convincingly conveyed a spirit of adventure with regard to research and scholarship. Anima is not only a role model, a career guide but also a friend who shares life experience and offers excellent advice. I couldn't have fought through the tough times in my Ph.D. pursuit without her inspiration or support.

During my graduate studies, I have been lucky to have collaborated with some smart and innovative minds who inspired me profoundly. My collaborator Rong Ge has impressed me by his enthusiasm, intensity and incredible ability to disentangle complicated research problems. I would also like to acknowledge Chi Jin and Yang Yuan for always being available for discussions and brainstorming. I am especially grateful for working with Srini Turaga and Ernest Fraenkel. They provided comments and advice from fresh angles and stimulated me to think differently. 
I appreciate insightful and sparkling discussions with  Sham Kakade, Daniel Hsu, David Mimno, David Blei, Qirong Ho, Alex Smola, Paul Mineiro, Nikos Karampatziakis and others.  During my internship in Microsoft Research New England, I have met the most wonderful mentors Jennifer Chayes and Christian Borgs, whose support has powered me to chase my academic dreams.

 I would like to thank my committee members, Professor Athina Markopoulou, and Professor Carter Butts, who are always there for me whenever I need advice. In addition, a thank you to Professor Max Welling and Professor Alexander Ihler, who introduced me to machine learning, and stimulated my long lasting enthusiasm for machine learning. I also appreciate the efforts of Professor Padhraic Smyth, who started the Data Science Initiative, a growing interdisciplinary machine learning community,  in UC Irvine.

The members of the MEGADatA group, 
Majid Janzamin, Hanie Sedghi, Niranjan UN, Forough Arabshahi, Yang Shi, Kamyar Azizzade, and Saeed Karimi Bidhendi, 
have brought immense amount of joy to my personal and
professional time at UC Irvine. I am grateful for the nights we spent working on paper deadlines, as well as the fun times we had wearing bean sprout hair clips  in the lab and posing for group profile pictures. The group has been a source of friendships and collaborations.

I thank MIT Press for permission to include Chapter 2 of my thesis, which was originally published in Conference of Learning Theory.  And I thank MIT Press for permission to include Chapter 3 and 4 of my thesis, which was originally published in Journal of Machine Learning.  I gratefully acknowledge the funding sources that made my Ph.D. work possible. I
was funded by the EECS Department fellowship. My work was also supported by the National Science Foundation BIGDATA award.

Lastly, I would like to thank my family for all their unconditional love and faithful support. 
Thank my parents, Jinsong Huang and Shaoyun Liu, for raising me with hard-working spirit and a love of science.  Wenchao Xi, thank you for always being by my side, sharing joy and sorrow,  in the years of adventure. 

     }

% Some custom commands for your list of publications and software.
\newcommand{\mypubentry}[3]{
  \begin{tabular*}{1\textwidth}{@{\extracolsep{\fill}}p{5.6 in}r}
    \textbf{#1} & \textbf{#2} \\ 
    \multicolumn{2}{@{\extracolsep{\fill}}p{.95\textwidth}}{\emph{#3}}\vspace{6pt} \\
  \end{tabular*}
}
\newcommand{\mysoftentry}[3]{
  \begin{tabular*}{1\textwidth}{@{\extracolsep{\fill}}lr}
    %\textbf{#1} & \url{#2} \\
    \textbf{#1} & \textcolor{blue}{#2} \\
    \multicolumn{2}{@{\extracolsep{\fill}}p{.95\textwidth}}
    {\emph{#3}}\vspace{-6pt} \\
  \end{tabular*}
}

% Include, at minimum, a listing of your degrees and educational
% achievements with dates and the school where the degrees were
% earned. This should include the degree currently being
% attained. Other than that it's mostly up to you what to include here
% and how to format it, below is just an example.
\curriculumvitae
{

\textbf{EDUCATION}
  
  \begin{tabular*}{1\textwidth}{@{\extracolsep{\fill}}lr}
    \textbf{Doctor of Philosophy in ECE} & \textbf{2016} \\
    \vspace{6pt}
    University of California Irvine & \emph{Irvine, CA, USA} \\
     \textbf{Master of Science in ECE} & \textbf{2012} \\
    \vspace{6pt}
    University of California Irvine & \emph{Irvine, CA, USA} \\

    \textbf{Bachelor of Science in EECS} & \textbf{2010} \\
    \vspace{6pt}
   Zhejiang University & \emph{Hangzhou, Zhejiang, China} \\
  \end{tabular*}

\vspace{12pt}
\textbf{RESEARCH EXPERIENCE}

  \begin{tabular*}{1\textwidth}{@{\extracolsep{\fill}}lr}
    \textbf{Graduate Research Assistant} & \textbf{2010--2016} \\
    \vspace{6pt}
    University of California Irvine & \emph{Irvine, California} \\
  \end{tabular*}
  
  \begin{tabular*}{1\textwidth}{@{\extracolsep{\fill}}lr}
    \textbf{Research Intern} & \textbf{2014.3--2014.5} \\
    \vspace{6pt}
    Microsoft Research & \emph{Redmond, Washington} \\
  \end{tabular*}
  
   \begin{tabular*}{1\textwidth}{@{\extracolsep{\fill}}lr}
    \textbf{Research Intern} & \textbf{2014.6--2014.12} \\
    \vspace{6pt}
    Microsoft Research New England & \emph{Cambridge, Massachusetts} \\
  \end{tabular*}
\vspace{12pt}

%\textbf{TEACHING EXPERIENCE}
%
%  \begin{tabular*}{1\textwidth}{@{\extracolsep{\fill}}lr}
%    \textbf{Teaching Assistant} & \textbf{2009--2010} \\
%    \vspace{6pt}
%    University name & \emph{City, State} \\
%  \end{tabular*}

%\pagebreak

\textbf{REFEREED JOURNAL PUBLICATIONS}

  \mypubentry{F. Huang, U.N. Niranjan, M.U. Hakeem and A. Anandkumar, ``Online Tensor Methods for Learning Latent Variable Models''}{2014}{Journal of Machine Learning}

 \mypubentry{A. Anandkumar, V.Y.F Tan, F. Huang and A.S. Willsky, ``High-Dimensional Structure Learning of Ising Models: Local Separation Criterion''}{2012}{Annals of Statistics}

 \mypubentry{A. Anandkumar, V.Y.F Tan, F. Huang and A.S. Willsky, ``High-Dimensional Gaussian Graphical Model Selection: Walk-Summability and Local Separation Criterion''}{2012}{Journal of Machine Learning}

\vspace{12pt}
\textbf{REFEREED CONFERENCE PUBLICATIONS}
    
    \mypubentry{F. Huang, A. Anandkumar, C. Borgs, J. Chayes, E. Fraenkel, M. Hawrylycz, E. Lein, A. Ingrosso, S. Turaga, ``Discovering Neuronal Cell Types and Their Gene Expression Profiles Using a Spatial Point Process Mixture Model''}{2015}{NIPS BigNeuro workshop 2015}

\mypubentry{F. Huang, U.N. Niranjan, J. Perros, R. Chen, J. Sun, A. Anandkumar,``Scalable Latent Tree Model and its Application to Health Analytics''}{2015}{NIPS 2015 Workshop on Machine Learning in Healthcare}

      \mypubentry{F. Huang, A. Anandkumar, ``Convolutional Dictionary Learning through Tensor Factorization''}{2015}{JMLR conference and workshop proceedings}
      
       \mypubentry{F. Arabshahi, F. Huang, A. Anandkumar, C. Butts, ``Are you going to the party: depends, who else is coming? --Learning hidden group dynamics via conditional latent tree models''}{2015}{2015 IEEE International Conference on Data Mining (ICDM)}

          \mypubentry{F. Huang, S. Matusevych, A.Anandkumar, N. Karampatziakism and P. Mineiro, ``Distributed Latent Dirichlet Allocation via Tensor Factorization''}{2014}{NIPS Optimization for Machine Learning workshop}

    \mypubentry{A. Anandkumar, D. Hsu, F. Huang and S.M. Kakade, ``Learning High-Dimensional Mixtures of Graphical Models''}{2012}{Proc. of NIPS 2012}

          \mypubentry{F. Huang and A. Anandkumar, ``FCD: Fast-Concurrent-Distributed Load Balancing under Switching Costs and Imperfect Observations"}{2013}{In Proc. of the 32nd  IEEE INFOCOM}

            \mypubentry{F. Huang, W. Wang and Z. Zhang, ``Prediction-based Spectrum Aggregation with Hardware Limitation in Cognitive Radio Networks''}{2010}{IEEE Vehicular Technology Conference}

\vspace{12pt}
\textbf{SOFTWARE}

  \mysoftentry{TensorDecom4TopicModeling}{\href{https://github.com/FurongHuang/TensorDecomposition4TopicModeling}{Link to Github repository}}
  {C++ algorithm that solves topic modeling LDA using tensor decomposition on single node workstations.}
  
  \vspace{4pt}
  
  \mysoftentry{OnlineTensorCommunity}{\href{https://github.com/FurongHuang/OnlineTensorCommunity}{Link to Github repository}}
  {C++ and CUDA algorithms that solves community detection problem using tensor decomposition on single node CPU and GPU.}
  
    \vspace{4pt}
    
\mysoftentry{SpectralLDA-TensorSpark}{\href{https://github.com/FurongHuang/SpectralLDA-TensorSpark}{Link to Github repository}}
  {Spark spectral LDA algorithms in Scala  that solves large scale tensor decomposition.}
  
 \vspace{4pt}
 
\mysoftentry{ConvDicLearnTensorFactor}{\href{https://github.com/FurongHuang/ConvDicLearnTensorFactor}{Link to Github repository}}
  {Tensor decomposition algorithms that learns convolutional dictionary models.}

\vspace{12pt}
\textbf{AWARDS}

 \begin{tabular*}{1\textwidth}{@{\extracolsep{\fill}}lr}
    \textbf{MLconf Industry Impact Student Research Winner} & \textbf{2015} \\
    \vspace{6pt}
    Google & \emph{San Francisco, California} \\
  \end{tabular*}
  
   \begin{tabular*}{1\textwidth}{@{\extracolsep{\fill}}lr}
    \textbf{Travel Grant} & \textbf{2015} \\
    \vspace{6pt}
    NIPS & \emph{Montreal, Canada} \\
  \end{tabular*}
  
     \begin{tabular*}{1\textwidth}{@{\extracolsep{\fill}}lr}
    \textbf{Travel Grant} & \textbf{2013} \\
    \vspace{6pt}
    WiML & \emph{ Lake Tahoe, Nevada} \\
  \end{tabular*}
  
       \begin{tabular*}{1\textwidth}{@{\extracolsep{\fill}}lr}
    \textbf{Fellowship} & \textbf{2010} \\
    \vspace{6pt}
    University of California Irvine & \emph{Irvine, California} \\
  \end{tabular*}
}

% The abstract should not be over 350 words, although that's
% supposedly somewhat of a soft constraint.
\thesisabstract
{
Unsupervised learning aims at the discovery of hidden structure that drives the observations in the real world. It is essential for success in modern machine learning and artificial intelligence.  Latent variable models are versatile in unsupervised learning and have applications in almost every domain, e.g., social network analysis, natural language processing, computer vision and computational biology. Training latent variable models is challenging due to the non-convexity of the likelihood objective function. An alternative method is based on the spectral decomposition of low order moment matrices and tensors.  This versatile framework is guaranteed to estimate the correct model consistently. My thesis spans both theoretical analysis of tensor decomposition framework and practical implementation of various applications. 
 
This thesis presents theoretical results on convergence to globally optimal solution of tensor decomposition using the stochastic gradient descent, despite non-convexity of the objective. This is the first work that gives global convergence guarantees for the stochastic gradient descent on non-convex functions with exponentially many local minima and saddle points.

This thesis also presents large-scale deployment of spectral methods (matrix and tensor decomposition) carried out on CPU, GPU and Spark platforms.  Dimensionality reduction techniques such as random projection are incorporated for a highly parallel and scalable tensor decomposition algorithm. We obtain a gain in both accuracies and in running times by several orders of magnitude compared to the state-of-art variational methods.

To solve real world problems, more advanced models and learning algorithms are proposed. After introducing tensor decomposition framework under latent Dirichlet allocation (LDA) model, this thesis discusses generalization of LDA model to mixed membership stochastic block model for learning hidden user commonalities or communities in social network, convolutional dictionary model for learning phrase templates and word-sequence embeddings, hierarchical tensor decomposition and latent tree structure model for learning disease hierarchy in healthcare analytics, and spatial point process mixture model for detecting cell types in neuroscience. 
}

%%% Local Variables: ***
%%% mode: latex ***
%%% TeX-master: "thesis.tex" ***
%%% End: ***

% Add PDF document info fields
\hypersetup{
	pdftitle={\Thesistitle},
	pdfauthor={\Authorname},
	pdfsubject={\Degreefield},
}

% Uncomment the following to have numbered subsubsections (by default
% numbering goes only to subsections).
%\setcounter{secnumdepth}{4}

% Set this to only select a subset of the includes directives below.
% Very handy to speed up compilation if you're working on a certain
% part of your thesis. It conserves page numbers, references, etc.
% even for non-included files.
%\includeonly{chapter1}

\begin{document}

% Preliminary pages are always loaded (TOC, CV, etc.)
\preliminarypages

% Include the different components of your thesis, in separate files.
% Using \include allows you to set \includeonly above.
\chapter{Introduction}

There has been tremendous excitement about machine learning and artificial intelligence over the last few years. 
We are now able to do automated classification of images, where there are a predefined set of image categories. 
Due to the enormous amount of available labeled data, and powerful computation resources, we can train massive neural networks and obtain features for classification in domains such as image classification, speech recognition, and text understanding. 
However, all these tasks fall under what we call \emph{supervised learning}, where the training data provides label information.   
What if such labeled information about the categories is absent? 
Can we have automated discovery of the features and categories?

This problem is known as \emph{unsupervised learning}, and experts agree that it is one of the hardest problems in machine learning. Unsupervised learning is usually the foundation for the success of supervised learning in many real world problems, and it aims at summarizing key features in the data. Human beings are known to be good at unsupervised learning, as we accumulate ``general knowledge'' or ``common sense.''   But can we have ``intelligent'' machines that mimic such capabilities?

We live in a world with explosively growing data; as we receive more data, not only do we get more information but also are we confronted with more variables or ``unknowns''.
  In other words, as the data grows, the number of variables also grows, and this is known as the high-dimensional regime.
  Learning the \emph{data patterns} or the \emph{model} in high dimensions is extremely challenging due to curse of dimensionality.
However, the useful information that we need to gain an insightful understanding of the data usually hides in a low dimensional space. 
 Finding these hidden structures is computationally challenging since it is akin to finding ``a needle in a haystack".

The hidden structures in data can be efficiently expressed with the use of probabilistic latent variable models. 
The computational task of searching for hidden structures is then expressed as learning a probabilistic latent variable model. 
Once the model is learned, the hidden variables can be inferred based on the model parameters, as depicted in Figure~\ref{fig:unsupervisedlearning}.

There exit numerous popular approaches for probabilistic latent variable model learning algorithms, among which two families of approaches are particularly successful: randomized algorithms (such as MCMC) and deterministic algorithms (such as maximum likelihood based variational inference). However, randomized algorithms are typically slow due to the exponential mixing time. The deterministic maximum likelihood based estimators tend to be faster than randomized algorithms, but the likelihood function is often intractable. One solution is to substitute the likelihood objective with its approximation and search for the optima. However, local search methods are susceptible to spurious local optima as the surrogate likelihoods are usually non-convex.  

\begin{figure}[!htbp]
\includegraphics[width = \textwidth]{\fighomeTalk/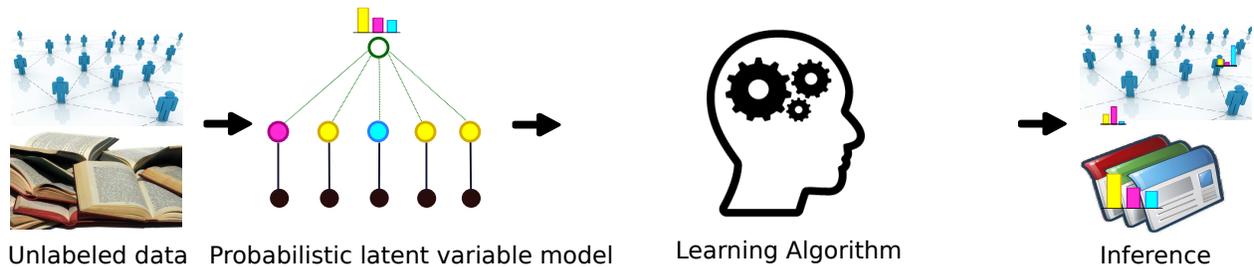}
\caption[Unsupervised learning general framework]{A general framework of unsupervised learning framework.}\label{fig:unsupervisedlearning}
\end{figure}

In this thesis, we analyze and deploy an alternative tensor decomposition framework for learning latent variable models. 
The basic paradigm of tensor decomposition framework dates back to 1894 when Pearson~\cite{pearson1894contributions} proposed the \emph{method of moments}, a classical parameter estimation technique using data statistics. 
The method of moments identifies the model whose parameters give rise to the observed aggregated statistics of the data (such as empirical moments)~\cite{anandkumar2014tensor}. 
Although matching the model parameters to the observed moments may involve solving computationally intractable systems of multivariate polynomial equations,  low-order moments (typically third or fourth order) completely characterize the distribution for many classes of latent variable models~\cite{cattell1944parallel, cardoso1991super,chang1996full,mossel2005learning,hsu2012spectral,anandkumar2012method,hsu2013learning}, and decomposition of the low-order statistics of the data (tensors)  reveals the consistent model parameters asymptotically.
Therefore, the inverse method of moments is solved efficiently with consistency guarantees (both in terms of computational and sample complexity), in contrast to the computationally prohibitive maximum likelihood estimators which require non-convex optimization and are subject to local optimality.

\section{Summary of Contributions}

\subsection{Globally Guaranteed Online Tensor Decomposition}
Learning latent variable models via method of moments involves a challenging non-convex optimization problem in the high-dimensional regime as tensor decomposition is NP-hard in general. 
We identify {\em \name} property for non-convex problem that allows for efficient optimization. 
Using this property, we show that from an {\em arbitrary} starting point,  noisy stochastic gradient descent converges to a local minimum in a polynomial number of iterations. 
To the best of our knowledge, this is the first work that gives {\em global} convergence guarantees for stochastic gradient descent on non-convex functions with exponentially many local minima and saddle points.  
Our analysis is applied to orthogonal tensor decomposition, and we propose a new optimization formulation for the tensor decomposition problem that has \name~property. 
As a result, we get the first online algorithm for orthogonal tensor decomposition with global convergence guarantee~\cite{GeHuangJinYuan:COLT15}.
By employing this algorithm, we obtain an efficient unsupervised learning algorithm for a wide class of latent variable models.

\subsection{Deployment of Scalable Tensor Decomposition Framework}
Tensor decomposition framework is tailored for automated categorization of documents (that is finding the hidden topics of articles) and prediction of social actors' common interests or communities (using the connectivity graph) in social networks efficiently, see Figure~\ref{fig:TDversatile_LDACommunity}. 
Compared to the state of the art variational inference, which optimizes the lower bound on the likelihood, our results are surprisingly accurate and much faster~\cite{huang2014distributed,huang2014online}. For instance, we implemented our tensor decomposition on spark to learn topics in the PubMed data, which consists of 8 million documents and 700 million words. Tensor method achieves much more accurate results (better likelihood) compared to variational inference although we never compute or optimize over the likelihood function. Furthermore, tensor method requires much less computation time and is at least an order of magnitude faster. 

 Another comparison is carried out on graph data to evaluate the performance of discovering hidden communities. On the Facebook friendship network, yelp bipartite review graph and DBLP co-authorship system,  tensor decomposition framework continues to be both accuracy and fast compared to the state-of-the-art variational methods~\cite{huang2014online}. 
\begin{figure}[H]
\begin{minipage}{0.5\textwidth}
\includegraphics[width=\textwidth]{\fighomeTalk/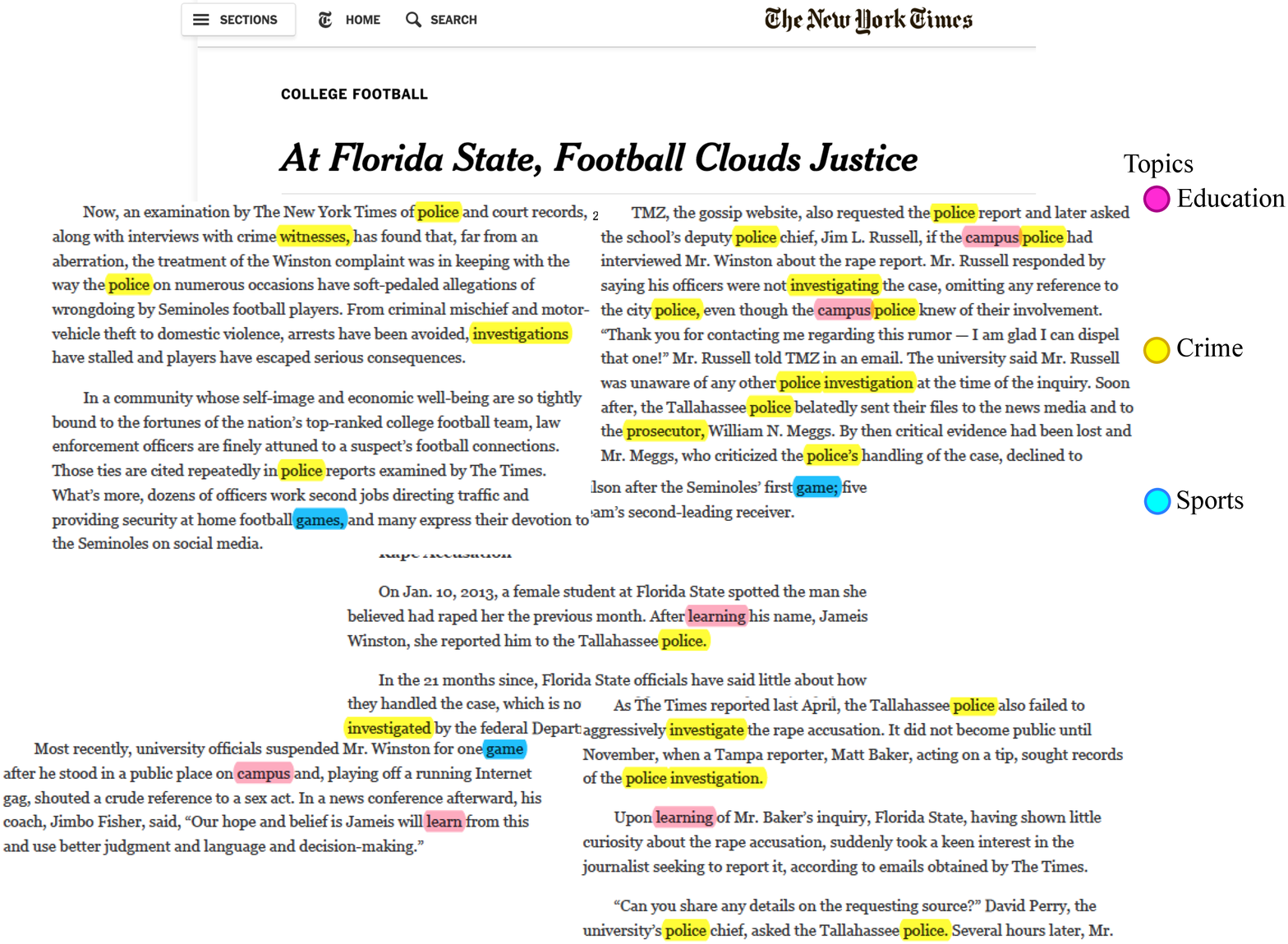}
\end{minipage}
\hfil
\begin{minipage}{0.45\textwidth}
\includegraphics[width=\textwidth]{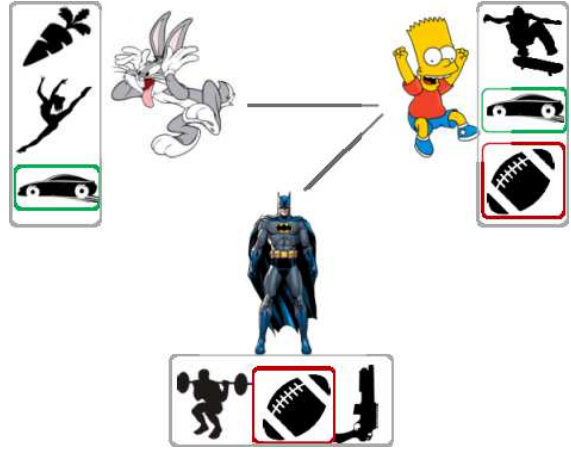}
\end{minipage}
\caption[Tensor decomposition framework is versatile]{Tensor decomposition framework is versatile. (a) Automated hidden topic discovery. (b) Scalable community membership detection via connectivity graph. }\label{fig:TDversatile_LDACommunity}
\end{figure}

\begin{figure}[H]
\begin{minipage}{0.5\textwidth}
\centering{\bp
\psfrag{Perplexity}[Bl]{\tiny{Perplexity}}
\psfrag{Tensor}[Bl]{\tiny{Tensor}}
\psfrag{Variational}[Bl]{\tiny{Variational}}
\includegraphics[width=0.9\textwidth]{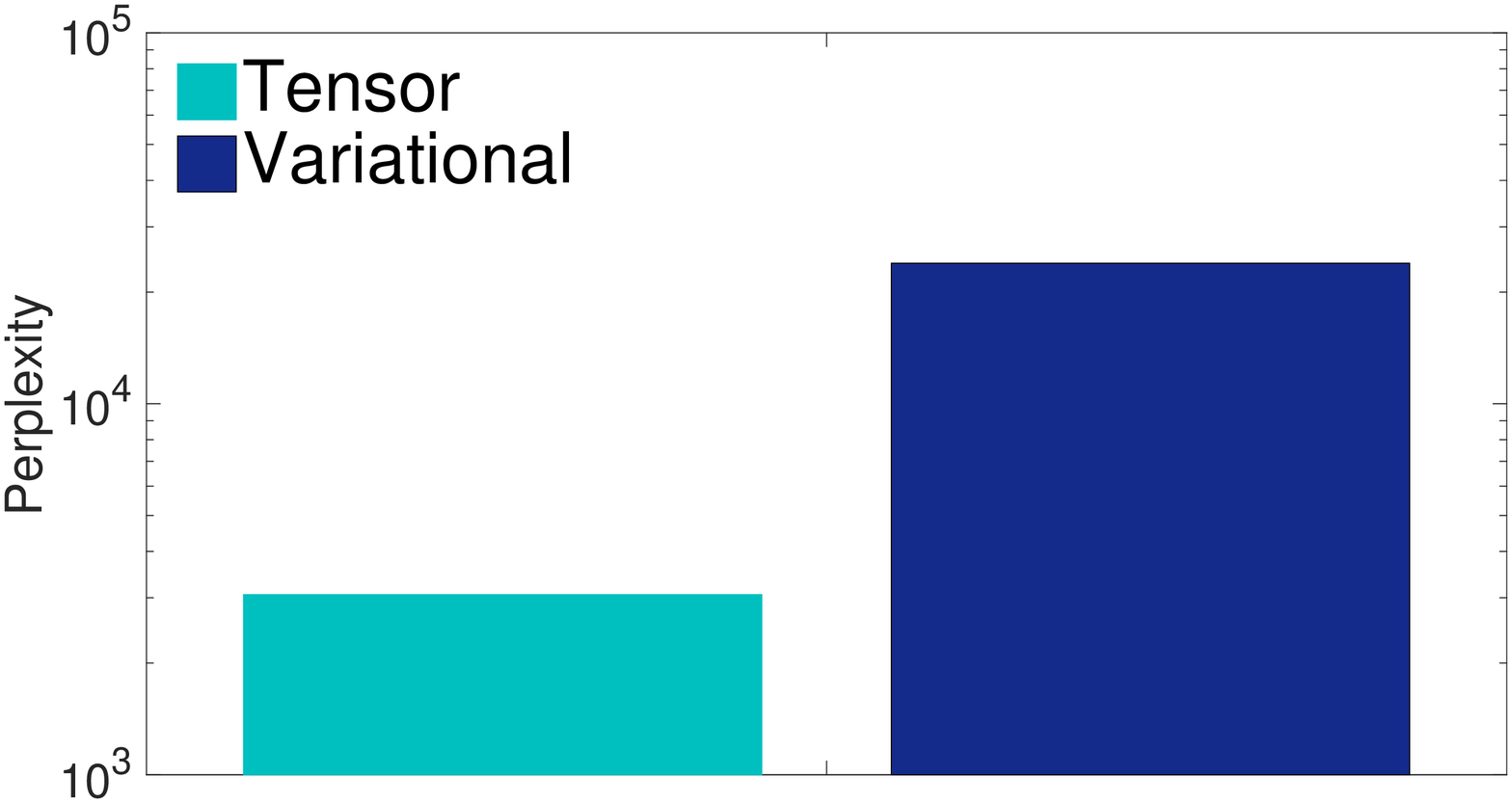}
\ep}
\end{minipage}
\hfil
\begin{minipage}{0.5\textwidth}
\centering{
\psfrag{Running Time (seconds)}[Bl]{\tiny{Running Time (s)}}
\psfrag{Tensor}[Bl]{\tiny{Tensor}}
\psfrag{Variational}[Bl]{\tiny{Variational}}
\includegraphics[width=0.9\textwidth]{\fighomeTalk/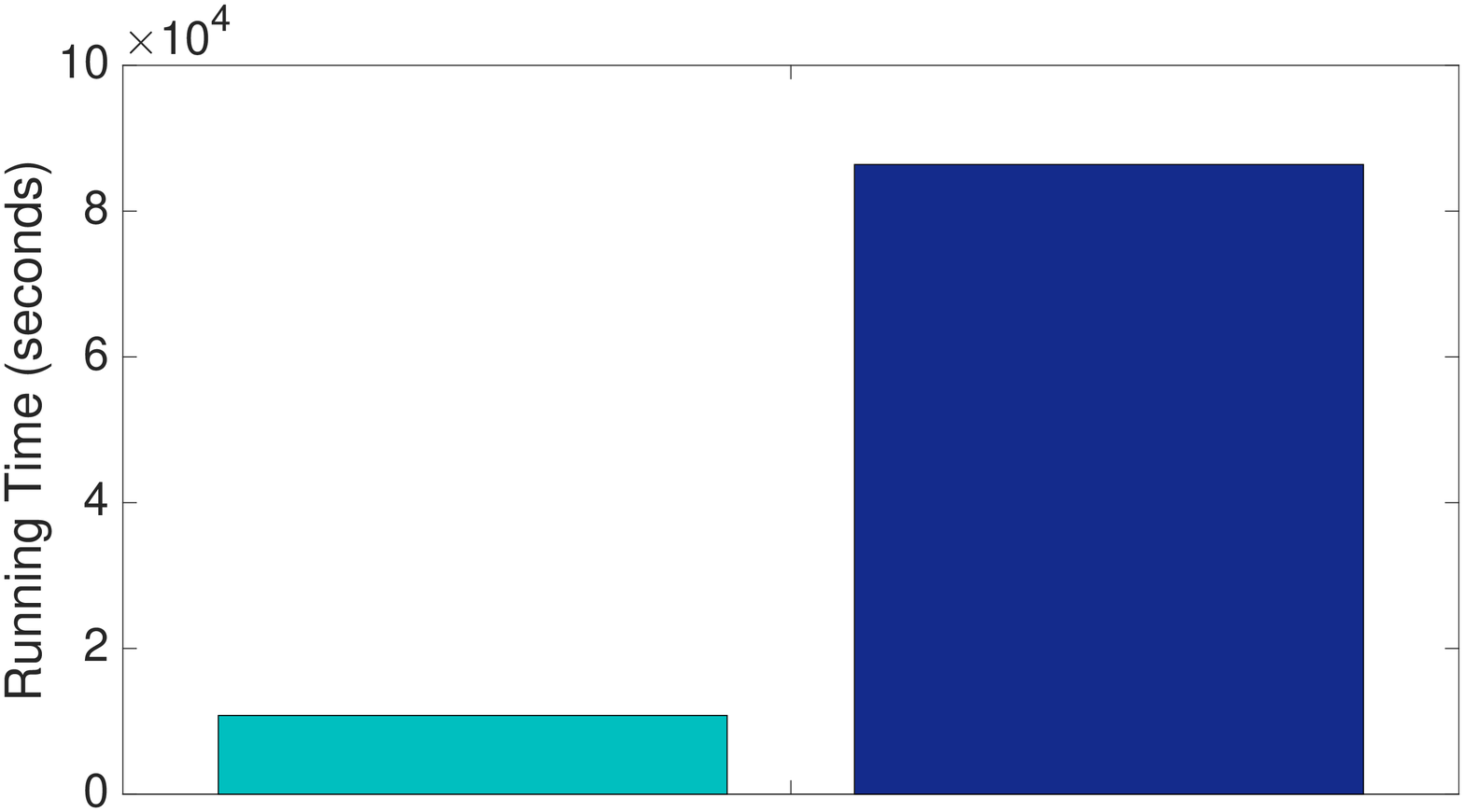}
}\end{minipage}
\caption[Tensor decomposition vs variational inference on PubMed]{Tensor decomposition framework vs variational inference on PubMed. }
\end{figure}

\begin{figure}[H]
\begin{minipage}{0.5\textwidth}
\bc\psfrag{Variational}[Bl]{\tiny Variational}
\psfrag{Tensor}[Bl]{\tiny Tensor}
\psfrag{datasets}[Bl]{\tiny  }
\psfrag{Error / group}[Bl]{\tiny Error /group}
\psfrag{FB}[Bl]{\tiny FB}
\psfrag{YP}[Bl]{\tiny  YP}
\psfrag{DBLP sub}[Bl]{\tiny  DBLPsub}
\psfrag{DBLP}[Bl]{\tiny DBLP}
\includegraphics[width=0.9\textwidth]{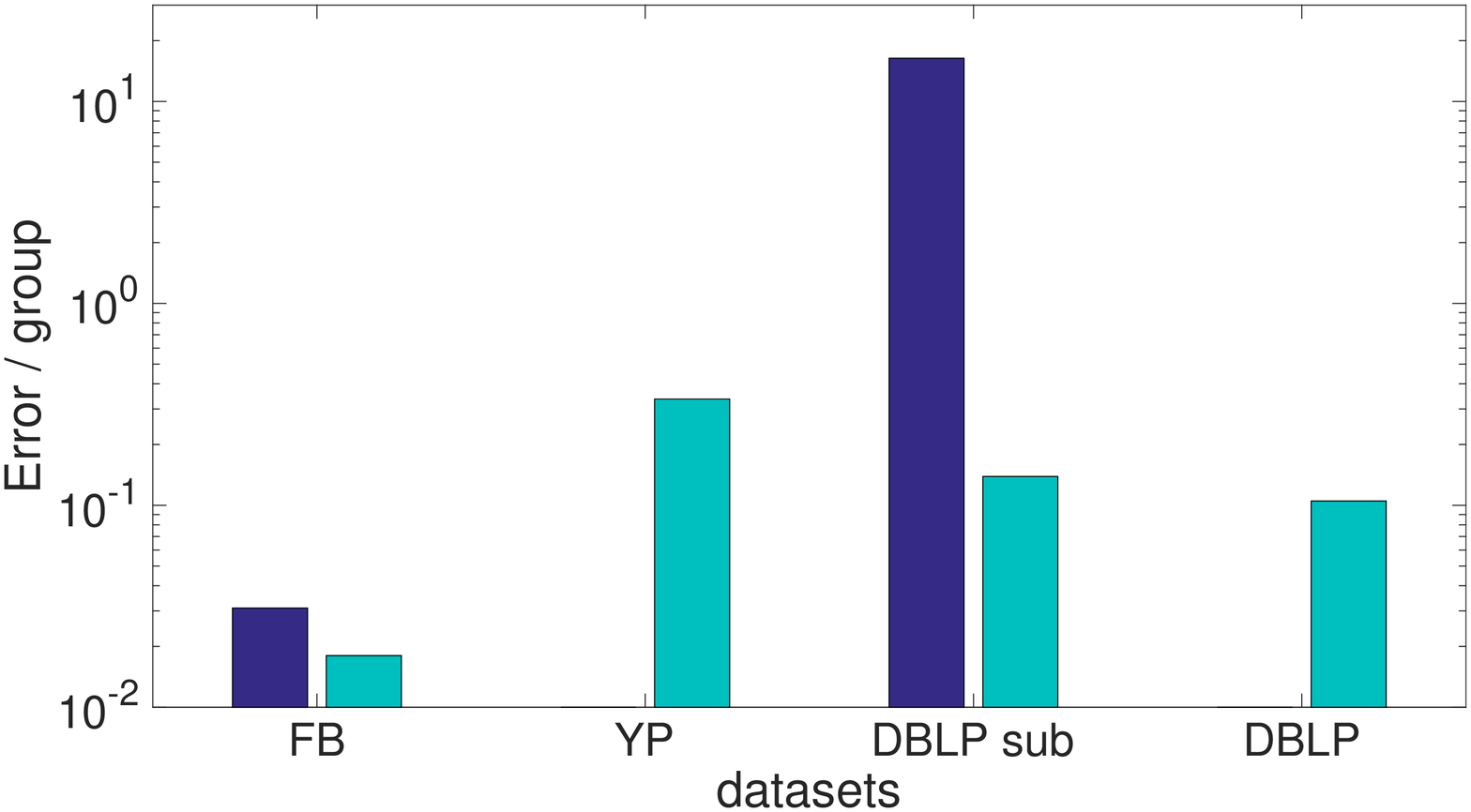}
\ec\end{minipage}
\hfil
\begin{minipage}{0.5\textwidth}
\bc\psfrag{Variational}[Bl]{\tiny Variational}
\psfrag{Tensor}[Bl]{\tiny Tensor}
\psfrag{datasets}[Bl]{\tiny  }
\psfrag{Running Times (seconds)}[Bl]{\tiny Running Times (s)}
\psfrag{FB}[Bl]{\tiny FB}
\psfrag{YP}[Bl]{\tiny  YP}
\psfrag{DBLP sub}[Bl]{\tiny  DBLPsub}
\psfrag{DBLP}[Bl]{\tiny DBLP}
\includegraphics[width=0.9\textwidth]{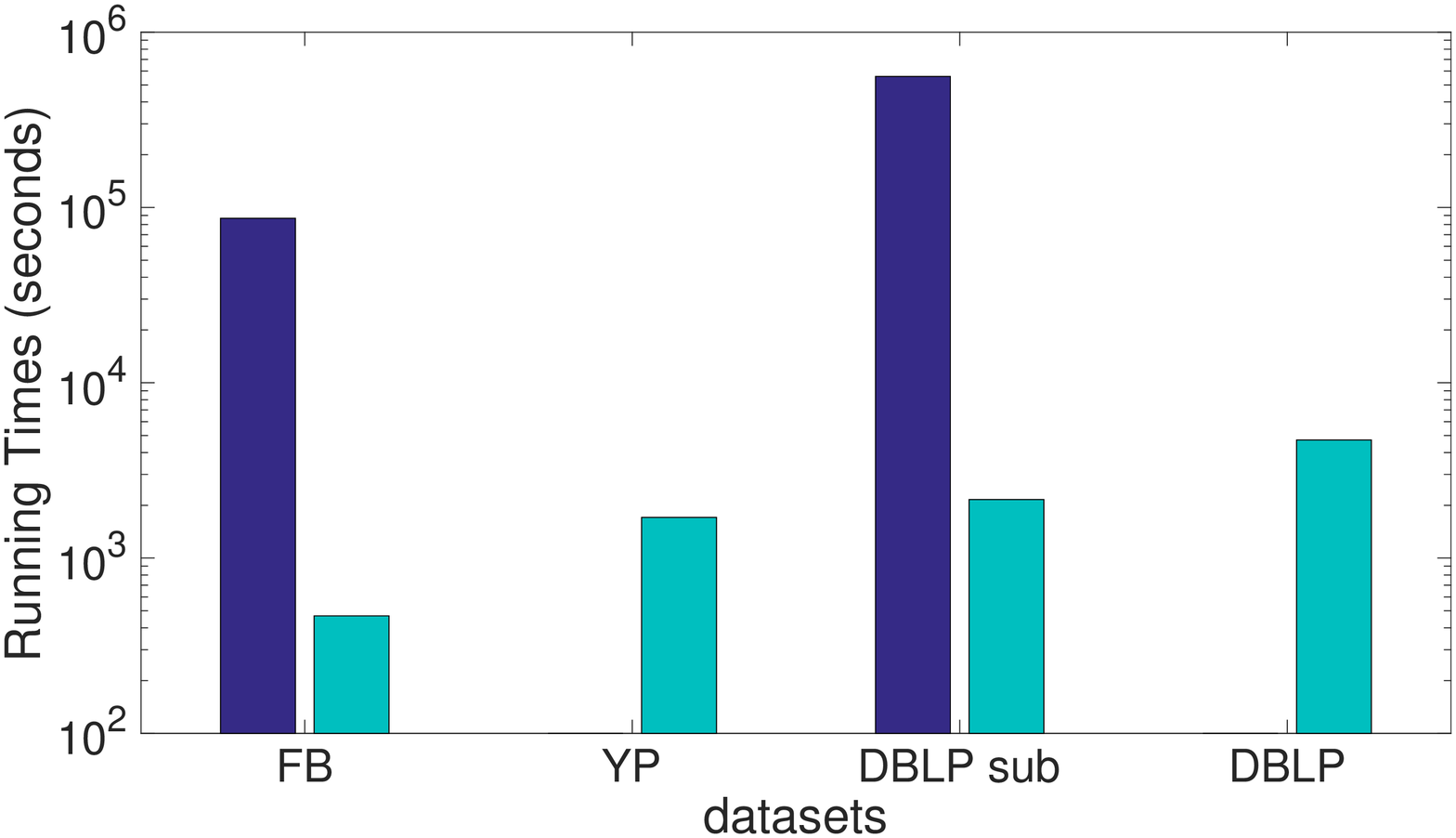}
\ec
\end{minipage}
\caption[Tensor decomposition vs variational inference on social networks]{Tensor decomposition framework vs variational inference on Facebook, Yelp and DBLP. }
\end{figure}

\subsection{Learning Invariant Models Using Convolutional Tensor Decomposition}
Tensor methods can also be extended to solving the problem of learning shift invariant dictionary elements. The data is modeled as linear combinations of filters/templates convolved with activation maps. The filters are shift invariant dictionary elements due to the convolution. 
A tensor decomposition algorithm with additional shift invariance constraints on the factors is introduced, and it converges to models with better reconstruction error and is much faster, compared to the popular alternating minimization heuristic, where the filters and activation maps are alternately updated. 

This convolutional tensor decomposition framework successfully solves challenging natural language processing tasks such as learning phrase templates and extracting word-sequence embeddings, as in Figure~\ref{fig:embedding}. 
Convolutional tensor decomposition learns a good set of filters/templates~\cite{huang15convolutional} and discriminative features (such as word-sequence embeddings) which yield successful automated understanding and classification of word-sequences. 

\begin{figure}[H]
\begin{minipage}{0.4\textwidth}
\centering{
\includegraphics[height=1.2in]{\fighomeTalk/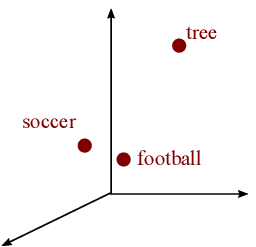}
\\ \tcdkr{Word Embedding}
}
\end{minipage}
\hfil
\begin{minipage}{0.6\textwidth}
\centering{
\includegraphics[height=1.2in]{\fighomeTalk/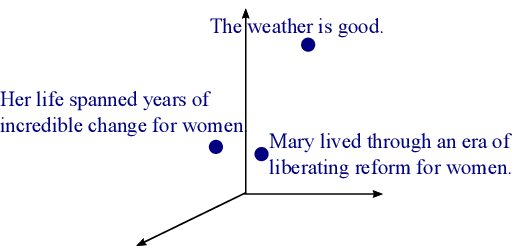}
\\ \tcdkb{Word Sequence Embedding}
}
\end{minipage}
\caption[Word embedding and sentence embedding]{Word embedding and sentence embedding. Word embeddings are vector representations of words, such that words with similar semantic meanings are closer in the vector space. Therefore, a machine can ``comprehend'' the words. Similarly, a more challenging task is to extract word sequence embeddings, where sentences or arbitrary length word-sequences that share semantic and syntactic properties are mapped to similar vector representations.}\label{fig:embedding}
\end{figure}

\subsection{Learning Latent Tree Models Using Hierarchical Tensor Decomposition}

Tensor decomposition framework is also extended to learning models with hierarchy. 
This thesis presents an integrated approach to structure and parameter estimation in latent tree models.
The proposed algorithm automatically learns the latent variables and their locations and achieves consistent structure estimation with logarithmic computational complexity. 
Meanwhile, the inverse method of moments is carried out on smartly selected local neighborhoods with linear computational complexity. 
A rigorous proof of the global consistency of the structure and parameter estimation under the ``divide-and-conquer'' framework is presented. 
The consistency guarantees apply to
a broad class of linear multivariate latent tree models including discrete distributions, continuous multivariate distributions (e.g. Gaussian), and mixed distributions such as Gaussian mixtures~\cite{huang2014scalable}.
This model class is much more general than discrete models, prevalent in most of the previous works on latent tree models~\cite{mossel2005learning,mossel2007distorted,erdos1999few,anandkumar2013learning}.

\begin{figure}[H]
\bc
\includegraphics[width=0.7\textwidth]{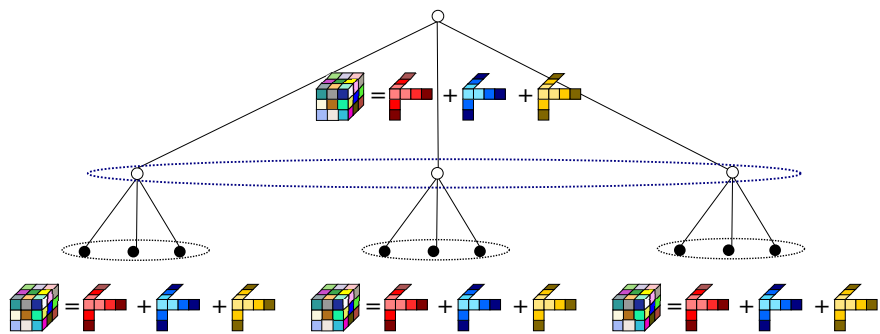}
\ec
\caption{Hierarchical tensor decomposition.}
\end{figure}

This efficient approach is shown to be useful in healthcare analytics~\cite{huang2014scalable}, where we account for the co-occurrence of diseases on individuals and learn a clinical meaningful human disease hierarchy, using big electronic hospital records which involve millions of patients, hundreds of millions diagnostic events, and tens of thousands of diseases. The learned hierarchy on human diseases is clinically meaningful and can help doctors prevent potential diseases according to partial information on patients' health condition. 

\subsection{Discovering Neuronal Cell Types Using Spectral Methods}
The above advances in unsupervised learning have rich applications in neuroscience.
Using spectral decomposition framework, we analyze challenging tasks. For instance, cataloging neuronal cell types in the brain, which has been the number one goal of the brain initiative and modern neuroscience.
It is an extremely challenging task partly due to the petabyte-scale size brain-wide single-cell resolution \emph{in situ hybridization} imagery.  Previous methods average over image intensity in local voxels for a rough estimation of gene expression levels. The success of these methods rely on a precise neuron level image alignment across different brains, which is computationally prohibitive. 

\begin{figure}[H]
\subfloat[]{\begin{minipage}{0.5\textwidth}
\includegraphics[width=\textwidth]{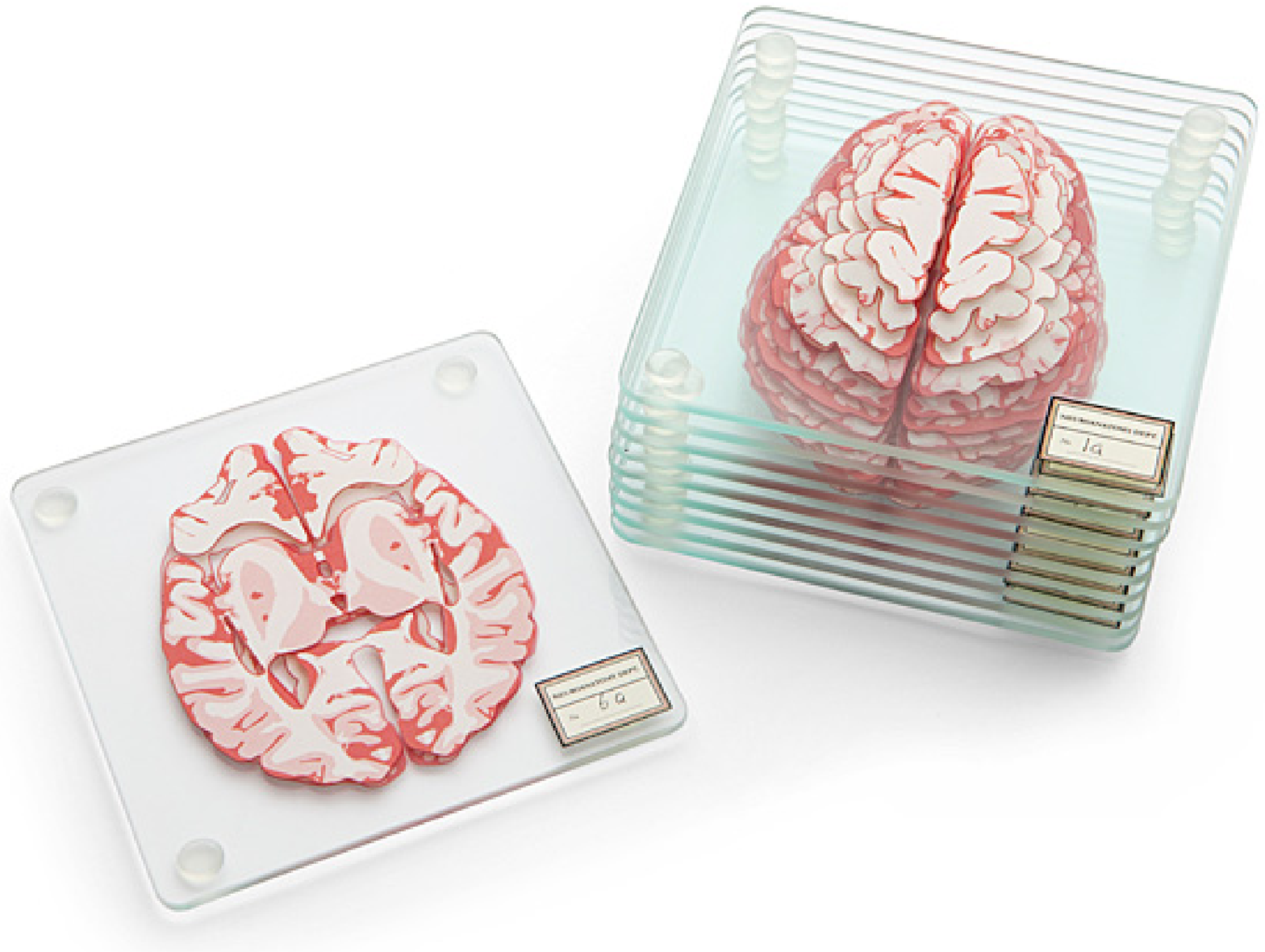}\end{minipage}}
\hfil
\subfloat[]{\begin{minipage}{0.4\textwidth}
\includegraphics[width=\textwidth]{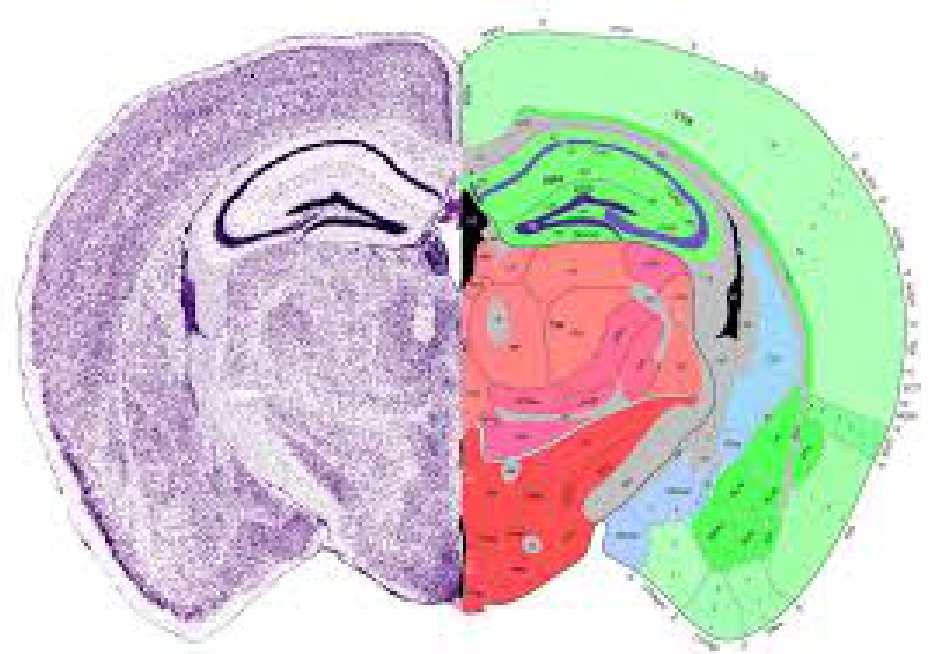}\end{minipage}}
\caption{Examples of brain slices.}
\end{figure}

In this thesis,  we resolve the above problem using a spatial point process mixture model. We measure the spatial distribution of neurons labeled in the ISH image for each gene and model it as a spatial point process mixture, whose mixture weights are given by the cell types which express that gene. By fitting a point process mixture model jointly to the ISH images, we infer both the spatial point process distribution for each cell type and their gene expression profile. We validate our predictions of cell type-specific gene expression profiles using single cell RNA sequencing data, recently published for the mouse somatosensory cortex. Jointly with the gene expression profiles, cell features such as cell size, orientation, intensity and local density level are inferred per cell type. 
Compared with the state-of-the-art approaches, our method~\cite{huang2016discovering}  yields lower/better perplexity scores. In addition, 8 cell types are detected and their cell features are estimated.

\section{Tensor Preliminaries}
\paragraph{What is a tensor? } A $p^\tha$ order tensor is a $p$-dimensional array. We will use $4^\tha$ order tensor as an example.
If $T\in \R^{d^4}$ is a $4^\tha$ order tensor, we use $T_{i_1,i_2,i_3,i_4} (i_1,...,i_4\in [d])$ to denote its $(i_1,i_2,i_3,i_4)^\tha$ entry.

Tensors can be constructed from tensor products. We use $(u\otimes v)$ to denote a $2$nd order tensor where $(u\otimes v)_{i,j} = u_iv_j$. This generalizes to higher order and we use $u^{\otimes 4}$ to denote the $4^\tha$ order tensor
$$
[u^{\otimes 4}]_{i_1,i_2,i_3,i_4} = u_{i_1}u_{i_2}u_{i_3}u_{i_4}.
$$ 
We say a $4^\tha$ order tensor $T\in \R^{d^4}$ has an {\em orthogonal decomposition} if it can be written as
\begin{equation}
T = \sum_{i=1}^d a_i^{\otimes 4}, \label{eq:orthodecomp}
\end{equation}
where $a_i$'s are orthonormal vectors that satisfy $\|a_i\| = 1$ and $a_i^T a_j = 0$ for $i\ne j$. We call the vectors $a_i$'s the components of this decomposition. Such a decomposition is unique up to permutation of $a_i$'s and sign-flips.

A tensor also defines a multilinear form (just as a matrix defines a bilinear form), for a $p^\tha$ order tensor $T\in \R^{d^p}$ and matrices $M_i\in \R^{d\times n_i}, i\in[p]$, we define
$$
[T(M_1,M_2,...,M_p)]_{i_1,i_2,...,i_p} = \sum_{j_1,j_2,...,j_p\in[d]} T_{j_1,j_2,...,j_p} \prod_{t\in[p]} M_t[j_t,i_t].
$$
That is, the result of the multilinear form $T(M_1,M_2,...,M_p)$ is another tensor in $\R^{n_1\times n_2\times \cdots \times n_p}$. We will most often use vectors or identity matrices in the multilinear form. In particular, for a $4^\tha$ order tensor $T\in \R^{d^4}$ we know $T(I,u,u,u)$ is a vector and $T(I,I,u,u)$ is a matrix. In particular, if $T$ has the orthogonal decomposition in (\ref{eq:orthodecomp}), we know $T(I,u,u,u) = \sum_{i=1}^d (u^T a_i)^3 a_i$ and $T(I,I,u,u) = \sum_{i=1}^d (u^Ta_i)^2 a_ia_i^T$.

\paragraph{Why are tensors powerful? } Let us start with the simple matrix decomposition, where the goal is to discover the orthogonal eigenvectors of a matrix. However, it is known that if the eigenvalues of the matrix are equal to each other, one can not uniquely identify the eigenvectors. For instance, an identity matrix can be decomposed as the set of basis vector $e_1$ and $e_2$, as well as $u_1$ and $u_2$, who are 45 degree rotated $e_1$ and $e_2$:  
\begin{equation*}
\left[\begin{tabular}{cc}1 & 0\\0 &1\end{tabular}\right] = \tcr{e_1e_1^\top} + \tcr{e_2 e_2^\top} =\tcb{u_1u_1^\top} + \tcb{u_2 u_2^\top}.
\end{equation*}

\begin{figure}[H]
\bc
\bp
\psfrag{e1}[Bl]{\tcr{$e_1$}}
\psfrag{e2}[Bl]{\tcr{$e_2$}}
\psfrag{u1=(r2,r1)}[Bl]{\tcb{$u_1=[\frac{\sqrt{2}}{2},\frac{-\sqrt{2}}{2}]$}}
\psfrag{u2=(r1,r2)}[Bl]{\tcb{$u_2=[\frac{\sqrt{2}}{2},\frac{\sqrt{2}}{2}]$}}
\includegraphics[width=0.3\textwidth]{\fighomeTalk/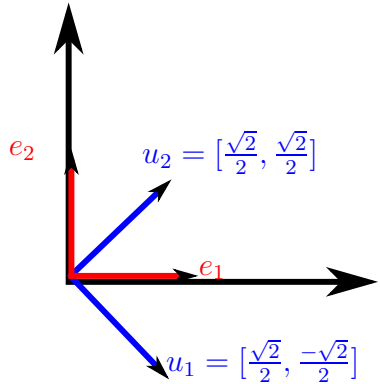}
\ep
\ec
\caption{Orthogonal matrix decomposition is not unique without eigenvalue gap.}
\end{figure}

However, in tensors, there exists a unique decomposition even without eigenvalue gap. Let a third order tensor (a cube) be decomposed as a linear combination of 2 rank-1 tensors as in red and blue, see Figure~\ref{fig:uniquetensordecomp}a. The eigenvectors of the tensor are this red vector and this blue vector who are orthogonal to each other, and the eigenvalues of the tensor are equal. Consider taking a slice of the tensor, which yields matrix. This matrix shares the same eigenvectors with the tensor, but the eigenvalues of this matrix will be different depending on the direction of the slice. Therefore, the slice of tensor has eigenvalue gap. And thus we are able to identify the eigenvectors for the tensor uniquely.  Since higher order tensors have additional dimensions and contains more information, it is more powerful than second-order matrices. 
\begin{figure}[H]
\subfloat[]{
\begin{minipage}{0.3\textwidth}
\fbox{
\includegraphics[width=\textwidth]{\fighomeTalk/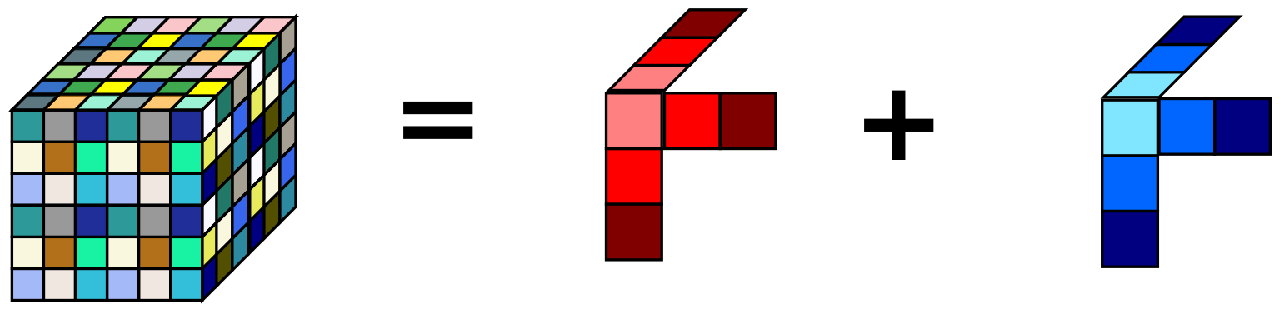}}
\end{minipage}}
\hfil
\subfloat[]{
\begin{minipage}{0.3\textwidth}
\fbox{
\includegraphics[width=\textwidth]{\fighomeTalk/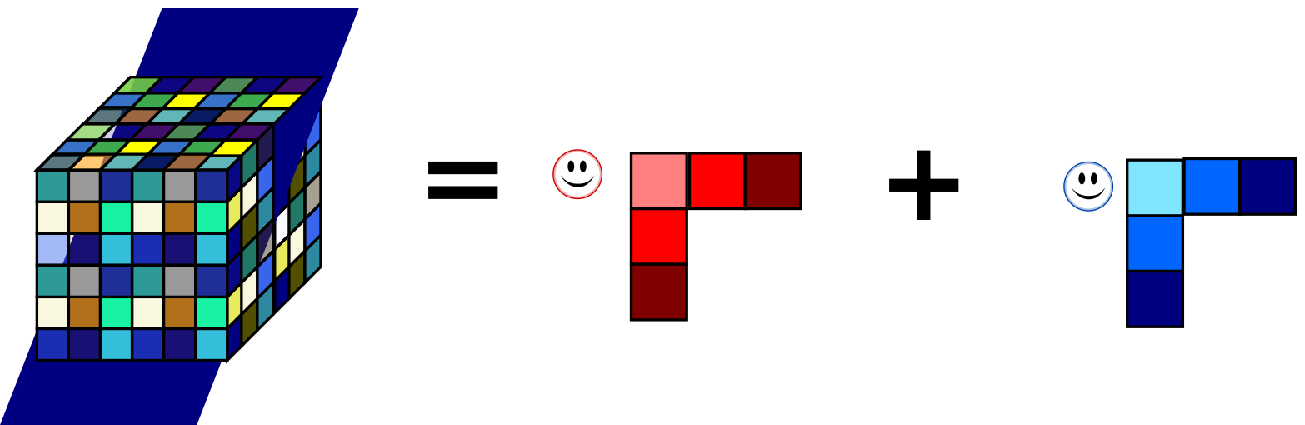}}
\end{minipage}}
\hfil
\subfloat[]{
\begin{minipage}{0.3\textwidth}
\fbox{
\includegraphics[width=\textwidth]{\fighomeTalk/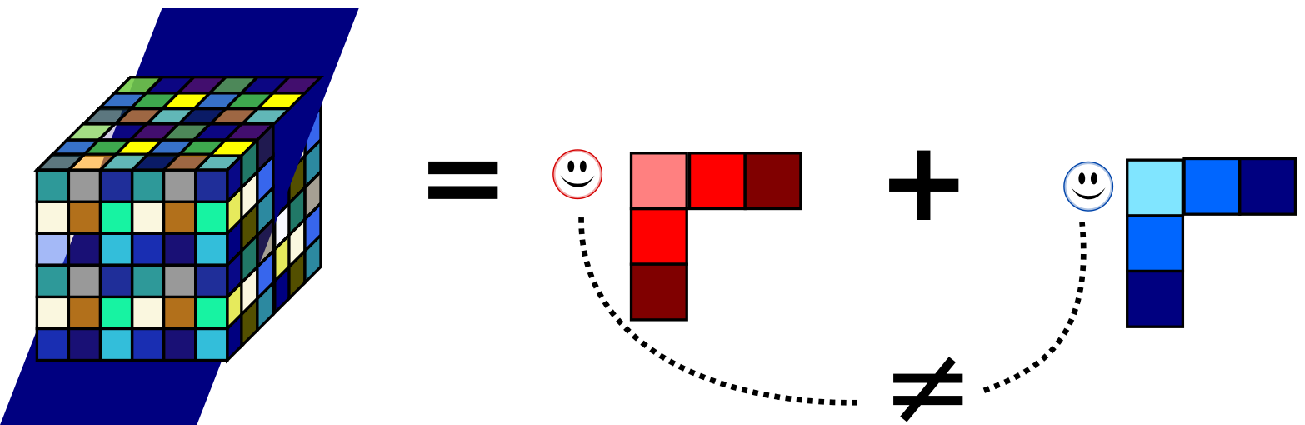}}
\end{minipage}}
\caption[Orthogonal tensor decomposition is unique with or without eigenvalue gap.]{Orthogonal tensor decomposition is unique with or without eigenvalue gap. (a) A third order tensor equals to a linear combination of rank 1 tensors, where each rank 1 tensor is a third order tensor product of the tensor's eigenvector.  (b) A slice of the tensor results in a matrix. The matrix shares the same set of eigenvectors with the original tensor, with a different scaling factor, i.e., different eigenvalues. (c) Tensor eigenvectors are uniquely identified when there is a eigenvalue gap in the slice.}\label{fig:uniquetensordecomp}
\end{figure}

\paragraph{Orthogonal tensor decomposition } Given a tensor $T$ with an orthogonal decomposition, the orthogonal tensor decomposition problem asks to find the individual components $a_1,...,a_d$. This is a central problem in learning many latent variable models, including Hidden Markov Model, multi-view models, topic models, mixture of Gaussians and Independent Component Analysis (ICA). See the discussion and citations in \cite{JMLR:v15:anandkumar14b}. Orthogonal tensor decomposition problem can be solved by many algorithms even when the input is a noisy estimation $\tilde{T} \approx T$ ~\cite{harshman1970foundations,kolda2001orthogonal,JMLR:v15:anandkumar14b}. 
 In practice this approach has been successfully applied to ICA~\cite{comon2002tensor}, topic models~\cite{zou2013contrastive} and community detection~\cite{huang2013fast}.

 \section{Background and Related Works}
\subsection{Online Stochastic Gradient for Tensor Decomposition}

Stochastic gradient descent is one of the basic algorithms in optimization. It is often used to solve the following stochastic optimization problem
\begin{equation}
w = \arg\min_{w\in \R^d} f(w), \textrm{~where~} f(w) = \E_{x\sim \mathcal{D}}[\phi(w,x)]
\label{eq:opt}
\end{equation}
Here $x$ is a data point that comes from some unknown distribution $\mathcal{D}$, and $\phi$ is a {\em loss function} that is defined for a pair $(x,w)$ of sample and parameters. We hope to minimize the expected loss $\E[\phi(w,x)]$.

When the function $f(w)$ is convex, convergence of stochastic gradient descent is well-understood \\ \cite{shalev2009stochastic, ICML2012Rakhlin_261}. However, the stochastic gradient descent is not only limited to convex functions. Especially, in the context of neural networks, the stochastic gradient descent is known as the ``backpropagation'' algorithm~\cite{rumelhart1988learning}, and has been the main algorithm that underlies the success of deep learning~\cite{bengio2009learning}. However, the guarantees in the convex setting do not transfer to the non-convex settings.

Optimizing a non-convex function is NP-hard in general. The difficulty comes from two aspects. First, the function may have many local minima, and it can be hard to find the best one (global minimum) among them. Second, even finding a local minimum can be hard as there can be many saddle points which have $0$-gradient but are not local minima\footnote{See Section~\ref{sec:sgd} for the definition of saddle points.}. In the most general case, there is no known algorithm that guarantees to find a local minimum in a polynomial number of steps. The discrete analog (finding a local minimum in domains like $\{0,1\}^n$) has been studied in complexity theory and is PLS-complete~\cite{johnson1988easy}.

In many cases, especially in those related to deep neural networks~\cite{dauphin2014identifying}\\ \cite{choromanska2014loss}, the main bottleneck in optimization is not due to local minima, but the existence of many saddle points. Gradient-based algorithms are in particular susceptible to saddle point problems as they only rely on the gradient information. The saddle point problem is alleviated for second-order methods that also rely on the Hessian information~\cite{dauphin2014identifying}.

However, using Hessian information usually increases the memory requirement and computation time per iteration. As a result,  many applications still use stochastic gradient and empirically get reasonable results. In this paper we investigate why stochastic gradient methods can be effective even in presence of saddle point, in particular, we answer the following question:

\medskip
{\noindent \textbf{Question:}} Given a non-convex function $f$ with many saddle points, what properties of $f$ will guarantee stochastic gradient descent to converge to a local minimum efficiently?
\medskip

We identify a property of non-convex functions which we call {\em \name}. Intuitively, it guarantees local progress if we have access to the Hessian information. Surprisingly we show that, with only first order (gradient) information, the stochastic gradient escape from the saddle points efficiently. 
We provide a framework for analyzing stochastic gradient in both unconstrained and equality-constrained case using this property.

We apply our framework to {\em orthogonal tensor decomposition}, which is a core problem in learning many latent variable models. 
The tensor decomposition problem is inherently susceptible to the saddle point issues, as the problem asks to find $d$ different components and any permutation of the true components yields a valid solution. Such symmetry creates exponentially many local minima and saddle points in the optimization problem.  Using our new analysis of stochastic gradient, we give the first online algorithm for orthogonal tensor decomposition with global convergence guarantee. This is a key step towards making tensor decomposition algorithms more scalable.

\paragraph{Relaxed notions of convexity} In optimization theory and economics, there are extensive works on understanding functions that behave similarly to convex functions (and in particular can be optimized efficiently). Such notions involve pseudo-convexity~\cite{mangasarian1965pseudo}, quasi-convexity~\cite{kiwiel2001convergence}, invexity\cite{hanson1999invexity} and their variants. More recently there are also works that consider classes that admit more efficient optimization procedures like RSC (restricted strong convexity)~\cite{agarwal2010fast}. Although these classes involve functions that are non-convex, the function (or at least the function restricted to the region of analysis) still has a unique stationary point that is the desired local/global minimum. Therefore, these works cannot be used to prove global convergence for problems like tensor decomposition, where there are exponentially many local minima and saddle points by the symmetry of the problem.

\paragraph{Second-order algorithms} The most popular second-order method is the Newton's method. Although Newton's method converges fast near a local minimum, its global convergence properties are less understood in the more general case. For non-convex functions,~\cite{frieze1996learning} gave a concrete example where second-order method converges to the desired local minimum in a polynomial number of steps (interestingly the function of interest is trying to find one component in a $4^\tha$ order orthogonal tensor, which is a simpler case of our application). As Newton's method often converges also to saddle points, to avoid this behavior, different trusted-region algorithms are applied~\cite{dauphin2014identifying}.

\paragraph{Stochastic gradient and symmetry} The tensor decomposition problem we consider in this paper has the following symmetry: the solution is a set of $d$ vectors $v_1,...,v_d$. If $(v_1,v_2,...,v_d)$ is a solution, then for any permutation $\pi$ and any sign flips $\kappa \in \{\pm 1\}^d$, $(.., \kappa_i v_{\pi(i)}, ...)$ is also a valid solution. In general, symmetry is known to generate saddle points, and variants of gradient descent often perform reasonably in these cases (see~\cite{saad1995line},~\cite{rattray1998natural},~\cite{inoue2003line}). The settings in these work are different from ours, and none of them give bounds on number of steps required for convergence.

Many other problems have the same symmetric structure as the tensor decomposition problem, including the sparse coding problem~\cite{olshausen1997sparse} and many deep learning applications~\cite{bengio2009learning}. In these problems, the goal is to learn multiple ``features'' where the solution is invariant under permutation. Note that there are many recent papers on iterative/gradient-based algorithms for problems related to matrix factorization~\cite{jain2013low,saxe2013exact}. These problems
often have very different symmetry, as if $Y = AX$ then for any invertible matrix $R$ we know $Y = (AR)(R^{-1} X)$. In this case, all the equivalent solutions are in a connected low dimensional manifold, and there need not be saddle points between them.

 \subsection{Applying Online Tensor Methods for Learning Latent Variable Models}

The spectral or moment-based approach involves decomposition of certain empirical moment tensors, estimated from observed data to obtain the parameters of the proposed probabilistic model. Unsupervised learning for a wide range of latent variable models can be carried out efficiently via tensor-based techniques with low sample and computational complexities~\cite{AGHKT12}. In contrast, usual methods employed in practice such as expectation maximization (EM) and variational Bayes do not have such consistency guarantees. While the previous works~\cite{AnandkumarEtal:community12COLT} focused on theoretical guarantees, in chapter~\ref{chapter:online} of this thesis, we focus on the implementation of the tensor methods, study its performance on several datasets.

We introduce an online tensor decomposition based approach for two latent variable modeling problems namely, (1) community detection,  in which we learn the latent communities that the social actors in social networks belong to, and (2) topic modeling, in which we infer hidden topics of text articles. We consider decomposition of moment tensors using stochastic gradient descent. We conduct   optimization of multilinear operations in SGD and avoid directly forming the tensors, to save computational and storage costs. 
We present optimized algorithm in two platforms.  Our GPU-based implementation exploits the parallelism of SIMD architectures to allow for maximum speed-up by a careful optimization of storage and data transfer, whereas our CPU-based implementation uses efficient sparse matrix computations and is suitable for large sparse data sets. For the community detection problem, we demonstrate accuracy and computational efficiency on Facebook, Yelp, and DBLP data sets, and for the topic modeling problem, we also demonstrate good performance on the New York Times data set. 
We compare our results to the state-of-the-art algorithms such as the variational method and report a gain of accuracy and a gain of several orders of magnitude in the execution time.

Chapter~\ref{chapter:online} builds on the recent works of Anandkumar et al~\cite{AGHKT12,AnandkumarEtal:community12COLT} which establishes the correctness of tensor-based approaches for learning MMSB~\cite{ABFX08} models and other latent variable models.
While, the earlier works provided a theoretical analysis of the method, the current paper considers a careful implementation of the method. Moreover, there are a number of algorithmic improvements in this thesis. For instance, while \cite{AGHKT12,AnandkumarEtal:community12COLT} consider tensor power iterations, based on batch data and deflations performed serially, here, we adopt a  stochastic gradient descent approach for tensor decomposition, which provides the flexibility to trade-off  sub-sampling with accuracy. Moreover, we use randomized methods for dimensionality reduction in the preprocessing stage of our method which enables us to scale our method to graphs with millions of nodes.

There are other known methods for learning the stochastic block model based on techniques such as spectral clustering~\cite{McSherry01} and convex optimization~\cite{chen2012clustering}. However, these methods are not applicable for learning overlapping communities.
We note that learning the mixed membership model can be reduced to a matrix factorization problem~\cite{Zhang:2012:OCD:2339530.2339629}. While collaborative filtering techniques  such as~\cite{mnih2007probabilistic,salakhutdinov2008bayesian}
focus on matrix factorization and the prediction accuracy of recommendations on an unseen test set, we recover the underlying latent communities, which helps with the interpretability, and the statistical model can be employed for other tasks.

Although there have been other fast implementations for community detection before~\cite{soman2011fast,lancichinetti2009community}, these methods are not statistical and do not yield descriptive statistics such as bridging nodes~\cite{nepusz2008fuzzy}, and cannot perform predictive tasks such as link classification which are the main strengths of the MMSB model. With the implementation of our tensor-based approach, we record huge speed-ups compared to existing approaches for learning the MMSB model.

To the best of our knowledge, while stochastic methods for matrix decomposition have been considered earlier~\cite{oja1985stochastic,6483308}, this is the first work incorporating stochastic optimization for tensor decomposition, and paves the way for further investigation on many theoretical and practical issues.
We also note that we never explicitly form or store the subgraph count tensor, of size $O(n^3)$ where $n$ is the number of nodes, in our implementation, but directly manipulate the neighborhood vectors to obtain tensor decompositions through stochastic updates. This is a crucial departure from other works on tensor decompositions on GPUs~\cite{ballard2011efficiently,schatz2013exploiting}, where the tensor needs to be stored and manipulated directly.
 
 \subsection{Dictionary Learning through Convolutional Tensor  Decomposition}
 Feature or representation learning forms a cornerstone of modern machine learning. Representing the data in the relevant feature space is critical to obtaining good performance in challenging machine learning tasks in speech, computer vision and natural language processing.  A popular representation learning framework is based on dictionary learning. Here, the input data is modeled as a linear combination of dictionary elements. However, this model fails to incorporate natural domain-specific invariances such as shift invariance and results in highly redundant dictionary elements, which makes inference in these models expensive.

These shortcomings can be remedied by incorporating invariances   into the dictionary model, and  such models are known as convolutional models.  Convolutional models are ubiquitous in machine learning for image, speech and sentence representations~\cite{zeiler2010deconvolutional,kavukcuoglu2010learning,bristow2013fast},
and in neuroscience for modeling neural spike trains~\cite{olshausen2002sparse,ekanadham2011blind}. Deep convolutional neural networks are a multi-layer extension of these models with non-linear activations. Such models have revolutionized performance in image, speech and natural language processing~\cite{zeiler2010deconvolutional,kalchbrenner2014convolutional}. 
The convolutional dictionary learning model posits that the input signal  $x$ is generated as   a linear combination of convolutions of unknown dictionary elements or {\em  filters} $f_1^*, \ldots f_L^*$ and unknown  {\em activation maps} $w_1^*, \ldots w_L^*$:
\beq\label{eqn:sparsedef} x = \sum\limits_{i\in [L]}  f_i^* \Conv w_i^*,\eeq where $[L]:=1,\ldots, L$. The vector $w_i^*$ denotes the activations at locations, where the corresponding filter $f_i^*$ is active.

In order to learn the model in \eqref{eqn:sparsedef}, usually a square loss reconstruction criterion is employed: \beq\label{eqn:alt-min}\min_{f_i,w_i: \|f_i\|=1} \|x - \sum\limits_{i\in [L]}  f_i \Conv w_i\|^2 .\eeq The constraints $(\|f_i\|=1)$  are enforced, since otherwise, the scaling can be exchanged between the filters $f_i$ and the activation maps $w_i$. Also,  an additional regularization term (for example an $\ell_1$ term on the $w_i'$s) is usually added to the above objective to promote sparsity on $w_i$.

A popular heuristic for solving \eqref{eqn:alt-min}   is based on  alternating minimization~\cite{bristow2014optimization}, where the filters $f_i$ are optimized, while keeping the activations $w_i$ fixed, and vice versa.  Each alternating update can be solved efficiently (since it is linear in each of the variables). However,  the method is computationally expensive in the large sample setting since each iteration requires a pass over all the samples, and in modern machine learning applications, the number of samples can run into billions. Moreover,  alternating minimization has multiple spurious local optima, and reaching the global optimum of~\eqref{eqn:alt-min} is NP-hard in general. This problem is severely amplified in the convolutional setting due to additional symmetries, compared to the usual dictionary learning setting (without the convolutional operation). Due to shift invariance of the convolutional operator, shifting a filter $f_i$ by some amount, and applying a corresponding negative shift on the activation $w_i$ leaves the objective in \eqref{eqn:alt-min} unchanged. Can we design alternative methods for convolutional dictionary learning that are scalable to huge datasets?

The special case of \eqref{eqn:sparsedef} with one filter $(L=1)$ is a well studied problem, and is referred to as {\em blind deconvolution}~\cite{hyvarinen2004independent}.  In general, this problem is not identifiable, i.e. multiple equivalent solutions can exist~\cite{choudhary2014sparse}. It has been documented that in many cases alternating minimization produces trivial solutions, where the filter $f=x$ is the signal itself and the activation is the identity function~\cite{levin2009understanding}. Therefore, alternative techniques have been proposed, such as
convex programs, based on nuclear norm minimization~\cite{ahmed2014blind} and  imposing hierarchical Bayesian priors for activation maps~\cite{wipf2013revisiting}. %The convex program is guaranteed to succeed under certain randomness assumptions.
However, there is no analysis for settings with more than one filter.  Incorporating Bayesian priors has shown to reduce the number of local optima, but not eliminate them~\cite{wipf2013revisiting,krishnan2013blind}. Moreover, Bayesian techniques are in general more expensive than alternating minimization.

The extension of blind deconvolution to multiple filters is known as convolutive blind source separation or convolutive independent component analysis (ICA)~\cite{hyvarinen2004independent}. Previous methods directly reformulate convolutive ICA as an ICA model, without incorporating the shift constraints. Moreover, reformulation leads to an increase in the number of hidden sources from $L$ to $nL$ in the new model, where $n$ is the input dimension, which is harder to separate and computationally more expensive. Other methods are based on performing ICA in the Fourier domain, but the downside is that the new mixing matrix depends on the angular frequency, and leads to permutation and sign indeterminacies of the sources across frequencies.
Complicated interpolation methods~\cite{hyvarinen2004independent} overcome these indeterminacies. In contrast, our method avoids all these issues. We do not perform Fourier transform on the input. Instead, we employ FFTs at different iterations of our method to estimate the filters efficiently.

The dictionary learning problem without convolution has received much attention. Recent results show that simple iterative methods can learn the globally optimal solution~\cite{AnandkumarEtal:COLT14,arora2013new}. Also, tensor decomposition methods provably learn the model, when the activations are independently drawn (the ICA model)~\cite{anandkumar2014tensor} or are sparse (the sparse coding model)~\cite{anandkumar2014provable}. 
In this work, we extend the tensor decomposition methods to efficiently incorporate the shift invariance constraints imposed by the convolution operator. This framework is applied to word-sequence embedding learning in natural language processing.

We have recently witnessed the tremendous success of word embeddings or word vector representations in natural language processing. This involves mapping words to vector representations such that words which share similar semantic or syntactic meanings are close to one another in the vector space~\cite{bengio2006neural,collobert2008unified,collobert2011natural,mikolov2013efficient,pennington2014glove}. 
Word embeddings have attained state-of-the-art performance in tasks such as part-of-speech (POS) tagging, chunking, named entity recognition (NER), and semantic role labeling.
Despite this impressive performance,   word embeddings do not suffice for more advanced tasks which require context-aware information or word orders, e.g. paraphrase detection, sentiment analysis, plagiarism detection, information retrieval and machine translation.  Therefore, extracting word-sequence vector representations is crucial for expanding the realm of automated text understanding.

Previous works on word-sequence embeddings are based on a variety of mechanisms. A popular method is to learn the composition operators in sequences~\cite{mitchell2010composition,yu2015learning}. The complexity of the compositionality  varies widely: from simple operations such as addition~\cite{mitchell2010composition,yu2015learning} to complicated recursive neural networks~\cite{socher2011parsing,socher2013recursive,belanger2015linear}, convolutional neural networks~\cite{kalchbrenner2014convolutional,kalchbrenner2014convolutional},  long short-term memory (LSTM) recurrent neural networks~\cite{tai2015improved}, or combinations of these architectures~\cite{wieting2015towards}.  All  these methods produce sentence representations that depend on  a supervised task, and the class labels are back-propagated to update the composition weights~\cite{DBLP:conf/acl/KalchbrennerGB14}.

Since the above methods rely heavily on the downstream task and the domain of the training samples, they can hardly be used as universal embeddings across domains,  and require intensive pre-training and hyper-parameter tuning.  The state-of-the-art unsupervised framework is  Skip-thought~\cite{kiros2015skip}, based on an objective function that abstracts the skip-gram model to the sentence level, and encodes a sentence to predict the sentences around it.  However, the skip-thought model requires a large corpus of contiguous text, such as the book corpus with more than 74 million sentences. 
Can we instead efficiently learn sentence embeddings using small amounts of samples without supervision/labels or annotated features(such as parse trees)? Also, can the sentence embeddings be context-aware, can handle variable lengths, and is not limited to specific domains?

We propose an unsupervised \ourframework  framework that satisfies all the above constraints. It is composed of two phases, a \emph{comprehension phase} which summarizes template phrases using \emph{convolutional dictionary} elements, followed by a \emph{feature-extraction phase} which extracts   activations  using \emph{deconvolutional decoding}. We propose a novel learning algorithm for the comprehension phase based on convolutional tensor decomposition.  Note that in the \emph{comprehension phase}, phrase templates are learned over fixed length small patches (patch length is equal to phrase template length), whereas entire word-sequence is decoded to get the final word-sequence embedding in the \emph{feature-extraction phase}.
 
We employ our sentence embeddings in the tasks of sentiment classification, semantic textual similarity estimation, and paraphrase detection over eight datasets from various domains. These are challenging tasks since they require a contextual understanding of text relationships rather than bags of words. We learn the embeddings from scratch without using any auxiliary information. While previous works use information such as parse trees, Wordnet or pre-train on a much larger corpus, we train from scratch on small amounts of text and obtain competitive results, which are close or even better than the state-of-the-art. 

This is due to the combination of efficient modeling and learning approaches in our work. The convolutional model incorporates word orders and phrase representations, and our tensor decomposition algorithm can efficiently learn a set of parameters (phrase templates) for the convolutional model.

 \subsection[Latent Tree Model Learning via Hierarchical Tensor Decomposition]{Latent Tree Model Learning through Hierarchical Tensor Decomposition}
Latent variable graphical models span flat models and hierarchical models, see Figure~\ref{fig:flathierarchical} for a flat multi-view model and a hierarchical model.   Latent tree graphical models are a popular class of latent variable models, where a probability distribution involving observed and hidden variables are Markovian on a tree.
Due to the fact that structure of (observable and hidden) variable interactions are approximated as a tree, inference on latent trees can be carried out exactly through a simple belief propagation~\cite{pearl1988probabilistic}.
Therefore, latent tree graphical models present a good trade-off between model accuracy and computational complexity.
They are applicable in many domains, where it is natural to expect hierarchical or
sequential relationships among the variables (through a hidden-Markov model).
For instance,  latent tree models have been employed for phylogenetic reconstruction~\cite{Durbin:book}, object recognition~\cite{choi2012context}, ~\cite{choi2012context2} and human pose estimation~\cite{wang2013beyond}.

\begin{figure}[!htbp]
\bc
\subfloat[Multi-view]{
\begin{minipage}{0.3\textwidth}
\psfrag{h}[]{}\psfrag{x1}[]{}\psfrag{x2}[]{}\psfrag{x3}[]{}\psfrag{x4}[]{}\psfrag{x5}[]{}
\includegraphics[width=\textwidth]{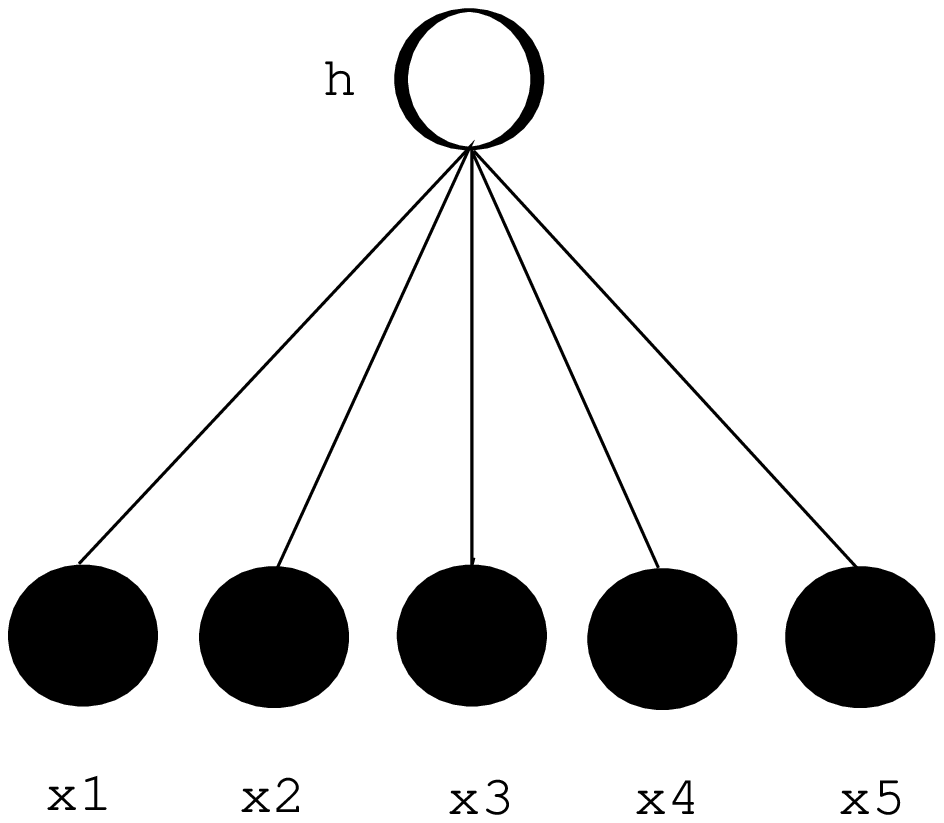}
\end{minipage}}
\hfil
\subfloat[Hierarchical tree]{
\begin{minipage}{0.3\textwidth}
\psfrag{h}[]{}\psfrag{h1}[]{}\psfrag{h2}[]{}\psfrag{x1}[]{}\psfrag{x2}[]{}\psfrag{x3}[]{}\psfrag{x4}[]{}\psfrag{x5}[]{}
\includegraphics[width=\textwidth]{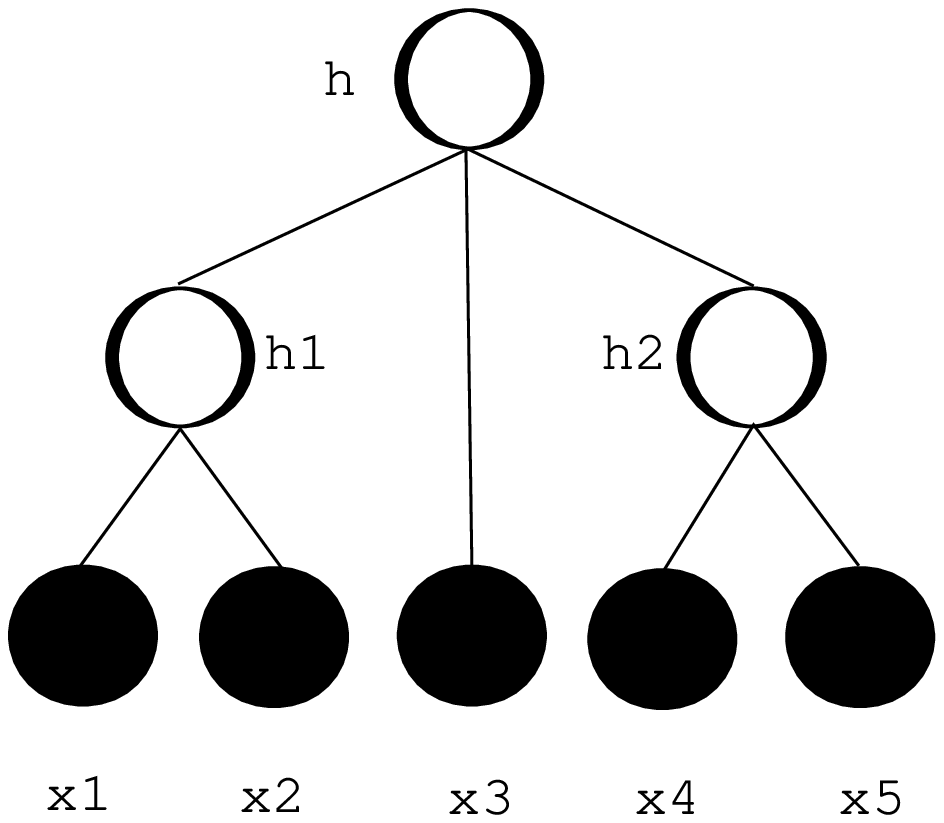}
\end{minipage}}
\ec
\caption[Flat multi-view vs hierarchical latent variable graphical model]{Flat multi-view latent variable graphical model vs hierarchical latent variable graphical model.}\label{fig:flathierarchical}
\end{figure}
The task of learning a latent tree model consists of two parts: learning the tree structure, and learning the parameters of the tree.
There exist many challenges which {prohibit} efficient or guaranteed learning of the latent tree graphical model, which will be addressed in this thesis:
\begin{compactenum}
\item The location and the number of latent variables are hidden, and the marginalized graph over
the observable variables no longer conforms to a tree structure.
\item Structure learning algorithms are typically of computational complexity polynomial with $p$ (number of variables) as discussed in~\cite{anandkumar2011spectral, choi2011learning}. These methods are serial in nature and therefore are not scalable for large $p$.
\item Parameter estimation in latent tree models is typically carried out through Expectation Maximization (EM) or other local search heuristics~\cite{choi2011learning}.
These methods have no consistency guarantees, suffer from the problem of local optima and are not easily parallelizable.
\item Typically structure learning and parameter estimation are carried out one after another.
\end{compactenum}

There has been widespread interest in developing distributed learning techniques, e.g.,  the recent works of~\cite{smola2010architecture} and~\cite{2013arXiv1312.7869W}.
These works consider parameter estimation via likelihood-based optimizations such as Gibbs sampling, while our method involves more challenging tasks where both the structure and the parameters are estimated.
Simple methods such as local neighborhood selection through $\ell_1$-regularization~\cite{Mei06} or local conditional independence testing~\cite{AnandkumarTanWillsky:Ising11} can be parallelized, but these methods do not incorporate hidden variables.
Finally, note that the latent tree models provide a statistical description, in addition to revealing the hierarchy.
In contrast, hierarchical clustering techniques are not based on a statistical model~\cite{krishnamurthy2012efficient} and cannot provide valuable information such as the level of correlation between observed and hidden variables.

\section{Thesis Structure}

 In my thesis, I will first prove that simple noisy gradient descent on a carefully selected objective function yields global convergence guarantee in chapter~\ref{chapter:saddle}. Based on the theoretical guarantees, I will show how to make tensor decomposition highly scalable, highly parallel in chapter~\ref{chapter:online}.  Furthermore, I extend the framework to learn dictionary or templates with additional constraints such as shift invariance in image or text dictionary learning using convolutional dictionary tensor decomposition in chapter~\ref{chapter:convolutional}.   I do not limit myself to shallow models where observations are conditional independent on the hidden dimension. On the contrary,  I extend the \emph{multi-view} tensor decomposition framework to a hierarchical tensor decomposition framework to analyze data with complicated hierarchical structure. A latent tree model is therefore proposed in chapter~\ref{chapter:tree}, where latent variable graphical model structure learning technique is combined with hierarchical tensor decomposition for a  consistent learning of the hierarchical model structure and parameter. Finally, I conclude my thesis with a challenging but important task in chapter~\ref{chapter:brain}, discovering cell types in the brain. This work brings together the techniques used in all previous chapters, such as image processing to extract cells and cell features from brain slices, learning a point process admixture model. 

\chapter{Online Stochastic Gradient for Tensor Decomposition}\label{chapter:saddle}
It is established in the previous work~\cite{JMLR:v15:anandkumar14b} that a wide class of latent variable graphical models can be learned through tensor decomposition, and model parameters are obtained by decomposing higher order data aggregates or modified data moments. Therefore, learning latent variable graphical model is reduced to tensor decomposition problem. Tensor decomposition is a non-convex optimization problem, and it is known that non-convex optimization problem is NP hard in general. Now the question is: could we use efficient methods such as stochastic gradient descent to reach local optima for a class of function under mild conditions? Could we fit tensor decomposition problem into the class of function?

We analyze stochastic gradient descent for optimizing non-convex functions. In many cases for non-convex functions the goal is to find a reasonable local minimum, and the main concern is that gradient updates are trapped in {\em saddle points}. In this chapter we identify {\em \name} property for non-convex problem that allows for efficient optimization. Using this property we show that from an {\em arbitrary} starting point,  stochastic gradient descent converges to a local minimum in a polynomial number of iterations. To the best of our knowledge this is the first work that gives {\em global} convergence guarantees for stochastic gradient descent on non-convex functions with exponentially many local minima and saddle points. 

Our analysis can be applied to orthogonal tensor decomposition, which is widely used in learning a rich class of  latent variable models. We propose a new optimization formulation for the tensor decomposition problem that has \name~property. As a result we get the first online algorithm for orthogonal tensor decomposition with global convergence guarantee.

%!TEX root = saddlepoint.tex

\paragraph{\bigname~functions}

Given a function $f(w)$ that is twice differentiable, we call $w$ a stationary point if $\nabla f(w) = 0$. A stationary point can either be a local minimum, a local maximum or a saddle point. %Here for simplicity we assume for all the local minimum points have positive definite hessian $\Hess f(x) \succ 0$, and all other stationary points are treated as saddle points. 
We identify an interesting class of non-convex functions which we call \name. For these functions the Hessian of every saddle point has a negative eigenvalue. 
In particular, this means that local second-order algorithms which are similar to the ones in \cite{dauphin2014identifying} can always make some progress. 

It may seem counter-intuitive why stochastic gradient can work in these cases: in particular if we run the basic gradient descent starting from a stationary point then it will not move. However, we show that the saddle points are not stable and that the randomness in stochastic gradient helps the algorithm to escape from the saddle points.
%
%\begin{definition}
%\label{def:nice}
%Given a twice differentiable function $f(x)$ whose local minimum all have $\Hess f(x) \succ 0$. We say $x$ is a saddle point of $f(x)$ if $\nabla f(x) = 0$ and $\Hess f(x)$ is not positive definite.
%\end{definition}
%
%Note that this definition is slightly different from the traditional definition, as it also includes local maxima as saddle points. This is just for simplicity of presentation.
%
%Intuitively, if we are not at a stationary point, then we can always follow the gradient and reduce the value of the function. However, if we are at a saddle point, we need to consider a second order Taylor expansion:
%$$
%f(x+\delta x) \approx x + \delta x^T \Hess f(x) \delta x.
%$$
%In order to find a nearby point that has strictly smaller function value, we need $\delta x^T \Hess f(x) \delta x < 0$. This inspires the following condition:
%\begin{definition}
%A twice differentiable function $f(x)$ is \name, if all its local minimum have $\Hess f(x) \succ 0$ and all its saddle points satisfy $\lambda_{min} \Hess f(x) < -C$ for some fixed constant $C > 0$.
%\end{definition}
%
%This assumption shows it is possible to make local improvements as long as we have access to second order information. We prove that even without using the second order information, a stochastic gradient algorithm can optimize such a function:

\begin{theorem}[informal] Suppose $f(w)$ is \name~(see Definition~\ref{def:robustcondition}), Noisy Gradient Descent (Algorithm~\ref{algo:sgdwn}) outputs a point that is close to a local minimum in polynomial number of steps.
\end{theorem}

%\fhcomment{gets to is different from converges to, as gets to means it can still go out}
%RG: I changed it to outputs.

\paragraph{Online tensor decomposition} Requiring all saddle points to have a negative eigenvalue may seem strong, but it already allows non-trivial applications to natural non-convex optimization problems. %\fhcomment{A wide class of non-convex optimization problems satisfies the condition that all saddles points have a negative eigenvalue. }
%RG: I don't want to overclaim... We don't have ``a wide class'' of problems.
As an example, we consider the orthogonal tensor decomposition problem. This problem is the key step in spectral learning for many latent variable models.

We design a new objective function for tensor decomposition that is \name. 
\begin{theorem}
Given random variables $X$ such that $T = \E[g(X)] \in \R^{d^4}$ is an orthogonal $4$-th order tensor, there is an objective function $f(w) = \E[\phi(w,X)]$ $w\in \R^{d\times d}$ such that every local minimum of $f(w)$ corresponds to a valid decomposition of $T$. Further, function $f$ is \name.
\end{theorem}

Combining this new objective with our framework for optimizing \name functions,  we get the first online algorithm for orthogonal tensor decomposition with global convergence guarantee.

%!TEX root = saddlepoint.tex
\section{Preliminaries}

The stochastic gradient aims to solve the stochastic optimization problem (\ref{eq:opt}), which we restate here:
$$
w = \arg\min_{w\in \R^d} f(w), \textrm{~where~} f(w) = \E_{x\sim \mathcal{D}}[\phi(w,x)].
$$
Recall $\phi(w,x)$ denotes the loss function evaluated for sample $x$ at point $w$.
The algorithm follows a stochastic gradient 
\begin{equation}
w_{t+1} = w_{t} - \eta \nabla_{w_t} \phi(w_t,x_t),
\end{equation}
where $x_t$ is a random sample drawn from distribution $\mathcal{D}$ and $\eta$ is the {\em learning rate}.

In the more general setting, stochastic gradient descent can be viewed as optimizing an arbitrary function $f(w)$ given a stochastic gradient oracle.

\begin{definition}
For a function $f(w):\R^d \to \R$, a function $SG(w)$ that maps a variable to a random vector in $\R^d$ is a stochastic gradient oracle if $\E[SG(w)] = \nabla f(w)$ and $\|SG(w) - \nabla f(w)\| \le Q$.
\end{definition}

In this case the update step of the algorithm becomes $w_{t+1} = w_t - \eta SG(w_t)$.

\paragraph{Smoothness and Strong Convexity } Traditional analysis for stochastic gradient often assumes the function is smooth and strongly convex. A function is $\beta$-smooth if for any two points $w_1,w_2$,
\begin{equation}
\|\nabla f(w_1) - \nabla f(w_2)\| \le \beta\|w_1-w_2\|.
\end{equation}
When $f$ is twice differentiable this is equivalent to assuming that the spectral norm of the Hessian matrix is bounded by $\beta$.
We say a function is $\alpha$-strongly convex if the Hessian at any point has smallest eigenvalue at least $\alpha$ ($\lambda_{min}(\Hess f(w)) \ge \alpha$).

Using these two properties, previous work~\cite{ICML2012Rakhlin_261} shows that stochastic gradient converges at a rate of $1/t$. In this thesis we consider non-convex functions, which can still be $\beta$-smooth but cannot be strongly convex.

\paragraph{Smoothness of Hessians }  
It is  common to assume the Hessian of the function $f$ to be smooth.  We say a function $f(w)$ has $\rho$-Lipschitz Hessian if for any two points $w_1,w_2$ we have
\begin{equation}
\|\Hess f(w_1) - \Hess f(w_2)\| \le \rho \|w_1-w_2\|.
\end{equation}
This is a third order condition that is 
true if the third order derivative exists and is bounded.

%!TEX root = saddlepoint.tex

\section{Stochastic Gradient Descent for \bigname~Function}

\label{sec:sgd}

In this section we discuss the properties of saddle points, and show if all the saddle points are well-behaved then stochastic gradient descent finds a local minimum for a non-convex function in polynomial time.

\paragraph{Notation} Throughout the chapter we use $[d]$ to denote set $\{1,2,...,d\}$. We use $\|\cdot\|$ to denote the $\ell_2$ norm of vectors and spectral norm of matrices. For a matrix we use $\lambda_{min}$ to denote its smallest eigenvalue. For a function $f:\R^d\to \R$,  $\nabla f$ and $\nabla^2 f$ denote its gradient vector and Hessian matrix.

\subsection{\bigname~Property}
\label{subsec:strictsaddleproperty}
For a twice differentiable function $f(w)$, we call a point {\em stationary point} if its gradient is equal to $0$. Stationary points could be local minima, local maxima or saddle points. By local optimality conditions~\cite{wright1999numerical}, in many cases we can tell what type a point $w$ is by looking at its Hessian: if $\Hess f(w)$ is positive definite then $w$ is a local minimum; if $\Hess f(w)$ is negative definite then $w$ is a local maximum; if $\Hess f(w)$ has both positive and negative eigenvalues then $w$ is a saddle point. These criteria do not cover all the cases as there could be degenerate scenarios: $\Hess f(w)$ can be positive semidefinite with an eigenvalue equal to 0, in which case the point could be a local minimum or a saddle point.

If a function does not have these degenerate cases, then we say the function is \name:

\begin{definition}
A twice differentiable function $f(w)$ is {\em \name}, if all its local minima have $\Hess f(w) \succ 0$ and all its other stationary points satisfy $\lambda_{min} (\Hess f(w)) < 0$.
\end{definition}

Intuitively, if we are not at a stationary point, then we can always follow the gradient and reduce the value of the function. If we are at a saddle point, we need to consider a second order Taylor expansion:
$$
f(w+\Delta w) \approx w + (\Delta w)^T \Hess f(w) (\Delta w) + O(\|\Delta w\|^3).
$$
Since the \name~property guarantees $\Hess f(w)$ to have a negative eigenvalue, there is always a point that is near $w$ and has strictly smaller function value. It is possible to make local improvements as long as we have access to second order information. However it is not clear whether the more efficient stochastic gradient updates can work in this setting. 

To make sure the local improvements are significant, we use a robust version of the \name~property:

\begin{definition}
\label{def:robustcondition}
A twice differentiable function $f(w)$ is $(\alpha, \gamma, \epsilon, \delta)$-{\em\name}, if for any point $w$ at least one of the following is true
\begin{enumerate}
\item $\|\nabla f(w)\| \ge \epsilon$.
\item $\lambda_{min}(\Hess f(w)) \le -\gamma$.
\item There is a local minimum $w^\star$ such that $\|w-w^\star\| \le \delta$, and the function $f(w')$ restricted to $2\delta$ neighborhood of $w^\star$ ($\|w'-w^\star\|\le 2\delta$) is $\alpha$-strongly convex.
\end{enumerate}
\end{definition}

Intuitively, this condition says for any point whose gradient is small, it is either close to a robust local minimum, or is a saddle point (or local maximum) with a significant negative eigenvalue. 
%Note that the $\epsilon$ in Conditions 1 and 3 need not be the same quantity, here we use the same $\epsilon$ to simplify presentation.

%!TEX root = 0colt2015-main.tex

\begin{algorithm}[ht]
 \caption{Noisy Stochastic Gradient}
 \label{algo:sgdwn}
 \begin{algorithmic}[1]
 \REQUIRE Stochastic gradient oracle $SG(w)$, initial point $w_0$, desired accuracy $\kappa$.
  \ENSURE $w_t$ that is close to some local minimum $w^\star$.
  \STATE Choose $\eta = \min\{\tilde{O}(\kappa^2 / \log (1/\kappa)), \eta_{\max}\}$
 	\FOR{$t = 0$ to $\tilde{O}(1/\eta^2)$}
 	\STATE Sample noise $n$ uniformly from unit sphere.
	\STATE $w_{t+1} \leftarrow w_{t} - \eta (SG(w) + n)$
% 		\FORALL {$1$}
% 		\IF {$ll $}
% 		\STATE blah
% 		\STATE blah blah.
% 		\ENDIF
% 		\ENDFOR
 		\ENDFOR
 \end{algorithmic}
 \end{algorithm}

We purpose a simple variant of stochastic gradient algorithm, where the only difference to the traditional algorithm is we add an extra noise term to the updates. The main benefit of this additional noise is that we can guarantee there is noise in every direction, which allows the algorithm to effectively explore the local neighborhood around saddle points. If the noise from stochastic gradient oracle already has nonnegligible variance in every direction, our analysis also applies without adding additional noise. We show noise can help the algorithm escape from saddle points and optimize \name~functions.

\begin{theorem} [Main Theorem]\label{thm:sgdmain}
Suppose a function $f(w):\R^d\to \R$ that is $(\alpha, \gamma, \epsilon, \delta)$-\name, and has a stochastic gradient oracle with radius at most $Q$. Further, suppose the function is bounded by $|f(w)| \le B$, is $\beta$-smooth and has $\rho$-Lipschitz Hessian. 
Then there exists a threshold $\eta_{\max} = \tilde{\Theta}(1)$, so that for any $\zeta>0$, and
for any $\eta \le \eta_{\max} / \max\{1, \log(1/\zeta)\}$, with probability at least $1-\zeta$ in $t = \tlO(\eta^{-2}\log (1/\zeta))$ iterations, Algorithm~\ref{algo:sgdwn} (Noisy Gradient Descent) outputs a point $w_t$ that is $\tlO(\sqrt{\eta\log(1/\eta\zeta)})$-close to some local minimum $w^\star$.
\end{theorem}

Here (and throughout the rest of the chapter) $\tlO(\cdot)$ ($\tilde{\Omega},\tilde{\Theta}$) hides the factor that is polynomially dependent on all other parameters (including $Q$, $1/\alpha$, $1/\gamma$, $1/\epsilon$, $1/\delta$, $B$, $\beta$, $\rho$, and $d$), but independent of $\eta$ and $\zeta$. So it focuses on the dependency on $\eta$ and $\zeta$. 
Our proof technique can give explicit dependencies on these parameters however we hide these dependencies for simplicity of presentation. \footnote{
Currently, our number of iteration is a large polynomial in the dimension $d$. We have not tried to optimize the degree of this polynomial. Empirically the dependency on $d$ is much better, whether the dependency on $d$ can be improved to $\mbox{poly}\log d$ is left as an open problem.% and we believe the current analysis is an interesting starting point that can be further improved. 
}

%\jccomment{This dependence on $\zeta$ is not correct, should be poly($\zeta$)
%Also, why not just use $O$ to hide constant? and our proof did not give explicit dependence?}
%\grcomment{The dependence on $\zeta$ has to be correct if in the ``never escaping local minimum'' lemma you make the probability also depend on $\zeta$, which only incurs a $\log \zeta$ loss.}

\begin{remark} [Decreasing learning rate]
Often analysis of stochastic gradient descent uses decreasing learning rates and the algorithm converges to a local (or global) minimum. Since the function is strongly convex in the small region close to local minimum, we can use Theorem \ref{thm:sgdmain} to first find a point that is close to a local minimum, and then apply standard analysis of SGD in the strongly convex case (where we decrease the learning rate by $1/t$ and get $1/\sqrt{t}$ convergence in $\|w-w^\star\|$).
\end{remark}
%
%\jccomment{I don't quite understand this remark. If we do this, we no longer have $\tlO(\eta^{-2})$ guarantee (which needs more steps), and
%why step is $\eta \approx t^{-r/2}$, do you mean $(t/2)^{-r}$}

In the next part we sketch the proof of the main theorem. Details are deferred to Appendix~\ref{sec:unconstrained}.

\subsection{Proof Sketch}

In order to prove Theorem~\ref{thm:sgdmain}, we analyze the three cases in Definition~\ref{def:robustcondition}. When the gradient is large, we show the function value decreases in one step (see Lemma~\ref{lem:gradient}); when the point is close to a local minimum, we show with high probability it cannot escape in the next polynomial number of iterations (see Lemma~\ref{lem:minimum}). 

\begin{lemma}[Gradient]
\label{lem:gradient}
Under the assumptions of Theorem~\ref{thm:sgdmain}, for any point with $\|\nabla f(w_t)\| $ $\ge$ $ C\sqrt{\eta}$ (where $C = \tilde{\Theta}(1)$) and $C\sqrt{\eta} \le \epsilon$, after one iteration we have $\E[f(w_{t+1})] \le f(w_t) - \tlOmega(\eta^2)$.
\end{lemma}

The proof of this lemma is a simple application of the smoothness property.

\begin{lemma}[Local minimum]
\label{lem:minimum}
Under the assumptions of Theorem~\ref{thm:sgdmain}, for any point $w_t$ that is $\tlO(\sqrt{\eta}) < \delta$ close to local minimum $w^\star$, in $\tlO(\eta^{-2}\log (1/\zeta))$ number of steps all future $w_{t+i}$'s are $\tlO(\sqrt{\eta\log(1/\eta\zeta)})$-close with probability at least $1-\zeta/2$.
\end{lemma}

The proof of this lemma is similar to the standard analysis \cite{ICML2012Rakhlin_261} of stochastic gradient descent in the smooth and strongly convex setting, except we only have local strong convexity. The proof appears in Appendix~\ref{sec:unconstrained}.

%\jccomment{This is not exact, since their analysis is for global convex function, 
%which need not to tackle with local property. I think our deveploped indicator martingale is 
%a good tool for analysis function with local properties.
%Also their result is extremely loose in our case, I would say similar to the standard analysis, but needs take care of local property}

The hardest case is when the point is ``close'' to a saddle point: it has gradient smaller than $\epsilon$ and smallest eigenvalue of the Hessian bounded by $-\gamma$. In this case we show the noise in our algorithm helps the algorithm to escape:

\begin{lemma}[Saddle point]
\label{lem:saddle}
Under the assumptions of Theorem~\ref{thm:sgdmain}, for any point $w_t$ where $\|\nabla f(w_t)\| \le C\sqrt{\eta}$ (for the same $C$ as in Lemma~\ref{lem:gradient}), and $\lambda_{\min}(\Hess f(w_t)) \le -\gamma$, there is a number of steps $T$ that depends on $w_t$ such that $\E[f(w_{t+T})] \le f(w_t)-\tlOmega(\eta)$. The number of steps $T$ has a fixed upper bound $T_{max}$ that is independent of $w_t$ where $T \le T_{max} = \tilde{O}(1/\eta)$.
\end{lemma}

Intuitively, at point $w_t$ there is a good direction that is hiding in the Hessian. The hope of the algorithm is that the additional (or inherent) noise in the update step makes a small step towards the correct direction, and then the gradient information will reinforce this small perturbation and the future updates will ``slide'' down the correct direction.

To make this more formal, we consider a coupled sequence of updates $\tilde{w}$ such that the function to minimize is just the local second order approximation $$\tilde{f}(w) = f(w_t) + \nabla f(w_t)^T (w-w_t) + \frac{1}{2}(w-w_t)^T\Hess f(w_t) (w-w_t).$$

%The dynamics of stochastic gradient descent for this quadratic function is easy to analyze as the expectation point satisfy linear recursions, and indeed we show the expectation will decrease. We then use the smoothness of the function to show that as long as the points did not go very far from $w_t$, the two update sequences $\tilde{w}$ and $w$ will remain close to each other. Finally we prove the future $w_{t+i}$'s (in the next $T$ steps) will remain close to $w_t$ with high probability by Martingale bounds. The detailed proof appears in Appendix~\ref{}.

	The dynamics of stochastic gradient descent for this quadratic function is easy to analyze as 
	$\tilde{w}_{t+i}$ can be calculated analytically. Indeed, we show the expectation of 
	$\tilde{f}(\tilde{w})$ will decrease. More concretely we show the point $\tilde{w}_{t+i}$ will move substantially in the negative curvature directions and remain close to $w_t$ in positive curvature directions.
	We then use the smoothness of the function to show that as long as the points did not go very far from $w_t$, the two update sequences $\tilde{w}$ and $w$ will remain close to each other, and thus $\tilde{f}(\tilde{w}_{t+i}) \approx f(w_{t+i})$. Finally we prove the future $w_{t+i}$'s (in the next $T$ steps) will remain close to $w_t$ with high probability by Martingale bounds. The detailed proof appears in Appendix~\ref{sec:unconstrained}.

With these three lemmas it is easy to prove the main theorem. Intuitively, as long as there is a small probability of being $\tlO(\sqrt{\eta})$-close to a local minimum, we can always apply Lemma~\ref{lem:gradient} or Lemma~\ref{lem:saddle} to make the expected function value decrease by $\tlOmega(\eta)$ in at most $\tlO(1/\eta)$ iterations, this cannot go on for more than $\tlO(1/\eta^2)$ iterations because in that case the expected function value will decrease by more than $2B$, but $\max f(x) - \min f(x) \le 2B$ by our assumption. Therefore in $\tlO(1/\eta^2)$ steps with at least constant probability $w_t$ will become $\tilde{O}(\sqrt{\eta})$-close to a local minimum. By Lemma~\ref{lem:minimum} we know once it is close it will almost always stay close, so after $q$ epochs of $\tilde{O}(1/\eta^2)$ iterations each, the probability of success will be $1-\exp(-\Omega(q))$. Taking $q = O(\log (1/\zeta))$ gives the result. More details appear in Appendix~\ref{sec:unconstrained}.
%
%\begin{proof}(sketch for the main theorem)
%We consider the dynamics of the stochastic gradient updates, and keep track of the expectation of the function value. In order to prove the main theorem, it is essential to look at points that are at different iterations depending on the history of the updates. In particular we define $\tau_i$ recursively as follows: $\tau_0 = 0$, and the next value is
%$$
%\tau_{i+1} = \left\{ \begin{array}{ll}\tau_i+1 & \nabla f(w_{\tau_i}) \ge ?? \\
%\tau_i & \exists w^\star, \|w-w^\star\| \le \tlO(\sqrt{\eta})\\
%\tau_i+T & T\mbox{ in Lemma~\ref{lem:saddle},otherwise} \end{array}\right.
%$$
%
%We also define $p_i$ to be the probability that $w_{\tau_i}$ is $\tlO(\sqrt{\eta})$ close to a local minimum. Now we first prove after a small number of iterations, $p_i \ge 1/2$. This is because by Lemma~\ref{lem:gradient} and \ref{lem:saddle}, we always have
%$$
%\E[f(w_{\tau_{i+1}})] \le \E[f(w_{\tau_i}] - (1-p_i) \Omega(\eta).
%$$
%
%Therefore suppose $p_i$'s are always smaller than $1/2$ for $i$ up to $O(R/\eta)$, we will have $\E[f(w_{\tau_i})] \le f(w_0) - 2R$ which is impossible because the function is bounded.
%
%Now since in $\tlO(\eta^{-2})$ iterations with probability at least $1/2$ the point $w_t$ becomes small enough, and by Lemma~\ref{lem:minimum} we know once it is small enough it does not escape, we get the desired guarantee in main theorem.
%\end{proof}
\subsection{Constrained Problems}
\label{sec:constrainedproblem}

In many cases, the problem we are facing are constrained optimization problems. In this part we briefly describe how to adapt the analysis to problems with equality constraints (which suffices for the tensor application). Dealing with general inequality constraint is left as future work.

For a constrained optimization problem:
\begin{align}
&\min_{w\in \R^d} \quad \quad  f(w) \\
&\text{s.t.} \quad \quad c_i(w) = 0, \quad \quad i\in[m]\nonumber
\end{align}
in general we need to consider the set of points in a low dimensional manifold that is defined by the constraints. In particular, in the algorithm after every step we need to project back to this manifold (see Algorithm~\ref{algo:psgdwn} where $\Pi_\mathcal{W}$ is the projection to this manifold).

%!TEX root = 0colt2015-main.tex

\begin{algorithm}[ht]
 \caption{Projected Noisy Stochastic Gradient}
 \label{algo:psgdwn}
 \begin{algorithmic}[1]
 \REQUIRE Stochastic gradient oracle $SG(w)$, initial point $w_0$, desired accuracy $\kappa$.
  \ENSURE $w_t$ that is close to some local minimum $w^\star$.
  \STATE Choose $\eta = \min\{\tilde{O}(\kappa^2/\log (1/\kappa)), \eta_{\max}\}$
 	\FOR{$t = 0$ to $\tilde{O}(1/\eta^2)$}
 	\STATE Sample noise $n$ uniformly from unit sphere.
	\STATE $v_{t+1} \leftarrow w_{t} - \eta (SG(w) + n)$
	\STATE $w_{t+1} = \Pi_{\mathcal{W}}(v_{t+1})$
% 		\FORALL {$1$}
% 		\IF {$ll $}
% 		\STATE blah
% 		\STATE blah blah.
% 		\ENDIF
% 		\ENDFOR
 		\ENDFOR
 \end{algorithmic}
 \end{algorithm}

For constrained optimization it is common to consider the Lagrangian:
\begin{equation}
\mathcal{L}(w, \lambda) =  f(w) - \sum_{i=1}^m \lambda_i c_i(w).
\end{equation}

%\jccomment{Since we already identified the regularity conditions, 
%should we talk a bit about $\alpha_c$-RICQ? I think that's also
%one of our inovation, and we may also need to check it in the following chapter.}
%\grcomment{I prefer keeping it short and high level. It is hard to expect reviewers to pick up contributions at this level.}

Under common regularity conditions, it is possible to compute the value of the Lagrangian multipliers: $$\lambda^*(w)=\arg\min_{\lambda} \|\nabla_w \mathcal{L}(w, \lambda)\|.$$ We can also define the tangent space, which contains all directions that are orthogonal to all the gradients of the constraints: $\mathcal{T}(w) =\{v: \nabla c_i(w)^T v = 0; ~ i=1, \cdots, m \}$. In this case the corresponding gradient and Hessian we consider are the first-order and second-order partial derivative of Lagrangian $\mathcal{L}$
at point $(w, \lambda^*(w))$:
% \begin{align*}
% \xi(w) & = \mbox{Proj}_{\mathcal{T}(w)} \nabla f(w) =  \nabla f(w) - \sum_i \lambda_i^*(w)\nabla c_i(w)\\
% \mathfrak{M}(w) & = \Hess_{ww}\mathcal{L}(w,\lambda^*(w)) = \Hess f(w) - \sum_i \lambda_i^*(w)\Hess c_i(w).
% \end{align*}
\begin{align}
&\chi(w) = \nabla_w \mathcal{L}(w, \lambda) |_{(w, \lambda^*(w))} =\nabla f(w) - \sum_{i=1}^m \lambda^*_i(w) \nabla c_i(w)  \\
&\mathfrak{M}(w) = \Hess_{ww} \mathcal{L}(w, \lambda) |_{(w, \lambda^*(w))} = \Hess f(w) - \sum_{i=1}^m \lambda^*_i(w) \nabla^2 c_i(w) 
\end{align}

We replace the gradient and Hessian with $\chi(w)$ and $\mathfrak{M}(w)$, and when computing eigenvectors of $\mathfrak{M}(w)$ we focus on its projection on the tangent space. In this way, we can get a similar definition for \name ~(see Appendix~\ref{sec:constrained}), and the following theorem.

%\jccomment{if we talk this simple, I think most people don't understand what we are talking about.
%we can write down some formula like 
%``replacing $\lambda_{\min}(\Hess f(w))$ in Definition for \name by
%$\min_{\hat{d}} \{\hat{d}^T\mathfrak{M}(w) \hat{d} | \hat{d} \in\mathcal{T}(w), \|\hat{d}\| = 1 \}$''
%Under common regularity condition, $\chi(w) = 0$ is necessary condition for stationary point,
%$\min_{\hat{d}} \{\hat{d}^T\mathfrak{M}(w) \hat{d} | \hat{d} \in\mathcal{T}(w), \|\hat{d}\| = 1 \}>0$ is sufficient condition for local minima}
%\grcomment{We will talk simple and refer to the appendix, where there should be an appropriate new definition and a theorem}

% Under common regularity conditions, it is possible to compute the $\lambda_i$'s for a particular point (which we denote by $\lambda^*(w)$). We can also define the tangent space $\mathcal{T}(w) =\{v: \nabla c_i(w)^T v = 0; ~ i=1, \cdots, m \}$. In this case the gradients and Hessian we consider are
% \begin{align*}
% \xi(w) & = \mbox{Proj}_{\mathcal{T}(w)} \nabla f(w) =  \nabla f(w) - \sum_i \lambda_i^*(w)\nabla c_i(w)\\
% \mathfrak{M}(w) & = \Hess_{ww}\mathcal{L}(w,\lambda^*(w)) = \Hess f(w) - \sum_i \lambda_i^*(w)\Hess c_i(w).
% \end{align*}
% Also when computing eigenvectors of Hessian matrix we focus on its projection on the tangent space. Replacing the gradient and Hessian with $\xi(w)$ and $\mathfrak{M}(w)$, we can get a similar definition for \name, and the following theorem

\begin{theorem}(informal)\label{thm:constrainedinformal}
Under regularity conditions and smoothness conditions, if a constrained optimization problem satisfies \name~property, then for a small enough $\eta$, in $\tlO(\eta^{-2}\log 1/\zeta)$ iterations Projected Noisy Gradient Descent (Algorithm~\ref{algo:psgdwn}) outputs a point $w$ that is $\tlO(\sqrt{\eta}\log (1/\eta\zeta))$ close to a local minimum with probability at least $1-\zeta$.
\end{theorem}

Detailed discussions and formal version of this theorem are deferred to Appendix~\ref{sec:constrained}.

%\jccomment{Should we give more detailed theorem?}\grcomment{in the appendix}

%!TEX root = saddlepoint.tex

\section{Online Tensor Decomposition}\label{sec:tensors}

In this section we describe how to apply our stochastic gradient descent analysis to tensor decomposition problems. We first give a new formulation of tensor decomposition as an optimization problem, and show that it satisfies the \name~property. Then we explain how to compute stochastic gradient in a simple example of Independent Component Analysis (ICA)~\cite{icabook}.

\subsection{Optimization Problem for Tensor Decomposition}

Given a tensor $T\in \R^{d^4}$ that has an orthogonal decomposition
\begin{equation}
T = \sum_{i=1}^d a_i^{\otimes 4},
\end{equation}
where the components $a_i$'s are orthonormal vectors ($\|a_i\| = 1$, $a_i^Ta_j = 0$ for $i\ne j$), the goal of orthogonal tensor decomposition is to find the components $a_i$'s. This problem has inherent symmetry: for any permutation $\pi$ and any set of $\kappa_i\in \{\pm 1\},i\in[d]$, we know $u_i = \kappa_i a_{\pi(i)}$ is also a valid solution. This symmetry property makes the natural optimization problems non-convex.

In this section we will give a new formulation of orthogonal tensor decomposition as an optimization problem, and show that this new problem satisfies the \name~property. Previously, \cite{frieze1996learning} solves the problem of finding one component, with the following objective function
\begin{equation}\label{eq:findone}
\max\limits_{\|u\|^2 = 1}  \quad T(u,u,u,u).
\end{equation}
In Appendix \ref{sec:warmup}, as a warm-up example we show this function is indeed \name, and we can apply Theorem~\ref{thm:constrainedinformal} to prove global convergence of stochastic gradient descent algorithm.

It is possible to find all components of a tensor by iteratively finding one component, and do careful {\em deflation}, as described in \cite{JMLR:v15:anandkumar14b} or \cite{arora2012provable}. However, in practice the most popular approaches like Alternating Least Squares \cite{comon2009tensor} or FastICA \cite{hyvarinen1999fast} try to use a single optimization problem to find all the components. Empirically these algorithms are often more robust to noise and model misspecification.

The most straight-forward formulation of the problem aims to minimize the {\em reconstruction error}

\begin{equation}\label{eq:reconstruction}
\min\limits_{\forall i, \|u_i\|^2 = 1} \quad \| T - \sum_{i=1}^d u_i^{\otimes 4}\|_F^2.
\end{equation}
Here $\|\cdot \|_F$ is the Frobenius norm of the tensor which is equal to the $\ell_2$ norm when we view the tensor as a $d^4$ dimensional vector. However, it is not clear whether this function satisfies the \name~property, and empirically stochastic gradient descent is unstable for this objective.

We propose a new objective that aims to minimize the correlation between different components:
\begin{equation}\label{eq:hardprob}
\min\limits_{\forall i, \|u_i\|^2 = 1}  \quad \sum_{i\ne j} T(u_i,u_i,u_j,u_j),
\end{equation}
To understand this objective intuitively, we first expand vectors $u_k$ in the orthogonal basis formed by $\{a_i\}$'s. That is, we can write $u_k = \sum_{i=1}^{d}z_k(i) a_i$, where $z_k(i)$ are scalars that correspond to the coordinates in the $\{a_i\}$ basis. In this way we can rewrite $T(u_k,u_k,u_l,u_l) = \sum_{i=1}^{d} (z_k(i))^2 (z_l(i))^2$. From this form it is clear that the $T(u_k,u_k,u_l,u_l)$ is always nonnegative, and is equal to $0$  only when the support of $z_k$ and $z_l$ do not intersect. For the objective function, we know in order for it to be equal to 0 the $z$'s must have disjoint support. Therefore, we claim that $\{u_k\}, \forall k\in[d]$ is equivalent to $\{a_i\}, \forall i\in[d]$ up to permutation and sign flips when the global minimum (which is 0) is achieved. 
%\textcolor[rgb]{1,0,0}{Suppose we express $u_1 = \sum_{i=1}^d z_i a_i$ and $u_2 = \sum_{i=1}^d w_i a_i$, then by the definition of multilinear form we know $T(u_1,u_1,u_2,u_2) = \sum_{i=1}^d z_i^2 w_i^2$, which is always nonnegative and is only equal to $0$ when the support of $z$ and $w$ do not intersect. Using this property it is easy to see that the only solutions that have $0$ value correspond to correct decompositions. }

We further show that this optimization program satisfies the \name~property and all its local minima in fact achieves global minimum value. The proof is deferred to Appendix \ref{sec:hardcase}.
\begin{theorem}
The optimization problem (\ref{eq:hardprob}) is $(\alpha, \gamma, \epsilon,\delta)$-\name, for $\alpha = 1$ and $\gamma,\epsilon,\delta = 1/\mbox{poly}(d)$. Moreover, all its local minima have the form $u_i = \kappa_i a_{\pi(i)}$ for some $\kappa_i = \pm 1$ and permutation $\pi(i)$.
\end{theorem}

Note that we can also generalize this to handle 4th order tensors with different positive weights on the components, or other order tensors, see Appendix~\ref{app:tensorextension}.

\subsection{Implementing Stochastic Gradient Oracle}
\label{sec:icagrad}
%Stochastic gradient oracle requires online observations of the gradient of the objective function~\eqref{eq:hardprob},  this entails a detailed analysis of the generative model of the tensors. 
To design an online algorithm based on objective function \eqref{eq:hardprob}, we need to give an implementation for the stochastic gradient oracle.

In applications, the tensor $T$ is oftentimes the expectation of multilinear operations of samples $g(x)$ over $x$ where $x$ is generated from some distribution $\mathcal{D}$. In other words, for any $x\sim \mathcal{D}$, the tensor is $T =\E[g(x)] $. Using the linearity of the multilinear map, we know $\E[g(x)] (u_i,u_i,u_j,u_j) = \E[g(x)(u_i,u_i,u_j,u_j)]$. Therefore we can define the loss function $\phi(u,x) = \sum_{i\ne j} g(x)(u_i,u_i,u_j,u_j)$, and the stochastic gradient oracle $SG(u) = \nabla_u \phi(u,x)$.

For concreteness, we look at a simple ICA example. In the simple setting we consider an unknown signal $x$ that is uniform\footnote{In general ICA the entries of $x$ are independent, non-Gaussian variables.} in $\{\pm 1\}^d$, and an unknown orthonormal linear transformation\footnote{In general (under-complete) ICA this could be an arbitrary linear transformation, however usually after the ``whitening'' step (see \cite{cardoso1989source}) the linear transformation becomes orthonormal.} $A$ ($AA^T = I$). The sample we observe is $y :=  Ax \in \R^d$. Using standard techniques (see \cite{cardoso1989source}), we know the $4$-th order cumulant of the observed sample is a tensor that has orthogonal decomposition. Here for simplicity we don't define 4-th order cumulant, instead we give the result directly.

Define tensor $Z\in \R^{d^4}$ as follows:
%\fhcomment{choose from the two? I prefer the first one.}
%\fhcomment{1:}
%\begin{equation*}
%Z(i,j,k,l) = \left\{
%\begin{array}{ll}
% 3, &    \forall i=j=k=l \in [d]  \\
% 1, & \forall i=j, k=l, i\neq k \in [d]\\
% 1, & \forall i=k,j=l, i\neq j \in [d]\\
% 1, & \forall i=l, j=k, i\neq j \in [d]\\
% 0, &o.w.
%\end{array}
%\right.
%\end{equation*}
%\fhcomment{2:}
\begin{equation*}
\begin{array}{ll}
Z(i,i,i,i) =3, &   \forall i\in [d] \\
Z(i,i,j,j) = Z(i,j,i,j) = Z(i,j,j,i) = 1, &\forall i\ne j\in [d]\\
\end{array}
\end{equation*}
where all other entries of $Z$ are equal to $0$. 
The tensor $T$ can be written as a function of the auxiliary tensor $Z$ and multilinear form of the sample $y$.
\begin{lemma}\label{lm:constructTensorZ}
The expectation $\E[\frac{1}{2}(Z - y^{\otimes 4})] = \sum_{i=1}^d a_i^{\otimes 4}=T$, where $a_i$'s are columns of the unknown orthonormal matrix $A$.
\end{lemma}

This lemma is easy to verify, and is closely related to cumulants~\cite{cardoso1989source}.  Recall that $\phi(u,y)$ denotes the loss (objective) function evaluated at sample $y$ for point $u$. Let $\phi(u,y) = \sum_{i\ne j} \frac{1}{2}(Z - y^{\otimes 4})(u_i,u_i,u_j,u_j)$. By Lemma~\ref{lm:constructTensorZ}, we know that  $\E[\phi(u,y)]$ is equal to the objective function as in Equation~\eqref{eq:hardprob}.
Therefore we rewrite objective (\ref{eq:hardprob}) as the following stochastic optimization problem
\begin{equation*}
\min\limits_{\forall i, \|u_i\|^2 = 1} \quad  \E[\phi(u,y)] ,~\text{where}~ \phi(u,y) = \sum_{i\ne j} \frac{1}{2}(Z - y^{\otimes 4})(u_i,u_i,u_j,u_j)
\end{equation*}
%We emphasize that the expectation of the loss function $\E[\phi(u,y)]$ is equivalent to $T(u_i,u_i,u_j,u_j)$ due to Lemma~\ref{lm:constructTensorZ}.
The stochastic gradient oracle is then 
\begin{equation}\label{eq:icasg}
\nabla_{u_i} \phi(u,y) = \sum\limits_{j\neq i}\left(\left\langle u_j ,u_j \right\rangle u_i  + 2 \left\langle u_i ,u_j \right\rangle u_j  - \left\langle u_j , y\right\rangle^2 \left\langle u_i ,y\right\rangle y \right).
\end{equation}
Notice that computing this stochastic gradient does not require constructing the $4$-th order tensor $T - y^{\otimes 4}$. In particular, this stochastic gradient can be computed very efficiently: 
\begin{claim}
The stochastic gradient (\ref{eq:icasg}) can be computed for all $u_i$'s in $O(d^3)$ time for one sample or $O(d^3+d^2k)$ for average of $k$ samples. 
\end{claim}

\begin{proof} The proof is straight forward as the first two terms  on the right hand side take $O(d^3)$ and is shared by all samples. The third term can be efficiently computed once the inner-products between all the $y$'s and all the $u_i$'s are computed (which takes $O(kd^2)$ time).
\end{proof}

%!TEX root = saddlepoint.tex
\section{Experiments}\label{sec:experi}

We run simulations for Projected Noisy Gradient Descent (Algorithm~\ref{algo:psgdwn}) applied to orthogonal tensor decomposition.
The results show that the algorithm converges from random initial points efficiently (as predicted by the theorems), and our new formulation (\ref{eq:hardprob}) performs better than reconstruction error (\ref{eq:reconstruction}) based formulation.

\paragraph{Settings} We set dimension $d = 10$, the input tensor $T$ is a random tensor in $\R^{10^4}$ that has orthogonal decomposition (\ref{eq:orthodecomp}). The step size is chosen carefully for respective objective functions. The performance is measured by normalized reconstruction error $\mathcal{E} =\left({\|T - \sum_{i=1}^{d} u_i ^{\otimes4}\|_F^2}\right)/{\| T\|_F^2}$. 

\paragraph{Samples and stochastic gradients}
We use two ways to generate samples and compute stochastic gradients. In the first case we generate  sample $x$ by setting it equivalent to $d^{\frac{1}{4}} a_i$ with probability $1/d$. It is easy to see that $\E[x^{\otimes 4}] = T$. This is a very simple way of generating samples, and we use it as a sanity check for the objective functions.

In the second case we consider the ICA example introduced in Section~\ref{sec:icagrad}, and use Equation (\ref{eq:icasg}) to compute a stochastic gradient. In this case the stochastic gradient has a large variance, so we use mini-batch of size 100 to reduce the variance.

\paragraph{Comparison of objective functions} We use the simple way of generating samples for our new objective function (\ref{eq:hardprob}) and reconstruction error objective (\ref{eq:reconstruction}). The result is shown in Figure~\ref{fig:obj}. Our new objective function is empirically more stable (always converges within 10000 iterations); the reconstruction error do not always converge within the same number of iterations and often exhibits long periods with small improvement (which is likely to be caused by saddle points that do not have a significant negative eigenvalue).

\paragraph{Simple ICA example} As shown in Figure~\ref{fig:ICA}, our new algorithm also works in the ICA setting. When the learning rate is constant the error stays at a fixed small value. When we decrease the learning rate the error converges to 0.

%%%%%%%%%%%%%%%%%%%%%%%%%%%%%%%%%%%%

%\begin{figure}[!htb]
%\hfill
%\subfigure[New Objective (\ref{eq:hardprob})]{\psfrag{reconstruction error}[Bc]{\scriptsize  Reconstruction Error}\psfrag{iter}[c]{\scriptsize  Iteration}\includegraphics[width=0.36\textwidth]{\fighomeSaddle/NG.eps}}\label{fig:NG1}
%\hfill
%\subfigure[Reconstruction Error Objective (\ref{eq:reconstruction})]{\psfrag{reconstruction error}[Bc]{\scriptsize  Reconstruction Error}\psfrag{iter}[c]{\scriptsize  Iteration}\includegraphics[width=0.36\textwidth]{\fighomeSaddle/IO.eps}}\label{fig:IO1}
%\hfill
%\caption{Comparison of different objective functions}\label{fig:obj}
%\end{figure}
%
%\begin{figure}[!htb]
%\hfill
%\subfigure[Constant Learning Rate $\eta$]{\psfrag{reconstruction error}[Bc]{\scriptsize  Reconstruction Error}\psfrag{iter}[c]{\scriptsize  Iteration}\includegraphics[width=0.36\textwidth]{\fighomeSaddle/ICA.eps}}\label{fig:ICA1}
%\hfill
%\subfigure[Learning Rate $\eta/t$ (in $\log$ scale)]{\psfrag{reconstruction error}[Bc]{\scriptsize  Reconstruction Error}\psfrag{iter}[c]{\scriptsize  Iteration}\includegraphics[width=0.36\textwidth]{\fighomeSaddle/ICA_decrease.eps}}\label{fig:ICA2}
%\hfill
%\caption{ICA setting performance with mini-batch of size 100}\label{fig:ICA}
%\end{figure}

\begin{figure}[!htb]
\begin{minipage}{0.5\textwidth}
\subfloat[New Objective (\ref{eq:hardprob})]
{\psfrag{reconstruction error}[Bc]{\scriptsize  Reconstruction Error}\psfrag{iter}[c]{\scriptsize  Iteration}\includegraphics[width=\textwidth]{\fighomeSaddle/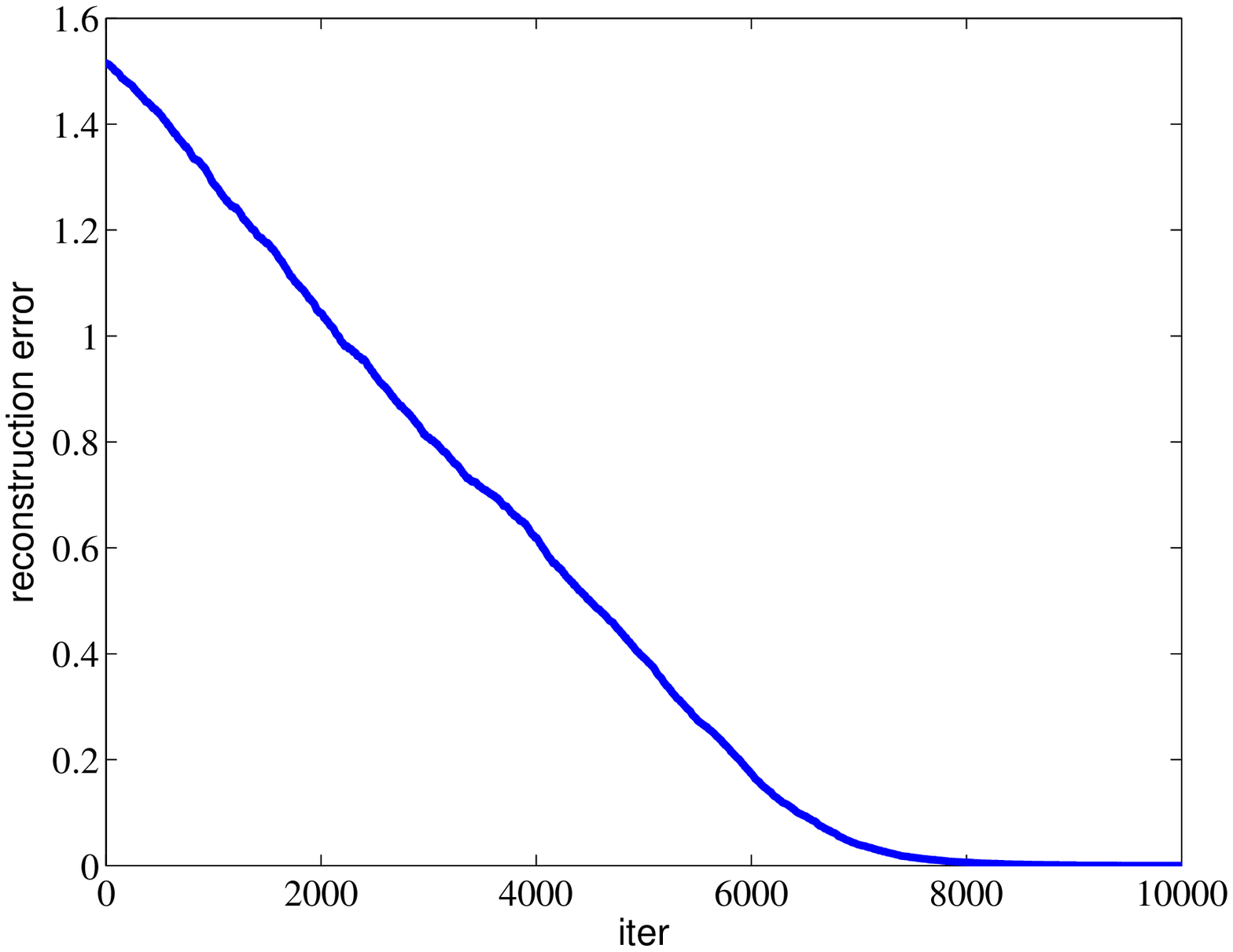}}
\end{minipage}\label{fig:NG1}
\hfil
\begin{minipage}{0.5\textwidth}
\subfloat[Reconstruction Error Objective (\ref{eq:reconstruction})]{\psfrag{reconstruction error}[Bc]{\scriptsize  Reconstruction Error}\psfrag{iter}[c]{\scriptsize  Iteration}\includegraphics[width=\textwidth]{\fighomeSaddle/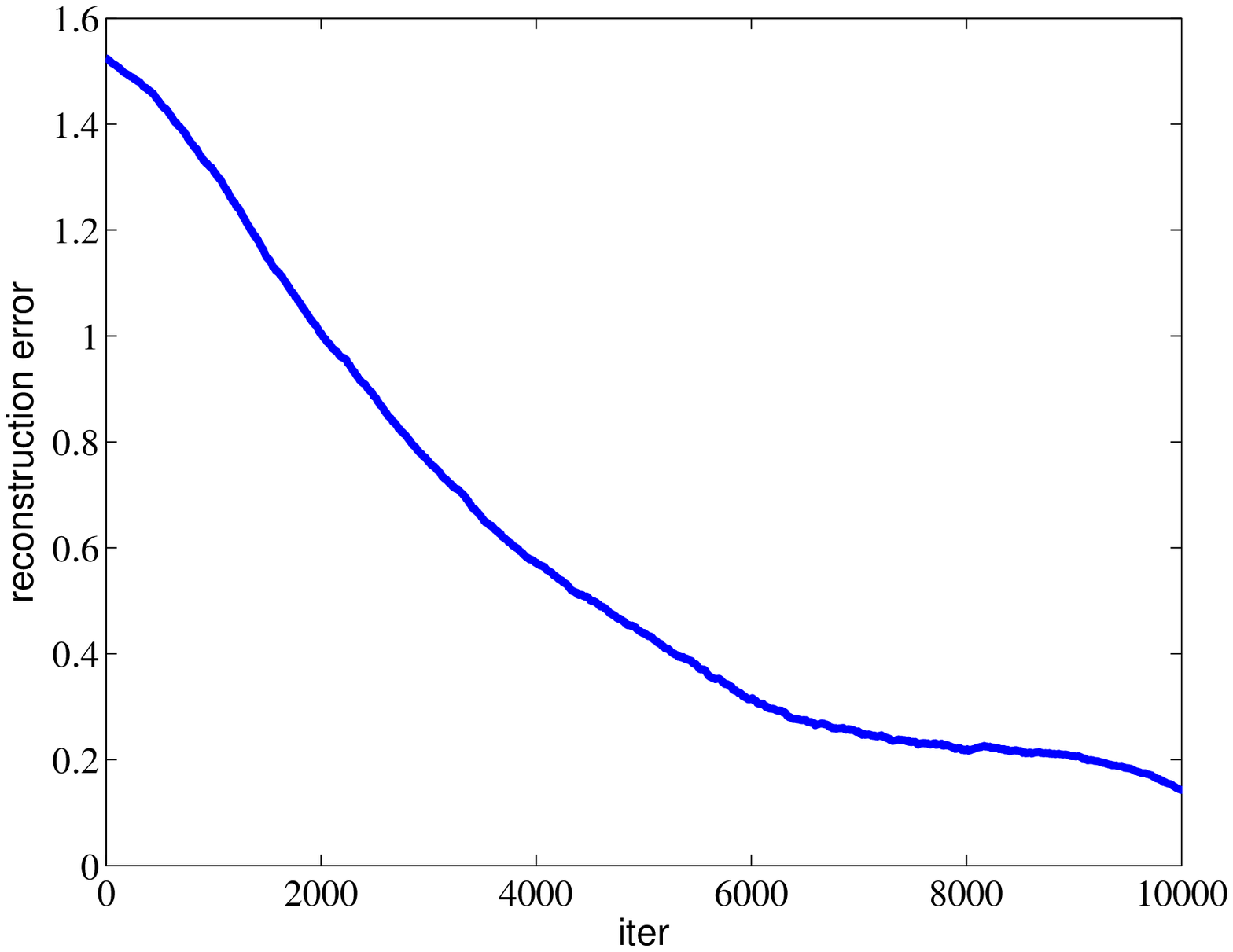}}
\end{minipage}\label{fig:IO1}
\caption{Comparison of different objective functions}\label{fig:obj}
\end{figure}

\begin{figure}[!htb]
\begin{minipage}{0.5\textwidth}
\subfloat[Constant Learning Rate $\eta$]{\psfrag{reconstruction error}[Bc]{\scriptsize  Reconstruction Error}\psfrag{iter}[c]{\scriptsize  Iteration}\includegraphics[width=\textwidth]{\fighomeSaddle/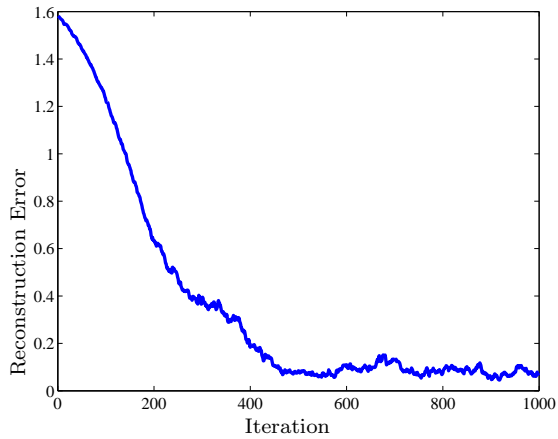}}
\end{minipage}\label{fig:ICA1}
\hfil
\begin{minipage}{0.5\textwidth}
\subfloat[Learning Rate $\eta/t$ (in $\log$ scale)]{\psfrag{reconstruction error}[Bc]{\scriptsize  Reconstruction Error}\psfrag{iter}[c]{\scriptsize  Iteration}\includegraphics[width=\textwidth]{\fighomeSaddle/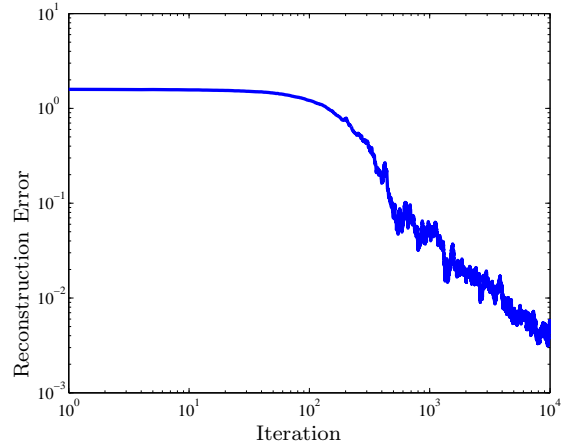}}
\end{minipage}\label{fig:ICA2}
\caption{Comparison of different objective functions}\label{fig:ICA}
\end{figure}

%!TEX root = saddlepoint.tex

\section{Conclusion}
In this chapter we identify the \name~property and show stochastic gradient descent converges to a local minimum under this assumption. This leads to new online algorithm for orthogonal tensor decomposition. We hope this is a first step towards understanding stochastic gradient for more classes of non-convex functions. We believe \name~property can be extended to handle more functions, especially those functions that have similar symmetry properties.

\chapter{Applying Online Tensor Methods for Learning Latent Variable Models}\label{chapter:online}

In Chapter~\ref{chapter:saddle}, we have established a guaranteed online stochastic gradient descent algorithm for tensor decomposition. Theoretically, it is solid and well justified. We will now fill in the gap of theoretical findings and practical applications by applying the algorithm to real world problems. 

We consider two problems: (1) community detection (wherein we compute the  decomposition of a tensor which relates to the count of $3$-stars in a graph) and (2) topic modeling (wherein we consider the tensor related to co-occurrence of triplets of words in documents); decomposition of the these tensors allows us to learn the hidden communities and topics from observed data.

\paragraph{Community detection: } We recover hidden communities in several real datasets with high accuracy. When ground-truth communities are available, we propose a new error score based on the hypothesis testing methodology involving $p$-values and false discovery rates~\cite{strimmer2008fdrtool} to validate our results. %Although these notions are standard in statistics, especially in bio-statistics, their use in social network analysis has been limited.
The use of $p$-values eliminates the need to carefully tune the number of communities output by our algorithm, and hence, we obtain a flexible trade-off between the fraction of communities recovered and their estimation accuracy.  %In fact, the overlapping version of NMI does not have an information-theoretic interpretation,  differs from the non-overlapping version of NMI, and is thus incorrect.
We find that our method has very good accuracy on a range of network datasets: Facebook, Yelp and DBLP. We summarize the datasets used in this chapter in Table~\ref{tab:data_info}. To get an idea of our running times, let us consider the larger DBLP collaborative data set for a moment. It consists of $16$ million edges, one million nodes and $250$ communities. We obtain an error of $10\%$ and the method runs in about two minutes, excluding the $80$ minutes taken to read the edge data from files stored on the hard disk and converting it to sparse matrix format.

Compared to the state-of-the-art method for learning MMSB models using the stochastic variational inference algorithm of~\cite{gopalan2012scalable}, we obtain several orders of magnitude speed-up in the running time on multiple real datasets. This is because our method consists of efficient matrix operations which are \emph{embarrassingly parallel}. Matrix operations are   carried out in the sparse format which is efficient especially for social network settings involving large sparse graphs. Moreover, our code is flexible to run on a range of graphs such as directed, undirected and bipartite graphs, while the code of~\cite{gopalan2012scalable} is designed for homophilic networks, and cannot  handle bipartite graphs in its present format. Note that bipartite networks occur in the recommendation setting such as the Yelp data set. Additionally, the variational implementation in~\cite{gopalan2012scalable} assumes a homogeneous connectivity model, where any pair of communities connect with the same probability and the probability of intra-community connectivity is also fixed. Our framework does not suffer from this restriction. We also provide arguments to show that the Normalized Mutual Information (NMI) and other scores, previously used for evaluating the recovery of overlapping community, can underestimate the errors.

\paragraph{Topic modeling: }We also employ the tensor method for topic-modeling, and there are many similarities between the topic and community settings. For instance, each document has  multiple topics, while in the network setting, each node has membership in multiple communities. The words in a document are generated based on the latent topics in the document, and similarly, edges are generated based on the community memberships of the node pairs. The tensor method is even faster for topic modeling, since the word vocabulary size is typically much smaller than the size of real-world networks.   We learn interesting hidden topics in New York Times corpus from UCI bag-of-words data set\footnote{\url{https://archive.ics.uci.edu/ml/datasets/Bag+of+Words}} with around $100,000$ words and $300,000$ documents in about two minutes. We present the important words for recovered topics, as well as interpret ``bridging'' words, which occur in many topics. 

\paragraph{Implementations: }We present two implementations, \viz a GPU-based implementation which exploits the parallelism of SIMD architectures and a CPU-based implementation for larger datasets, where the GPU memory does not suffice. We discuss various aspects   involved such as implicit manipulation of tensors since explicitly forming tensors would be unwieldy for large networks,   optimizing for communication bottlenecks in a parallel deployment, the need for sparse matrix and vector operations since real world networks tend to be sparse, and a careful statistical approach to validating the results, when ground truth is available.

\section{Tensor Forms for Topic and Community  Models}
\label{sec:sysmodel}
In this section, we briefly recap the topic and community models, as well as the tensor forms for their exact moments, derived in~\cite{AGHKT12,AnandkumarEtal:community12COLT}.

\subsection{Topic Modeling}
In  topic modeling, a document is viewed as a bag of words. Each document has a latent set of topics, and $h=(h_1,h_2,\ldots,h_k)$ represents the proportions of $k$ topics in a given document.  Given the topics $h$, the words are independently drawn and are exchangeable, and hence, the term ``bag of words'' model. We represent the
words in the document by $d$-dimensional random vectors $x_1, x_2, \ldots x_l \in \mathbb{R}^d$, where $x_i$ are coordinate basis vectors in $\mathbb{R}^d$ and $d$ is the size of the word vocabulary. Conditioned on $h$, the words in a document satisfy $\Ebb[x_i|h]=\mu h$, where $\mu : = [\mu_1,\ldots,\mu_k]$ is the topic-word matrix. And thus $\mu_j $ is the topic vector satisfying $\mu_j = \Pr\left(x_i \vert h_j\right)$,  $\forall j\in[k]$.
Under the Latent Dirichlet Allocation (LDA) topic model~\cite{blei2012probabilistic}, $h$ is drawn from a Dirichlet distribution with concentration parameter vector $\alpha = [\alpha_1,\ldots,\alpha_k]$. In other words, for each document $u$,  $h_u \stackrel{iid}{\sim}\Dir(\alpha),\ \forall u\in [n]$ with parameter vector $\alpha \in \R_{+}^k$. We define the Dirichlet concentration (mixing) parameter
\[
\alpha_0:=\sum_{i\in[k]}{\alpha_i}.
\]
  The Dirichlet distribution allows us to specify the extent of overlap among the topics by controlling for sparsity in topic density function. A larger $\alpha_0$ results in   more overlapped (mixed) topics. A special case of $\alpha_0=0$ is the single topic model.

Due to exchangeability, the order of the words does not matter, and it suffices to consider the frequency vector for each document, which counts the number of occurrences of each word in a document. Let $c_t:= (c_{1,t}, c_{2,t},\ldots, c_{d,t})\in\Rbb^{d}$ denote the frequency vector for $t^{\tha}$ document, and let $n$ be the number of documents.

We consider the first three order empirical moments,  given by
{\small
\begin{align}
\label{eq:1moment_topic}
M_1^{\topic} &: = \frac{1}{n} \sum\limits_{t=1}^{n} c_t\\
\label{eq:2moment_topic}
M_2^{\topic} &:=
\frac{\alpha_0+1}{n} \sum\limits_{t=1}^{n}{\left(c_t\otimes c_t - \diag\left(c_t\right)\right)} - {\alpha_0}M_1^{\topic}\otimes M_1^{\topic}\\
\label{eq:3moment_topic}
M_3^{\topic} & :=
 \frac{(\alpha_0+1)(\alpha_0+2)}{2n}\sum\limits_{t=1}^{n}\left[  c_t\otimes c_t\otimes c_t - \sum\limits_{i=1}^{d} \sum\limits_{j=1}^{d}c_{i,t}c_{j,t}(e_i\otimes e_i\otimes e_j) \right. \nonumber \\
& \left. - \sum\limits_{i=1}^{d} \sum\limits_{j=1}^{d}c_{i,t}c_{j,t}(e_i\otimes e_j\otimes e_i) - \sum\limits_{i=1}^{d} \sum\limits_{j=1}^{d} c_{i,t}c_{j,t}(e_i\otimes e_j\otimes e_j) + 2 \sum\limits_{i=1}^{d} c_{i,t}(e_i\otimes e_i\otimes e_i) \right] \nonumber\\
& -\frac{\alpha_0(\alpha_0+1)}{2n} \sum\limits_{t=1}^{n}\left[   \sum\limits_{i=1}^{d} c_{i,t}(e_i\otimes e_i\otimes M_1^{\topic}) +   \sum\limits_{i=1}^{d} c_{i,t}(e_i \otimes M_1^{\topic}\otimes e_i)  \right.\nonumber\\
&\left.   +   \sum\limits_{i=1}^{d} c_{i,t}(M_1^{\topic} \otimes e_i \otimes e_i)    \right]
+ {\alpha_0^2}M_1^{\topic} \otimes M_1^{\topic} \otimes M_1^{\topic}.
\end{align}
}
We recall Theorem 3.5 of~\cite{AGHKT12}:
\begin{lemma}\label{lemma:topic}
The exact moments can be factorized as
\begin{align}
\label{eq:1Emoment_topic}
\Ebb[M_1^{\topic}] & = \sum\limits_{i=1}^{k} \frac{\alpha_i}{\alpha_0}\mu_i\\
\label{eq:2Emoment_topic}
\Ebb[M_2^{\topic}] & = \sum\limits_{i=1}^{k} \frac{\alpha_i}{\alpha_0}\mu_i \otimes \mu_i\\
\label{eq:3Emoment_topic}
\Ebb[M_3^{\topic}] & =\sum\limits_{i=1}^{k}  \frac{\alpha_i}{\alpha_0}\mu_i\otimes \mu_i \otimes \mu_i.
\end{align}

where $\mu = [\mu_1,\ldots,\mu_k]$ and $\mu_i = \Pr\left(x_t \vert h=i\right)$, $\forall t\in[l]$. In other words, $\mu$ is the topic-word matrix.
\end{lemma}

From the Lemma~\ref{lemma:topic}, we observe that the first three moments of a LDA topic model have a simple form involving the topic-word matrix $\mu$ and Dirichlet parameters $\alpha_i$. In~\cite{AGHKT12}, it is shown that these parameters  can be recovered under a weak non-degeneracy assumption. We will employ tensor decomposition techniques to learn the parameters.

\subsection{Mixed Membership Model}
In the mixed membership stochastic block model (MMSB), introduced by~\cite{ABFX08}, the edges in a social network are related to the hidden communities of the nodes.
A batch tensor decomposition technique for learning MMSB was derived in~\cite{AnandkumarEtal:community12COLT}.

Let $n$ denote the number of nodes, $k$ the number of communities and $G\in \mathbb{R}^{n \times n}$ the adjacency matrix of the graph.  Each node $i\in [n]$ has an associated community membership vector $\pi_i \in \Rbb^k$, which is a latent variable, and the vectors are contained in a simplex, i.e., \[\sum_{i\in [k]} \pi_u(i)=1, \ \forall u\in [n]\]
where the notation $[n]$ denotes the set $\{ 1, \ldots, n \}$.  Membership vectors are sampled from the Dirichlet distribution $\pi _u \stackrel{iid}{\sim} \Dir(\alpha), \ \forall u\in [n]$ with parameter vector $\alpha \in \R_{+}^k$ where $\alpha_0:=\sum_{i\in[k]}{\alpha_i}$.
As in the topic modeling setting, the Dirichlet distribution allows us to specify the extent of overlap among the communities by controlling for sparsity in community membership vectors. A larger $\alpha_0$ results in   more overlapped (mixed) memberships. A special case of $\alpha_0=0$ is the stochastic block model~\cite{AnandkumarEtal:community12COLT}.

The \emph{community connectivity matrix} is denoted by $P\in [0,1]^{k \times k}$ where $P(a,b)$ measures the connectivity between communities $a$ and $b$, $\forall a,b \in [k]$.  We model the adjacency matrix entries as either of the two settings given below:

\paragraph{Bernoulli model: }This models a network with unweighted edges. It is used for Facebook and DBLP datasets in Section~\ref{sec:results} in our experiments.
\[
G_{ij} \stackrel{iid}{\sim} \Ber (\pi_i ^\top P \pi_j),\, \  \forall i,j\in[n] .
\]

\paragraph{Poisson model~\cite{karrer2011stochastic}: }This models  a network with weighted edges. It is used for the Yelp data set in Section~\ref{sec:results} to incorporate the review ratings. \[
G_{ij} \stackrel{iid}{\sim} \Poi (\pi_i ^\top P \pi_j),\,\ \forall i,j\in[n].\]

The tensor decomposition approach involves up to third order moments, computed from the observed network. In order to compute the moments, we  partition the  nodes randomly into sets $X,A,B,C$. Let $F_A := \Pi_A^\top P^\top$, $F_B := \Pi_B^\top P^\top$, $F_C := \Pi_C^\top P^\top$ (where $P$ is the community connectivity matrix and $\Pi$ is the membership matrix) and $\hat{\alpha}:=\left( \frac{\alpha_1}{\alpha_0},\ldots,\frac{\alpha_k}{\alpha_0} \right)$ denote the normalized Dirichlet concentration parameter.  We define pairs over $Y_1$ and $Y_2$ as $\Pairs(Y_1,Y_2): = G_{X,Y_1}^\top \otimes G_{X,Y_2}^\top$.  Define the following matrices
\begin{align}\label{eq:transitionMat}
Z_B & := \Pairs\left(A,C\right) \left(\Pairs\left(B,C\right)\right)^\dag,\\
Z_C & := \Pairs\left(A,B\right) \left(\Pairs\left(C,B\right)\right)^\dag.
\end{align}

We consider the first three empirical moments, given by
\begin{align}
\label{eq:1moment_graph}
{M_{1}}^{\community} & : = \frac{1}{\nx}\sum\limits_{x\in X} G_{x,A}^\top \\
\label{eq:2moment_graph}
{M_2}^{\community} &: = \frac{\alpha_0 + 1}{\nx} \sum\limits_{x\in X} Z_C G_{x,C}^\top G_{x,B} Z_B^\top -  {\alpha_0}{M_{1}}^{\community} {{M_{1}}^{\community}}^\top \\
\label{eq:3moment_graph}
{M_3}^{\community}  & :=  \frac{(\alpha_0 + 1)(\alpha_0 + 2)}{2\nx}\sum_{x\in X}G^\top_{x,A}\otimes Z_B G^\top_{x,B}\otimes Z_C G^\top_{x,C}\nonumber\\
& + \alpha_0^2 {M_{1}}^{\community}  \otimes {M_{1}}^{\community}  \otimes {M_{1}}^{\community}  \nonumber \\
& - \frac{\alpha_0 (\alpha_0 + 1)}{2 \nx} \sum_{x \in X}\left( G_{x,A}^\top \otimes Z_B G_{x,B}^\top \otimes {M_{1}}^{\community}  +  G_{x,A}^\top  \otimes {M_{1}}^{\community}  \otimes Z_C G_{x,C}^\top \right. \nonumber\\
&\left.+{M_{1}}^{\community} \otimes Z_B G_{x,B}^\top \otimes Z_C G_{x,C}^\top \right)
\end{align}

We now recap Proposition 2.2 of~\cite{AnandkumarEtal:community12} which provides the form of these moments under expectation.

\begin{lemma}
The exact moments can be factorized as
\begin{align}
\label{eqn:single}
\mathbb{E} [{M_1}^{\community} | \Pi_A, \Pi_B, \Pi_C] & :=  \sum_{i\in [k]} \hat{\alpha}_i (F_A)_i\\
\label{eqn:pair}
\mathbb{E} [{M_2}^{\community} | \Pi_A, \Pi_B, \Pi_C] & :=  \sum_{i\in [k]} \hat{\alpha}_i (F_A)_i \otimes (F_A)_i \\
\label{eqn:triples}
\mathbb{E} [{M_3}^{\community} | \Pi_A, \Pi_B, \Pi_C] & := \sum_{i\in [k]} \hat{\alpha}_i (F_A)_i \otimes (F_A)_i \otimes (F_A)_i
\end{align}
where $\otimes$ denotes the {\em Kronecker product} and $(F_A)_i$ corresponds to the $i^{th}$ column of $F_A$.
\end{lemma}

We observe that the moment forms above for the MMSB model have a similar form as the moments of the topic model in the previous section. Thus, we can employ a unified framework for both topic and community modeling  involving decomposition of the third order moment tensors $M_3^{\topic}$ and $M_3^{\community}$. Second order moments $M_2^{\topic}$ and $M_2^{\community}$ are used for \emph{preprocessing} of the data (i.e., whitening, which is introduced in detail in Section~\ref{sec:DRandWhite}). For the sake of the simplicity of the notation, in the rest of the chapter, we will use $M_2$ to denote empirical second order moments for both $M_2^{\topic}$ in topic modeling setting, and $M_2^{\community}$ in the mixed membership model setting. Similarly, we will use $M_3$ to denote empirical third order moments for both $M_3^{\topic}$ and $M_3^{\community}$.

\section{Learning using Third Order Moment}

Our learning algorithm uses up to the third-order moment to estimate the topic word matrix $\mu$ or the community membership matrix $\Pi$. First, we obtain co-occurrence of triplet words or subgraph counts (implicitly).
 Then, we perform preprocessing using second order moment $M_2$. Then we perform tensor decomposition efficiently using {\em stochastic gradient descent}~\cite{kushner2003stochastic} on $M_3$. We note that, in our implementation of the algorithm on the Graphics Processing Unit (GPU), linear algebraic operations are extremely fast.
 We also implement our algorithm on the CPU for large datasets which exceed the memory capacity of GPU and use sparse matrix operations which results in large gains in terms of both the memory and the running time requirements. The overall approach is summarized in Algorithm~\ref{alg:otmllvm}.
\begin{algorithm}
\begin{algorithmic}[1]
\REQUIRE Observed data: social network graph or document samples.
\ENSURE Learned latent variable model and infer hidden attributes.
%\item Partition the graph and e
\STATE Estimate the third order moments tensor $M_3$ (implicitly). The tensor is not formed explicitly as we break down the tensor operations into vector and matrix operations.
     % i.e., tensor operation in \eqref{eqn:triples} are not carried out explicitly of length $O(n)$ vectors.
\STATE Whiten the data, via SVD of $M_2$, to reduce dimensionality via symmetrization and orthogonalization.  The third order moments $M_3$ are whitened as $\mathcal{T}$.
%(Appendix~\ref{sec:apdx_white}).
\STATE Use stochastic gradient descent to estimate spectrum of whitened (implicit) tensor $\mathcal{T}$.
\STATE Apply post-processing to obtain the topic-word matrix or the community memberships.
\STATE If ground truth is known, validate the results using various evaluation measures.
\end{algorithmic}
\caption[Moment-based spectral learning of latent variable models]{Overall approach for learning latent variable models via a moment-based approach.}
\label{alg:otmllvm}
\end{algorithm}
%\end{enumerate}

\subsection{Dimensionality Reduction and Whitening}\label{sec:DRandWhite}

Whitening step utilizes linear algebraic manipulations to make the tensor symmetric and orthogonal (in expectation).  Moreover, it leads to dimensionality reduction since it (implicitly) reduces tensor $M_3$ of size $O(n^3)$ to a tensor of size $k^3$, where $k$ is the number of communities. Typically we have $k \ll n$. The whitening step also converts the tensor $M_3$ to a symmetric orthogonal tensor. The whitening matrix $W\in \Rbb^{n_A \times k}$ satisfies $W^\top M_2 W = I$. The idea is that if the bilinear projection of the second order moment onto $W$ results in the identity matrix, then a trilinear projection of the third order moment onto $W$ would result in an orthogonal tensor. We  use multilinear operations  to get an orthogonal   tensor $\mathcal{T} :=M_3(W,W,W)$.

The whitening matrix $W$ is computed via truncated $k-$svd of the second order moments.
\begin{equation*}
W = U_{M_2} \Sigma_{M_2}^{-1/2},
\end{equation*}
where $U_{M_2}$ and $\Sigma_{M_2}=\diag(\sigma_{M_2,1},\ldots,\sigma_{M_2,k})$ are the top $k$ singular vectors and singular values of $M_2$ respectively.
We then perform multilinear transformations on the triplet data using the whitening matrix. The whitened data is thus
\begin{align*}
y^t_A  &: = \left<W, {c^t}\right>,\\
 y^t_B &:= \left<W, c^t  \right>,\\
y^t_C &: = \left<W, c^t\right>,
\end{align*} for the topic modeling, where $t$ denotes the index of the documents. Note that $y^t_A$, $y^t_B$ and $y^t_C$ $\in \Rbb^{k}$. Implicitly, the whitened tensor is $\mathcal{T} = \frac{1}{\nx} \sum\limits_{t\in X} y^t_A \otimes y^t_B \otimes y^t_C$ and is a $k\times k \times k$ dimension tensor.
Since  $k \ll n$, the dimensionality reduction is crucial for our speedup.

\subsection{Stochastic Tensor Gradient Descent}
\label{sec:sto_ten_grad_des}

In  \cite{AnandkumarEtal:community12COLT} and \cite{AGHKT12}, the power method with deflation is used for tensor decomposition where the eigenvectors are recovered by iterating over multiple loops in a serial manner. Furthermore, batch data is used in their iterative power method which makes that algorithm slower than its stochastic counterpart.  In addition to implementing a stochastic spectral optimization algorithm, we achieve further speed-up by efficiently parallelizing the stochastic updates.  %(Appendix~\ref{sec:apdx_imp}).% Refer to the appendix for the derivation of the update equations.

Let $\mathbf{v}=[v_1|v_2|\ldots|v_k]$ be the true eigenvectors. Denote the cardinality of the sample set as $\nx$, i.e., $\nx:=|X|$.  Now that we have  the whitened tensor, we propose the \emph{Stochastic Tensor Gradient Descent} (STGD) algorithm for  tensor decomposition.
Consider the tensor $\mathcal{T} \in \R^{k \times k \times k}$ using whitened samples, i.e.,
\begin{align*}
\mathcal{T} & = \frac{1}{\nx}\sum_{t\in X}{\mathcal{T}^t} = \frac{(\alpha_0+1)(\alpha_0+2)}{2\nx}\sum_{t \in X} y^t_A \otimes y^t_B \otimes y^t_C \\
&- \frac{\alpha_0(\alpha_0+1)}{2\nx}\sum_{t\in X} \left[ y^t_A\otimes y^t_B \otimes\bar{y}_C + y^t_A\otimes \bar{y}_B \otimes y^t_C + \bar{y}_A\otimes y^t_B \otimes y^t_C\right] + \alpha_0^2 \bar{y}_A \otimes \bar{y}_B \otimes \bar{y}_C,
\end{align*} where $t\in X$ and denotes the index of the online data and $\bar{y}_A$,  $\bar{y}_B$, and $\bar{y}_C$ denote the mean of the whitened data.
Our goal is to find a symmetric CP decomposition of the whitened tensor, and this will be extensively discussed in the next chapter.

After learning the decomposition of the third order moment, we perform post-processing to estimate $\widehat{\Pi}$.% Refer to Appendix~\ref{sec:post_process} for details.

\subsection{Post-processing}\label{sec:post_process}

Eigenvalues $\Lambda:=[\lambda_1,\lambda_2,\ldots, \lambda_k]$  are estimated as the norm of the eigenvectors $\lambda_i = {\lVert{\phi_i}\rVert}^3$.
\begin{lemma}\label{lemma:postprocessing}
After we obtain $\Lambda$ and $\Phi$, the estimate for the topic-word matrix is given by%\red{we compute the topic word matrix via}
\[
\hat{\mu} = {W^\top}^\dag \Phi,%\diag(\Lambda),
\]
and in the community setting,   the community membership matrix is given by
\[
\hat{\Pi}_{A^c} = %\left(
 \diag (\gamma)^{1/3}\diag(\Lambda)^{-1} \Phi^\top\hat{
W}^\top G_{A,A^c}.
\]
 where $A^c : = X \cup B \cup C$. Similarly, we estimate $\hat{\Pi}_A$ by exchanging the roles of $X$ and $A$. Next, we obtain the Dirichlet distribution parameters%\aacomment{changed below}
\begin{equation*}
\hat{\alpha_i} =  \gamma^2\lambda_i^{-2},  %\lambda_i^{-2},
\forall i \in [k].
\end{equation*}\end{lemma}
 where $\gamma^2$ is chosen such that we have normalization$
\sum_{i\in[k]}\hat{ \alpha}_i : =\sum_{i\in [k]}\frac{\alpha_i}{\alpha_0}=1.
$

Thus, we perform STGD method to estimate the eigenvectors and eigenvalues of the whitened tensor, and then use these to estimate the topic word matrix $\mu$ and community membership matrix $\widehat{\Pi}$ by thresholding.

\section{Implementation Details}

\subsection{Symmetrization Step to Compute $M_2$}Note that for the topic model, the second order moment $M_2$ can be computed easily from the word-frequency vector. On the other hand, for the community setting, computing $M_2$ requires additional linear algebraic operations. It requires computation of matrices $Z_B$ and $Z_C$ in equation~\eqref{eq:transitionMat}. This requires computation of pseudo-inverses of ``Pairs'' matrices.
Now, note that pseudo-inverse of $\left(\Pairs\left(B,C\right)\right)$ in Equation~\eqref{eq:transitionMat} can be computed using rank $k$-SVD:
\begin{align*}
& \text{k-SVD}\left(\Pairs\left(B,C\right)\right) = U_B(:,1:k) \Sigma_{BC}(1:k) V_C(:,1:k)^\top.
\end{align*}We exploit the low rank property to have efficient running times and storage. We first implement the k-SVD of Pairs, given by $G_{X,C}^\top G_{X,B} $. Then the order in which the matrix products are carried out plays a significant role in terms of both memory and speed. Note that  $Z_C$ involves the multiplication of a sequence of matrices of sizes $\mathbb{R}^{n_A\times n_B}$, $\mathbb{R}^{n_B\times k}$, $\mathbb{R}^{k\times k}$, $\mathbb{R}^{k\times n_C}$,  $G_{x,C}^\top G_{x,B} $ involves products of sizes $\mathbb{R}^{n_C\times k}$,  $\mathbb{R}^{k\times k}$, $\mathbb{R}^{k\times n_B}$, and $Z_B$ involving products of sizes $\mathbb{R}^{n_A\times n_C}$, $\mathbb{R}^{n_C\times k}$, $\mathbb{R}^{k\times k}$, $\mathbb{R}^{k\times n_B}$. While performing these products, we avoid products of sizes $\mathbb{R}^{O(n)\times O(n)}$ and $\mathbb{R}^{O(n)\times O(n)}$. This allows us to have efficient storage requirements. Such manipulations are represented in Figure~\ref{fig:dim_reduc}.

\begin{figure*}[h]
\centering
{\begin{minipage}{4in}
\centering
\def\svgwidth{\textwidth}
\begingroup%
  \makeatletter%
  \providecommand\color[2][]{%
    \errmessage{(Inkscape) Color is used for the text in Inkscape, but the package 'color.sty' is not loaded}%
    \renewcommand\color[2][]{}%
  }%
  \providecommand\transparent[1]{%
    \errmessage{(Inkscape) Transparency is used (non-zero) for the text in Inkscape, but the package 'transparent.sty' is not loaded}%
    \renewcommand\transparent[1]{}%
  }%
  \providecommand\rotatebox[2]{#2}%
  \ifx\svgwidth\undefined%
    \setlength{\unitlength}{747.20448486bp}%
    \ifx\svgscale\undefined%
      \relax%
    \else%
      \setlength{\unitlength}{\unitlength * \real{\svgscale}}%
    \fi%
  \else%
    \setlength{\unitlength}{\svgwidth}%
  \fi%
  \global\let\svgwidth\undefined%
  \global\let\svgscale\undefined%
  \makeatother%
  \begin{picture}(1,0.23765382)%
    \put(0,0){\includegraphics[width=\unitlength]{\fighomeCommunity/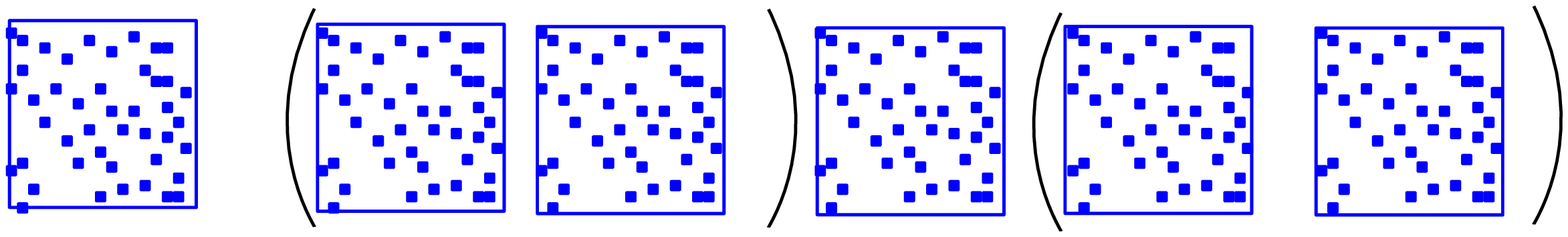}}%
    \put(0.15651648,0.12646764){\color[rgb]{0,0,0}\makebox(0,0)[lb]{\smash{$=$}}}%
    \put(0.46220125,0.18491793){\color[rgb]{0,0,0}\makebox(0,0)[lb]{\smash{$\dag$}}}%
    \put(0.9106669,0.18384786){\color[rgb]{0,0,0}\makebox(0,0)[lb]{\smash{$\top$}}}%
    \put(0.76865456,0.18491793){\color[rgb]{0,0,0}\makebox(0,0)[lb]{\smash{$\dag$}}}%
    \put(0.78514754,0.18384786){\color[rgb]{0,0,0}\makebox(0,0)[lb]{\smash{$\top$}}}%
    \put(-0.01606648,0.13703326){\color[rgb]{0,0,0}\makebox(0,0)[lb]{\smash{$\lvert A\rvert$}}}%
    \put(0.06959692,0.19913139){\color[rgb]{0,0,0}\makebox(0,0)[lb]{\smash{$\lvert A\rvert$}}}%
  \end{picture}%
\endgroup%
\end{minipage}}
\hfil
{\begin{minipage}{4in}
\centering
\def\svgwidth{\textwidth}
\begingroup%
  \makeatletter%
  \providecommand\color[2][]{%
    \errmessage{(Inkscape) Color is used for the text in Inkscape, but the package 'color.sty' is not loaded}%
    \renewcommand\color[2][]{}%
  }%
  \providecommand\transparent[1]{%
    \errmessage{(Inkscape) Transparency is used (non-zero) for the text in Inkscape, but the package 'transparent.sty' is not loaded}%
    \renewcommand\transparent[1]{}%
  }%
  \providecommand\rotatebox[2]{#2}%
  \ifx\svgwidth\undefined%
    \setlength{\unitlength}{720.00027734bp}%
    \ifx\svgscale\undefined%
      \relax%
    \else%
      \setlength{\unitlength}{\unitlength * \real{\svgscale}}%
    \fi%
  \else%
    \setlength{\unitlength}{\svgwidth}%
  \fi%
  \global\let\svgwidth\undefined%
  \global\let\svgscale\undefined%
  \makeatother%
  \begin{picture}(1,0.24667124)%
    \put(0,0){\includegraphics[width=\unitlength]{\fighomeCommunity/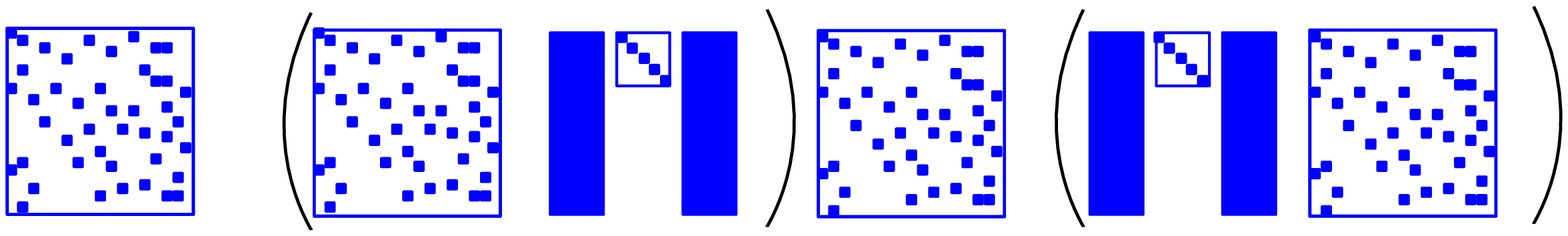}}%
    \put(0.14428852,0.12068267){\color[rgb]{0,0,0}\makebox(0,0)[lb]{\smash{$=$}}}%
    \put(0.45436066,0.1765092){\color[rgb]{0,0,0}\makebox(0,0)[lb]{\smash{$\top$}}}%
    \put(0.77777514,0.18445862){\color[rgb]{0,0,0}\makebox(0,0)[lb]{\smash{$\top$}}}%
    \put(0.91555287,0.18445862){\color[rgb]{0,0,0}\makebox(0,0)[lb]{\smash{$\top$}}}%
  \end{picture}%
\endgroup%
\end{minipage}}
\hfil
{\begin{minipage}{4in}
\centering
\def\svgwidth{\textwidth}
\begingroup%
  \makeatletter%
  \providecommand\color[2][]{%
    \errmessage{(Inkscape) Color is used for the text in Inkscape, but the package 'color.sty' is not loaded}%
    \renewcommand\color[2][]{}%
  }%
  \providecommand\transparent[1]{%
    \errmessage{(Inkscape) Transparency is used (non-zero) for the text in Inkscape, but the package 'transparent.sty' is not loaded}%
    \renewcommand\transparent[1]{}%
  }%
  \providecommand\rotatebox[2]{#2}%
  \ifx\svgwidth\undefined%
    \setlength{\unitlength}{720.00075303bp}%
    \ifx\svgscale\undefined%
      \relax%
    \else%
      \setlength{\unitlength}{\unitlength * \real{\svgscale}}%
    \fi%
  \else%
    \setlength{\unitlength}{\svgwidth}%
  \fi%
  \global\let\svgwidth\undefined%
  \global\let\svgscale\undefined%
  \makeatother%
  \begin{picture}(1,0.24663905)%
    \put(0,0){\includegraphics[width=\unitlength]{\fighomeCommunity/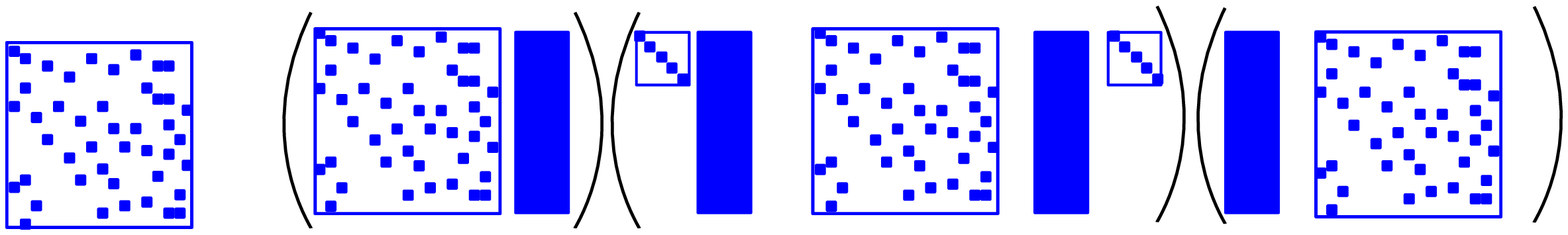}}%
    \put(0.14411109,0.13121129){\color[rgb]{0,0,0}\makebox(0,0)[lb]{\smash{$=$}}}%
    \put(0.47093097,0.19135582){\color[rgb]{0,0,0}\makebox(0,0)[lb]{\smash{$\top$}}}%
    \put(0.78426398,0.19580026){\color[rgb]{0,0,0}\makebox(0,0)[lb]{\smash{$\top$}}}%
    \put(0.91981939,0.19357804){\color[rgb]{0,0,0}\makebox(0,0)[lb]{\smash{$\top$}}}%
  \end{picture}%
\endgroup%
\end{minipage}}
\caption[Efficient computation in smart order]{By performing the matrix multiplications in an efficient order (Equation~\eqref{eq:2moment_graph}), we avoid products involving $O(n) \times O(n)$ objects. Instead, we use objects of size $O(n) \times k$ which improves the speed, since $k \ll n$. Equation~\eqref{eq:2moment_graph} is equivalent to $ M_2 = $ $\left(\Pairs_{A,B} \Pairs_{C,B}^\dag\right)$ $\Pairs_{C,B}$ $\left(\Pairs_{B,C}^\dag\right)^\top \Pairs_{A,C}^\top$  $ -\text{shift}$, where the $\text{shift}=\frac{\alpha_0}{\alpha_0+1} \left({M_1  }{M_1  }^\top- \diag\left({M_1  } {M_1  }^\top\right) \right)$. We do not explicitly calculate the pseudoinverse but maintain the low rank matrix decomposition form.}\label{fig:dim_reduc}
\end{figure*}

We then orthogonalize the third order moments to reduce the dimension of its modes to $k$.
%For this purpose, we compute the whitening matrix $W$ such that $W^\top M_2 W = I$.
%Therefore, the whitening matrix is computed using the second order moments , i.e., $2$-stars.
We perform linear transformations on the data corresponding to the partitions $A$, $B$ and $C$ using the whitening matrix.
% The resulting tensor  is denoted by $\mathcal{T} \in \mathbb{R}^{k\times k\times k}$.
The whitened data is thus $y^t_A  : = \left<W, G^\top_{t,A}\right>$, $y^t_B := \left<W, Z_B G^\top_{t,B}\right>$, and
$y^t_C : = \left<W, Z_C G^\top_{t,C}\right>$, where $t\in X$ and denotes the index of the online data.
Since  $k \ll n$, the dimensionality reduction is crucial for our speedup.

\subsection{Efficient Randomized SVD Computations}\label{sec:Apdx_sparse}
When we consider very large-scale data, the whitening matrix is a bottleneck to handle when we aim for fast running times. We obtain the low rank approximation of matrices using random projections. In the CPU implementation, we use \emph{tall-thin SVD} (on a sparse matrix) via the Lanczos algorithm after the projection and in the GPU implementation, we use \emph{tall-thin QR}. %We use two techniques in our implementation, namely:
%\begin{enumerate}
%\item Low rank approximation of matrices using the random projection method.
%\item Sparse SVD computation via the Lanczos algorithm.
%\end{enumerate}
We give the overview of these methods below. Again, we use graph community membership model without loss of generality.

\paragraph{Randomized low rank approximation: }
From~\cite{gittens2013revisiting}, for the $k$-rank positive semi-definite matrix ${M_2} \in \mathbb{R}^{n_A \times n_A}$ with $n_A \gg k$, we can perform random projection to reduce dimensionality. More precisely, if we have a random matrix $S\in \mathbb{R}^{n_A \times \tilde{k}}$ with unit norm (rotation matrix), we project $M_2 $ onto this random matrix to get $\mathbb{R}^{n\times \tilde{k}}$ tall-thin matrix. Note that we choose $\tilde{k}=2k$ in our implementation.  We will obtain lower dimension approximation of $M_2$ in $\mathbb{R}^{\tilde{k} \times \tilde{k}}$.
%If we define $Z:= \left( {M_2}  Q \right)^\top {M_2} Q \in \mathbb{R}^{\tilde{k} \times \tilde{k}}$, then SVD of ${M_2} $ is obtained by SVD of  $Z$.
%\begin{enumerate}
%\item Generate random matrix $S\in \mathbb{R}^{n,k}$ for dense $M_2 $
%\item Generate column selection matrix with random sign $S\in \{0,1\}^{n\times k}$ for sparse $M_2 $
%\item $Q $= orth($M_2 S$)
%\item $Z = (M_2 Q)^\top M_2 Q \in\mathbb{R}^{\tilde{k}\times \tilde{k}}$
%\item $[U_z, L_z, V_z] = $SVD($Z$)
%\item $V_{M_2 } \approx Q V_{z}$
%\item $L_{M_2 }\approx \sqrt{L_{z}}$
%\end{enumerate}
%Therefore, $W = V_{M_2 }L_{M_2 }^{-\frac{1}{2}} \approx \left(Q V_{z} L_{z}^{-\frac{1}{4}}\right)$. See Appendix~\ref{sec:apdx_Rsvd} for proof.
Here we emphasize that $S\in \mathbb{R}^{n\times \tilde{k}}$ is a random matrix for dense ${M_2} $. However for sparse ${M_2} $, $S\in \{0,1\}^{n\times \tilde{k}}$ is a column selection matrix with random sign for each entry.

%There are alternatives which require
After the projection, one approach we use is SVD on this tall-thin ($\mathbb{R}^{n\times \tilde{k}}$) matrix. Define $O:= {M_2}  S \in \mathbb{R}^{n \times \tilde{k}}$ and $\Omega := S^\top {M_2} S \in \mathbb{R}^{\tilde{k} \times \tilde{k}}$. A low rank approximation of ${M_2} $ is given by $O\Omega^\dag O^\top$~\cite{gittens2013revisiting}.
Recall that the definition of a whitening matrix $W$ is that $W^\top {M_2}  W = I$. We can obtain the whitening matrix of ${M_2} $ without directly doing a SVD on ${M_2}  \in \mathbb{R}^{n_A \times n_A}$.

\smallskip
\emph{Tall-thin SVD: }This is used in the CPU implementation.
The whitening matrix can be obtained by
\begin{equation}\label{eq:nystrom_whiten}
W \approx(O^\dag)^\top (\Omega^{\frac{1}{2}})^\top.
\end{equation}
The pseudo code for computing the whitening matrix $W$ using tall-thin SVD is given in Algorithm~\ref{alg:pinv}.
%\floatname{algorithm}{Procedure}
\begin{algorithm}
\caption[Randomized tall-thin SVD]{Randomized Tall-thin SVD}
\label{alg:pinv}
\begin{algorithmic}[1]
\REQUIRE Second moment matrix $M_2$.
\ENSURE Whitening matrix $W$.
\STATE Generate random matrix $S\in \mathbb{R}^{n\times \tilde{k}}$ if  $M_2 $ is dense.
\STATE Generate column selection matrix with random sign $S\in \{0,1\}^{n\times \tilde{k}}$ if $M_2 $ is sparse.
\STATE $O = M_2 S\in \mathbb{R}^{n \times \tilde{k}}$
\STATE $[U_O, L_O, V_O] =$SVD$(O)$
\STATE $\Omega = S^\top O\in \mathbb{R}^{\tilde{k} \times \tilde{k}}$
\STATE $[U_\Omega, L_\Omega, V_\Omega] = $SVD$(\Omega)$
\STATE $W = U_O L_O^{-1} V_O^\top V_\Omega L_\Omega^{\frac{1}{2}} U_\Omega^\top$
\end{algorithmic}
\end{algorithm}
Therefore, we only need to compute SVD of a tall-thin matrix $O\in\mathbb{R}^{n_A \times \tilde{k}}$.  Note that $\Omega \in \mathbb{R}^{\tilde{k} \times \tilde{k}}$, its square-root is easy to compute.
Similarly, pseudoinverses can also be obtained without directly doing SVD. For instance, the pseudoinverse of the $\Pairs\left(B,C\right)$ matrix is given by
\[
\left(\Pairs\left(B,C\right)\right)^\dag = (J^\dag)^\top \Psi J^\dag,
\]
where $\Psi = S^\top \left(\Pairs\left(B,C\right)\right) S$ and $J = \left(\Pairs\left(B,C\right)\right) S$.
The pseudo code for computing pseudoinverses is given in Algorithm~\ref{alg:ttsvd}.
\begin{algorithm}
\caption{Randomized Pseudoinverse}
\label{alg:ttsvd}
\begin{algorithmic}[1]
\REQUIRE Pairs matrix $\Pairs\left(B,C\right)$.
\ENSURE Pseudoinverse of the pairs matrix $\left(\Pairs\left(B,C\right)\right)^\dag$.
\STATE Generate random matrix $S\in \mathbb{R}^{n,k}$ if  $M_2 $ is dense.
\STATE Generate column selection matrix with random sign $S\in \{0,1\}^{n\times k}$ if $M_2 $ is sparse.
\STATE $J = \left(\Pairs\left(B,C\right)\right) S$
\STATE $\Psi = S^\top J$
\STATE $[U_J, L_J, V_J] =$SVD$(J)$
\STATE $\left(\Pairs\left(B,C\right)\right)^\dag = U_J L_J^{-1}V_J^\top \Psi V_J L_J^{-1}U_J^\top$
\end{algorithmic}
\end{algorithm}

%We also take advantage of the sparse representation to resolve memory issues that arise when running large scale datasets such as the DBLP dataset. % We implement our algorithm using the sparse matrix format and operations available in the Eigen toolkit\footnote{\scriptsize{\url{http://eigen.tuxfamily.org/index.php?title=Main_Page}}} and we use the SVDLIBC~\cite{svdlibc2002} library to compute sparse SVD via the Lanczos algorithm.
The sparse representation of the data allows for scalability on a single machine to datasets having millions of nodes.
%Note that the complexity of computing the top eigenvector of a symmetric matrix with $s$ non-zero elements is $O(s/\sqrt{gamma}$ where $\gamma = \max \{ \frac{\lambda_1}{\lambda_1}\}$
%To resolve the  memory issue of GPU and incorporate the computational challenge of the DBLP dataset with a million nodes, we use sparse manipulations, since the edge density is extremely small.
Although the GPU has SIMD architecture which makes parallelization efficient, it lacks advanced libraries with sparse SVD operations and out-of-GPU-core implementations. We therefore implement the sparse format on CPU for sparse datasets. We implement our algorithm using random projection for efficient dimensionality reduction~\cite{DBLP:journals/corr/abs-1207-6365} along with the sparse matrix operations available in the Eigen toolkit\footnote{\scriptsize{\url{http://eigen.tuxfamily.org/index.php?title=Main_Page}}}, and we use the SVDLIBC~\cite{svdlibc2002} library to compute sparse SVD via the Lanczos algorithm.
%We employ random projection for efficient dimensionality reduction~\cite{DBLP:journals/corr/abs-1207-6365} and employ sparse format multiplication available in the Eigen toolkit and SVD~\cite{svdlibc2002} which uses the Lanczos algorithm on the CPU.
Theoretically, the Lanczos algorithm~\cite{zbMATH06159604} on a $n \times n$ matrix takes around $(2d+8)n$ flops for a single step where $d$ is the average number of non-zero entries per row.

%\begin{enumerate}
%\item Generate random matrix $S\in \mathbb{R}^{n,k}$ if  $M_2 $ is dense.
%\item Generate column selection matrix with random sign $S\in \{0,1\}^{n\times k}$ if $M_2 $ is sparse.
%\item $\Xi = \left(\Pairs\left(B,C\right)\right) S$
%\item $\Psi = S^\top \Xi$
%\item $[U_\Xi, L_\Xi, V_\Xi] =$SVD$(\Xi)$
%\item $\left(\Pairs\left(B,C\right)\right)^\dag = U_\Xi L_\Xi^{-1}V_\Xi^\top \Psi V_\Xi L_\Xi^{-1}U_\Xi^\top$
%\end{enumerate}
%%%%%%%%%%%%%

\smallskip
\emph{Tall-thin QR: }
This is used in the GPU implementation due to the lack of library to do sparse tall-thin SVD. The difference is that we instead implement a tall-thin QR on $O$,  therefore the whitening matrix is obtained as
\[
W \approx Q (R^\dag)^\top (\Omega^{\frac{1}{2}})^\top.
 \]

%\paragraph{GPU Memory Issues: }
The main bottleneck for our GPU implementation is device storage, since GPU memory is highly limited and not expandable. Random projections help in reducing the dimensionality from $O(n \times n)$ to $O(n \times k)$ and hence, this fits the data in the GPU memory better. Consequently, after the whitening step, we project the data into $k$-dimensional space. Therefore, the STGD step is dependent only on $k$, and hence can be fit in the GPU memory. So, the main bottleneck is computation of large SVDs. In order to support larger datasets such as the DBLP data set % (involving matrices of size about $30000 \times 30000$ wherein a single $30000 \times 30000$ matrix, in double precision, occupies around 6.7GB)
which exceed the GPU memory capacity, we extend our implementation with out-of-GPU-core matrix operations and the Nystrom method~\cite{gittens2013revisiting} for the whitening matrix computation and the pseudoinverse computation in the pre-processing module. %This can be solved via distributed QR method~\cite{constantine2011tall}, which involves the (dense) SVD of a $k \times k(k \ll n)$ matrix leading to  $O(k^3)$ in Table~\ref{tab:complexity} followed by multiplication to find $Q$, which takes $O(n)$ time.

% We emphasize that we implement our algorithm on a single machine whose specifications are given in Table~\ref{tab:sys_spec} in Appendix~\ref{sec:apdx_arch}.

%\paragraph{Sparsity:}

\subsection{Stochastic Updates}\label{sec:stgd}

\begin{figure}%[H]%{l}{2.1in}
\centering{
\def\svgwidth{3.2in}%\def\svgheight{1.2in}
\begingroup%
  \makeatletter%
  \providecommand\color[2][]{%
    \errmessage{(Inkscape) Color is used for the text in Inkscape, but the package 'color.sty' is not loaded}%
    \renewcommand\color[2][]{}%
  }%
  \providecommand\transparent[1]{%
    \errmessage{(Inkscape) Transparency is used (non-zero) for the text in Inkscape, but the package 'transparent.sty' is not loaded}%
    \renewcommand\transparent[1]{}%
  }%
  \providecommand\rotatebox[2]{#2}%
  \ifx\svgwidth\undefined%
    \setlength{\unitlength}{784.8bp}%
    \ifx\svgscale\undefined%
      \relax%
    \else%
      \setlength{\unitlength}{\unitlength * \real{\svgscale}}%
    \fi%
  \else%
    \setlength{\unitlength}{\svgwidth}%
  \fi%
  \global\let\svgwidth\undefined%
  \global\let\svgscale\undefined%
  \makeatother%
  \begin{picture}(1,0.36799185)%
    \put(0,0){\includegraphics[width=\unitlength]{\fighomeCommunity/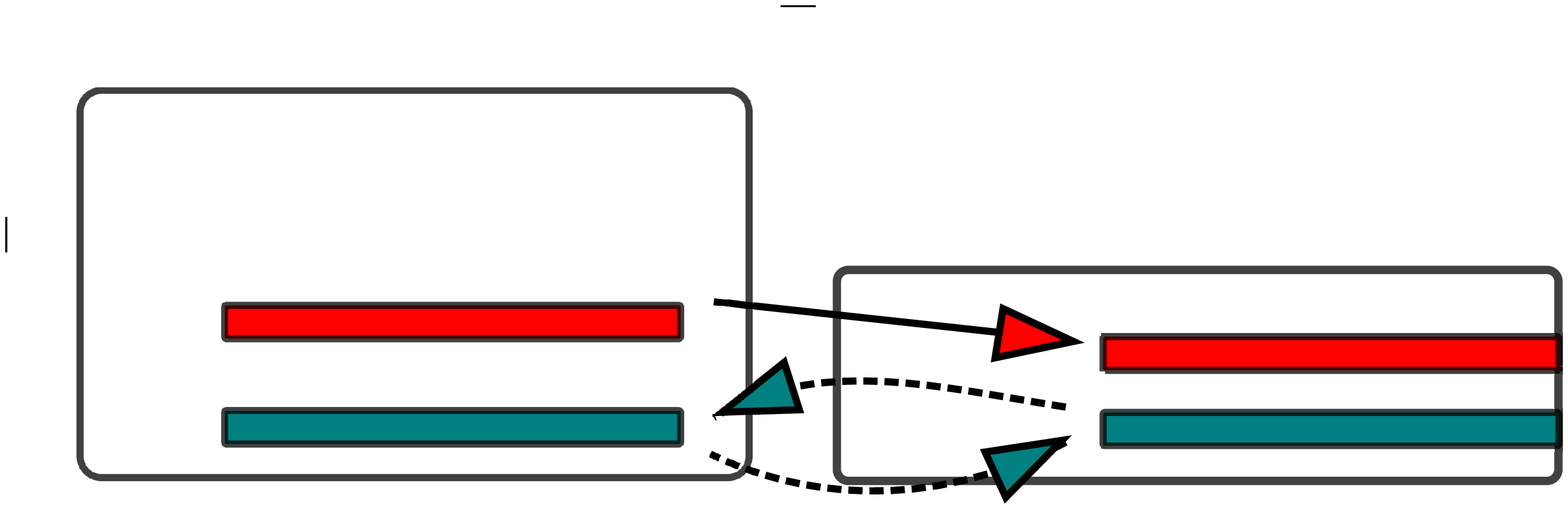}}%
    \put(0.04335005,0.11879389){\color[rgb]{0,0,0}\makebox(0,0)[lb]{\smash{$v_i^t$}}}%
    \put(0.04314793,0.21615305){\color[rgb]{0,0,0}\makebox(0,0)[lb]{\smash{$y_A^t$,$y_B^t$,$y_C^t$}}}%
    \put(0.18612624,0.33110272){\color[rgb]{0,0,0}\makebox(0,0)[lb]{\smash{CPU}}}%
    \put(0.7069755,0.22916592){\color[rgb]{0,0,0}\makebox(0,0)[lb]{\smash{GPU}}}%
    \put(0.33555723,0.02770536){\color[rgb]{0,0,0}\makebox(0,0)[lb]{\smash{Standard Interface}}}%
    \put(0.90796023,0.10734348){\color[rgb]{0,0,0}\makebox(0,0)[lb]{\smash{$v_i^t$}}}%
  \end{picture}%
\endgroup%
}
%\caption{}
%\end{figure}

%\begin{figure}%[H]%{l}{2.1in}
%\label{fig:device}
\centering{
\def\svgwidth{3.2in}%\def\svgheight{1.2in}
\begingroup%
  \makeatletter%
  \providecommand\color[2][]{%
    \errmessage{(Inkscape) Color is used for the text in Inkscape, but the package 'color.sty' is not loaded}%
    \renewcommand\color[2][]{}%
  }%
  \providecommand\transparent[1]{%
    \errmessage{(Inkscape) Transparency is used (non-zero) for the text in Inkscape, but the package 'transparent.sty' is not loaded}%
    \renewcommand\transparent[1]{}%
  }%
  \providecommand\rotatebox[2]{#2}%
  \ifx\svgwidth\undefined%
    \setlength{\unitlength}{784.8bp}%
    \ifx\svgscale\undefined%
      \relax%
    \else%
      \setlength{\unitlength}{\unitlength * \real{\svgscale}}%
    \fi%
  \else%
    \setlength{\unitlength}{\svgwidth}%
  \fi%
  \global\let\svgwidth\undefined%
  \global\let\svgscale\undefined%
  \makeatother%
  \begin{picture}(1,0.36799185)%
    \put(0,0){\includegraphics[width=\unitlength]{\fighomeCommunity/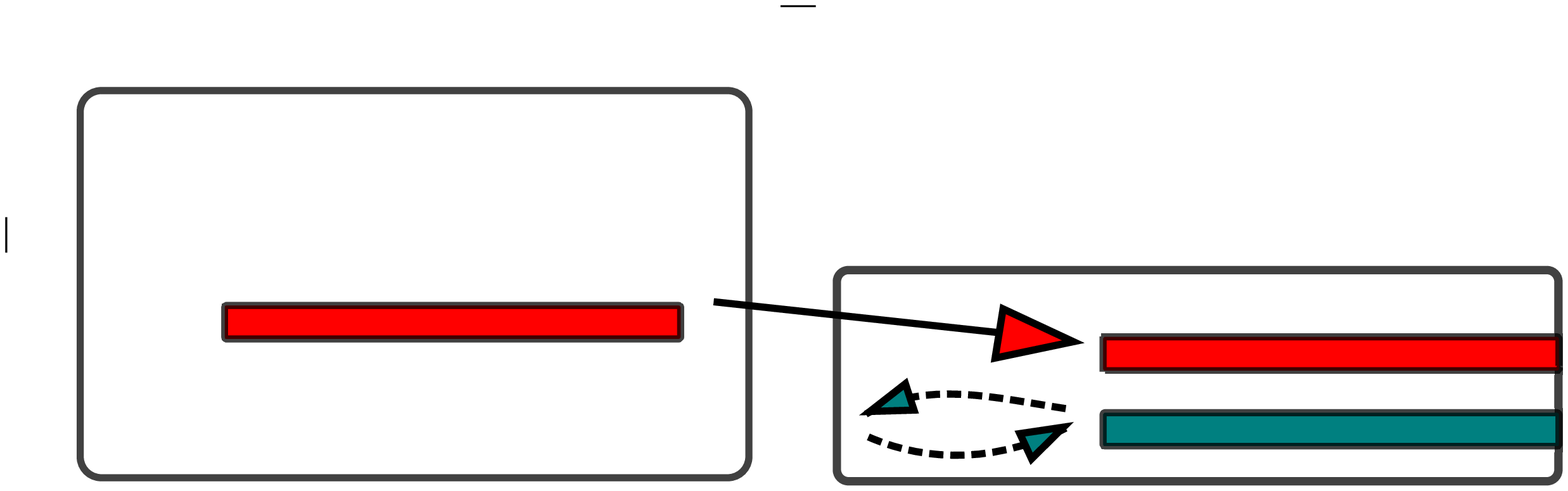}}%
    \put(0.04314793,0.21615305){\color[rgb]{0,0,0}\makebox(0,0)[lb]{\smash{$y_A^t$,$y_B^t$,$y_C^t$}}}%
    \put(0.18612624,0.33110272){\color[rgb]{0,0,0}\makebox(0,0)[lb]{\smash{CPU}}}%
    \put(0.7069755,0.22916592){\color[rgb]{0,0,0}\makebox(0,0)[lb]{\smash{GPU}}}%
    \put(0.33490021,0.02770536){\color[rgb]{0,0,0}\makebox(0,0)[lb]{\smash{Device Interface}}}%
    \put(0.90796023,0.10734348){\color[rgb]{0,0,0}\makebox(0,0)[lb]{\smash{$v_i^t$}}}%
  \end{picture}%
\endgroup%
}
\caption[Data transfer between CPU and GPU]{Data transfers in the standard and device interfaces of the GPU implementation.}
\label{fig:standard_device}
\end{figure}

STGD can potentially be the most computationally intensive task if carried out naively since the storage and manipulation of a $O(n^3)$-sized tensor makes the method not scalable. However we   overcome this problem since we never form the tensor explicitly; instead, we collapse the tensor modes implicitly. We gain large speed up by optimizing the implementation of STGD.% and expressing it in terms of efficient vector and matrix products.  %We approach this optimization at two levels of design, namely efficient linear algebraic operations and GPU level optimization. % The details are give in the following paragraphs.
%\paragraph{Design for efficient linear algebraic operations:}
%Our learning algorithm never explicitly forms or stores the tensor to be decomposed which is crucial for storage and time efficiency.
To implement the tensor operations efficiently we convert them into matrix and vector operations so that they are implemented using BLAS routines.  We obtain whitened   vectors $y_A, y_B$ and $y_C$ and manipulate these vectors efficiently to obtain tensor eigenvector updates  using the gradient scaled by a suitable learning rate.

\paragraph{Efficient STGD via stacked vector operations: }
We convert the BLAS II into BLAS III operations by stacking the vectors to form matrices, leading to more efficient operations.
Although the updating equation for the stochastic gradient update is presented serially, we can update the $k$ eigenvectors simultaneously in parallel. The basic idea is to stack the $k$ eigenvectors $\phi_i\in\mathbb{R}^k$ into a matrix $\mathbf{\Phi}$, then using the internal parallelism designed for BLAS III operations.
%Since the eigenvectors are just $k$ length $k$ vectors, we store them in the device memory and perform the iterations.

Overall, the STGD step involves $1+k+i(2+3k)$ BLAS II over $\mathbb{R}^k$ vectors, 7N BLAS III over $\mathbb{R}^{k\times k}$ matrices and 2 QR operations over $\mathbb{R}^{k \times k}$ matrices, where $i$ denotes the number of iterations. We  provide a count of BLAS operations for various steps in Table~\ref{tab:blas_count}.

\begin{table}[htbp]
\small
 \centering
   \begin{tabular}{@{} l|c|c|c|c|c @{}}
%      \toprule
\hline
Module & BLAS I & BLAS II & BLAS III & SVD & QR\\
\hline
\hline
Pre & $0$ & $8$ & $19$ & $3$ & $0$\\
%STGD & 0 & $1+k+N(2+3k)$ & $7N$ & $0$ & $2$\\
STGD & 0 & $Nk$ & $7N$ & $0$ & $2$\\
Post & $0$ & $0$ & $7$ & $0$ & $0$\\
\hline
%      \bottomrule
   \end{tabular}
   \caption[Linear algebraic operation counts]{Linear algebraic operation counts: $N$ denotes the number of iterations for STGD and $k$, the number of communities. %For details on how we manipulate matrices for STGD, see Appendix~\ref{sec:apdx_imp}.
   }
   \label{tab:blas_count}
\end{table}

\paragraph{Reducing communication in GPU implementation: }
In STGD, note that the storage needed for the iterative part does not depend on the number of nodes in the data set, rather, it depends on the parameter $k$, i.e., the number of communities to be estimated, since whitening performed before STGD leads to dimensionality reduction.  This makes it suitable for storing the required buffers in the GPU memory, and using the CULA device interface for the BLAS operations. In Figure~\ref{fig:standard_device}, we illustrate the data transfer involved in the GPU standard and device interface codes. While the standard interface involves data transfer (including whitened neighborhood vectors and the eigenvectors) at each stochastic iteration between the CPU memory and the GPU memory, the device interface involves allocating and retaining the eigenvectors at each stochastic iteration which in turn speeds up the spectral estimation.

\begin{figure}%[H]%{l}{2.1in}
\centering
\psfrag{Number of communities}[l]{{Number of communities $k$}}
\psfrag{Time (in seconds) for 100 stochastic iterations}[l]{{ Running time(secs)}}
\psfrag{Scaling of the stochastic algorithm with the rank of the tensor}[c]{}
\psfrag{MATLAB Tensor Toolbox}[l]{{\textbf{MATLAB Tensor Toolbox}}}
\psfrag{CULA Standard Interface}[l]{{\textbf{CULA Standard Interface}}}
\psfrag{CULA Device Interface}[l]{{\textbf{CULA Device Interface}}}
\psfrag{Eigen Sparse}[l]{{\textbf{Eigen Sparse}}}
\includegraphics[width=\textwidth]{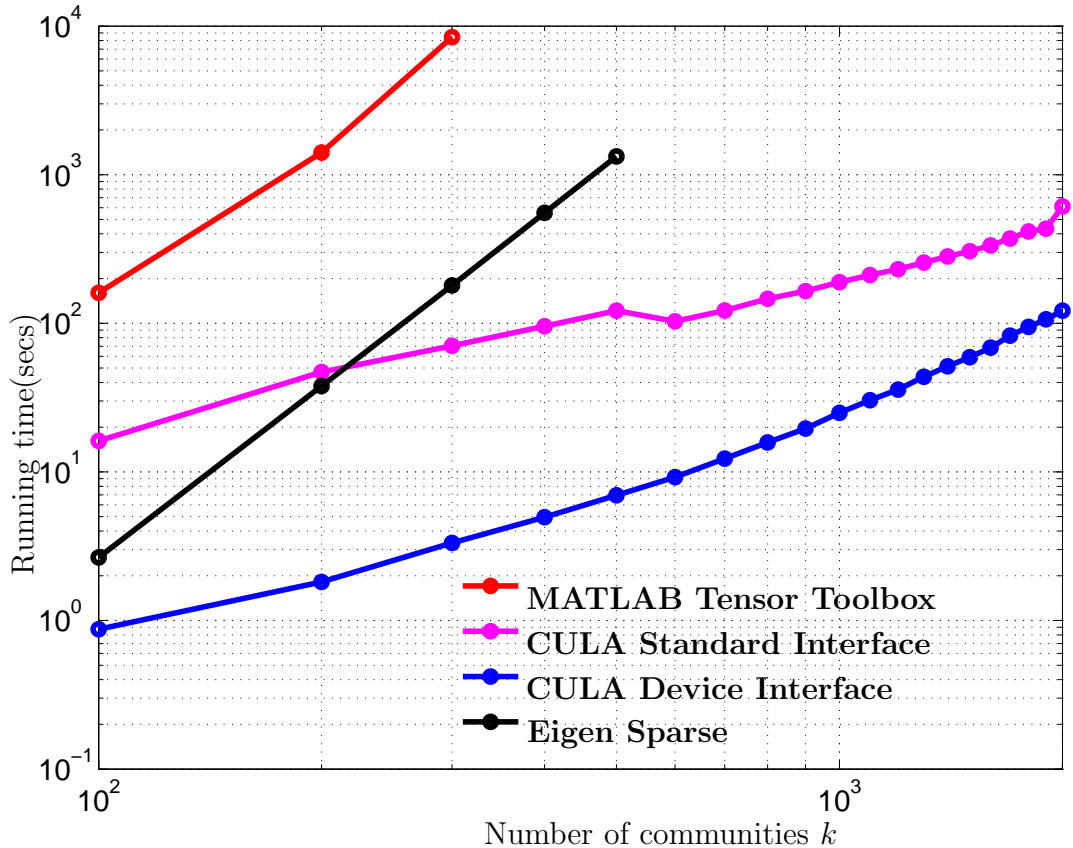}
\caption[STGD running time comparison]{Comparison of the running time for STGD under different $k$ for $100$ iterations.}
\label{fig:k_vs_t}
\end{figure}

We compare the running time of  the CULA device code  with the MATLAB code (using the tensor toolbox~\cite{TTB_Software}), CULA standard code and Eigen sparse code in Figure~\ref{fig:k_vs_t}. %From the plot, we can see that the CULA device code is faster than the other two.
As expected, the GPU implementations of matrix operations are much faster and scale much better than the CPU implementations. Among the CPU codes, we notice that sparsity and optimization offered by the Eigen toolkit gives us huge gains.
We obtain orders of magnitude of speed up for the GPU device code as we place the buffers in the GPU memory and transfer minimal amount of data involving the whitened vectors only once at the beginning of each iteration.  %Doing this eliminates all the unnecessary PCI bus transfers and keeps the GPU pipeline busy.
%Although the MATLAB implementation is flexible and rapid to prototype, it runs out of memory at around $k=300$ because of limited system resources.
The running time for the CULA standard code is more than the device code because of the CPU-GPU data transfer overhead. For the same reason, the sparse CPU implementation, by avoiding the data transfer overhead, performs better than the GPU standard code for very small number of communities. We note that there is no performance degradation due to the parallelization of the matrix operations. After whitening, the STGD requires the most code design and optimization effort, and so we convert that into BLAS-like routines.

\subsection{Computational Complexity}

\begin{table}[htbp]
 \centering
   \begin{tabular}{@{} l|l|l@{}}
\hline
Module & Time & Space\\% & Var\\
\hline
\hline
%Cores & $O(c)$ & $O(c)$ & $O(c)$\\% & $O(nk)$\\
Preprocessing (Matrix Multiply) & $O\left(\max(nsk/c, \log s)\right)$ & $O\left(\max(s^2,sk)\right)$\\
\hline
Preprocessing (CPU SVD) & $O\left(\max(nsk/c,\log s) + \max(k^2/c,k) \right)$ & $O(sk)$\\
\hline
Preprocessing (GPU QR) & $O\left( \max(sk^2/c, \log s) + \max(sk^2/c,\log k) \right)$ & $O(sk)$\\
\hline
Preprocessing(short-thin SVD) & $O\left(\max(k^3/c,\log k) +\max(k^2/c,k)\right)$ & $O(k^2)$\\
\hline
STGD & $O\left(\max(k^3/c, \log k)   \right)$ & $O(k^2)$\\
\hline
Post-processing & $O\left(\max(nsk/c, \log s)\right)$ & $O(nk)$\\
\hline
   \end{tabular}
   \caption[Time and space complexity]{The time and space complexity (number of compute cores required) of our algorithm. Note that $k \ll n$, $s$ is the average degree of a node (or equivalently, the average number of non-zeros per row/column in the adjacency sub-matrix); note that the STGD time is per iteration time. We denote the number of cores as $c$ - the time-space trade-off depends on this parameter.}
   \label{tab:complexity}
\end{table}

We partition the execution of our algorithm into three main modules namely, pre-processing, STGD and post-processing, whose various matrix operation counts  are listed above in Table~\ref{tab:blas_count}.

The theoretical asymptotic complexity of our method is summarized in Table~\ref{tab:complexity} and is best addressed by considering the parallel model of computation~\cite{jaja1992introduction}, i.e., wherein a number of processors or compute cores are operating on the data simultaneously in parallel. This is justified considering that we implement our method on GPUs and matrix products are embarrassingly parallel. Note that this is different from serial computational complexity. We now break down the entries in Table~\ref{tab:complexity}. First, we recall a basic lemma regarding the lower bound on the time complexity for parallel addition along with the required number of cores to achieve a speed-up.
\begin{lemma}
~\cite{jaja1992introduction}
\label{lem:add}
Addition of $s$ numbers in serial takes $O(s)$ time; with $\Omega(s / \log s )$ cores, this can be improved to $O( \log s)$ time in the best case.
\end{lemma}
%\begin{proof}
Essentially, this speed-up is achieved by recursively adding pairs of numbers in parallel. %Applying Lemma~\ref{lem:add} for inner product of , we obtain the following
\begin{lemma}
~\cite{jaja1992introduction}
\label{lem:mat_mul}
Consider $M \in \mathbb{R}^{p \times q}$ and $N \in \mathbb{R}^{q \times r}$ with $s$ non-zeros per row/column. Naive serial matrix multiplication requires $O(psr)$ time; with $\Omega(psr / \log s)$ cores, this can be improved to $O( \log s)$ time in the best case.
\end{lemma}
Lemma~\ref{lem:mat_mul} follows by simply parallelizing the sparse inner products and applying Lemma~\ref{lem:add} for the addition in the inner products. Note that, this can be generalized to the fact that given $c$ cores, the multiplication can be performed in $O(\max(psr/c, \log s))$ running time.
%\end{proof}
\subsubsection{Pre-processing}

\paragraph{Random projection: }In preprocessing, given $c$ compute cores, we first do random projection using matrix multiplication. We multiply an $O(n) \times O(n)$ matrix $M_2$ with  an $O(n) \times O(k)$ random matrix $S$. Therefore, this requires $O(nsk)$ serial operations, where $s$ is the number of non-zero elements per row/column of $M_2$.  Using Lemma~\ref{lem:mat_mul}, given $c = \frac{nsk}{\log s}$ cores, we could achieve $O(\log s)$ computational complexity. However, the parallel computational complexity is not further reduced with more than $\frac{nsk}{\log s}$ cores.

After the multiplication, we use \emph{tall-thin SVD} for CPU implementation, and \emph{tall-thin QR} for GPU implementation.

\paragraph{Tall-thin SVD: }We perform Lanczos SVD on the tall-thin sparse $O(n)\times O(k)$ matrix, which involves a tri-diagonalization followed with the QR on the tri-diagonal matrix. Given $c = \frac{nsk}{\log s}$ cores, the computational complexity of the tri-diagonalization is $O(\log s)$. We then do QR on the tridiagonal matrix which is as cheap as $O(k^2)$ serially. Each orthogonalization requires $O(k)$ inner products of constant entry vectors, and there are $O(k)$ such orthogonalizations to be done. Therefore given $O(k)$ cores, the complexity is               $O(k)$.  More cores does not help since the degree of parallelism is $k$.

\paragraph{Tall-thin QR: }Alternatively, we perform QR in the GPU implementation which takes $O(sk^2)$. To arrive at the complexity of obtaining $Q$, we analyze the Gram-Schmidt orthonormalization procedure under sparsity and parallelism conditions. Consider a serial Gram-Schmidt on $k$ columns (which are $s$-dense) of $O(n) \times O(k)$ matrix. For each of the columns $2$ to $k$, we perform projection on the previously computed components and subtract it. Both inner product and subtraction operations are on the $s$-dense columns and there are $O(s)$ operations which are done $O(k^2)$ times serially. The last step is the normalization of $k$ $s$-dense vectors with is an $O(sk)$ operation. This leads to a serial complexity of $O(sk^2 + sk) = O(sk^2)$. Using this, we may obtain the parallel complexity in different regimes of the number of cores as follows.

\emph{Parallelism for inner products }: For each component $i$, we need $i-1$ projections on previous components which can be parallel.  Each projection involves scaling and inner product operations on a pair of $s$-dense vectors.  Using Lemma~\ref{lem:add},  projection for component $i$ can be performed in $O(\max(\frac{sk}{c}, \log s) )$ time.  $O(\log s)$ complexity is obtained using $O(sk/\log s)$ cores.

\emph{Parallelism for subtractions}:  For each component $i$, we need $i-1$ subtractions on a $s$-dense vector after the projection. Serially the subtraction requires $O(sk)$ operations, and this can be reduced to $O(\log k)$ with $O(sk/\log k)$ cores in the best case. The complexity is $O(\max(\frac{sk}{c}, \log k) )$.

Combing the inner products and subtractions, the complexity is $O\left(\max(\frac{sk}{c}, \log s)\right.$\\$\left. +\max(\frac{sk}{c}, \log k)\right)$ for component $i$. There are $k$ components in total, which can not be parallel. In total, the complexity for the parallel QR is $O\left(\max(\frac{sk^2}{c}, \log s) +\max(\frac{sk^2}{c}, \log k)\right)$.

\paragraph{Short-thin SVD: }SVD of the smaller $O( \Rbb^{k \times k})$ matrix  time requires $O(k^3)$ computations in serially. We note that this is the bottleneck for the computational complexity, but we emphasize that $k$ is sufficiently small in many applications. Furthermore, this $k^3$ complexity  can be reduced by using distributed SVD algorithms~e.g.~\cite{kannan2014principal,feldman2013turning}. An analysis with respect to Lanczos parallel SVD is similar with the discussion in the Tall-thin SVD paragraph. The complexity is $O(\max(k^3/c,\log k) + \max(k^2/c,k))$. In the best case, the complexity is reduced to $O(\log k + k)$.

The serial time complexity of SVD is $O(n^2 k)$ but with randomized dimensionality reduction~\cite{gittens2013revisiting} and parallelization~\cite{constantine2011tall}, this is significantly reduced.
\subsubsection{STGD}
In STGD, we perform implicit stochastic updates, consisting of a constant number of matrix-matrix and matrix-vector products, on the set of eigenvectors and whitened samples which is of size $k \times k$. When $c \in [1, k^3 / \log k]$, we obtain a running time of $O({k^3/c})$ for computing inner products in parallel with $c$ compute cores since each core can perform an inner product to compute an element in the resulting matrix independent of other cores in linear time. For $c \in (k^3 / \log k, \infty]$, using Lemma~\ref{lem:add}, we obtain a running time of $O(\log k)$. Note that the STGD time complexity is calculated per iteration.%, i.e., each core performs the multiplication of a single row vector in the first matrix with all the columns of the second matrix completely in parallel.%; also note that there is a trade-off between the number of cores required and the time complexity (for example, increasing the number of available cores to $O(k^2)$ reduces the running time to $O(k)$ as each core can independently perform an inner product in linear time). In general, using $O(c)$ cores, the time complexity is $O({k^3/c})$.
\subsubsection{Post-processing}
Finally, post-processing consists of sparse matrix products as well. Similar to pre-processing, this consists of multiplications involving the sparse matrices. Given $s$ number of non-zeros per column of an $O(n) \times O(k)$ matrix, the effective number of elements reduces to $O(sk)$. Hence, given $c \in [1, nks / \log s]$ cores, we need $O({nsk/c})$ time to perform the inner products for each entry of the resultant matrix. For $c \in (nks / \log s, \infty]$, using Lemma~\ref{lem:add}, we obtain a running time of $O(\log s)$.%More generally, using $O(c)$ cores, the time complexity is $O({nsk/c})$.

\bigskip
Note that $nk^2$ is the complexity of computing the exact SVD and we reduce it to $O(k)$ when there are sufficient cores available. This is meant for the setting where $k$ is small.
This $k^3$ complexity of SVD on $O(k\times k)$ matrix  can be reduced to $O(k)$  using distributed SVD algorithms~e.g.~\cite{kannan2014principal,feldman2013turning}.
We note that the variational inference algorithm complexity, by Gopalan and Blei~\cite{gopalan2013efficient},  is $O(mk)$ for each iteration, where $m$ denotes the number of edges in the graph, and  $n< m < n^2$.  In the regime that $n\gg k$, our algorithm is more efficient. Moreover, a big difference is in the scaling with respect to the size of the network and ease of parallelization of our method compared to variational one.

\section{Validation methods}\label{sec:val_meth}
% \fhcomment{mv the following to appendix, give a sentence to maintain the flow. }
%--------------------------------------------------------------
\subsection{$P$-value Testing }\label{sec:apdx_pval}
%\begin{wrapfigure}{r}{0.4\textwidth}
\begin{figure}[hbtp]%{\textwidth}
   \centering
   \psfrag{Matched}[c]{ }
   \psfrag{\$\\Pi\_\{1\}\$}[c]{\textcolor[rgb]{1,1,1}{\scriptsize{ $\Pi_{1}$}}}
   \psfrag{\$\\Pi\_\{2\}\$}[c]{\textcolor[rgb]{1,1,1}{\scriptsize{ $\Pi_{2}$}}}
   \psfrag{\$\\Pi\_\{3\}\$}[c]{\textcolor[rgb]{1,1,1}{\scriptsize{ $\Pi_{3}$}}}
   \psfrag{\$\\Pi\_\{4\}\$}[c]{\textcolor[rgb]{1,1,1}{\scriptsize{ $\Pi_{4}$}}}
   \psfrag{\$\\hat\{\\Pi\}\_\{1\}\$}[l]{\textcolor[rgb]{1,1,1}{\scriptsize{$\widehat{\Pi}_{1}$}}}
   \psfrag{\$\\hat\{\\Pi\}\_\{2\}\$}[l]{\textcolor[rgb]{1,1,1}{\scriptsize{$\widehat{\Pi}_{2}$}}}
   \psfrag{\$\\hat\{\\Pi\}\_\{3\}\$}[l]{\textcolor[rgb]{1,1,1}{\scriptsize{$\widehat{\Pi}_{3}$}}}
   \psfrag{\$\\hat\{\\Pi\}\_\{4\}\$}[l]{\textcolor[rgb]{1,1,1}{\scriptsize{$\widehat{\Pi}_{4}$}}}
   \psfrag{\$\\hat\{\\Pi\}\_\{5\}\$}[l]{\textcolor[rgb]{1,1,1}{\scriptsize{$\widehat{\Pi}_{5}$}}}
   \psfrag{\$\\hat\{\\Pi\}\_\{6\}\$}[l]{\textcolor[rgb]{1,1,1}{\scriptsize{$\widehat{\Pi}_{6}$}}}
    \includegraphics[height=3in]{\fighomeCommunity/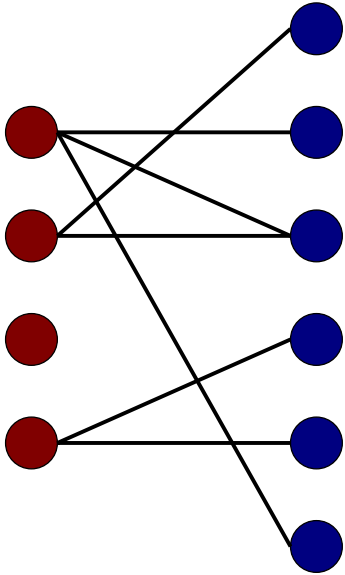}
   \caption[$P$-value matching]{Bipartite graph $G_{\{\Pvalue\}}$ induced by $p$-value testing. Edges represent statistically significant relationships between ground truth and estimated communities.}
      \label{fig:match}
\end{figure}
%\end{wrapfigure}
We recover the estimated community membership matrix $\widehat{\Pi}\in \mathbb{R}^{\widehat{k}\times n}$, where $\widehat{k}$ is the number of communities specified to our method. Recall that the true community membership matrix is $\Pi$, and we consider datasets where ground truth is available. Let $i$-th row of $\widehat{\Pi}$ be denoted by $\widehat{\Pi}_i$. Our community detection method is unsupervised, which inevitably results in row permutations between $\Pi$ and $\widehat{\Pi}$ and $\widehat{k}$ may not be the same as $k$. To validate the results, we need to find a good match between the rows of $\widehat{\Pi}$ and $\Pi$. We use the notion of $p$-values %(Appendix~\ref{sec:apdx_pval})
 to test for statistically significant dependencies among a set of random variables.
The $p$-value denotes the probability of not rejecting the null hypothesis that the random variables under consideration are independent and we use the  Student's\footnote{Note that Student's $t$-test is robust to the presence of unequal variances when the sample sizes of the two are equal which is true in our setting.} $t$-test statistic~\cite{fadem2012high} to compute the $p$-value. We use multiple hypothesis testing for different pairs of estimated and ground-truth communities $\widehat{\Pi}_i, \Pi_j$ and  adjust the $p$-values to ensure a small enough false discovery rate (FDR)~\cite{strimmer2008fdrtool}.

The test statistic used for the $p$-value testing of the estimated communities is
\[
T_{ij} : = \frac{\rho \left(\widehat{\Pi}_i,\Pi_j\right)\sqrt{n-2}}{\sqrt{1-\rho \left(\widehat{\Pi}_i,\Pi_j\right)^2}}.
\]
%\fhcomment{right $p$-value equation.}
The right $p$-value is obtained via the probability of obtaining a value (say $t_{ij}$) greater than the test statistic $T_{ij}$, and it is defined as
\[
\Pvalue(\Pi_i , \widehat{\Pi}_j):= 1- \mathbb{P}\left( t_{ij} > T_{ij}\right).
\]%We compute the probability of obtaining a value less extreme than the test statistic $T_{ij}$, $1-\mathrm{p}_{ij}$, and
Note that $T_{ij}$ has Student's $t$-distribution with degree of freedom $n-2$ (i.e. $T_{ij}\sim t_{n-2} $).  Thus, we obtain the right $p$-value\footnote{The right $p$-value accounts for the fact that when two communities are anti-correlated they are not paired up. Hence note that in the special case of block model in which the estimated communities are just permuted version of the ground truth communities, the pairing results in a perfect matching accurately.}.

In this way, we compute the $\Pvaluem$ matrix as\[
\Pvaluem(i,j):=\Pvalue\left[\widehat{\Pi}_i,\Pi_j\right], \forall i\in [k] \text{ and } j \in [\widehat{k}].
\]

\subsection{Evaluation Metrics}\label{sec:defineourscores}
%\paragraph{Recovery ratio and error function: }%\label{sec:defineourscores}
%\aacomment{need to edit this after discussion}
\paragraph{Recovery ratio: }Validating the results requires a matching of the true membership $\Pi$ with estimated membership $\widehat{\Pi}$.
Let $\Pvalue(\Pi_i , \widehat{\Pi}_j)$ denote the right $p$-value under the null  hypothesis that $\Pi_i$ and $\widehat{\Pi}_j$ are statistically independent. We use the $p$-value test to find out pairs $\Pi_i , \widehat{\Pi}_j$ which pass a specified $p$-value threshold, and we denote such pairs using  a bipartite graph $G_{\{\Pvalue\}}$. Thus, $G_{\{\Pvalue\}}$ is defined as
\[
G_{\{\Pvalue\}}:=\left(\left\{V^{(1)}_{\{\Pvalue\}},V^{(2)}_{\{\Pvalue\}}\right\}, E_{\{\Pvalue\}} \right),
\]
where the nodes in the two node sets are
\begin{align*}
&V^{(1)}_{\{\Pvalue\}}=\left\{\Pi_1,\ldots,\Pi_k\right\},\quad \\
&V^{(2)}_{\{\Pvalue\}}=\left\{\widehat{\Pi}_1,\ldots,\widehat{\Pi}_{\widehat{k}}\right\}
\end{align*}
 and the edges of $G_{\{\Pvalue\}}$ satisfy
\[(i,j)\in
E_{\{\Pvalue\}} \text{ s.t. } \Pvalue\left[\widehat{\Pi}_i,\Pi_j\right] \le 0.01
.
\]
%, \quad \forall i\in[\widehat{k}] \ge 0, \forall j\in[k], \]
%where $\rho$ denotes correlation coefficient.

A simple example is shown in Figure~\ref{fig:match}, in which $\Pi_{2}$  has statistically significant dependence with $\widehat{\Pi}_1$, i.e., the probability of not rejecting the null hypothesis is small (recall that null hypothesis is that they are independent). If no estimated membership vector has a significant overlap with $\Pi_3$, then $\Pi_3$ is not recovered. There can also be multiple pairings such as for $\Pi_{1}$ and $\{\widehat{\Pi}_2,\widehat{\Pi}_3,\widehat{\Pi}_6\}$. The $p$-value test between $\Pi_{1}$ and $\{\widehat{\Pi}_2,\widehat{\Pi}_3,\widehat{\Pi}_6\}$ indicates that probability of not rejecting the null hypothesis is small, i.e., they are independent. We use $0.01$ as the threshold. The same holds for $\Pi_{2}$ and $\{\widehat{\Pi}_1\}$ and for $\Pi_{4}$ and $\{\widehat{\Pi}_4,\widehat{\Pi}_5\}$. There can be a perfect one to one matching like for $\Pi_{2}$ and $\widehat{\Pi}_1$ as well as a multiple matching such as for $\Pi_{1}$ and $\{\widehat{\Pi}_2,\widehat{\Pi}_3,\widehat{\Pi}_6\}$. Or another multiple matching such as for $\{\Pi_{1},\Pi_{2}\} $ and $\widehat{\Pi}_3$.%No estimated membership vector has a significant overlapping with $\Pi_3$, then $\Pi_3$ is not recovered.

Let $\degree_i$ denote the degree of ground truth community $i\in[k]$ in $G_{\{\Pvalue\}}$, we define the recovery ratio as follows.
\begin{definition}
The \emph{recovery ratio} is defined as
%\begin{equation}\label{eq:matchratio}
%\mathcal{R}:=\frac{1}{k}\sum_{j\in[k]} {\mathbb{I}\left\{\sum_{i\in[\widehat{k}]}\mathbb{I}\left\{\Pvalue\left[\widehat{\Pi}_i,\Pi_j\right] \le 0.01\right\}\ge 1\right\}}.
\begin{equation*}
\mathcal{R}:=\frac{1}{k} \sum\limits_{i}\mathbb{I}\left\{\degree_i >0 \right\}, \quad i\in[k]
%{(i,j)\in \lvert E_{\{\Pvalue\}}\rvert%}
%,k\right\}
  \end{equation*}
   %    \end{equation}
    where $\mathbb{I}(x)$ is the indicator function whose value equals one if $x$ is true.
\end{definition}
The perfect case is that all the memberships have at least one significant overlapping estimated membership, giving a recovery ratio of $100\%$.
%A simple example is shown in Figure~\ref{fig:match}, in which $\Pi_{2}$  has statistically significant dependence with $\{\widehat{\Pi}_1\}$, i.e., the probability of not rejecting the null hypothesis is small (recall that null hypothesis is that they are independent).   If no estimated membership vector has a significant overlapping with $\Pi_3$, $\Pi_3$ is not recovered. There can also be multiple pairings such as for $\Pi_{1}$ and $\{\widehat{\Pi}_2,\widehat{\Pi}_3,\widehat{\Pi}_6\}$.
%$p$-value test between $\Pi_{1}$ and $\{\widehat{\Pi}_2,\widehat{\Pi}_3,\widehat{\Pi}_6\}$ indicates that probability of not rejecting the null hypothesis (i.e., they are independent) is small (we use $0.01$ as the threshold). The same is for $\Pi_{2}$ and $\{\widehat{\Pi}_1\}$ and for $\Pi_{4}$ and $\{\widehat{\Pi}_4,\widehat{\Pi}_5\}$. There can be a perfect one to one matching like for $\Pi_{2}$ and $\{\widehat{\Pi}_1\}$ as well as a multiple matching such as for $\Pi_{1}$ and $\{\widehat{\Pi}_2,\widehat{\Pi}_3,\widehat{\Pi}_6\}$. No estimated membership vector has a significant overlapping with $\Pi_3$, then $\Pi_3$ is not recovered.
\paragraph{Error function: }For performance analysis of our learning algorithm, we use an error function given as follows:
\begin{definition}
The average error function is defined as
%\begin{equation}\label{eq:error}

\[
\mathcal{E}: =\frac{1}{k}\sum\limits_{
      (i,j)\in E_{\{\Pvalue\}}
}
 \left\{\frac{1}{n}\sum\limits_{x\in |X|} {\llvert \widehat{\Pi}_i (x)- \Pi_j(x) \llvert}\right\},
%\end{equation}
\] where $E_{\{\Pvalue\}}
$ denotes the set of edges based on thresholding of the $p$-values.
\end{definition}

The error function incorporates two aspects, namely the $l_1$ norm error between each estimated community  and the corresponding paired ground truth community, and the error induced by false pairings between the estimated and ground-truth communities through $p$-value testing. For the former $l_1$ norm error, we normalize with $n$ which is reasonable and results in the range of the error in $[0,1]$. For the latter, we define the average error function as the summation of all paired memberships errors divided by the true number of communities $k$. In this way we penalize falsely discovered pairings by summing them up. Our error function can be greater than 1 if there are too many falsely discovered pairings through $p$-value testing (which can be as large as $k\times \widehat{k}$).

\paragraph{Bridgeness: }
Bridgeness in overlapping communities is an interesting measure to evaluate.  A bridge is defined as a vertex that crosses structural holes between discrete groups of people and bridgeness analyzes the extent to which a given vertex is shared among different communities~\cite{nepusz2008fuzzy}. Formally, the bridgeness of a vertex $i$ is defined as %the distance of its membership vector from the reference vector $[\frac{1}{\widehat{k}},\ldots,\frac{1}{\widehat{k}}]$ in the Euclidean vector norm, inverted and normalized to the interval [0,1] as follows:
\begin{equation}\label{eq:bridgeness}
b_i:=1-\sqrt{\frac{\widehat{k}}{\widehat{k}-1}\sum\limits_{j=1}^{\widehat{k}}{\left(\widehat{\Pi}_i(j)-\frac{1}{\widehat{k}}\right)}^2}.
\end{equation}
Note that centrality measures should be used in conjunction with bridge score to distinguish outliers from genuine bridge nodes~\cite{nepusz2008fuzzy}.
The \emph{degree-corrected bridgeness} is  used to evaluate our results and is defined as
\begin{equation}\label{eq:degreebridgeness}
\mathcal{B}_i:=D_ib_i,
\end{equation}
where $D_i$ is  degree of node $i$.

%\section{Results}\label{sec:results}
\section{Experimental Results}
\label{sec:results}
%We perform synthetic experiments for both stochastic block model and mixed membership model and  we obtain highly accurate recovery.% (Appendix~\ref{sec:apdx_synth}).
%The specifications of the machine on which we run our code are given in Table~\ref{tab:sys_spec}.

%\section{Additional Results on Datasets}

\paragraph{Results on Synthetic Datasets: }

We perform experiments for both the stochastic block model ($\alpha_0=0$) and the mixed membership model. For the mixed membership model, we set the concentration parameter $\alpha_0 =1$. We note that the error is around $8\% - 14\%$ and the running times are under a minute, when $n \leq 10000$ and $n \gg k$.
%\footnote{The code is available at\\ \url{http://github.com/FurongHuang/Fast-Detection-of-Overlapping-Communities-via-Online-Tensor-Methods}}.

We observe that more samples result in a more accurate recovery of memberships which matches intuition and theory. Overall, our learning algorithm performs better in the stochastic block model case than in the mixed membership model case although we note that the accuracy is quite high for practical purposes. Theoretically, this is expected since smaller concentration parameter $\alpha_0$ is easier for our algorithm to learn~\cite{AnandkumarEtal:community12COLT}. Also, our algorithm is scalable to an order of magnitude larger in $n$ as illustrated by experiments on real-world large-scale datasets.

Note that we threshold the estimated memberships to clean the results. There is a tradeoff between match ratio and average error via different thresholds. In synthetic experiments, the tradeoff is not evident since a perfect matching is always present. However, we need to carefully handle this in experiments involving real data.

%%%%%%%%%%%%%%%%%%%%%%%%%%%%%%%%%%%%%%%%%%%%%%%%
\paragraph{Results on Topic Modeling:} We perform experiments for the bag of words data set~\cite{Bache+Lichman:2013} for The New York Times. We set the concentration parameter to be $\alpha_0=1$ and observe top recovered words in numerous topics. The results are in Table~\ref{tab:Nytimes}. Many of the results are expected. For example, the top words in topic \# 11 are all related to some bad personality.

We also present the words with most spread membership, i.e., words that belong to many topics as in Table~\ref{tab:nytimesbridge}. As expected, we see minutes, consumer, human, member and so on. These words can appear in a lot of topics, and we expect them to connect topics.

% Requires the booktabs if the memoir class is not being used
\tiny
\begin{table}[htbp]
\scriptsize
   \centering
   %\topcaption{Table captions are better up top} % requires the topcapt package
   \begin{tabular}{@{} l|lllll@{}} % Column formatting, @{} suppresses leading/trailing space
%  fleshing	&	density	&	fungicides	&	woodshop	&	fashioning	&	prevention	&	damping	&	bluebird	&	stave	&	deteriorated	\\	
%\toprule
\hline
Topic \#&  & Top Words& & &  \\
\hline
\hline
1& prompting	&	complicated	&	eviscerated	&	predetermined	&	lap\\
	&	renegotiating	&	loose	&	entity	&	legalese	&	justice\\
	\hline	
%zorro	&	lap	&	ethnically	&	clubface	&	abide	&	trioval	&	rminter	&	demilitarize	&	pit	&	smoothly	\\	
%humanitarianism	&	cynically	&	hallucinatory	&	monotone	&	betraying	&	evidently	&	flighty	&	marginal	&	extradite	&	complicated	\\	
%patently	&	inverse	&	deductible	&	pathologically	&	rakish	&	operetta	&	prohibitive	&	mise	&	punish	&	deathless	\\	
%ascendance	&	wondrous	&	voiding	&	seyman	&	impulsively	&	refutation	&	bette	&	prose	&	haunted	&	reportage	\\	
2 & hamstrung	&	airbrushed	&	quasi	&	outsold	&	fargo	\\
&	ennobled	&	tantalize	&	irrelevance	&	noncontroversial	&	untalented	\\	
\hline
%ditches	&	everglades	&	heavily	&	underlining	&	massing	&	slackened	&	roguish	&	largess	&	levees	&	rowdiness	\\
%estranged	&	unanswerable	&	beguiled	&	plodding	&	autobiographical	&	parrot	&	enrollees	&	biologist	&	forwarding	&	musketeer	\\
%tokyo	&	teleconference	&	asses	&	proceed	&	modernistic	&	suicidal	&	million	&	net2phone	&	divestitures	&	totaling	\\	
%seduces	&	earthier	&	preludes	&	sameness	&	scriptwriter	&	contamination	&	orchestration	&	farmer	&	visage	&	wittier	\\	
%baleful	&	intriguing	&	affect	&	repulsive	&	megadeal	&	sermonizing	&	flashpoint	&	obstructive	&	hissing	&	psi	\\
%passionately	&	laudatory	&	soup	&	spare	&	scotched	&	tirelessly	&	effete	&	bulldog	&	fittingly	&	embolden	\\
%peppy	&	midtempo	&	refile	&	evangelizing	&	waltzes	&	concisely	&	venting	&	freighted	&	posturing	&	bedpan\\
%affiliation	&	hypothetical	&	cantata	&	plotting	&	outset	&	bovine	&	pincer	&	surfaced	&	firepower	&	retitled\\
%husker	&	buffeting	&	richer	&	deploying	&	underpass	&	ensures	&	unobtrusively	&	quick	&	shaking	&	jump	\\
3 &scariest	&	pest	&	knowingly	&	causing	&	flub\\
	&	mesmerize	&	dawned	&	millennium	&	ecological	&	ecologist	\\
	\hline	
%prohibitive	&	embed	&	proving	&	unused	&	traceable	&	security	&	paranormal	&	concretely	&	bgee	&	salaam\\
4 & reelection	&	quixotic	&	arthroscopic	&	versatility	&	commanded\\
	&	hyperextended	&	anus	&	precipitating	&	underhand	&	knee\\
	\hline
%brutality	&	decoder	&	personifies	&	incredibly	&	intimating	&	unexplored	&	antennas	&	pathetic	&	tactical	&	tremor\\
5 &believe	&	signing	&	ballcarrier	&	parallel	&	anomalies\\
	&	munching	&	prorated	&	unsettle	&	linebacking	&	bonus\\
	\hline
6 &gainfully	&	settles	&	narrator	&	considerable	&	articles\\
	&	narrative	&	rosier	&	deviating	&	protagonist	&	deductible\\
	\hline
7 &faithful	&	betcha	&	corrupted	&	inept	&	retrench \\
	&	martialed	&	winston	&	dowdy	&	islamic	&	corrupting	\\
	\hline
%denier	&	vignette	&	astray	&	wrongly	&	unearthed	&	diminution	&	concentration	&	blander	&	stoplight	&	telecommute\\
%tarred	&	heroin	&	swiveling	&	oafish	&	potholes	&	joking	&	calm	&	sizable	&	anger	&	reconditioned\\
8 &capable	&	misdeed	&	dashboard	&	navigation	&	opportunistically\\
	&	aerodynamic	&	airbag	&	system	&	braking	&	mph\\
	\hline
%annoying	&	satirist	&	cursor	&	complicit	&	immunologist	&	incrimination	&	collateral	&	surely	&	stupidity	&	forbearance\\
%parred	&	antagonize	&	cardiovascular	&	councilor	&	dogged	&	synopsis	&	fairway	&	shifty	&	syphilis	&	greenside	\\
9 &apostles	&	oracles	&	believer	&	deliberately	&	loafer	\\
&	gospel	&	apt	&	mobbed	&	manipulate	&	dialogue\\
\hline
%legitimizes	&	absorbed	&	budging	&	whiter	&	flexed	&	dissonant	&	centenary	&	compensates	&	sonata	&	appearances\\
10 & physique	&	jumping	&	visualizing	&	hedgehog	&	zeitgeist	\\
&	belonged	&	loo	&	mauling	&	postproduction	&	plunk\\
\hline
%semblance	&	tarred	&	bungee	&	tolerably	&	autocratic	&	upended	&	offensives	&	unprincipled	&	covering	&	defiance\\
11 &smirky	&	silly	&	bad	&	natured	&	frat	\\
&	thoughtful	&	freaked	&	moron	&	obtuse	&	stink	\\
\hline
12 &offsetting	&	preparing	&	acknowledgment	&	agree	&	misstating\\
	&	litigator	&	prevented	&	revoked	&	preseason	&	entomology\\
	\hline
%fifty	&	ensuing	&	truth	&	splintering	&	torturous	&	knock	&	inevitable	&	applauding	&	relying	&	flog	\\
%curtail	&	muddied	&	subordination	&	overabundance	&	licorice	&	zapped	&	mapping	&	repetition	&	pugnacity	&	nonpublic	\\
13 &undertaken	&	wilsonian	&	idealism	&	brethren	&	writeoff	\\
&	multipolar	&	hegemonist	&	multilateral	&	enlargement	&	mutating	\\	
\hline
%faithless	&	backsliding	&	unrestrained	&	nastiness	&	fish	&	reloaded	&	renounces	&	contribute	&	tendency	&	subtext	\\
%guarantees	&	h2o	&	awkward	&	catechism	&	slag	&	text	&	newsletter	&	addictive	&	steadied	&	meantime	\\
14 & athletically	&	fictitious	&	myer	&	majorleaguebaseball	&	familiarizing\\
	&	resurrect	&	slug	&	backslide	&	superseding	&	artistically	\\
	\hline
%annotation	&	anchovies	&	kinship	&	biographies	&	undersea	&	discovered	&	dueling	&	anchovy	&	knifed	&	thy	\\
%wanna	&	virginian	&	woeful	&	drubbed	&	avoidance	&	chimes	&	uncertainly	&	underused	&	faintly	&	frames	\\	
%trembles	&	revealingly	&	oompah	&	organizer	&	bros	&	dramatize	&	televises	&	deviated	&	give	&	redial	\\
%overpay	&	bopper	&	uptempo	&	rubbish	&	commemorative	&	preordained	&	exposes	&	dishonored	&	confidant	&	funk	\\
%wiggle	&	distancing	&	desensitized	&	uncompetitive	&	satiate	&	smoother	&	gird	&	cofounded	&	retooling	&	agile	\\
%circuited	&	prewar	&	exploiting	&	shard	&	blamed	&	employer	&	imdb	&	lubricant	&	interfered	&	cahir	\\
%untouched	&	music	&	diminishment	&	segues	&	mainstream	&	miso	&	releasing	&	catchy	&	rumba	&	reverses	\\
%multiply	&	unwind	&	derriere	&	minority	&	painlessly	&	redirect	&	recompense	&	cartilage	&	vocational	&	indigo\\	
%radiate	&	burkas	&	forbearance	&	wharf	&	soviet	&	modernistic	&	burka	&	unjustly	&	furthered	&	comprehend\\
%ghettos	&	fasting	&	contending	&	informed	&	ballroom	&	listening	&	snipping	&	appropriately	&	masthead	&	churchman	\\
%claymation	&	squabbles	&	briefed	&	intimation	&	dreaminess	&	unbowed	&	cautioning	&	disclosures	&	drone	&	gauzy\\
15 & dialog	&	files	&	diabolical	&	lion	&	town	\\
&	password	&	list	&	swiss	&	coldblooded	&	outgained	\\
\hline
%storylines	&	railway	&	rotor	&	biplane	&	slithered	&	hitchhiker	&	gripping	&	benzes	&	speechifying	&	smoothly\\
16 & recessed	&	phased	&	butyl	&	lowlight	&	balmy\\
	&	redlining	&	prescription	&	marched	&	mischaracterization	&	tertiary\\
	\hline
%atlantis	&	nothin	&	edition	&	brazilian	&	garbled	&	studio	&	schoolyard	&	sincerest	&	tomboy	&	airplay\\
%swordsman	&	progenitor	&	testes	&	denies	&	endocrine	&	replicas	&	regenerating	&	widthwise	&	flyboy	&	ized\\
%retirees	&	portion	&	strum	&	steamrolling	&	pontoon	&	contributes	&	songstress	&	quizzes	&	commiserated	&	trivia	\\
%p53	&	retracted	&	bette	&	hacker	&	weight	&	bromides	&	adult	&	entrusting	&	hammered	&	cost\\
17 & sponsor	&	televise	&	sponsorship	&	festival	&	sullied\\
	&	ratification	&	insinuating	&	warhead	&	staged	&	reconstruct	\\
	\hline
18 &trespasses	&	buckle	&	divestment	&	schoolchild	&	refuel	\\
&	ineffectiveness	&	coexisted	&	repentance	&	divvying	&	overexposed	\\	
%\bottomrule
\hline
 \end{tabular}
   \caption[New York Times results: topics]{Top recovered topic groups from the New York Times dataset along with the words present in them.}
   \label{tab:Nytimes}
\end{table}
\normalsize

% Requires the booktabs if the memoir class is not being used
\begin{table}[htbp]
   \centering
   %\topcaption{Table captions are better up top} % requires the topcapt package
   \begin{tabular}{@{} l @{}} % Column formatting, @{} suppresses leading/trailing space
\hline
%      \toprule
Keywords \\%      TOP WORDS\\
\hline
      \hline%\\
   minutes, %	\\
consumer, %	\\
human, %	\\
member, %	\\
friend, %	\\
program, %	\\
board, %	\\
cell, %	\\
insurance, %	\\
shot	\\
\hline
%      \bottomrule
   \end{tabular}
   \caption[New York Times results: words]{The top ten words which occur in multiple contexts in the New York Times dataset.}
   \label{tab:nytimesbridge}
\end{table}

\paragraph{Results on Real-world Graph Datasets: }We describe the results on real datasets summarized in Table~\ref{tab:data_info} in detail below. The simulations are summarized in Table~\ref{table:businessresults}.

\begin{table}[htbp]
\small
   \centering
      \begin{tabular}{@{} l|l|l|l|l@{}} % Column formatting, @{} suppresses leading/trailing space
%\toprule
\hline
Statistics				& Facebook 	& Yelp  			& DBLP sub & DBLP \\
\hline
\hline
%\midrule
$\lvert E \rvert$		&	766,800	&672,515			&5,066,510 	& 16,221,000\\
$\lvert V\rvert$		&	18,163   	&10,010$+$28,588	& 116,317		& 1,054,066\\
GD 				& 0.004649 	& 0.000903 		& 0.000749 	& 0.000029\\
$k$				&	360		&159				&250 		& 6,003\\
%MCS				&	21		&12				&1,101 		& 1 \\
AB                  		& 0.5379           &  0.4281      		& 0.3779 		& 0.2066\\
ADCB			&	47.01	&30.75			&48.41 		& 6.36\\
%      \bottomrule
\hline
   \end{tabular}
   \caption[Datasets summary]{Summary of real datasets used in our thesis: $\lvert V\rvert$	is the number of nodes in the graph, $\lvert E \rvert$ is the number of edges, GD is the graph density given by $\frac{2\lvert E\lvert}{\lvert V\rvert\left(\lvert V\rvert-1\right)}$,  %MCS is the minimum community size,
   $k$ is the number of communities,
   AB is the average bridgeness and ADCB is the average degree-corrected bridgeness(explained in Section~\ref{sec:val_meth}).}
 \label{tab:data_info}
\end{table}

\begin{table}[h]
\centering
{\scriptsize
\begin{tabular}{@{} l|l|l|l|l|l|l @{}}
\hline
%      \toprule
Data 	& Method 		& $\widehat{k}$ 	& Thre 	& $\mathcal{E}$ 	&$\mathcal{R} (\%)$		&  Time(s)  \\
\hline
 \hline
		&Ten(sparse) 		&$10$		&$0.10$  		& $0.063$     	& $13$   		&$35$  \\
		&Ten(sparse) 		&$100$		&$0.08$ 		&$0.024$ 		& $62$  		&$309$ \\
		&Ten(sparse) 		&$100$		&$0.05$ 		& $0.118$ 	& $95$ 		&$309$ \\
 		&Ten(dense)  		&$100$		&$0.100$  	& $0.012$     	& $39$ 		 &$190$  \\
		&Ten(dense)  	 	&$100$		&$0.070$ 		& $0.019$ 	& $100$ 	 	&$190$ \\
FB		& Variational 		&$100$ 		&--			& $0.070$     	& $100$ 	 	&$10,795$  \\
		& Ten(dense)  		&$500$		&$0.020$  	& $0.014$  	& $71$ 		&$468$  \\
		& Ten(dense) 		&$500$		&$0.015$  	& $0.018$ 	& $100$		&$468$ \\
		& Variational 			&$500$    		&--			& $0.031$  	& $100$	 	&$86,808$  \\

 \hline
		&Ten(sparse) 		&$10$    		&$0.10$  		&$0.271$     	&$43$ 		&$10$  \\
		&Ten(sparse) 		&$100$ 		&$0.08$ 		&$0.046$ 		&$86$ 		&$287$ \\
		&Ten(dense) 		&$100$    		& $0.100$  	& $0.023$     	&  $43$ 		&$1,127$  \\
YP		&Ten(dense) 		&$100$ 		&$0.090$ 		& $ 0.061$ 	& $80$ 		&$1,127$ \\
		&Ten(dense)		&$500$    		&$0.020$   	& $0.064$  	& $72$		&$1,706$  \\
		&Ten(dense)		&$500$ 		&$0.015$  	& $0.336$ 	& $100$		&$1,706$ \\
\hline
		&Ten(dense) 		&$100$ 		&$0.15$  		&$0.072$ 		&$36$ 		&$7,664$ \\
		&Ten(dense) 		&$100$    		&$0.09$  		&$0.260$     	&$80$ 		&$7,664$  \\
			&Variational				&$100$		&--			&$7.453$		&$99$		&$69,156$\\
DB sub		&Ten(dense)  		&$500$    		&$0.10$  		&$0.010$  	&$19$		&$10,157$  \\
			&Ten(dense) 		&$500$ 		&$0.04$   		&$0.139$ 		&$89$		&$10,157$ \\
			&Variational				&$500$		&--			&$16.38$		&$99$		&$558,723$\\
\hline
		&Ten(sparse)  		&$10$    		&$0.30$ 		&$0.103$  	&$73$     		&$4716$  \\
DB		&Ten(sparse) 		&$100$ 		&$0.08$  		&$0.003$   	&$57$ 		&$5407$ \\
		&Ten(sparse) 		&$100$ 		&$0.05$  		&$0.105$ 	       	&$95$ 		&$5407$ \\
\hline
%    \bottomrule
   \end{tabular}
   }
\caption[Compare community detection results against variational method]{Yelp, Facebook and DBLP main quantitative evaluation of the tensor method versus the variational method:
%(Since Yelp is bipartite, variational method is not applicable. And DBLP data is too large for variational method):
 $\widehat{k}$ is the community number specified to our algorithm, Thre is the threshold for picking significant estimated membership entries. Refer to Table~\ref{tab:data_info} for statistics of the datasets.%  the
 %$\mathcal{E}$ is the classification error per node per community, $\mathcal{R}$ is the fraction of ground truth communities recovered.
 }
\label{table:businessresults}
\end{table}

%\begin{table}[htbp]
%%\begin{table*}[htbp]
%   \scriptsize
%   \centering
%   %\topcaption{Table captions are better up top} % requires the topcapt package
%   \begin{tabular}{@{} l|l|l@{}} % Column formatting, @{} suppresses leading/trailing space
%%   \begin{tabularx}{\textwidth}{|l|l|l|} % Column formatting, @{} suppresses leading/trailing space
%\hline
%%      \toprule
%      \textbf{Business} & \textbf{RC} & \textbf{Categories}\\
%\hline
%      \hline
%\textbf{Four Peaks Brewing Co}		&735			& Restaurants, Bars, American (New), Nightlife, Food, Pubs, Tempe\\
%\textbf{Pizzeria Bianco}			&803			& Restaurants, Pizza,Phoenix\\
%\textbf{FEZ}					&652			& Restaurants, Bars, American (New), Nightlife, Mediterranean, Lounges\\%,$\quad $Phoenix\\
% & & Phoenix\\
%\textbf{Matt's Big Breakfast}		&689			& Restaurants, Phoenix, Breakfast\& Brunch\\
%\textbf{Cornish Pasty Company}	&580			& Restaurants, Bars, Nightlife, Pubs, Tempe\\
%\textbf{Postino Arcadia}			&575			& Restaurants, Italian, Wine Bars, Bars, Nightlife, Phoenix\\
%\textbf{Cibo}					&594			& Restaurants, Italian, Pizza, Sandwiches, Phoenix\\
%\textbf{Phoenix Airport}			&862			&  Hotels \& Travel, Phoenix\\
%\textbf{Gallo Blanco Cafe}			&549			& Restaurants, Mexican, Phoenix\\
%\textbf{The Parlor}				&489			& Restaurants, Italian, Pizza, Phoenix\\
%\hline
%%          \bottomrule
%   \end{tabular}
%   \caption[Top bridging businesses in Yelp]{Top 10 bridging businesses in   Yelp    and categories they belong to. ``RC'' denotes review counts for that particular business.}
%   \label{tab:bridgeYELPwithCom}
%\end{table}

The results are presented in Table~\ref{table:businessresults}. We note that our method, in both dense and sparse implementations, performs very well compared to the state-of-the-art variational method.
For the Yelp dataset, we have a bipartite graph where the business nodes are on one side and user nodes on the other and use the review stars as the edge weights.  %Which users review which businesses depends both on the user attributes as well as business attributes. For instance, women generally review nail salons while men review tire businesses as seen in Figure~\ref{fig:tradeoff}.
%Refer Appendix~\ref{sec:apdx_yelp} for the preprocessing of the Yelp dataset.
In this bipartite setting, %our error metric is feasible while the other four scores such as ``separability'',``density'' are not applicable since the interior connectivity of a community is none.
the variational code provided by Gopalan et al~\cite{gopalan2012scalable} does not work on since it is not applicable to non-homophilic models. Our approach does not have this restriction. Note that we use our dense implementation on the GPU to run experiments with large number of communities $k$ as the device implementation is much faster in terms of running time of the STGD step.%, since  the GPU is SIMD architecture thus is fast at the STGD step .
On the other hand, the sparse implementation on CPU is fast and memory efficient in the case of sparse graphs with a small number of communities while the dense implementation on GPU is faster for denser graphs such as Facebook. Note that data reading time for DBLP is around 4700 seconds, which is not negligible as compared to other datasets (usually within a few seconds). Effectively, our algorithm, excluding the file I/O time, executes within two minutes for $k=10$ and within ten minutes for $k=100$.

\begin{figure}[h]
%\subfloat[a][\small{Distribution of business categories}]
%{\begin{minipage}{3.2in}\centering\label{fig:distri}
\psfrag{Number of categories}[l]{\tiny{Business Category ID}}
\psfrag{Number of business}[l]{\tiny{\# business $\ $}}
\psfrag{Distribution of Categories: }[c]{\scriptsize{ }}
 \includegraphics[width=0.49\columnwidth]{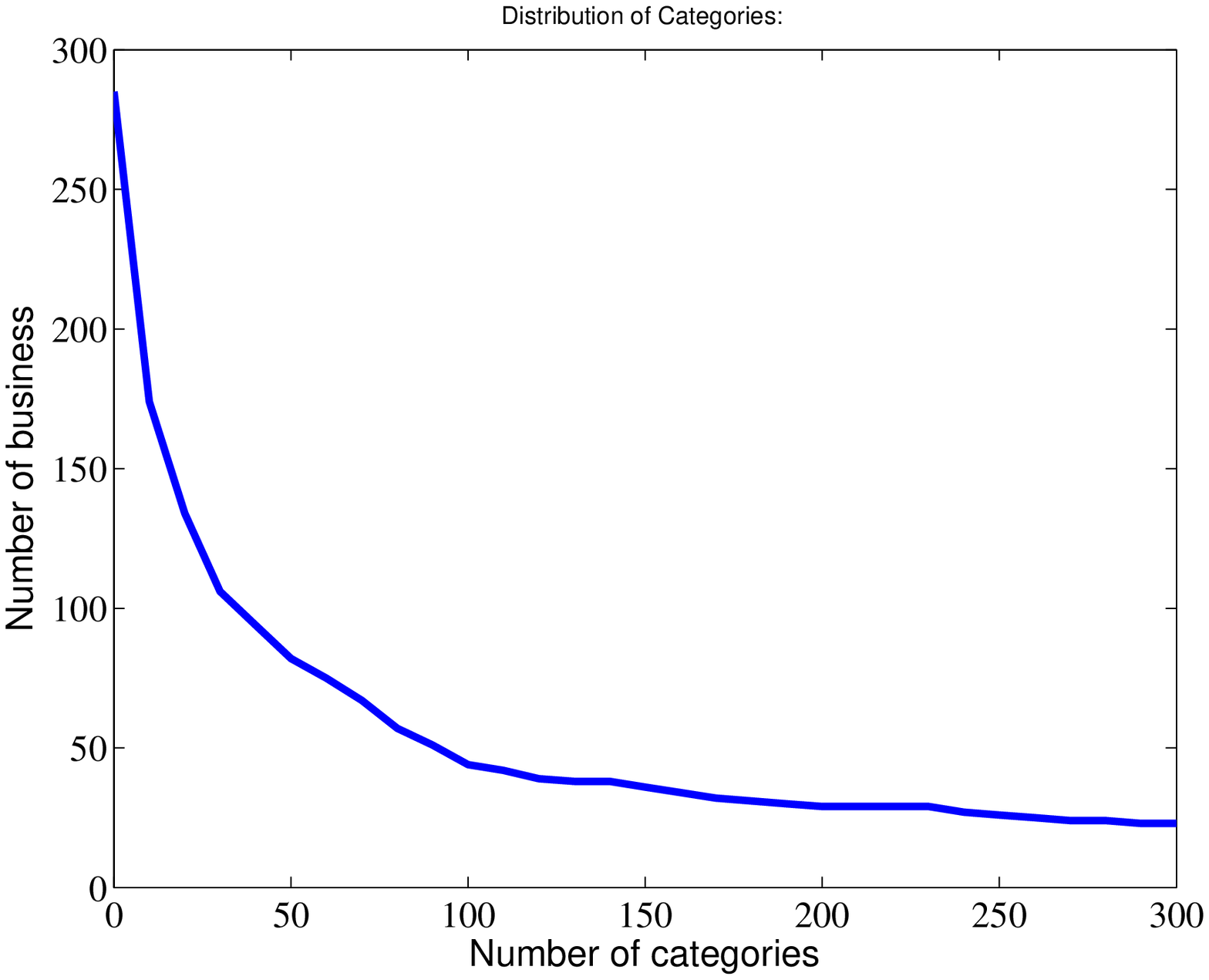}
% \end{minipage}}
% \hfil
%\subfloat[b][\small{Recovery ratio vs error}]
%{\begin{minipage}{3.2in}\centering\label{fig:tradeoff}
\psfrag{match ratio}[l]{\tiny{Recovery Ratio}}
\psfrag{average error}[l]{\tiny{Average Error $\quad \quad $}}
\includegraphics[width=0.49\columnwidth]{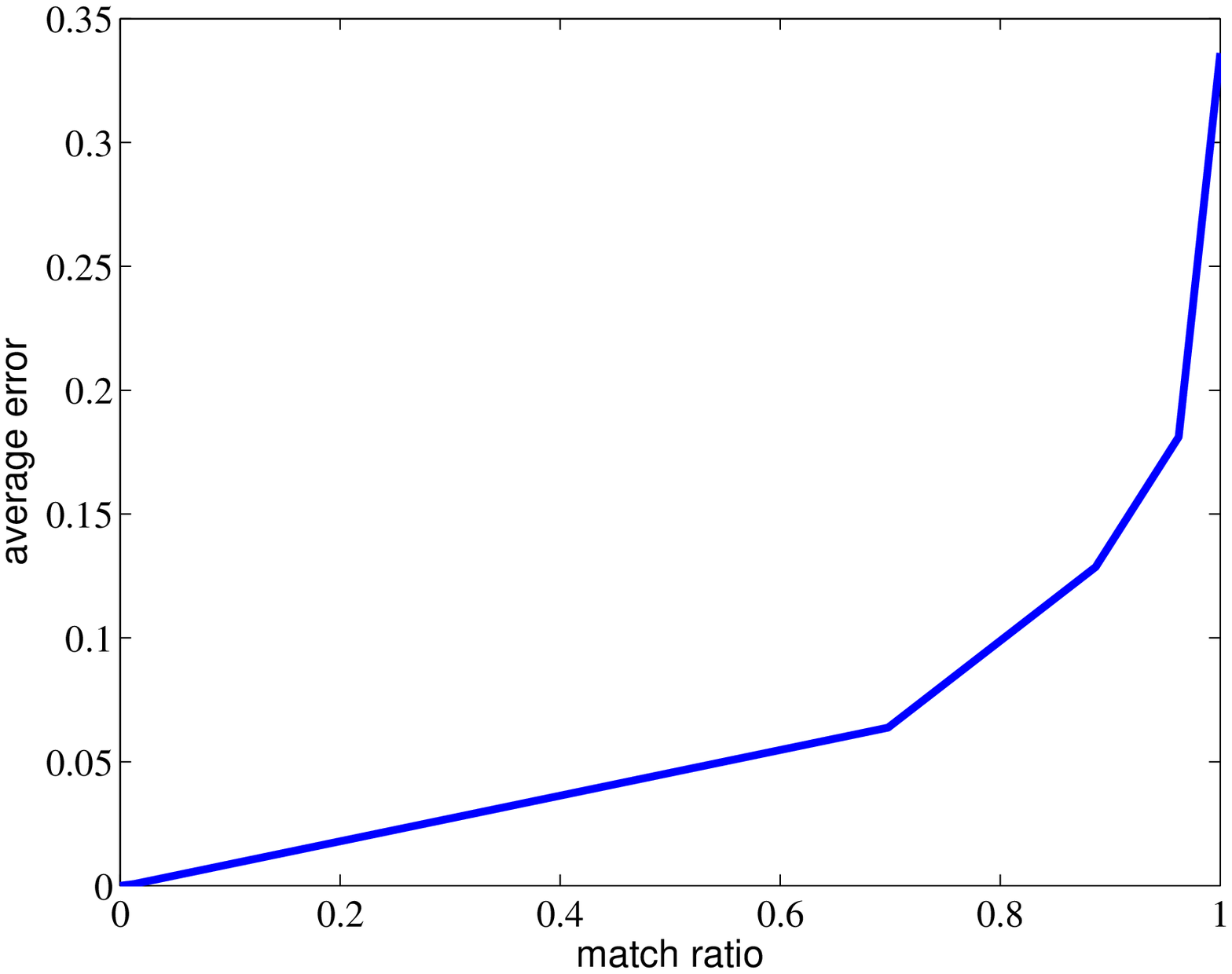}
%\end{minipage}}
\caption[Yelp result]{Distribution of business categories (left) and result tradeoff between recovery ratio and error for yelp (right).}
\label{fig:tradeoff}
\end{figure}

\paragraph{Interpretation on Yelp Dataset: }The ground truth on business attributes such as location and type of business are available (but not provided to our algorithm) and we provide the distribution in Figure~\ref{fig:tradeoff} on the left side. There is also a natural trade-off between recovery ratio and average error or between attempting to recover all the business communities and the accuracy of recovery. We can either recover top significant communities with high accuracy or recover more with lower accuracy. We demonstrate the trade-off in Figure~\ref{fig:tradeoff} on the right side.

We select the top ten categories recovered with the lowest error and report the business with  highest weights in $\widehat{\Pi}$. Among the matched communities, we find the business with the highest membership weight (Table~\ref{tab:topbusinesses}). We can see that most of the ``top'' recovered businesses are  rated   high. %The businesses with highest weights are typically the popular ones: notice that they have higher stars than the average stars for that category.
Many of the categories in the top ten list are restaurants as they have a large number of reviewers. Our method can recover restaurant category with high accuracy, and the specific  restaurant in the category is a popular result (with high number of stars). Also, our method can also recover many of the categories with low review counts accurately like hobby shops, yoga, churches, galleries and religious organizations which are the ``niche'' categories with a dedicated set of reviewers, who mostly do not review other categories. %Hence, it is possible to recover niche categories even when they have an extremely small number of reviews.

\begin{table}[htbp]%{r}{4.6in }%[h]
  \scriptsize
  \centering
   \begin{tabular}{@{} l|l|l|c|c|c @{}}
%\begin{tabularx}{\textwidth}{llllllll}
\hline
%      \toprule
      Category & Business & Star(B) & Star(C) &
      RC(B) & RC(C)\\
\hline
%      \midrule
      Latin American & Salvadoreno & $4.0$ & $3.94$	
      &$36$ 	&$93.8$\\
      Gluten Free & P.F. Chang's & $3.5$ & $3.72$		
      &$55$	&$50.6$\\
      Hobby Shops & Make Meaning & $4.5$ & $4.13$			
      &$14$	&$7.6$\\
      Mass Media & KJZZ $91.5$FM & $4.0$ & $3.63$			
      &$13$	&$5.6$\\
      Yoga & Sutra Midtown & $4.5$ & $4.55$					
      &$31$	&$12.6$\\
      Churches & St Andrew Church & $4.5$ & $4.52$
      &$3$	&$4.2$\\
      Art Galleries &Sette Lisa &$4.5$ & $4.48$         		
      &$4$		&$6.6$\\
      Libraries & Cholla Branch & $4.0$ & $4.00$			
      &$5$		&$11.2$\\
      Religious & St Andrew Church & $4.5$ &$4.40$	
      &$3$ &$4.2$\\
      Wickenburg & Taste of Caribbean & $4.0$ & $3.66$
      &$60$	& $6.7$\\
\hline
%      \bottomrule
   \end{tabular}
   \caption[Membership recovery in Yelp review data]{Most accurately recovered categories and businesses with highest membership weights for the Yelp dataset. ``Star(B)'' denotes the review stars that the business receive and ``Star(C)'', the average review stars that businesses in that category receive.  ``RC(B)'' denotes the review counts for that business and ``RC(C)'' , the average review counts in that category. }
   %Average stars denote the average review stars of estimated members in that category. }
   \label{tab:topbusinesses}
\end{table}

Our algorithm can also recover the attributes of users. However, the ground truth available about users
is far more limited than businesses, and we only have information on gender, average review counts and average stars (we infer the gender of the users through their  names). Our algorithm can  recover all these attributes. We observe that gender is the hardest to recover while review counts is the easiest.
We see that the other user attributes recovered by our algorithm correspond to valuable user information such as their interests, location, age, lifestyle, etc. This is useful, for instance, for businesses studying the characteristics of their users, for delivering better personalized advertisements for users, and so on.

\paragraph{Facebook Dataset: }A snapshot of the Facebook network of UNC~\cite{facebook} is provided with user attributes.
The ground truth communities are based on user attributes given in the dataset which are not exposed to the algorithm. There are $360$ top communities with sufficient (at least 20) users. Our algorithm can recover these attributes with high accuracy compared with variational inference result~\cite{gopalan2012scalable}. %We run both our tensor method and variational method on the same machine for fairness. The performance of our algorithm is summarized in Table~\ref{table:businessresults}.

%Note that the results reported are achieved by setting the Dirichlet concentration parameter $\alpha_0=1$, we therefore analyze the robustness of the results with $\alpha_0$. We both evaluate the recovery ratio and error function under different $\alpha_0$'s ($\alpha_0=[0.1:0.1:0.9]$ to be precise).
We also obtain results for a range of values of $\alpha_0$ (Figure~\ref{Fig:alpha0s}). We observe that the recovery ratio improves with larger $\alpha_0$ since a larger $\alpha_0$ can recover overlapping communities more efficiently  while the error score remains relatively the same.

%This is reasonable since smaller $\alpha_0$ results in small entries in most of the entries of $\widehat{\Pi}$. Thus thresholding the small entries will set them to zero which results in almost same errors.

\begin{figure*}[h]
%\subfloat[a][\small{Recovery ratio}]
%{\begin{minipage}{3.4in}\centering\label{Fig:alpha0s_m}
\small \centering
\psfrag{Recovery ratio vs threshold under alpha0s}[l]{\tiny{}}
\psfrag{Threshold}[c]{\small{Threshold}}
\psfrag{Rec. ratio}[c]{\small{Recovery ratio}}
\psfrag{alpha0.1}[l]{\small{$\alpha_0$:0.1}}
\psfrag{alpha0.2}[l]{\small{$\alpha_0$:0.2}}
\psfrag{alpha0.3}[l]{\small{$\alpha_0$:0.3}}
\psfrag{alpha0.4}[l]{\small{$\alpha_0$:0.4}}
\psfrag{alpha0.5}[l]{\small{$\alpha_0$:0.5}}
\psfrag{alpha0.6}[l]{\small{$\alpha_0$:0.6}}
\psfrag{alpha0.7}[l]{\small{$\alpha_0$:0.7}}
\psfrag{alpha0.8}[l]{\small{$\alpha_0$:0.8}}
\psfrag{alpha0.9}[l]{\small{$\alpha_0$:0.9}}
\includegraphics[width=0.49\textwidth]{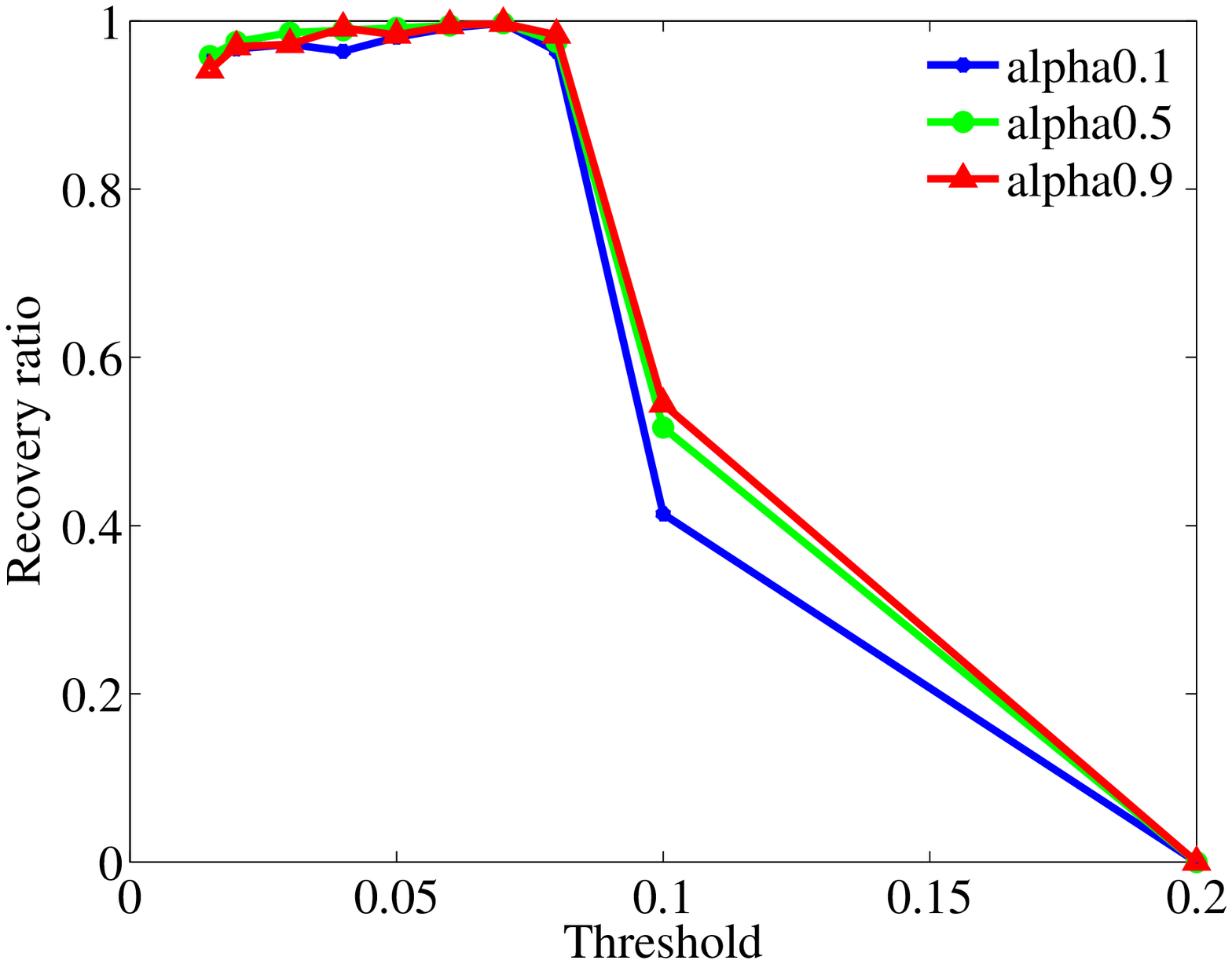}
%\end{minipage}}
%\hfil
%\subfloat[a][\small{Error function}]
%{\begin{minipage}{3.4in}\centering\label{Fig:alpha0s_e}
\small \centering
\psfrag{Error vs threshold under alpha0s}[l]{\tiny{}}
\psfrag{Threshold}[c]{\small{Threshold}}
\psfrag{Error}[c]{\small{Error}}
\psfrag{alpha0.1}[l]{\small{$\alpha_0$:0.1}}
\psfrag{alpha0.2}[l]{\small{$\alpha_0$:0.2}}
\psfrag{alpha0.3}[l]{\small{$\alpha_0$:0.3}}
\psfrag{alpha0.4}[l]{\small{$\alpha_0$:0.4}}
\psfrag{alpha0.5}[l]{\small{$\alpha_0$:0.5}}
\psfrag{alpha0.6}[l]{\small{$\alpha_0$:0.6}}
\psfrag{alpha0.7}[l]{\small{$\alpha_0$:0.7}}
\psfrag{alpha0.8}[l]{\small{$\alpha_0$:0.8}}
\psfrag{alpha0.9}[l]{\small{$\alpha_0$:0.9}}
\includegraphics[width=0.49\textwidth]{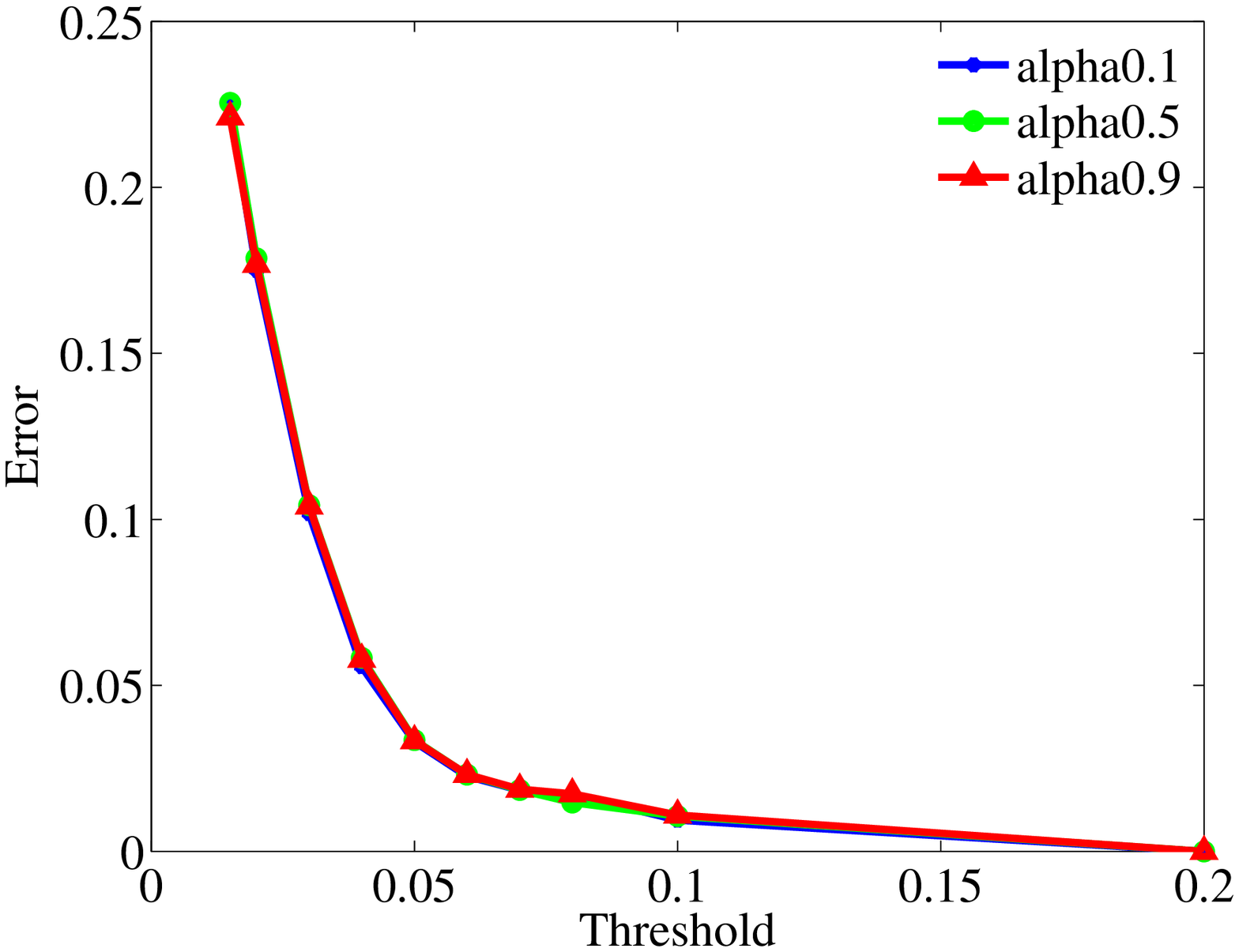}
%\end{minipage}}
\caption[Facebook result tunning]{Performance analysis of Facebook dataset under different settings of the concentration parameter ($\alpha_0$) for $\hat{k}=100$. %Recall that $\alpha_0$ is the Dirichlet concentration parameter which influences the extent of overlap of community memberships.
%\fhcomment{remind what is $\alpha_0$}
}
\label{Fig:alpha0s}
\end{figure*}

For the Facebook dataset, the top ten communities recovered with  lowest error consist of certain high schools, second majors and dorms/houses.
%\begin{enumerate}
%\item high school$-18573$
%\item major $-145$
%\item second major$-108$
%\item second major$-137$
%\item second major$-244$
%\item high school$-18432$
%\item dorm/house$-253$
%\item dorm/house$-256$
%\item high school$-18408$
%\item high school$-18460$
%\end{enumerate}
We observe that high school attributes are easiest to recover and second major and dorm/house are reasonably easy to recover by looking  at the friendship relations in Facebook.  This is reasonable:  college students from the same high school have a high probability of being friends; so do colleges students from the same dorm. %For more results, refer to Appendix~\ref{sec:apdx_fb_bridge}.

\paragraph{DBLP Dataset: }
%\fhcomment{refer to performance evaluation table (table 4 in main)}

The DBLP data contains bibliographic records\footnote{\url{http://dblp.uni-trier.de/xml/Dblp.xml} } with various publication venues, such as journals and conferences, which we model as communities. We then consider authors who have published at least one paper in a community (publication venue) as a member of it. Co-authorship is thus modeled as link in the graph in which authors are represented as nodes.  In this framework, we could recover the top authors in communities and bridging authors.

\section{Conclusion}
In this chapter, we presented a fast and unified moment-based framework for learning overlapping communities as well as topics in a corpus. There are several key insights involved. Firstly, our approach follows from a systematic and guaranteed learning procedure in contrast to several heuristic approaches which may not have strong statistical recovery guarantees. Secondly, though using a moment-based formulation may seem computationally expensive at first sight, implementing implicit ``tensor'' operations leads to significant speed-ups of the algorithm. Thirdly, employing randomized methods for spectral methods is promising in the computational domain, since the running time can then be significantly reduced.

This work paves the way for several interesting directions for further research. While our current deployment incorporates community detection in a single graph, extensions to multi-graphs and hypergraphs are possible in principle. A careful and efficient implementation for such settings will be useful in a number of applications.  It is natural to extend the deployment to even larger datasets by having cloud-based systems. The issue of efficient partitioning of data and reducing communication between the machines becomes significant there. Combining our approach with other simple community detection approaches to gain even more speedups can be explored.

\chapter{Dictionary Learning through Convolutional Tensor Decomposition}\label{chapter:convolutional}
In this chapter, we extend tensor decomposition framework to models with invariances, such as convolutional dictionary models.  Learning invariant dictionary elements is crucial to remove unnecessary model redundancy in a lot of settings. For instance, in image filter bank learning where image filters' activation locations in the image are ignored, in natural language process where the phrase templates are not distinguished by their location in the  sentence, and in neural  science where neural spikes consist of template spikes activated at different time.

We propose a tensor decomposition algorithm to solve this problem of learning shift invariant dictionary elements. 
Our tensor decomposition algorithm is based on the popular alternating least squares (ALS) method, but with additional shift invariance constraints on the factors. We demonstrate that each ALS update can be computed efficiently using simple operations such as fast Fourier transforms and matrix multiplications. Our algorithm converges to models with better reconstruction error and is much faster, compared to the popular alternating minimization heuristic, where the filters  and activation maps are alternately updated.

%Thus, solving \eqref{eqn:alt-min} is fundamentally   {\em ill-posed} and has a large number of equivalent solutions.  On the other hand, imposing shift-invariance constraints directly on the objective function in \eqref{eqn:alt-min} results in non-smooth optimization which is challenging to solve.

We propose a novel framework for  learning convolutional models through tensor decomposition.  We consider inverse method of moments to estimate the model parameters via decomposition of higher order (third or fourth order) moment tensors. When the inputs $x$ are generated from a convolutional model in \eqref{eqn:sparsedef}, with independent activation maps $w_i^*$, i.e. a convolutional ICA model, we show that the cumulant tensors have a CP decomposition, whose components correspond to filters and their {\em circulant} shifts.  We propose a novel method for tensor decomposition when such circulant   constraints  are imposed on the components of the tensor decomposition.

Our tensor decomposition method is a constrained form of the popular alternating least squares (ALS) method\footnote{The ALS method for tensor decomposition is not to be confused with the alternating minimization method for solving \eqref{eqn:alt-min}. While \eqref{eqn:alt-min} acts on data samples and alternates between updating filters and activation maps, tensor ALS operates on averaged moment tensors and alternates between different modes of the tensor decomposition.}. We show that the resulting optimization problem in each tensor ALS iteration can be solved in closed form, and uses simple operations such as Fast Fourier transforms (FFT) and matrix multiplications. These operations have a high degree of parallelism: for estimating $L$ filters, each of length $n$, we require $O(\log n +\log L)$ time and $O(L^2n^3)$ processors. Note that there is {\em no} dependence on the number of data samples $N$, since the empirical moment tensor can be computed in one data pass, and the ALS iterations only updates the filters. This is a huge saving in running time,  compared to  the alternate minimization method which requires a pass over data in each step to decode all the activation maps $w_i$. The  running time of alternating minimization is   $O(\max(\log n\log L, \log n \log N))$  per iteration with $O(\max(\frac{nNL}{\log N}, \frac{nNL}{\log L}))$ processors,  and when $N\gg Ln^2$, which is the typical scenario, our method is hugely advantageous.  Our method avoids decoding the activation maps in each iteration since they   are averaged out in the input moment tensor, on which the ALS method operates and  we only estimate the  filters $f_i$ in the learning step. In other words, the activation maps $w_i$'s are easily estimated using \eqref{eqn:alt-min} in one data pass after filter estimation. Thus, our method is highly parallel and scalable to huge datasets.

%Moreover, our method requires only one pass over data to compute the higher order cumulant of  the input data or its approximation through sketching algorithms.

We carefully optimize computation and memory costs by exploiting tensor algebra and circulant structure, due to the shift invariance of the convolutional model. We implicitly carry out many of the operations and do not form large (circulant) matrices and minimize storage requirements.  Preliminary experiments further demonstrate superiority of our method compared to alternating minimization. Our algorithm converges accurately and much faster to the true underlying filters compared to alternating minimization. Moreover, it results in much lower reconstruction error, while alternating minimization tends to get stuck in spurious local optima. Our algorithm is also orders of magnitude faster than the alternating minimization. 

\section{Model and Formulation}\label{sec:model}
\paragraph{Notation } Let $[n]:=\{1,2,\ldots,n\}$. For a vector $v$,   denote the $i^{\tha}$ element as $v(i)$. For a matrix $M$,   denote the $i^{\tha}$  row as $M^i$ and   $j^{\tha}$  column as $M_j$.  For a tensor   $T\in\mathbb{R}^{n\times n\times n}$, its $(i_1,i_2,i_3)^{\tha}$ entry is denoted by $[T]_{i_1,i_2,i_3}$.
A \emph{column-stacked} matrix $M$ consisting of $M_i'$s  (with same number of rows) is $M := [M_1,M_2,\ldots,M_L]$. Similarly, a \emph{row-stacked} matrix $M$ from $M_i'$s (with same number of columns)  is $M:= [M_1;M_2;\ldots;M_L]$.

 \paragraph{Cyclic Convolution } The 1-dimensional (1-D) $n$-cyclic convolution $f\Conv w$ between vectors $f$ and $w$ is defined as $ \label{eqn:cyclic}v=f\Conv_n w, \ v(i) =\sum_{j\in [n]} f(j) w((i-j+1)\mod n).$ Note that the linear convolution is the combination without the modulo operation (i.e. cyclic shifts) above.  $n$-Cyclic convolution is equivalent to linear convolution, when $n$ is at least twice the support length of both $f$ and $w$~\cite{oppenheim1997signals}, which will be assumed. We drop the notation $n$ in $\Conv$ for convenience. Cyclic convolution in \eqref{eqn:cyclic} is equivalent to $f\Conv w = \Toep(f)\cdot w, $
and \begin{equation}\label{eq:circulant}
\Toep(f) : = \sum_{p}f(p) G_p\in \Rbb^{n\times n} ,\quad  \left(G_p\right)^i_j : = \delta\left\{\left((i-j)\mod n\right)=p-1\right\}, \quad \forall p\in[n].
\end{equation} 
defines a  circulant matrix.   A circulant matrix $\Toep(f)$   is characterized by the vector $f$, and each column corresponds to a cyclic shift of $f$.

\paragraph{Properties of circulant matrices } Let $F$ be  the discrete Fourier transform matrix  whose $(m,k)$-th entry is $F_{k}^m = \omega_n^{(m-1)(k-1)}$, $\forall m,k\in[n]$ where $\omega_n = \exp(-\frac{2\pi i}{n})$.
If $U : = \sqrt{n} F^{-1}$, $U$ is the set of eigenvectors for all $n \times n$  circulant matrices~\cite{gray2005toeplitz}. Let the  Discrete Fourier Transform  of a vector $f$ be $\DFT(f)$,  we express  the circulant matrix $\Toep(f)$ as
 \begin{equation}\label{eq:circulant_fft}
 \Toep(f) = U \Diag(F\cdot f) U\hermconj = U \Diag(\DFT(f)) U\hermconj.
 \end{equation}   This is an important property we use in algorithm optimization to improve computational efficiency.

 \paragraph{Column stacked circulant matrices } We will extensively use column stacked circulant matrices $\mathcal{F}:=[\Toep(f_1),\ldots, \Toep(f_L)]$, where $\Toep(f_j)$ is the circulant matrix corresponding to filter $f_j$.

\subsection{Convolutional Dictionary Learning/ICA Model}\label{subsec:ConvolutionalICAmodel}
We assume that the  input $x\in \mathbb{R}^n$ is generated as
\begin{equation}
x = \sum_{j\in [L]}   f_j^*\Conv w_j^*
	   = \sum_{j\in [L]} \Toep(f_j^*) w_j^*
	   = \CToep^*\cdot w^* ,\label{eqn:gen}
\end{equation}
where $\CToep^*:=[\Toep(f_1^*), \Toep(f_2^*),\ldots,\Toep(f_L^*)]$ is the concatenation or column stacked version  of circulant matrices and $w^*$ is the  \emph{row-stacked} vector $w^* : =[w_1^*;w_2^*;\ldots w_L^*]\in \Rbb^{nL}$.  Recall that $\Toep(f_l^*)$ is circulant matrix corresponding to filter $f_l^*$, as given by \eqref{eq:circulant_fft}. Note that although $\mathcal{F}^*$ is a $n$ by $nL$  matrix, there are only $nL$ free parameters. We never explicitly form the estimates $\mathcal{F}$ of $\mathcal{F}^*$, but instead use filter estimates  $f_l$'s to characterize $\mathcal{F}$.
In addition, we can handle additive Gaussian noise in \eqref{eqn:gen}, but do not incorporate it for simplicity.
\paragraph{Activation Maps:}For each observed sample $x$, the activation map $w_i^*$ in \eqref{eqn:gen} indicates the locations where each filter $f_i^*$ is active and $w^*$ is the  \emph{row-stacked} vector $w^* : =[w_1^*;w_2^*;\ldots w_L^*]$. We assume that the coordinates of $w^*$ are   drawn from some product distribution, i.e. different entries are independent of one another and we have the independent component analysis (ICA) model in \eqref{eqn:gen}. When the distribution encourages sparsity, e.g. Bernoulli-Gaussian, only a small subset of locations are active, and we have the {\em sparse coding} model in that case. We can also extend to dependent distributions such as Dirichlet for $w^*$, along the lines of~\cite{blei2003latent}, but limit ourselves to ICA model for simplicity.
\paragraph{Learning Problem:}Given access to $N$ i.i.d. samples,  $X:=[x^1,x^2,\ldots,x^N]\in \Rbb^{n\times N}$, generated according to the above model, we aim to estimate the true filters $f_i^*$, for $i \in [L]$. Once the filters are estimated, we can use standard decoding techniques, such as the square loss criterion in \eqref{eqn:alt-min} to learn the activation maps for the individual maps. We focus on developing a novel method for filter estimation in this chapter.

\section{Form of  Cumulant Moment Tensors}\label{sec:CumForm}

\paragraph{Tensor Preliminaries} We consider 3rd order tensors in this chapter but the analysis is easily extended to higher order tensors. For tensor $T\in\mathbb{R}^{n\times n\times n}$, its $(i_1,i_2,i_3)^\tha$ entry is denoted by $[T]_{i_1,i_2,i_3}, \forall i_1\in[n], i_2\in[n], i_3\in[n]$. 
A flattening or unfolding of tensor $T\in \Rbb$ is the column-stacked matrix of all its slices, given by $ \flatten(T):=[[T]_{:,:,1},[T]_{:,:,2},\ldots,[T]_{:,:,n}]\in\mathbb{R}^{n\times n^2}$.  Define the Khatri-Rao product for vectors $u\in\mathbb{R}^{a}$ and $v\in\mathbb{R}^b$ as a \emph{row-stacked} vector $[u\odot v] := [u(1)v; u(2) v; \ldots; u(a)v] \in \mathbb{R}^{ab}$. Khatri-Rao product is also defined for matrices with same columns. For $M \in \mathbb{R}^{a\times c}$ and $M^\prime\in \mathbb{R}^{b\times c}$,  $M \odot M^\prime : = [M_1\odot M_1^\prime, \ldots, M_c\odot M_c^\prime, ]\in \mathbb{R}^{ab\times c}$, where $M_i$ denotes the $i^{\tha}$ column of $M$.
\paragraph{Cumulant}The third order cumulant of a multivariate distribution   is a third order tensor, which uses (raw) moments up to third order. Let $\Cum\in \Rbb^{n \times n^2}$ denote the unfolded version of third order cumulant tensor, it is given by
\begin{equation}
\Cum : = \Ebb[x(x \odot x )^\top] - \flatten(Z) \label{eq:thirdCum_unfolding}
\end{equation}
where   $ [Z]_{a,b,c} := \Ebb[x_{a}]\Ebb[x_{b}x_{c}] +  \Ebb[x_{b}]\Ebb[x_{a}x_{c}] +  \Ebb[x_{c}]\Ebb[x_{a}x_{b}] -2 \Ebb[x_{a}]\Ebb[x_{b}]\Ebb[x_{c}], \ \forall a,b,c \in [n]. $

Under the convolution ICA model in Section~\ref{subsec:ConvolutionalICAmodel}, we show that the third order cumulant has a nice tensor form, as given below.

\begin{lemma}[Form of Cumulants]\label{lm:3orderMom}
The unfolded third order cumulant $\Cum$ in \eqref{eq:thirdCum_unfolding} has the following decomposition form
\begin{equation}\label{eqn:cumform}
\Cum = \sum_{j\in [nL]}\lambda_j^* \CToep^*_j (\CToep^*_j \odot\CToep^*_j)^\top={\mathcal{F}^*} \Lambda^* \left({\mathcal{F}^*}\odot {\mathcal{F}^*}\right)^\top, \quad \mbox{where }\Lambda^* := \Diag(\lambda_1^*,\lambda_2^*,\ldots,\lambda_{nL}^*)
\end{equation}
where $\CToep^*_j$ denotes the $j^{\tha}$ column of the
 \emph{column-stacked} circulant matrix  $\CToep^*$  and $\lambda^*_j$ is the third order cumulant corresponding to the (univariate) distribution of $w^*(j)$.
\end{lemma}

For example, if the $l^{\tha}$ activation is drawn  from a   Poisson distribution  with mean $\tilde{\lambda}$, we have that  $\lambda_l^*=\tilde{\lambda}$.
Note that if the third order cumulants  of the activations, i.e.  $\lambda_j^*$'s, are zero, we need to consider higher order cumulants. This holds for zero-mean activations and we need to use fourth order cumulant instead.  Our method extends in a straightforward manner for higher order cumulants.

%%%%%%%%%%%%%%%%%%%%%%%%%%%%%%%                                                
\begin{figure}[htb]
\bc
{
\begin{minipage}{0.45\textwidth}
\tikzstyle{matrx}=[rectangle,
                                    thick,
                                    minimum size=1.0cm,
                                    draw=gray!80,
                                    fill=gray!10]
   \tikzstyle{matrx_begin}=[rectangle,
                                    thick,
                                    minimum size=0.8cm]                                 
\tikzstyle{backgroun}=[rectangle,
                                                fill=gray!10,
                                                inner sep=0cm]
                                                %rounded corners=5mm]

\begin{tikzpicture}[>=latex,text height=1.0ex,text depth=0.1ex]
  \matrix[row sep=0cm,column sep=0.0001cm] {
        % Second line: System noise & input matrix
        \node(F)[matrx_begin]{{ $\CCir =$}};&
        \node (F1) [matrx] {{ $\block_1(\mathcal{F})$}};       &
        \node (F3)   [matrx] {{ $\quad \ldots\quad$}};       &
        \node (FL) [matrx] {{ $\block_L(\mathcal{F})$}};   \\    
            };
\end{tikzpicture}
\end{minipage}
}
\hfil
%\subfloat[b][]
{\begin{minipage}{0.45\textwidth}
\tikzstyle{matrx}=[rectangle,
                                    thick,
                                    minimum size=0.8cm,
                                    draw=gray!80,
                                    fill=gray!10]
   \tikzstyle{matrx_begin}=[rectangle,
                                    thick,
                                    minimum size=0.8cm]                                 
% Everything is drawn on underlying gray rectangles with
% rounded corners.
\tikzstyle{backgroun}=[rectangle,
                                                fill=gray!10,
                                                inner sep=0cm]
                                                %rounded corners=5mm]
\begin{tikzpicture}[>=latex]%,text height=1ex,text depth=0.25ex]
    % "text height" and "text depth" are required to vertically
    % align the labels with and without indices.
  \matrix[row sep=0.0000cm,column sep=0.000cm] {
        & \node (F11) [matrx] {{$\ \block_1^1(\mathbf{\Psi})$}};       &
        \node (F12)   [matrx] {{$\quad\  \ldots\quad$}};        &
        \node (F13) [matrx] {{$\block_L^1(\mathbf{\Psi})$}};       
        \\
        \node(F)[matrx_begin]{{$\mathbf{\Psi} =$}}; & 
         \node (F21)   [matrx] {{$ \ \quad\  \ldots\quad$}};       &
          \node (F22)   [matrx] {{$\quad  \ \ldots\quad$}};       &
           \node (F23)   [matrx] {{$\ \quad  \ldots\ \quad\!$}};        &
           \\
           & \node (F31) [matrx] {{$\block_1^L(\mathbf{\Psi})\ $}};       &
        \node (F32)   [matrx] {{$\quad\  \ldots\quad$}};       &
        \node (F33) [matrx] {{$\block_L^L(\mathbf{\Psi})$}};   
        \\
};
%    \begin{pgfonlayer}{background}
%        \node [backgroun,
%                    fit=(F11) (F12)(F13)(F21)(F22)(F23)(F31)(F32)(F33),
%                    label=left:0] (myblock){};
%    \end{pgfonlayer}
\end{tikzpicture}
\end{minipage}}
\ec
\caption[Block structure]{(a) Blocks of the column-stacked circulant matrix $ \CCir$. (b) Blocks of the row-and-column-stacked diagonal matrices $ \mathbf{\Psi}$. $\block_j^i(\mathbf{\Psi})$ is diagonal.}
\end{figure}
The decomposition form in \eqref{eqn:cumform} is known as the CANDECOMP/PARAFAC  (CP) decomposition form~\cite{anandkumar2014tensor} (the usual form has the decomposition of the tensor and not its unfolding, as above). We now attempt to recover the unknown filters $f_i^*$ through decomposition of the  third order cumulants $ \Cum$. This is formally stated below.

\paragraph{Objective Function: }
Our goal is to obtain filter estimates $f_i$'s which minimize   the Frobenius norm $\|\cdot\|_{\Fbb}$ of reconstruction of the   cumulant tensor $ \Cum$,
\begin{align}\nn&\min\limits_{\mathcal{F}}\quad \lVert \Cum - {\mathcal{F}} \Lambda \left({\mathcal{F}}\odot {\mathcal{F}}\right)^\top   \rVert^2_{F},
 \\ &\mbox{s.t. }   \block_l(\mathcal{F}) = U \Diag(\DFT(f_{l})) U\hermconj, \ \lVert f_l \rVert_2=1,\quad  \forall  l\in[L], \quad \Lambda=\Diag(\lambda).\label{eq:condition}
\end{align}
where $\block_l(\mathcal{F})$ denotes the $l^{\tha}$ circulant matrix   in $\CCir$. The conditions in \eqref{eq:condition} enforce $\block_l(\mathcal{F})$ to be circulant  and for the filters to be normalized. Recall that   $U$ denotes   the eigenvectors for circulant matrices. The rest of the chapter is devoted to devising efficient methods to solve \eqref{eq:condition}.

Throughout the chapter, we will use $\mathcal{F}_j$ to denote the $j^{\tha}$ column of $\mathcal{F}$, and $\block_l(\mathcal{F})$ to denote the $l^{\tha}$ circulant matrix block in $\mathcal{F}$. 
 Note that $\mathcal{F}\in \Rbb^{n\times nL}$, $\mathcal{F}_{j}\in \Rbb^{n}$ and $\block_l(\mathcal{F})\in\Rbb^{n\times n}$.
\section{Alternating Least Squares for Convolutional Tensor Decomposition}
To solve the non-convex optimization problem in \eqref{eq:condition}, we consider the alternating least squares (ALS) method with \emph{column stacked} circulant constraint. We first consider  the asymmetric relaxation of \eqref{eq:condition} and introduce separate variables $\modeA, \modeB$ and $\modeC$ for  filter estimates   along each of the modes to fit the third order cumulant  tensor $\Cum$.   We then perform alternating updates by fixing two of the modes and updating the third one.   
\begin{equation}\label{eqn:modeAopt}
\min\limits_{ {\modeA}} \quad  \lVert \Cum-  {\modeA} \Lambda \left({\modeC}\odot {\modeB}\right)^\top  \rVert^2_{F}\,\,
 \mbox{s.t. }   \block_l(\modeA)
= U \cdot\Diag(\DFT(f_{l}))\cdot U\hermconj, \ \lVert f_{l} \rVert_2^2=1,  \forall  l\in[L]
\end{equation} Similarly, $\modeB$ and $\modeC$ have the same column-stacked circulant matrix constraint and  are updated similarly in alternating steps. The diagonal matrix $\Lambda$ is updated through normalization.

We now introduce the \emph{Convolutional Tensor } ($\ouralgorithm$) Decomposition algorithm to efficiently solve~\eqref{eqn:modeAopt} in closed form, using simple operations such as matrix multiplications and fast Fourier Transform (FFT). We do not form matrices $\modeA, \modeB$ and $\modeC\in \Rbb^{n \times nL}$, which are large, but only update them using filter estimates $f_1, \ldots, f_L, g_1, \ldots, g_L, h_1, \ldots h_L$. Denote
\begin{equation}\label{eqn:newCum}
\newCum : =  \Cum (({\modeC}\odot {\modeB})^\top)^\dag,
\end{equation}  where $\dag$ denotes
pseudoinverse. Let $\block_l(\newCum)$ and $\block_l(\Lambda)$   denote the $l^{\tha}$  blocks of $\newCum$ and $\Lambda$. We have a closed form solution for filter update, once we have computed $M$, and we present the main result as follows.

\begin{theorem} \label{theorem:main}[Closed form updates] The optimal solution $f_l^{\opt}$ for \eqref{eqn:newCumopt} is given by
\begin{align}\label{eq:circulantProjection}
f_l^{\opt} (p) &= \frac{\sum\limits_{i,j\in[n]} \| \block_l(\newCum)_{j}\|^{-1} \cdot  \block_l(\newCum)_{j}^i\cdot I_{p-1}^{q }}{\sum\limits_{i,j\in [n]}   I_{p-1}^{q} }, &\forall p\in[n], q:=(i-j)\mod n.
\end{align} Further $\Lambda=\Diag(\lambda)$ is updated as
$\lambda(i) = \|\newCum_i\|$, for all  $i\in[nL]$. Note that $I_{p-1}^{q }$ denotes the $(q, (p-1))^\tha$ element of the identity matrix.
\end{theorem}
\paragraph{Proof Sketch:} Using the property of least squares,  the optimization problem in \eqref{eqn:modeAopt}
is equivalent to \begin{equation}\label{eqn:ls}
\min\limits_{ {\modeA}}    \lVert \Cum (\left({\modeC}\odot {\modeB}\right)^\top)^\dag \Lambda^\dag-  {\modeA}    \rVert^2_{F}\,\,
 \mbox{s.t. }   \block_l(\modeA)
= U \cdot\Diag(\DFT(f_{l}))\cdot U\hermconj, \ \lVert f_{l} \rVert_2^2=1,  \forall  l\in[L]\eeq
when  $(\modeC\odot \modeB)$ and $\Lambda$ are full column rank.
The full rank condition requires $nL<n^2$ or $L <n$,  and it is a reasonable assumption since otherwise the filter estimates are redundant. In practice, we can additionally regularize the update to ensure full rank condition is met. Since \eqref{eqn:ls} has block constraints, it can be broken down in to solving $L$ independent sub-problems
\begin{equation}\label{eqn:newCumopt}
  \min_{f_{l}}
\left\lVert
 \block_l(\newCum)\cdot \block_l(\Lambda)^\dag - U \cdot\Diag(\DFT(f_l))\cdot U\hermconj
  \right\rVert^2_{F} \\
\quad  s.t. \quad \lVert f_l \rVert_2^2=1, \forall l\in[L]
\end{equation} Our proof for the closed form solution is similar to the analysis in ~\cite{eberle2003finding}, where they proposed a closed form solution for finding  the closest circulant/toeplitz matrix. For a detailed proof of Theorem~\ref{theorem:main}, see Appendix~\ref{appdx:maintheorem}.\qed

Thus, the reformulated problem in \eqref{eqn:newCumopt} can be solved in closed form efficiently. A bulk of the computational effort will go into computing $\newCum$ in \eqref{eqn:newCum}.
Computation of $ \newCum$ requires $2L$ fast Fourier Transforms of length $n$ filters and simple matrix multiplications without explicitly forming ${\modeB}$ or ${\modeC}$. We make this concrete in the next section. The closed form update after getting $\newCum$ is highly parallel. With $O(n^2L/\log n)$ processors, it takes $O(\log n)$ time. 
\section{Algorithm Optimization to Reduce Memory and Computational Costs}
We now focus on estimating    $\,\newCum:=  \Cum (({\modeC}\odot {\modeB})^\top)^\dag$  in \eqref{eqn:newCum}. If done naively, this requires inverting $n^2 \times nL$ matrix and multiplication of $n \times n^2$ and $n^2 \times nL$ matrices with $O(n^6)$ time. However, forming and computing with these matrices is very expensive when $n$ (and $L$) are large. Instead, we utilize the properties of circulant matrices and the Khatri-Rao product $\odot$ to efficiently carry out these computations implicitly. We present our final result on computational complexity of the proposed method. Recall that $n$ is the filter size and $L$ is the number of filters.
\begin{lemma}\label{lm:complexity}[\textbf{Computational Complexity}]\label{lm:computationalComplexity}
%The computational complexity for each iteration is $O(L^3n + L^2n^3 + Ln\log n )$.
With multi-threading,  the running time of our algorithm for $n$ dimensional input and $L$ number of filters is $O(\log n +\log L)$ per iteration using $O(L^2n^3)$ processors.
\end{lemma}
{Note that before the iterative updates, we   compute the third order cumulant\footnote{Instead of computing the cumulant tensor $\Cum$, a randomized sketch can be computed efficiently, following the recent work of~\cite{wang2015fast}, and the ALS updates can be performed efficiently without forming the cumulant tensor $\Cum$.} $\Cum$ once whose computational complexity is $O(\log N)$ with $\frac{N}{\log N}$ processors, where $N$ is the number of samples. However, this operation is not iterative. In contrast, alternating minimization (AM) requires pass over all the data samples in each iteration, while our algorithm requires only one pass of the data. 

The parallel computational complexity of AM is as follows.  In each iteration of AM, computing the derivative with respect to either filters or activation maps requires $NL$ number of FFTs (requires $O(NLn\log n)$ serial time), and the degrees of parallelism are $O(Nn\log L)$ and $O(Nn\log n)$ respectively. Therefore with multi-threading, the running time of AM is $ O(\max(\log n\log L, \log n \log N))$ per iteration using $O(\max(\frac{nNL}{\log N}, \frac{nNL}{\log L}))$ processors. Comparing with Lemma~\ref{lm:complexity}, we find that our algorithm is advantageous in the regime of $N \ge Ln^2$, which is the typical regime in applications.   }

Let us describe how we utilize various algebraic structures to obtain efficient computation.
\paragraph{Property 1}(Khatri-Rao product): $(({\modeC}\odot {\modeB})^\top)^\dag = ({\modeC}\odot{\modeB}) (({\modeC}^\top{\modeC}) .\star ({\modeB}^\top {\modeB}) )^\dag $, where $.\star$  denotes element-wise product.

\paragraph{Computational Goals: }Find   $(({\modeC}^\top{\modeC}) .\star ({\modeB}^\top {\modeB}) )^\dag$ first and multiply the result with $ \Cum({\modeC}\odot{\modeB})$ to find $\newCum$.

\smallskip
We now describe in detail how to carry out each of these steps.

\subsection{Challenge: Computing $(({\modeC}^\top{\modeC}) .\star ({\modeB}^\top {\modeB}) )^\dag$}

A naive implementation to find the matrix inversion  $(({\modeC}^\top{\modeC}) .\star ({\modeB}^\top {\modeB}))^\dag$ is %$O(n^3L^2)$ which is
very expensive. However, we incorporate the stacked circulant structure of $\modeB$ and $\modeC$ to reduce computation. Note that this is not completely straightforward since although $\modeB$ and $\modeC$ are column stacked circulant matrices, the resulting product whose inverse is required, is {\em not} circulant. Below, we show that however, it is partially circulant along different rows and columns.

\paragraph{ Property 2 } (Block circulant matrix): The matrix $({\modeC}^\top{\modeC}) .\star ({\modeB}^\top {\modeB})$ consists of row and column stacked circulant matrices.

We now make the above property precise by introducing some new notations. Define column stacked identity matrix $\mathbf{I} : = [I,\ldots,I]\in \mathbb{R}^{n\times nL}$, where $I$ is $n \times n $ identity matrix. Let $\bfU:=\blkdiag(U, U, \ldots U) \in \Rbb^{n L \times nL}$ be the block diagonal matrix with $U$ along the diagonal. The first thing to note is that $\modeB$ and $\modeC$, which are column stacked circulant matrices, can be written as \beq\label{eqn:conc} \modeB = \bfI \cdot \bfU \cdot \Diag(v) \cdot\bfU^{\hermconj}, \quad v:=[\DFT(g_1); \DFT(g_2);\ldots; \DFT(g_L)], \eeq where $g_1$, \ldots, $g_L$ are the filters corresponding to $\modeB$, and similarly for $\modeC$, where the diagonal matrix consists of FFT coefficients of the respective filters $h_1, \ldots, h_L$.

By appealing to the above form, we have the following result. We use the notation $\block_j^i(\mathbf{\Psi})$ for a matrix $\mathbf{\Psi}\in \Rbb^{nL\times nL}$ to denote $(i,j)^{\tha}$ block of size $n \times n$.

\begin{lemma}[Form of $({\modeC}^\top{\modeC}) .\star ({\modeB}^\top {\modeB}) $ ]We have
	 \begin{equation}\label{eq:rowcolumnstackdecomp}
	 (({\modeC}^\top{\modeC}) .\star ({\modeB}^\top {\modeB})  )^\dag = \mathbf{U} \cdot\mathbf{\Psi}^\dag\cdot \mathbf{U}\hermconj,
	 \end{equation}
	 where $\mathbf{\Psi}\in \Rbb^{nL \times nL}$ has $L$ by $L$ blocks, each block of size $n \times n$. Its $(j,l)^{\tha}$  block is given by
	 	\begin{equation}
	 	\block_{l}^j(\mathbf{\Psi})  = \Diag(\DFT(\gamma(g_{j},g_l).*\gamma(h_j,h_l)))\in \Rbb^{n \times n}
	 	\end{equation}
		where $\gamma(g_j,g_l):=\mathsf{reverse} (\mathsf{reverse}({g_j}) \Conv g_l)$ and $\gamma(h_j,h_l):=\mathsf{reverse} (\mathsf{reverse}({h_j}) \Conv h_l)$.
\end{lemma}

Therefore, the inversion of $({\modeC}^\top{\modeC}) .\star ({\modeB}^\top {\modeB}) $ can be reduced to the inversion of row-and-column stacked set of  diagonal matrices which form $\mathbf{\Psi}$. Computing $\mathbf{\Psi}$  simply requires  FFT on all $2L$ filters $g_1, \ldots, g_L$ and $h_1,\ldots, h_L$, i.e. $2L$ FFTs, each on length $n$ vector.
We propose an efficient iterative algorithm to compute $\mathbf{\Psi}^\dag$ via block matrix inversion theorem\cite{golub2012matrix} in Appendix~\ref{appdx:parallelBlockInversion}. 
\subsection{Challenge: Computing $\newCum= \Cum({\modeC}\odot{\modeB})\cdot(({\modeC}^\top{\modeC}) .\star ({\modeB}^\top {\modeB}) )^\dag$}

 Now that we have computed $(({\modeC}^\top{\modeC}) .\star ({\modeB}^\top {\modeB})  )^\dag$ efficiently, we need to compute the resulting matrix with $ \Cum({\modeC}\odot{\modeB})$ to obtain $\newCum$. 
We observe that the $m^{\tha}$ row of the result $\newCum$ is given by
\begin{align}\label{eqn:M}
M^m= \sum_{j\in [nL]} {\mathbf{U}}^j \Diag\hermconj\left(z\right) \Phi^{(m)} \Diag \left(v\right) ({\mathbf{U}}^j)\hermconj {\mathbf{U}}^j  \mathbf{\Psi}^\dag {\mathbf{U}}\hermconj, \quad \forall m\in [nL],
\end{align}
where $v:=[\DFT(g_1);\ldots; \DFT(g_L)]$, $z:=[\DFT(h_1);\ldots; \DFT(h_L)]$ are concatenated FFT coefficients of the filters, and
\begin{align} \Phi^{(m)}&:={\mathbf{U}}\hermconj \mathbf{I}^\top  \Gamma^{(m)} \mathbf{I} {\mathbf{U}}, \quad 
[\Gamma^{(m)}]_j^i:= [\Cum]_{i+(j-1)n}^m, \quad \forall i,j,m\in[n] \end{align}
Note that $\Phi^{(m)}$ and $\Gamma^{(m)}$ are fixed for all iterations and need to be computed only once. Note that $\Gamma^{(m)}$ is the result of taking $m^{\tha}$ row of the cumulant unfolding $\Cum$ and matricizing it. Equation~\eqref{eqn:M} uses the property that $\Cum^m(\modeC\odot \modeB)$ is equal to the diagonal elments of $ \modeC^\top \Gamma^{(m)} \modeB$.

We now bound the cost for computing \eqref{eqn:M}. (1) Inverting $ \mathbf{\Psi}$  takes $O(\log L +\log n)$ time with $O(n^2L^2/(\log n + \log L))$ processors %~\cite{pan1987complexity} 
according to appendix~\ref{appdx:parallelBlockInversion}.
(2) Since $\Diag(v)$ and $\Diag(z)$ are diagonal and $\mathbf{\Psi}$ is a matrix with diagonal blocks, the overall matrix multiplication in equation~\eqref{eqn:M} takes $O(L^2n^2)$ time serially with $O(L^2n^2)$ degree of parallelism for each row. Therefore the overall serial computation cost is $O(L^2n^3)$ with $O(L^2n^3)$ degree of parallelism.  With multi-threading, the running time is $O(1)$ per iteration using $O(L^2n^3)$ processes. (3) $\DFT$ requires $O(n\log n)$ serial time, with $O(n)$ degree of parallelism. Therefore computing $2L$ $\DFT$'s takes $O(\log n)$ time with $O(Ln)$ processors. 

Combining the above discussion, it takes $O(\log L+\log n)$ time with $O(L^2n^3)$ processors.

\section{Experiments: Comparison  with Alternating Minimization}
We compare our convolutional tensor decomposition framework with solving equation~\eqref{eqn:alt-min} using alternating (between filters and activation map) minimization method where gradient descent is employed to update $f_i$ and $w_i$ alternatively.  The error comparison between our proposed convolutional tensor algorithm and the alternating minimization algorithm is in figure~\ref{fig:error}. We evaluate the errors for both algorithms by comparing the reconstruction of error and filter recovery error\footnote{Note that  circulant shifts of the filters result in the same reconstruction error, and we report the lowest error between the estimated filters and all circulant shifts of the ground-truth.}. Our algorithm converges much faster to the solution than the alternating minimization algorithm. In fact, alternating minimization  leads to spurious solution where the reconstruction error is significantly larger compared to the error achieved by the tensor method.  The error bump in the reconstruction error curve in figure~\ref{fig:error} for tensor method is due to the random initialization following deflation of one filter, and estimation of the second one.  The running time is also reported in figure~\ref{fig:runtime1} and~\ref{fig:runtime2} between our proposed convolutional tensor algorithm and the alternating minimization. Our algorithm is orders of magnitude faster than the alternating minimization. Both our algorithm and alternating minimization scale linearly with number of filters. However convolutional tensor algorithm is almost constant time  with respect to  the number of samples, whereas the alternating minimization scales linearly. This results in huge savings in running time for large datasets.
\begin{figure}[!htb]
\bc
\subfloat%[a][Error Compare]
{\begin{minipage}{0.32\textwidth}\bc
\psfrag{error}{\scriptsize{error}}
\psfrag{filter 1 2 reconstruction error}{}
\psfrag{iteration}{\scriptsize{iteration}}
\psfrag{tensor method reconstruction error}[Bl]{\scriptsize{Proposed $\mathsf{CT}$: Reconst}}
\psfrag{alternating minimization reconstruction error}[Bl]{\scriptsize{Baseline $\mathsf{AM}$: Reconst}}
\psfrag{tensor method for filter 1}[Bl]{\scriptsize{Proposed $\mathsf{CT}$: $f_1$}}
\psfrag{alternating minimization for filter 1}[Bl]{\scriptsize{Baseline $\mathsf{AM}$: $f_1$}}
\psfrag{tensor method for filter 2}[Bl]{\scriptsize{Proposed $\mathsf{CT}$: $f_2$}}
\psfrag{alternating minimization for filter 2}[Bl]{\scriptsize{Baseline $\mathsf{AM}$: $f_2$}}
\includegraphics[width=\textwidth,height=2in]{\fighomeConv/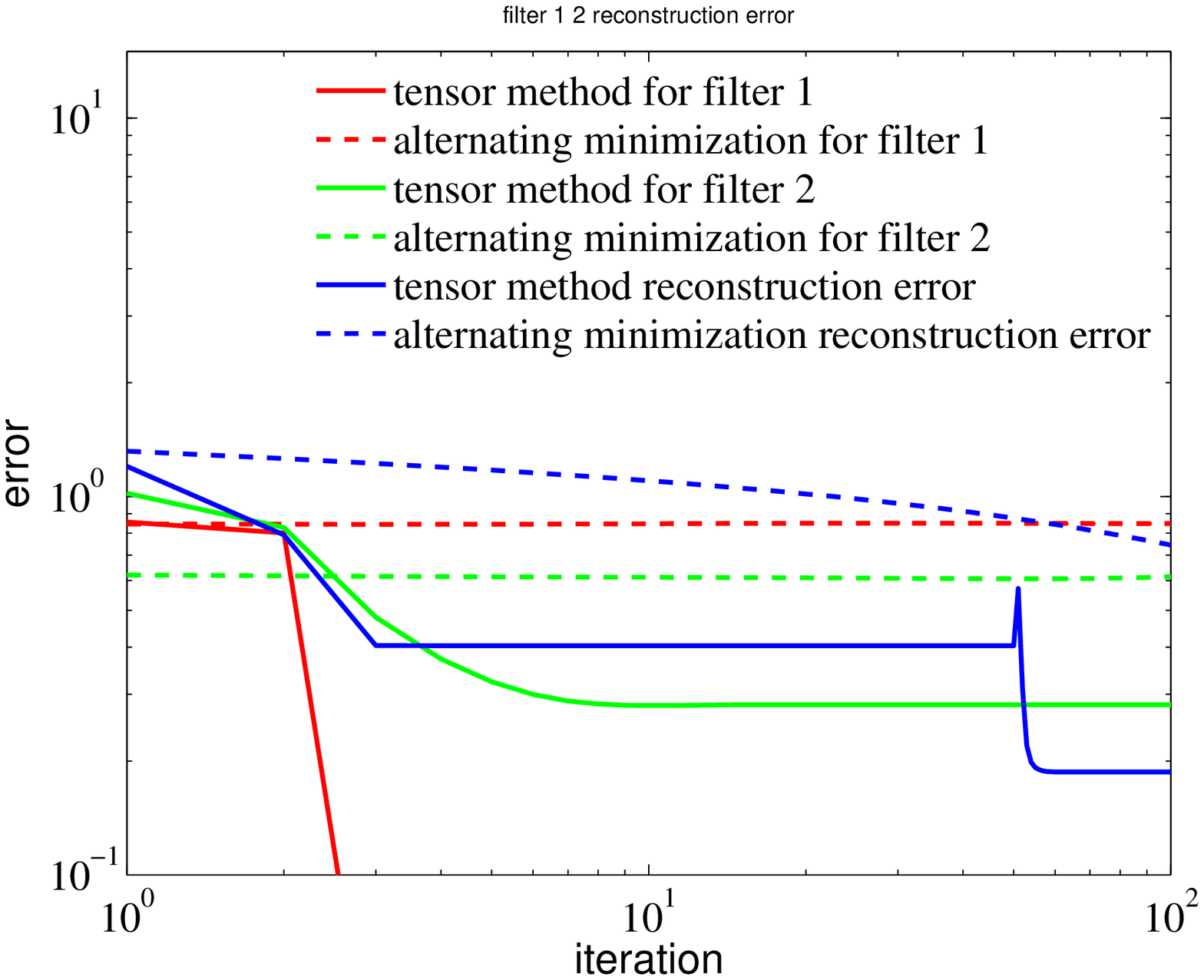}\label{fig:error}
\ec
\end{minipage}
}\hfil
\subfloat%[b][Run Times Scale with $L$]
{\begin{minipage}{0.32\textwidth}\bc 
\psfrag{Number of Filters L}[Bl]{\scriptsize{Number of Filters $L$}}
\psfrag{running time in seconds}[Bl]{\scriptsize{seconds}}
\psfrag{tensor method}[Bl]{\scriptsize{Proposed $\mathsf{CT}$}}
\psfrag{alternating minimization}[Bl]{\scriptsize{Baseline $\mathsf{AM}$}}
\includegraphics[width=\textwidth,height=2in]{\fighomeConv/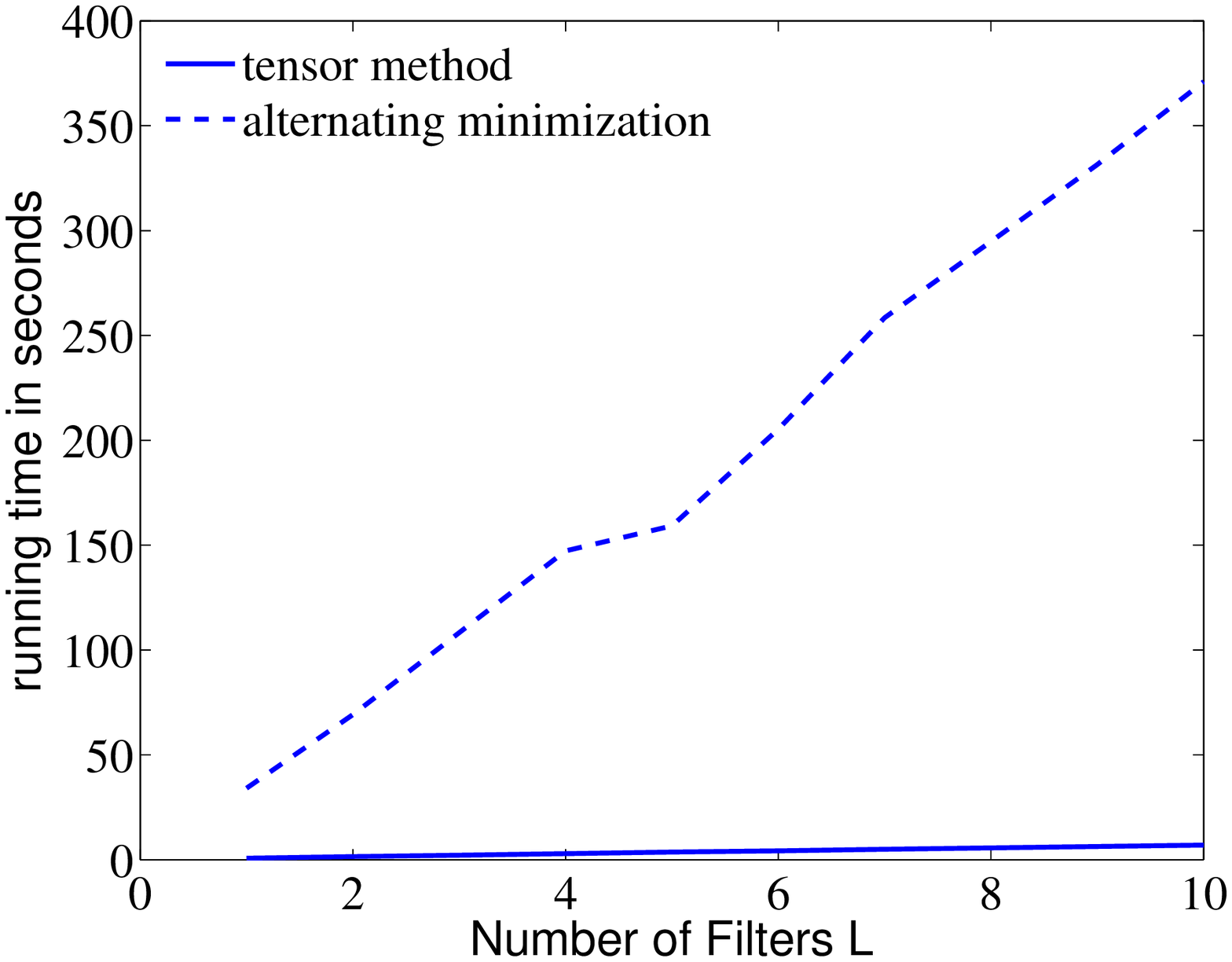}\label{fig:runtime1}\ec\end{minipage}}\hfil
%%%%%%%%%%%%%%%%%%%%%%%%%%%%%%%
\subfloat%[c][Run Times Scale with $N$]
{\begin{minipage}{0.32\textwidth}\bc  
\psfrag{Number of Samples N}[Bl]{\scriptsize{Number of Samples $N$}}
\psfrag{running time in seconds}[Bl]{\scriptsize{seconds}}
\psfrag{tensor method}[Bl]{\scriptsize{Proposed $\mathsf{CT}$}}
\psfrag{alternating minimization}[Bl]{\scriptsize{Baseline $\mathsf{AM}$}}
\includegraphics[width=\textwidth,height=2in]{\fighomeConv/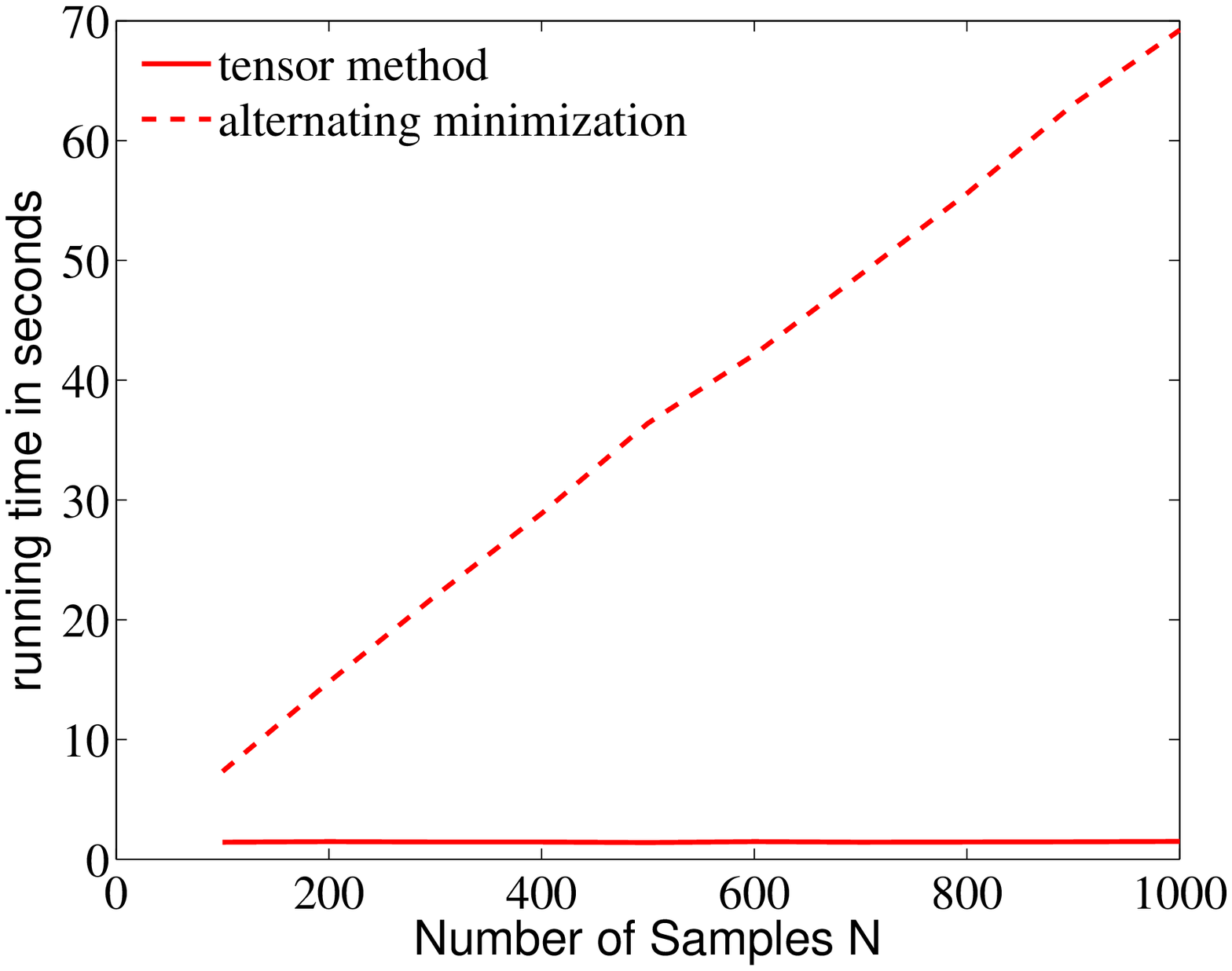}\label{fig:runtime2}\ec\end{minipage}}\\
		\caption[Error and running time comparison]{{ (a) Error comparison between our convolutional tensor method (proposed $\mathsf{CT}$) and the baseline alternate minimization method (baseline $\mathsf{AM}$). (b) Running time comparison between our proposed $\mathsf{CT}$ and the baseline $\mathsf{AM}$ method under varying $L$. (c) Running time comparison between $\mathsf{CT}$ and $\mathsf{AM}$ method under varying $N$.}}
		\ec
\end{figure}	 

\section{Application: Learning Word-sequence Embeddings}

\subsection{Word-Sequence Modeling and Formulation}\label{sec:model}
Our \ourframework framework focuses on a convolutional dictionary model to summarize phrase templates, and then decode word-sequence  signals to obtain the word-sequence embeddings. The first question is how to encode the word sequence into a signal, to be input to the convolutional model and we discuss that below.
\subsubsection{From raw text to signals}
\paragraph{Word encoding: } A word is represented as a \emph{one-hot encoding vector}, i.e. with vector $e_i\in \mathbb{R}^{d}$ whose $i^\tha$ entry is 1 and other entries are $0$, where $i$ is the index of the word in the dictionary. Alternatively, one could use the word2vec  embeddings instead of one-hot encodings. We then stack the one-hot encoding vectors of each sentence together to form a {\em encoding matrix}. The stacking order conforms the word-sequence order. 

\begin{figure}[htbp]
	\psfrag{d}[cc]{$d$}
	\psfrag{k}[cc]{$k$}
	\psfrag{N1}[cc]{$N_1$}
	\psfrag{N2}[cc]{$N_2$}
	\psfrag{N3}[cc]{$N_3$}
	\psfrag{svdeq}[lc]{$\mysvdeq$}\psfrag{equal}[lc]{$=$}\psfrag{eq}[cc]{$=$}
	\psfrag{C}[cc]{$\bfC$}
	\psfrag{U}[cc]{$U$}\psfrag{Ut}[cc]{$U^\top$}\psfrag{D}[cc]{$\Sigma$}\psfrag{Vt}[cc]{$V^\top$}
	\psfrag{Sseq1}[lc]{$\mathcal{S}_{\text{seq}_1}$}
	\psfrag{Sseq2}[cc]{$\mathcal{S}_{\text{seq}_2}$}
	\psfrag{Sseq3}[cc]{$\mathcal{S}_{\text{seq}_3}$}
	\psfrag{C1}[lc]{$\mathcal{Y}_1$}\psfrag{C2}[lc]{$\mathcal{Y}_2$}\psfrag{C3}[lc]{$\mathcal{Y}_3$}
	\psfrag{y11}[Bc]{\tiny{$y_1^{(1)}$}}
	\psfrag{y12}[Bc]{\tiny{$y_2^{(1)}$}}
	\psfrag{y13}[Bc]{\tiny{$y_3^{(1)}$}}
	\psfrag{y21}[Bc]{\tiny{$y_1^{(2)}$}}
	\psfrag{y22}[Bc]{\tiny{$y_2^{(2)}$}}
	\psfrag{y23}[Bc]{\tiny{$y_3^{(2)}$}}
	\includegraphics[width=\textwidth,height=2.5in]{\fighomeConv/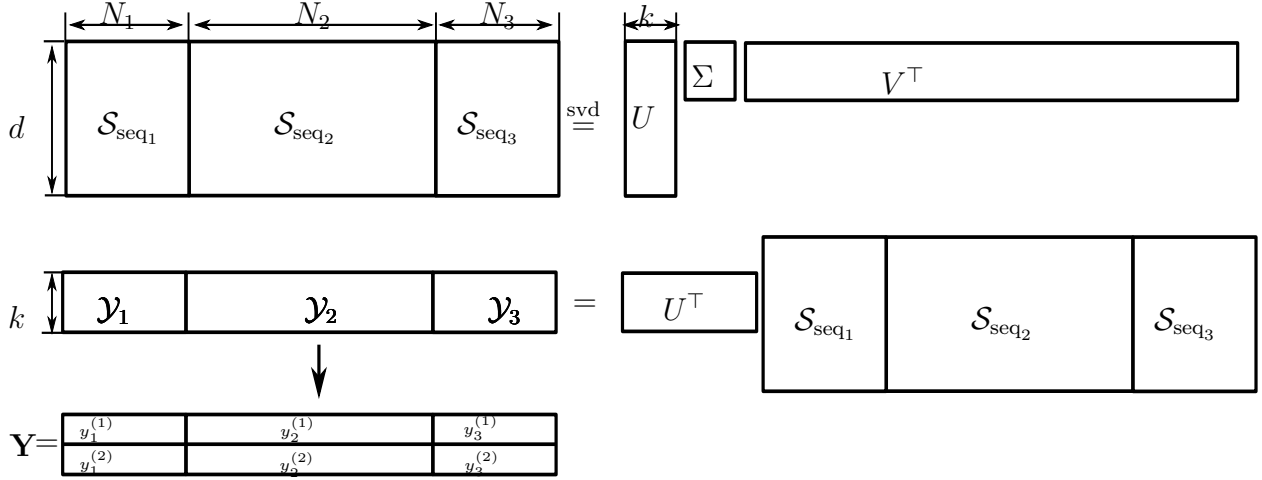}
	\caption[Principal component projection]{Principal component projection to obtain $ [\mathcal{Y}_1, \mathcal{Y}_2, \ldots, \mathcal{Y}_M]=  U^\top \bfS = U^\top[\mathcal{S}_{\text{seq}_1},\mathcal{S}_{\text{seq}_2}, \ldots, \mathcal{S}_{\text{seq}_M}] $ using $\mathcal{S}$. Note that $U$ is the top $k$ left eigenvectors of $\bfS$. }
\end{figure}

To be precise, let us consider sentenc  with  $N$ words. The \emph{encoding matrix} of this word-sequence $\mathcal{S}_{\text{seq}}$ is $\mathcal{S}_{\text{seq}}:= [s_{\text{word}_1},s_{\text{word}_2},\ldots,s_{\text{word}_N}]\in\mathbb{R}^{d\times N}$.

\paragraph{Principal components: } Now that we have encoded words in each sentence, we want to find a compact representation of them in terms of a dictionary model. However, the encoding matrices are too sparse to fit a convolutional model in the word space. Instead, we perform dimensionality reduction through PCA and carry out dictionary modeling in the projected space. 

Concretely,   we stack the encoding matrices side by side as $\bfS: = [\mathcal{S}_{\text{seq}_1},\mathcal{S}_{\text{seq}_2}, \ldots, \mathcal{S}_{\text{seq}_M}] \in \mathbb{R}^{d \times \left(\sum_{i=1}^{M}N_i\right)}$, assuming there are $M$ number of sentences  in the collection of varying lengths $N_1$, $N_2$ and so on. Let $U\in\mathbb{R}^{d \times  k}$ denote the top $k$ left eigenvectors of $\bfS$. We consider  $\mathcal{Y}_i:= U^\top \mathcal{S}_{\text{seq}_1} \in \mathbb{R}^{k \times N_i }$, for each sentence $i$.  We treat the rows of $\mathcal{Y}_i$ independently in parallel and fit convolutional model to each row. Denote $j^\tha$ row of $\mathcal{Y}_i$ as $y_i^{(j)}$,  and thus  $\mathcal{Y}_i = \left[\begin{tabular}{c}$y_i^{(1)}$\\ $\vdots$\\ $y_i^{(k)}$\end{tabular}\right]$.

 \begin{figure}[htbp]
 	\bc
 	\psfrag{Encoding}[lB]{$\mathcal{Y}_i$}
 	\psfrag{Activation}[]{}
 	\psfrag{y1}[cc]{$y_i^{(1)}$}
 	\psfrag{y2}[cc]{$y_i^{(2)}$}
 	\psfrag{y3}[cc]{$y_i^{(k)}$}
 	\psfrag{Maps}[cc]{\scriptsize{Activation Maps}}
 	\psfrag{Stack}[cc]{\scriptsize{stack}}\psfrag{Sequence}[]{}
 	\psfrag{coordinate1}[lc]{\scriptsize{Coordinate 1}}
 	\psfrag{coordinate2}[lc]{\scriptsize{Coordinate 2}}
 	\psfrag{coordinatek}[lc]{\scriptsize{Coordinate $k$}}
 	\psfrag{Embedding}[cc]{\scriptsize{$\ \ $ Word-sequence Embedding}}
 	\psfrag{Comprehension Phase}[lc]{\scriptsize{Comprehension Phase}}
 	\psfrag{Feature-extraction Phase}[lc]{\scriptsize{Feature-extraction Phase}}
 	\includegraphics[width = 0.9\textwidth]{\fighomeConv/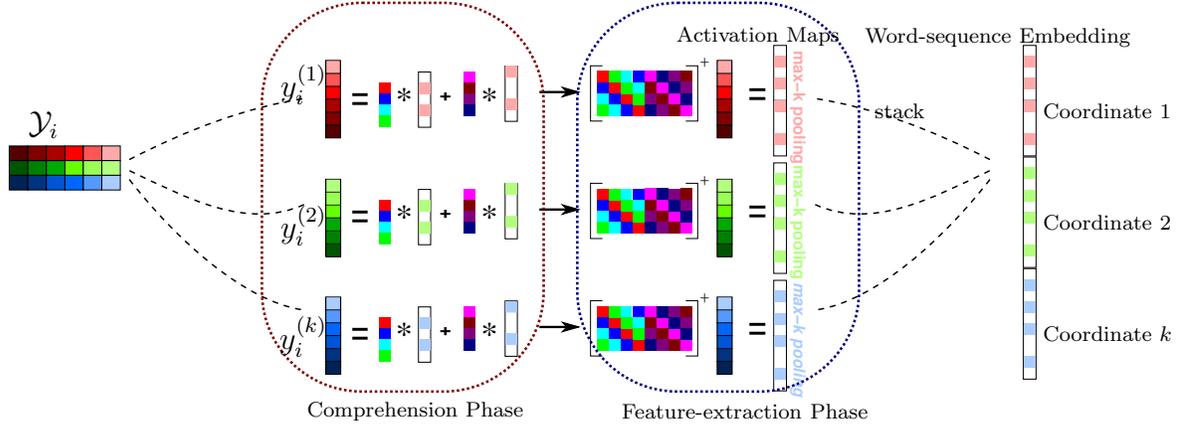}
 	\ec
 	\caption[Overview of our \ourframework framework]{Overview of our \ourframework framework for the $i^\tha$ word-sequence over $k$ coordinates. The Comprehension Phase learns phrase templates using tensor decomposition algorithm. The Feature-extraction Phase decodes activation maps using deconvolutional decoding algorithm. The activation maps are max-k pooled and stacked as the final word-sequence embedding.}\label{fig:wordsequenceembedding}
 \end{figure}

Each $y_i^{(j)}$ is generated through a convolutional dictionary model over phrase templates and activation maps. Our goal in the learning phase is to learn template phrases for the collection of  $[y_i^{(j)}]$ over all word-sequences  $\forall i\in[M]$ across all parallel directions $\forall j\in[k]$. We will state the learning problem formally in the next section. Since all the coordinates are independent and the phrase templates are learned in parallel over all the coordinates, we drop the index $j$ to denote a coordinate of the $i^\tha$ word sequence $y_i^{(j)}$. In the following subsection, a patch from $y_i^{(j)}$ will be denoted as $x$.
\subsubsection{Comprehension Phase -- Learning Phrase Templates}

\begin{figure}[htb]
	\subfloat%[a][Convolutional sparse coding model]
	{\begin{minipage}{\textwidth}
			\bc
			\bp
			\psfrag{=}[Bc]{$=$}
			\psfrag{*}[Bl]{$*$}
			\psfrag{X}[Bc]{\small $x$}
			\psfrag{sum}[Cl]{ $\sum$}
			\psfrag{f1}[Bl]{\small $f_1^*$}
			\psfrag{w2}[Bl]{\small $w_L^*$}
			\psfrag{f2}[Bl]{\small $f_L^*$}
			\psfrag{w1}[Bl]{\small $w_1^*$}
			\psfrag{circf}[Bc]{\small $\quad \mathcal{F}^*$}
			\psfrag{w}[Bl]{\small $w^*$}
			\psfrag{(a)}[Bl]{\small{\bf(a)} Convolutional model}
			\psfrag{(b)}[Bl]{\small{\bf(b)} Reformulated model}
			\includegraphics[height=1.5in]{\fighomeConv/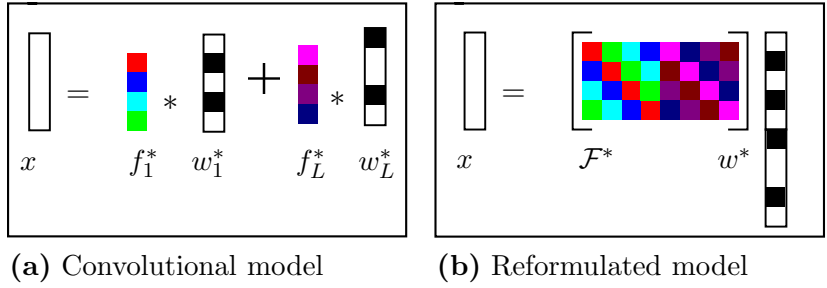}
			\ep
			\ec
		\end{minipage}
	}
\caption[Tensor decomposition for learning convolutional ICA models]{Convolutional tensor decomposition for learning convolutional ICA models~\cite{huang15convolutional}.(a) The convolutional generative model with template phrases. (b) Reformulated multiplicative model where $\CCir^*$ is column-stacked circulant matrix.}\label{fig:reform}
\end{figure}

\begin{figure}
			\bc
			\bp
			\psfrag{M3}[Bc]{$\Cum$}
			\psfrag{f1}[l]{$\lambda_1 (\mathcal{F}_1^*)^{\otimes 3}$}
			\psfrag{f2}[l]{$+\lambda_2 (\mathcal{F}_2^*)^{\otimes 3}$}
			\psfrag{sf1}[l]{ $\quad \  \ldots$}
			\psfrag{sf2}[l]{ }
			\includegraphics[width=\textwidth]{\fighomeConv/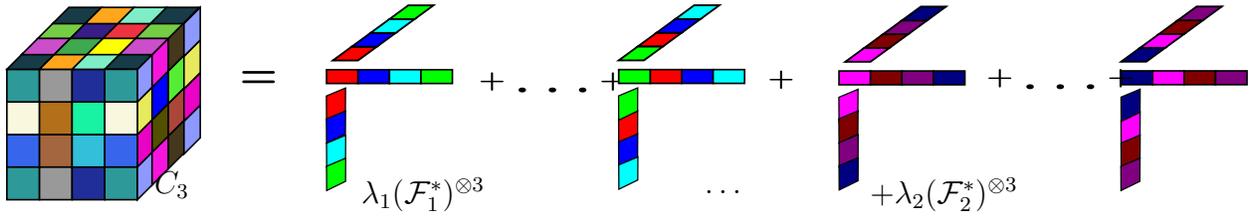}
			\ep
			\ec
			\caption[Third order cumulant]{The third order cumulant is decomposed superposition of third order outer product of template phrases and third order outer product of shifted template phrases.}\label{fig:reform}
\end{figure}

 A word sequence is composed of superposition of overlapping patches, therefore we are interested in learning a generative model over overlapping patches.  We can also view these patches as phrases. A length $n$ patch $x$ is generated as the superposition of $L$ phrase embeddings $f_l^*$ convolved at $L$ activation maps $w_l^*$, $\forall l\in [L]$. Due to the property of the convolution, the convolution is reformulated as the multiplication of $\CToep^*$ and $w^*$, where $\CToep^*:=[\Toep(f_1^*), \Toep(f_2^*),\ldots,\Toep(f_L^*)]$ is the concatenation of circulant matrices and $w^*$ is the  \emph{row-stacked} vector
 $w^* : =\left[
 \begin{tabular}{c}
 $w_1^*$\\
 $w_2^*$\\
$\vdots$\\
$w_L^*$
 \end{tabular}
 \right]\in \Rbb^{nL}$.
 To be precise, a patch
\begin{equation}
x = \sum_{l\in [L]}   f_l^*\Conv {w_l}^*
%= \sum_{j\in [L]} \Toep(f_j^*) w_j^*
%	  & = [\Toep(f_1),\Toep(f_2),\ldots,\Toep(f_L)] [h^i_1;h^i_2;\ldots;h^i_L]
= \CToep^*\cdot w^* ,\label{eqn:gen}
\end{equation}
This is illustrated in Fig~\ref{fig:reform}(a). $\Toep(f_l^*)$ is circulant matrix corresponding to phrase template $f_l^*$, whose columns are shifted versions of $f_l^*$ as shown in Fig~\ref{fig:reform}(a). Note that although $\mathcal{F}^*$ is a $n$ by $nL$  matrix, there are only $nL$ free parameters.  Given access to the collection of word-sequence sample patches,  $X:=[x^1,x^2,\ldots]$, generated according to the above model, we aim to estimate the true template phrases $f_i^*$, for $i \in [L]$. %In section~\ref{sec:convtensor} we will elaborate on our convolutional tensor decomposition dictionary learning method (\ourlearning).

 If the patches are in the same coordinate of the word sequence, these patches share a common set of phase templates, but their activation maps are different. The activation maps are the discriminative features that distinguish different patches.
 %We never explicitly form the estimates $\mathcal{F}$ of $\mathcal{F}^*$, but instead use filter estimates  $f_l$'s to characterize $\mathcal{F}$.
%In addition, we can handle additive Gaussian noise in \eqref{eqn:gen}, but do not incorporate it for simplicity.
 Once the template phrases are estimated, we can use standard decoding techniques, such as the square loss criterion in \eqref{eqn:alt-min} to learn the activation maps for the individual maps. %We focus on developing a novel method for template phrase estimation in this paper.

\subsubsection{Feature-extraction Phase -- Word-sequence Embeddings}
 \paragraph{Activation maps in a coordination: } After learning a good set of phrase templates $\{f_1,\ldots, f_L\}$ and thus $\CToep$, we use the deconvolutional decoding (\ourdecode) to obtain the activation maps for the $j^\tha$ coordinate. For each observed coordinate of the word-sequence $y_i^{(j)}$, the activation map $w_l^*$ in \eqref{eqn:gen} indicates the locations where $i^\tha$ template phrase $f_l^*$ is activated and $w^*$ is the  \emph{row-stacked} vector $w^* : =[w_1^*;w_2^*;\ldots w_L^*]$. An estimation of $w^*$, $w^{(j)}_i$, is achieved as follows
 \beq
 w^{(j)}_i =   \CToep^\dag {y^{(j)}_i}^\top.
 \eeq
Note that the estimated phrase templates are zero padded to match the length of the word-sequence.

 We assume that the elements of $w^*$ are drawn from some product distribution, i.e. different entries are independent of one another, and we have the independent component analysis (ICA) model in \eqref{eqn:gen}. When the distribution encourages sparsity, e.g. Bernoulli-Gaussian, only a small subset of locations are active, and we have the {\em sparse coding} model in that case. We can also extend to dependent distributions such as Dirichlet for $w^*$, along the lines of~\cite{blei2003latent}, but limit ourselves to ICA model for simplicity. This activation map $w^{(j)}_i  \in \mathbb{R}^{N_i\cdot L}$ contains sequence embeddings from coordinate $j$ only, and will be used as one coordinate of our final word-sequence embeddings.

 \paragraph{Varying sentence length: } One difficulty in learning the template phrases using our convolutional tensor decomposition model is that different word-sequence has a different length $N_i$, therefore the activation maps are of varying length as well.  We resolved this problem by \emph{max-k pooling}. In other words, we extract most informative global discriminative features from the activation maps, as illustrated in Figure~\ref{fig:wordsequenceembedding}.
 Finally, we concatenate all the max-k pooled coordinate sequence embeddings as a long vector as the final word-sequence embedding.

 The overall framework flow is depicted in Fig~\ref{fig:wordsequenceembedding}.

\subsection{Evaluating Embeddings through Downstream Tasks}
%We first evaluate the performance of our convolutional tensor decomposition framework on simulated data, and compare against solving equation~\eqref{eqn:alt-min} using alternating minimization method where gradient descent is employed to update $f_i$ and $w_i$ alternatively. The performance is illustrated in Appendix~\ref{app:synthetic}. Our method is both accurate and efficient compared to alternating minimization method.

We evaluate the quality of our word sequence embeddings using three challenging natural language process tasks: sentiment classification, paraphrase detection, and semantic textual similarity estimation. Eight datasets which cover various domains are used as shown in Table~\ref{tab:datasets}.
\begin{table}[htbp]
\small
\begin{tabular}{|l||l|c|c|l|}
\hline
\bf{Dataset} & \bf{Domain} & \bf{Label} & \bf{Label Distribution} & \bf{$M$} \\
\hline
Review & Moview Reviews &\{-1,1\}	& [0.49,0.51] 				&  64720 \\
\hline
SUBJ   & Obj/Subj comments 		   &\{-1,1\}  & [0.50,0.50]  & 1000 \\
\hline
MSRpara & news sources & \{-1,1\} & [0.33,0.67] & 5801$\times$2 \\
\hline
STS-MSRpar &  newswire& [0,5] & [0.00,0.02,0.10,0.24,0.47,0.17] &$1500\times$2 \\
\hline
STS-MSRvid & video caption  & [0,5] & [0.13,0.21,0.14,0.16,0.21,0.14] &1500$\times$2\\
\hline
STS-OnWN &  glosses& [0,5] & [0.01,0.02,0.04,0.12,0.35,0.47] & $750\times$2%test
\\
\hline
STS-SMTeuroparl & machine translation &[0,5] & [0.01,0.00,0.00,0.02,0.19,0.78] & $1193\times$2\\
\hline
STS-SMTnews & machine translation & [0,5] & [0.00,0.01,0.01,0.06,0.19,0.73] & $399\times$2%test
\\
\hline
\end{tabular}
\caption{Summary statistics of the datasets used. }
\label{tab:datasets}
\end{table}%

%\paragraph{Decode the patch representations}After running our convolutional tensor decomposition algorithm, we learn $L$ common templates $f_i, \forall i \in[L]$. We then decode for patch $\patch^j $'s representations by solving the least square optimization problem \begin{equation}\lVert \patch^j - \sum\limits_{i=1}^{L} f_i \Conv h_i^j\rVert^2\end{equation}After learning all the patch representation in sentence $\sentence$, we use the mean of the patch representations as the sentence representation. As for the $k$ different rows of the sentences, we concatenated them as they represent the $k$ different spaces of the sentence embeddings/representations.
\normalsize

For all the datasets, we train a simple logistic regression model on the training samples and report test classification accuracy using a 10-fold cross validation.
 Sentiment analysis and paraphrase detection belong to binary classification tasks. In a binary classification task, either accuracy or F score is used as evaluate metric. Recall that F-score is the harmonic mean of precision and recall, i.e.,
 ${\displaystyle F = 2 \cdot {(\mathrm{precision} \cdot \mathrm{recall})}/{ \mathrm{precision} + \mathrm{recall}}}$. Precision is the number of true positives divided by the total number of elements labeled as belonging to the positive class, and recall is the number of true positives divided by the total number of elements that belong to the positive class. 

Our \ourframework learns word-sequence embeddings from scratch and requires no pre-training. When working on a new dataset from a new domain, we train fresh set of phrase templates as called domain phrase templates. Using these domain phrase templates, we decode activation maps and then form phrase-embeddings. Our approach is different from skip thoughts, where universal phrase embeddings are generated~\cite{kiros2015skip}. 

\subsubsection{Evaluation Task: Sentiment Classification}
Sentiment analysis is an important task in natural language process as automated labeling of word sequences into positive and negative opinions is used in various settings. We evaluate our sentence embeddings on two datasets from different domains, such as movie review and subjective and objective comments, as in Table~\ref{tab:datasets}. Using word-sequence embeddings combined with NB features, we obtain the state-of-the-art classification results for both these datasets as in Table~\ref{tab:sentimentClass_SICK}.

\begin{table}[htbp]
\begin{center}
%\begin{table}[!htb]
\small{\begin{tabular}{|lll|}
\hline
Method & MR & SUBJ\\
\hline
NB-SVM~\cite{wang2012baselines}&79.4 & 93.2\\
MNB~\cite{wang2012baselines}& 79.0  & 93.6 \\
cBoW~\cite{zhao2015self} & 77.2  & 91.3 \\
\hline
GrConv~\cite{zhao2015self} & 76.3  & 89.5 \\
RNN~\cite{zhao2015self} & 77.2 &  93.7 \\
BRNN~\cite{zhao2015self} & 82.3 &  94.2 \\
CNN~\cite{kim2014convolutional} & 81.5  & 93.4 \\
AdaSent~\cite{zhao2015self} & 83.1  & 95.5 \\
\hline
Paragraph-vector~\cite{le2014distributed} &74.8  &90.5  \\
\hline
Skip-thought~\cite{kiros2015skip} & 75.5  &92.1\\
\hline
\textbf{\ourframework} &\textbf{78.9} &\textbf{92.4}\\
\hline
\end{tabular}}
%\caption{{\small Classification accuracies on standard benchmarks on movie review and subject dataset. The first group are results using bag-of-words models; the second group is supervised compositional models; the third group is paragraph vector; the fourth is the skip-thought result. }}\label{tab:sentimentClass}\end{table}

\end{center}
\caption[Classification tasks]{Binary classification tasks:  sentiment analysis task of cataloging a word-sequence into two different categories. \small{Classification accuracies in percentage on standard benchmarks (movie review and subject dataset) are displayed. The first group contains results using bag-of-words models; the second group exhibits some supervised compositional models; the third group is paragraph vector; the fourth is the skip-thought result.
}}
\label{tab:sentimentClass_SICK}
\end{table}

\begin{table}[htbp]
\begin{center}
%\begin{table}[!htb]
\small{\begin{tabular}{|lll|}
\hline
Method & %Description &
 Outside Information \tablefootnote{The word similarities information they use are either trained in  Wikipedia (4.4 million articles in contrast to the 4076 sentences of paraphrase dataset we use) or from \emph{WordNet} with expert knowledge. }
& F score\\
\hline
Vector Similarity~\cite{mihalcea2006corpus} %& cosine similarity with tf-idf weights	
&  word similarity
%& 65.4\%
& 0.75 \\
ESA~\cite{hassan2011measuring} %& explicit semantic space	
& word semantic profiles 
%& 67.0\%	
& 0.79\\
LSA~\cite{hassan2011measuring} %& latent semantic space	
&  word semantic profiles
%& 68.8\%	
& 0.80\\
RMLMG~\cite{rus2008paraphrase}%& graph subsumption	
& %lexical\&
syntactic%\&synonymy
 info
%& 70.6\%
& 0.81\\
\textbf{\ourframework} %& \textbf{convolutional dictionary learning}
& \textbf{none }
% & 67.57%\%
& \textbf{0.81}\\
Skip-thought~\cite{kiros2015skip} %& sentence skip-gram 
&train large book corpus &  0.82\\
%MCS~	\cite{mihalcea2006corpus} & combine  word similarity measures 
%&  word similarity 
%%&70.3\%
%& 81.3\%\\
%STS~\cite{islam2009semantic}&	combine semantic\&string similarity
%& semantic similarity
%%& 72.6\%
%&81.3\%\\
%SSA~\cite{hassan2011measuring}& salient semantic space	
%& word semantic profiles
%%& 72.5\%	
%& 81.4\%\\
%matrixJcn~\cite{fernando2008semantic}&JCN WordNet similarity with matrix 
%& word similarity 
%%&74.1\%	
%& 82.4\%\\
\hline
\end{tabular}}
%\caption{{\small Comparison of  F score with the state of art unsupervised sentence paraphrase detection. The higher F score indicates better performance. All other methods use word similarities trained on Wikipedia  or from \emph{WordNet}. In contrast, our algorithm learns sentence embedding with no outside information and detects paraphrases from scratch. We still achieve comparable results since our convolutional dictionary learning uses word order to generate sentence embeddings.}}\label{tab:paraphrase}\end{table}
\end{center}
\caption[Paraphrase detection tasks]{Binary classification tasks: paraphrase detection task, which operates on pairs of word-sequences and decides on whether they are a paraphrase of each other or not. Comparison of  F-score with other unsupervised sentence paraphrase approaches.  Other methods use auxiliary information such as word similarities trained on Wikipedia or from \emph{WordNet}. In contrast, our algorithm learns sentence embeddings from scratch.
}
\label{tab:sentimentClass_SICK_r}
\end{table}

\subsubsection{Evaluation Task: Paraphrase Detection}
We consider the \emph{paraphrase detection} task on the Microsoft paraphrase corpus~\cite{quirk2004monolingual,dolan2004unsupervised}. 
We employ 4076 sentence pairs as training data to learn the sentence embeddings and regress on the ground truth binary labels with our learned sentence embeddings. The remaining test data is used to calculate classification error.  

As discussed in ~\cite{tai2015improved}, we combine the pair of sentence embeddings produced earlier $w_L$ and $w_R$, i.e.,  the embedding for the right and the left sentences. We generate features for classification using both the distance (absolute difference) and the product between the pair ($w_L$, $w_R$):  $[w_L \odot w_R, \lVert w_L - w_R\rVert]$, where $\odot$ denotes the element-wise multiplication.

In contrast to other unsupervised methods which are trained using outside information such as wordnet and parse trees, our unsupervised approach use \textbf{no} extra information, and still achieves comparable results with the state of art~\cite{paraphraseStateofArt} as in table~\ref{tab:sentimentClass_SICK_r}. We show some examples of paraphrase and non-paraphrase we identified.

\textbf{Paraphrase detected: } \emph{ (1) Amrozi accused his brother, whom he called "the witness", of deliberately distorting his evidence.	(2) Referring to him as only "the witness", Amrozi accused his brother of deliberately distorting his evidence.} The two sentences are the ``difficult sentence'' to show how our algorithm detect paraphrases since they are not simple switching of clauses, and the sentence structures differ quite significantly in the two sentences.

\textbf{Non-paraphrase detected :}  \emph{ (1) I never organised a youth camp for the diocese of Bendigo.	(2) I never attended a youth camp organised by that diocese.}
Similarly with non-paraphrase detection, the two sentences share common words such as youth camp and organized, but our method is able to successfully detect them as non-paraphrase.

\subsubsection{Evaluation Task: Semantic Textual Similarity Estimation}
For the Semantic Textual Similarity (STS) task, the goal is to predict a real-valued similarity score in a range $[1,K]$ given a sentence pair. We include datasets from STS task in various domains including news, image and video description, glosses from WordNet/OntoNotes, the output of machine translation systems with reference translation.

To frame semantic test similarity estimation task into the multi-class classification framework, the gold rating $\tau \in[K_1,K_2]$ is discretized as $p\in\Delta^{K_2-K_1}$ in the follow manner~\cite{tai2015improved}, $p_i =\lfloor \tau \rfloor - \tau +1$ if $i = \lfloor \tau \rfloor + 1 - K_1$,  $p_i = \tau- \lfloor \tau \rfloor$ if $i = \lfloor \tau \rfloor +2 - K_1$, and $p_i=0$ otherwise.
This reduces to finding a predicted $\hat{p}_{\theta}\in \Delta^{K_2-K_1}$ given model parameters $\theta$ to be closest to $p$ in terms of KL divergence~\cite{tai2015improved}. We use a logistic regression classifier to predict $\hat{p}_{\theta}$ and estimate $\hat{\tau}_{\theta} = [K_1,\ldots,K_2]\hat{p}$.

 Results on STS task datasets are illustrated in Table~\ref{tab:paraphrase_STS}.
  As in~\cite{wieting2015towards}, Pearson's r of the median, 75th percentile, and highest score from the official task rankings are showed. We then compare our method against the performance of  supervised models in ~\cite{wieting2015towards}: PARAGRAM-PHRASE (PP), projection (proj.), deep-averaging network (DAN), recurrent neural network (RNN) and LSTM;  as well as the state-of-the-art unsupervised model skip-thought vectors~\cite{kiros2015skip}.%\fhcomment{STS task website and summarization}

 As we can see from the table, LST is performing poorly even though a back-propagation after seeing the training labelings is carried out for sequence embedding learning. Our method is an unsupervised approach as in skip-thought vectors. However, our algorithm doesn't output universal word-sequence embeddings across domains. We train a fresh model and a new set of domain phrase templates from scratch. Therefore our algorithm is performing better for these individual datasets on the STS task. 
%The STS competition...  LST performs poorly
%\aacomment{put a paragraph of conclusions regarding what u see from the last table}  
  
  %, averaged GloVe11 vectors~\cite{pennington2014glove}, and averaged PARAGRAM-SL999 vectors~\cite{wieting2015towards}, denoted ``PSL''.
%$
%p_i =\left\{
%\begin{tabular}{ll}
%$\lfloor \tau \rfloor - \tau +1$, & $i = \lfloor \tau \rfloor$\\
%$\tau- \lfloor \tau \rfloor$, & $i = \lfloor \tau \rfloor +1$\\
%0, & otherwise
%\end{tabular}
%\right..
%$

\begin{table}[!htb]
%\scriptsize{\begin{tabular}{l||cccccccc}
%\hline
%Dataset 	& 50\% & 75\% & Max &  DAN & RNN &  LSTM  & ST & \textbf{\ourframework} \\
%\hline
%MSRpar 	&51.5& 57.6& 73.4 &40.3 &18.6 &9.3& 16.8 &\\
%MSRvid 	&75.5 &80.3& 88.0 &70.0 &66.5 &71.3 &41.7& \\
%SMT-eur	&44.4 &48.1& 56.7 &43.8 &40.9& 44.3 &35.2 &\\
%OnWN 	&60.8 &65.9& 72.7 &65.9 &63.1& 56.4 &29.7 &\\
%SMT-news& 40.1 &45.4 &60.9& 60.0& 51.3& 51.0 &30.8 &\textbf{79.1}\\
%\hline
%\end{tabular}}
%
\scriptsize{\begin{tabular}{|l||ccc||ccc||cc|}
\hline
 		& Supervised &$+$& Unsupervised & Supervised     & 	Methods     &               & Unsupervised & Methods \\
\hline
\hline
Dataset  	& 50\% & 75\% & Max &  DAN 		& RNN &  LSTM  & Skip-thought & \tiny{\textbf{\ourframework}} \\
%\begin{tabular}{c}\textbf{ConvDic}$+$\\\textbf{DecovDec} \end{tabular} \\
\hline
\hline
MSRpar 	 &51.5& 57.6& 73.4 &40.3 &18.6 &9.3& 16.8 &\textbf{36.0}\\
MSRvid 	&75.5 &80.3& 88.0 &70.0 &66.5 &71.3 &41.7& \textbf{61.8}\\
SMT-eur	&44.4 &48.1& 56.7 &43.8 &40.9& 44.3 &35.2 &\textbf{37.5}\\
OnWN 	&60.8 &65.9& 72.7&65.9 &63.1& 56.4 &29.7 &\textbf{33.1}\\
SMT-news & 40.1 &45.4 &60.9& 60.0& 51.3& 51.0 &30.8 &\textbf{72.1}\\
\hline
\end{tabular}}

\caption[STS task]{{\small STS task results: Pearson's $r\times 100$ on MSRpar, MSRvid, OnWN, SMTeuroparl and SMTnews dataset.} The first three columns are official rankings reported in the STS2012 official website, so it combines both supervised and unsupervised methods. The second three columns are reported by~\cite{wieting2015towards}. Our comparison against the state-of-the-art unsupervised word-sequence embedding method is in the last two columns. }\label{tab:paraphrase_STS}
\end{table}

\section{Conclusion}
In this chapter, we proposed a novel tensor decomposition framework for learning convolutional dictionary models. Unlike the popular alternating minimization, our method avoids expensive decoding of activation maps in each step and can reach better solutions with faster run times. We derived efficient updates for   tensor decomposition based on modified  alternating least squares, and it consists of simple operations such as FFTs and matrix multiplications.  Our framework easily extends to convolutional models for higher dimensional signals (such as images), where the circulant matrix is replaced with block circulant matrices~\cite{gray2005toeplitz}. More generally, our framework can  handle general group structure, by replacing the FFT operation with the appropriate group FFT~\cite{kondor2008group}.  By combining the advantages of tensor methods with a general class of invariant representations, we thus have a powerful paradigm for learning efficient latent variable models and embeddings in a variety of domains.

%Deploying the frameworks in settings which incorporate Lie algebra is of great practical interest in computer vision and robotics. 

%\subsubsection*{Acknowledgments}
%We thank Cris Cecka for discussion on fast implementation of block matrix inverse, and we thank the initial discussions with Majid Janzamin and Hanie Sedghi on Toeplitz matrices~\cite{eberle2003finding}. F. Huang is supported by NSF BIGDATA award FG16455.

\chapter{Latent Tree Model Learning through Hierarchical Tensor Decomposition}\label{chapter:tree}
In previous chapters, we introduced latent dirichlet allocation and its variations to model data with ``shallow'' structure, for instance, multi-view model.  However, real world data is usually generated through more complicated models such as a latent (hierarchical) tree graphical model. Latent tree graphical models characterize a probability distribution involving observed and hidden variables which are Markovian on a tree. Learning is challenging as the number of latent variables and the location of them are not observed. We present an integrated approach to structure and parameter estimation in latent tree graphical models, where some nodes are hidden.

\begin{figure}[!htbp]
\bc
\subfloat[Latent tree]{
\begin{minipage}{0.29\textwidth}
\psfrag{h}[]{}\psfrag{h1}[]{}\psfrag{h2}[]{}\psfrag{x1}[]{}\psfrag{x2}[]{}\psfrag{x3}[]{}\psfrag{x4}[]{}\psfrag{x5}[]{}
\includegraphics[width=\textwidth]{\fighomeTalk/hierarchy-new}
\end{minipage}
}
\hfil
\subfloat[Hierarchical tensor decomposition]{
\begin{minipage}{0.64\textwidth}
\psfrag{h}[]{}\psfrag{x1}[]{}\psfrag{x2}[]{}\psfrag{x3}[]{}\psfrag{x4}[]{}\psfrag{x5}[]{}
\includegraphics[width=\textwidth]{\fighomeTalk/flat2hier-tensorDecomp-2}
\end{minipage}
}\\
\ec
\caption[Latent tree and hierarchical tensor decomposition]{Learning hierarchical latent variable graphical model parameter using hierarchical tensor decomposition.}\label{fig:treehierarchical}
\end{figure}

We present an integrated approach to structure and parameter estimation in latent tree models.
Our method overcomes all the above shortcomings simultaneously.  First, it automatically learns the latent variables and their locations.
Second,  our method achieves consistent structure estimation with $\log(p)$ computational complexity with enough computational resources via ``divide-and-conquer'' manner. We also present a rigorous proof on the global consistency of the structure and parameter estimation under the ``divide-and-conquer'' framework. Our consistency guarantees are applicable to
a broad class of linear multivariate latent tree models including discrete distributions, continuous multivariate distributions (e.g. Gaussian), and mixed distributions such as Gaussian mixtures.
This model class is much more general than discrete models, prevalent in most of the previous works on latent tree models~\cite{mossel2005learning,mossel2007distorted,erdos1999few,anandkumar2013learning}.
Third, our algorithm considers the inverse method of moments, and estimates the model parameters via tensor decomposition with low perturbation guarantees.
Moreover, we carefully integrate structure learning with parameter estimation, based on tensor spectral decompositions~\cite{anandkumar2012tensor}.
%The parameter learning for any triplet of observed nodes on the latent tree can be carried out via standard tensor decomposition, as in~\cite{anandkumar2012tensor}.
Finally,  our approach  has a high degree of parallelism, and is {\em bulk asynchronous }parallel~\cite{gerbessiotis1994direct}. 

In addition to the aforementioned technical contributions, we showcase the impact of our work by applying it to two real datasets originating from the healthcare domain. The algorithm was used to discover hidden patterns, or concepts reflecting co-occurrences of particular diagnoses in patients in outpatient and intensive care settings. While such a task is currently done through manual analysis of the data, our method provides an automated method for the discovery of novel clinical concepts from high dimensional, multi-modal data. 

Our overall approach follows a ``divide-and-conquer'' strategy that learns models  over small groups of variables and iteratively merges into a global solution.   The structure learning involves  combinatorial operations such as minimum spanning tree construction and local recursive grouping; the parameter learning is based on the method of moments and on tensor decompositions. Our method  is guaranteed to correctly recover the unknown tree structure and the model parameters with low sample complexity for the class of linear multivariate latent tree models which includes discrete and Gaussian distributions, and Gaussian mixtures. Our bulk asynchronous parallel algorithm is implemented in parallel using the OpenMP framework and scales logarithmically with the number of variables and linearly with dimensionality of each variable.

Our experiments confirm a high degree of efficiency and accuracy on large datasets of electronic health records. We use latent tree model for discovering a hierarchy among diseases based on comorbidities exhibited in patients' health records, i.e. co-occurrences of diseases in patients.  
In particular, two large healthcare datasets of 30K and 1.6M patients are used to build the latent disease trees, where clinically meaningful disease clusters
are identified as shown in fig~\ref{Fig:tree_mimic2_1} and ~\ref{Fig:tree_mimic2_2}. The proposed algorithm also generates intuitive and clinically meaningful disease hierarchies.

\section{Latent Tree Graphical Model Preliminaries}
We denote $[n]:=\{ 1, \ldots, n \}$. Let $\mathcal{T}:= \left( \mathcal{V}, \mathcal{E}\right)$ denote an undirected tree with  vertex set $\mathcal{V}$ and  edge set $\mathcal{E}$.
The \emph{neighborhood} of a node $v_i$, $\Nb(v_i)$, is the set of nodes to which $v_i$ is directly connected on the tree.
Leaves which have a common neighboring node are known as \emph{siblings}, and the common node is referred to as their {\em parent}. Let $N$ denote the number of samples. An example of latent tree is depicted in Figure~\ref{Fig:StructureLearning}(a).

There are two types of variables on the nodes, namely, the observable variables, denoted by $\mathcal{X} := \left\{x_1,\ldots,x_p\right\}$ ($p := \vert \mathcal{X} \vert$), and  hidden variables, denoted by $\mathcal{H}:=\left\{h_1,\ldots,h_m\right\}$ ($m := \vert \mathcal{H} \vert$).
Let $\mathcal{Y} := \mathcal{X} \cup \mathcal{H}$ denote the complete set of variables and let $y_i$ denote the random variable at node $v_i\in \mathcal{V}$, and similarly let $y_A$ denote the set of random variables in set $A$.

A graphical model is defined as follows: given the neighborhood $\Nb(v_i)$ of any node $v_i \in \mathcal{V}$, the variable $y_i$ is  conditionally independent  of the rest of the variables in $\mathcal{V}$, i.e., $y_i \perp y_j | y_{\Nb(v_i)},\  \forall v_j\in \mathcal{V}\backslash \left\{v_i \cup \Nb(v_i)
\right\}$.

\paragraph{Linear Models } We consider the class of linear latent tree models. %which includes discrete distributions, Gaussian multivariate models and Gaussian mixtures~\cite{anandkumar2011spectral}.  
The observed variables $x_i$ are random vectors of length $d_i$,
%\jscomment{all observed variables have to have the same dimensionality $d$? Is this easy to generalize to different dimensionalities for different variables? If easy, we should state this is a simple extension. If it is really easy, we should say how exactly to do that zero padding?}
 i.e., $x_i \in \Rbb^{d_i},\  \forall i\in [p]$ while the latent nodes  are $k$-state categorical variables, i.e., $h_i \in \{e_1, \ldots, e_k \}$, where $e_j\in \Rbb^k$ is the $j^{\tha}$ standard basis vector. Although $d_i$ can vary across variables, we use $d$ for notation simplicity. In other words, for notation simplicity, $x_i \in \Rbb^{d},\  \forall i\in [p]$ is equivalent to $x_i \in \Rbb^{d_i},\  \forall i\in [p]$. 
For any variable $y_i$ with neighboring hidden variable $ h_j$, we assume a linear relationship: 
\begin{equation} \Ebb[y_i |h_j] = A_{y_i\llvert h_j}  h_j,\end{equation}
where transition matrix $A_{y_i\llvert h_j} \in \mathbb{R}^{d\times k}$ is assumed to have full column rank, $\forall y_i,h_j\in \mathcal{V}$. 
This implies that $k\leq d$, which is natural if we want to enforce a parsimonious model for fitting the observed data.

For a pair of (observed or hidden)  variables $y_a$ and $y_b$, consider the \emph{pairwise correlation matrix}  $\mathbb{E}\left[y_a y_b^\top \right]$  where the expectation is over samples.
Since our model assumes that two observable variables interact through at least a hidden variable, we have 
\begin{equation}
\mathbb{E}[y_a y_b^\top] :=  \sum\limits_{e_i}\mathbb{E}[h_j=e_i]  A_{y_a\llvert h_j=e_i} A_{y_b\llvert h_j=e_i}^\top
\end{equation}
We see that $\mathbb{E}[y_a y_b^\top]$ is of rank $k$ since $A_{y_a\llvert h_j=e_i}$ or $A_{y_b\llvert h_j=e_i}$ is of rank $k$.

\section{Overview of Approach}\label{sec:overview}

\begin{figure*}[hbtp]
  \centering
\includegraphics[width=\textwidth]{\fighomeLT/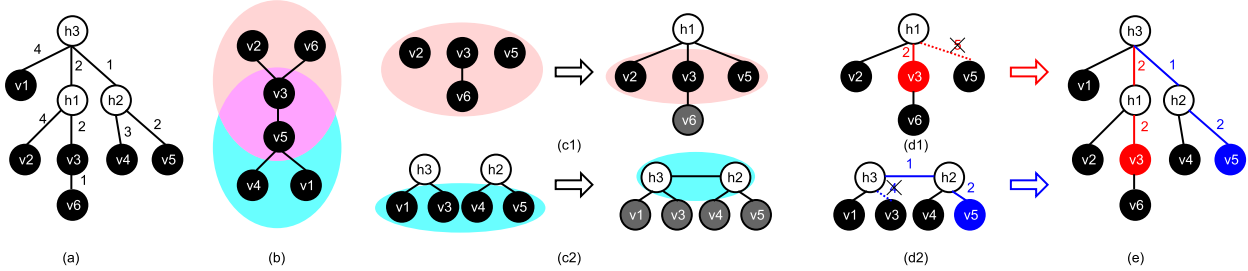}
\caption[Overall approach illustrated in a toy example]{\small \textbf{(a)} Ground truth latent tree to be estimated, numbers on edges are \emph{multivariate information distances}. \textbf{(b)} MST constructed using the \emph{multivariate information distances}. $v_3$ and $v_5$ are internal nodes (leaders). Note that \emph{multivariate information distances} are additive on latent tree, not on MST. \textbf{(c1)} LCR on $\Nb[v_3,\text{MST}]$ to get local structure $\mathcal{N}_3$. Pink shadow denotes the active set. Local parameter estimation is carried out over triplets with joint node, such as ($v_2$, $v_3$, $v_5$) with joint node $h_1$. \textbf{(c2)} LCR on $\Nb[v_5,\text{MST}]$ to get local structure $\mathcal{N}_5$. Cyan shadow denotes the active set. \textbf{(d1)}\textbf{(d2)} Merging local sub-trees. Path($v_3$,$v_5$; $\mathcal{N}_3$) and path($v_3$,$v_5$; $\mathcal{N}_5$) conflict.  \textbf{(e)} Final recovery. }
\label{Fig:StructureLearning}
\end{figure*}

The overall approach is depicted in Figure~\ref{Fig:StructureLearning}, where (a) and (b) show the data preprocessing step,  
 (c) - (e) illustrate the divide-and-conquer step for structure and parameter learning.

More specifically, we start with the parallel computation of 
pairwise \emph{multivariate information distances}. Information
distance roughly measures the extent of correlation between different pairs of observed variables and requires SVD computations in step (a). 
	Then in step (b) a Minimum Spanning Tree (MST) is constructed over observable variables in parallel~\cite{bader2006fast} using the \emph{multivariate information distance}. The local groups are also obtained through MST so that they are available for the structure and parameter learning
step that follows.

The structure and parameter learning is done jointly through a divide-and-conquer strategy. Step-(c) illustrates the divide step (or local learning), where local structure and parameter estimation is performed. 
%This results in huge computational savings. Our learning procedure consists of  localized combinatorial computations (local recursive grouping) for structure learning along with  tensor decompositions for  parameter estimation. 
%No coordination is needed across different groups during this process.
It also performs the local merge to obtain group level structure and parameter estimates. 
 After the local structure and parameter learning is finished within the groups,
 we perform merge operations among groups, again guided by the Minimum Spanning Tree structure. For the structure estimation it consists of  a  union operation of sub-trees;  for the parameter estimation, it consists of  linear algebraic operations. Since our method is unsupervised, an alignment procedure of the hidden states is carried out which finalizes the global estimates of the tree structure and the parameters.

\section{Structure Learning}\label{sec:structure}
%The information distance was first introduced and its additivity proven in~\cite{lake1994reconstructing}.
Structure learning in graphical models involves finding the underlying Markov graph, given the observed samples. For latent tree models, structure can be estimated via distance based methods. This involves computing certain {\em information} distances between any pair of observed variables,  and then finding a tree which fits the computed distances.

\textbf{ Multivariate information distances: }
%\fhcomment{shorten this information distance section. }
We propose an additive distance for  multivariate linear latent tree models.
For a pair of (observed or hidden)  variables $y_a$ and $y_b$, consider the pairwise correlation matrix  $\mathbb{E}\left[y_a y_b^\top \right]$ (the expectation is over samples). %\jscomment{why is rank $k$? not clear to me what the expectation is over. this is a fundamental definition which should be formally introduced, maybe in the section 2}.\fhcomment{Defined it in section 2 and discussed why it is of rank k.}  
Note that its rank is $k$, dimension of the hidden variables. %Below, we consider an information distance based on its rank-$k$ SVD.
\begin{definition}
\label{def:info_dist}
The multivariate information distance between nodes $i$ and $j$ is defined as
\begin{equation}
\label{eqn:info_dist}
\dist(v_a,v_b) := -\log \frac{\prod\limits_{i=1}^{k}\sigma_i\left(\mathbb{E}(y_a y_b^\top)\right)}{\sqrt{\det(\mathbb{E}(y_a y_a^\top)) \det(\mathbb{E}(y_b y_b^\top))}}
\end{equation}
where
$\{\sigma_1(\cdot),\ldots,\sigma_k(\cdot)\} $ are the top $k$ singular values.
\end{definition}
Note that definition~\ref{def:info_dist} suggests that this multivariate information distance allows heterogeneous settings where the dimensions of $y_a$ and $y_b$ are different (and $\geq k$).

For latent tree models, we can find information distances which are provably {\em additive} on the underlying tree  in expectation, i.e.  the expected distance between any two nodes in the tree is the sum of distances along the path between them. %For instance, for the special case when the variables on the tree are (scalar) Gaussian, the negative logarithm of the correlation coefficient  corresponds to an additive information distance.
\begin{lemma}\label{lem:additive}
The multivariate information distance is additive on the tree $\mathcal{T}$, i.e., $\dist(v_a,v_c) $ $= \dist(v_a,v_b) + \dist(v_b,v_c)$, where $v_b$ is a node in the path from $v_a$ to $v_c$ and $v_a$,$v_b$,$v_c\in \mathcal{V}$.
\end{lemma}
Refer to Appendix~\ref{apdx:additive} for proof. The empirical distances can be computed via rank-$k$ SVD of the empirical pairwise moment matrix $\hat{\Ebb}[y_a y_b^\top]$ %\jscomment{what is the exact definition? this is related to the early comment on a similar definition}. 
Note that the distances for all the pairs can be computed in parallel.

\textbf{Formation of local groups via MST: }Once the empirical distances are computed, we construct a    Minimum Spanning Tree (MST), based on those distances. Note that the MST can be computed efficiently in parallel~\cite{vineet2009fast,website:Boruvka}. We now form groups of observed variables over which we carry out learning independently, without any coordination. These groups are obtained by the (closed) neigborhoods in the MST, i.e. an internal node and its one-hop neighbors form a group. The corresponding internal node is referred to as the {\em group leader}. See Figure~\ref{Fig:StructureLearning}(b).

\textbf{Local recursive grouping (LRG): }Once the groups are constructed via neighborhoods of MST,  we construct a sub-tree with hidden variables in each group (in parallel) using the recursive grouping introduced in~\cite{choi2011learning}.    
%\fhcomment{We refer to it as  the \emph{ local recursive grouping test}, and provide it in Procedure~\ref{algo:lrg}.}
The recursive grouping uses the multivariate information distances and decides the locations and numbers of hidden nodes. 
It proceeds by deciding which nodes are siblings, which proceeds as follows: consider two observed nodes $v_i,v_j$ which are siblings on the tree with a common parent $v_l$, and consider any other observed node $v_a$. 
From additivity of the (expected) information distances, we have $\dist(v_i,v_a)= \dist(v_i,v_l) + \dist(v_l,v_a)$ and similarly for $\dist(v_j,v_a)$. 
Thus, we have $\Phi(v_i,v_j;v_a):=\dist(v_i,v_a)- \dist(v_j,v_a)= \dist(v_i,v_l) - \dist(v_j,v_l)$, which is independent of node $v_a$. 
Thus, comparing the quantity $\Phi(v_i,v_j;v_a)$ for different nodes $v_a$ allows us to conclude that $v_i$ and $v_j$ are siblings. 
Once the siblings are inferred, the hidden nodes are introduced, and the same procedure repeats to construct the higher layers. Note that whenever we introduce a new hidden node $h_{\text{new}}$ as a parent, we need to estimate multivariate information distance between $h_{\text{new}}$ and nodes in active set $\Omega$. This is discussed in~\cite{choi2011learning} with details.

We will describe the LRG in details with integrated parameters estimation in Procudure~\ref{algo:plrg} in Section~\ref{sec:merging}.
In the end, we obtain a sub-tree over the local group of variables. After this \emph{local recursive grouping test}, we store the neighborhood relationship for the leader $v_i$ using an adjacency list $\Adj_i$. We call the resultant local structure as \emph{latent sub-tree}.

\section{Parameter Estimation}\label{sec:parameter}
Along with the structure learning, we adopt a moment-based spectral learning technique for parameter estimation. This is a guaranteed and fast approach to recover parameters via moment matching for third order moments of the observed data. In contrast, traditional approaches such as Expectation Maximization (EM) suffer from spurious local optima and cannot provably recover the parameters.

\textbf{A latent tree with three leaves:} We first consider an example of three observable leaves $x_1,x_2, x_3$ (i.e., a triplet) with a common hidden parent $h$. We then clarify how this can be generalized to learn the parameters of the latent tree model.
Let $\otimes$ denote for the tensor  product. For example, if $x_1, x_2, x_3 \in \mathbb{R}^{d}$, we have $x_1 \otimes x_2\otimes x_3 \in \mathbb{R}^{d \times d\times d}$.
\begin{property}[Tensor decomposition for triplets]  For a linear latent tree model with three observed nodes $v_1,v_2,v_3$ with joint hidden node $h$, we have
\begin{equation}\mathbb{E}(x_1\otimes x_2 \otimes x_3) = \sum_{r=1}^{k} \Pbb[h=e_r] A_{x_1|h}^r \otimes A_{x_2|h}^r \otimes A_{x_3|h}^r,\end{equation}
where $A_{x_i|h}^r =\mathbb{E}(x_i|h=e_r)$, i.e., $r^{\text{th}}$ column of the transition matrices from $h$ to $x_i$.
The tensor decomposition method of~\cite{anandkumar2012tensor} provably recovers the parameters $A_{x_i|h}$, $\forall i\in [3]$, and $\Pbb[h]$.\end{property}\label{lem:para_est_def}

\textbf{Tensor decomposition for learning latent tree models: }We employ the above approach for learning latent tree model parameters as follows: for every triplet of variables $y_a$, $y_b$, and $y_c$ (hidden or observed), we consider the hidden variable $h_i$ which is the joining point of $y_a,y_b$ and $y_c$ on the tree. They form a {\em triplet} model, for which we employ the tensor decomposition procedure.   However, it is wasteful to do it over all the triplets in the latent tree. 

In the next section, we demonstrate how we efficiently estimate the parameters as we learn the structure, and minimize the tensor decompositions required for estimation. Issues such as alignment of hidden labels across different decompositions will also be addressed.

\section{Integrated Structure and Parameter Estimation}
\label{sec:merging}
So far, we described high-level procedures of structure estimation through local recursive grouping (LRG) and parameter estimation through tensor decomposition over triplets of variables, respectively. We now describe an integrated and efficient approach which brings all these ingredients together. In addition, we provide merging steps to obtain a global model, using the sub-trees and parameters learnt over local groups.

\subsection{Local Recursive Grouping with Tensor Decomposition}
%We focus on a given group of nodes, as obtained from the one-hop neighborhood of the MST. 
%We integrate the tensor decomposition approach for parameter learning, described in Section~\ref{sec:parameter}, with local recursive grouping for structure learning, described in Section~\ref{sec:structure}. 
Next we present an integrated procedure where the parameter estimation goes hand-in-hand with structure estimation.
Intuitively, we find efficient groups of triplets to carry out tensor decomposition simultaneously, as we estimate the structure through recursive grouping. In recursive grouping, pairs of nodes are recursively grouped as siblings or as parent-child. As this process continues, we carry out tensor decompositions whenever there are siblings present as triplets. If there are only a pair of siblings, we find an observed node with closest distance to the pair. Once the tensor decompositions are carried out on the observed nodes, we proceed to structure and parameter estimation of the  added  hidden variables. The samples of the hidden variables can be obtained via the posterior distribution, which is learnt earlier through tensor decomposition. This allows us to predict information distances and third order moments among the hidden variables as process continues. The full algorithm is given in Procedure~\ref{algo:plrg}. 

 \begin{algorithm}[h]
 \caption{LRG with Parameter Estimation}
 \label{algo:plrg}
 \begin{algorithmic}[1]
% \REQUIRE Internal nodes $\mathcal{X}_\text{int}$ on MST.
% \FORALL {}
 \REQUIRE for each $v_i \in \mathcal{X}_\text{int}$, active set $\Omega := \Nb[v_i;\text{MST}]$.
  \ENSURE for each $v_i \in \mathcal{X}_\text{int}$, local sub-tree adjacency matrix $\Adj_i$, and $\mathbb{E}[y_a| y_b]$ for all $(v_a,v_b)\in \Adj_i$.
 	\STATE Active set $\Omega \leftarrow \Nb[v_i;\text{MST}]$
 	\WHILE{$\lvert \Omega\rvert> 2$}
 		\FORALL {$v_a,v_b\in \Omega$}
 		\IF {$\Phi (v_a, v_b; v_c) = \dist(v_a,v_b), \; \forall$ $v_c \in \Omega \backslash \{v_a, v_b\}$}
 		\STATE $v_a$ is a leaf node and $v_b$ is its parent,
 		\STATE Eliminate $v_a$ from $\Omega$.
 		\ENDIF
 		\IF{$-\dist(v_a,v_b) < \Phi (v_a, v_b; v_c) = \Phi (v_a, v_b; v_c^\prime) < \dist(v_a, v_b),\forall v_c, v_c^\prime \in \Omega \backslash \{v_a, v_b\}$}
 		\STATE $v_a$ and $v_b$ are siblings,eliminate $v_a$ and $v_b$ from $\Omega$, add $h_{\text{new}}$ to $\Omega$.
 		\STATE Introduce new hidden node $h_{\text{new}}$ as parent of $v_a$ and $v_b$.
 		\IF {more than 3 siblings under $h_{\text{new}}$}
 		\STATE find $v_c$ in siblings,
 		\ELSE
 		 \STATE find $v_c = \arg \min_{v_c\in \Omega} \dist(v_a,v_c)$.
 		\ENDIF
 		\STATE Estimate empirical third order moments $\widehat{\mathbb{E}}(y_a\otimes y_b\otimes y_c)$
 		\STATE Decompose $\widehat{\mathbb{E}}(y_a\otimes y_b\otimes y_c)$ to get $\Pr[h_{\text{new}}]$ and $\mathbb{E}(y_r|h_{\text{new}})$, $\forall r=\{a,b,c\}$.
 		\ENDIF
 		\ENDFOR
 		\ENDWHILE
 %\ENDFOR
 \end{algorithmic}
 \end{algorithm} %}

The divide-and-conquer local spectral parameter estimation is superior compared to popular EM-based method~\cite{choi2011learning}, which is slow and prone to local optima.
More importantly, EM can only be applied on a stable structure since it is a global update procedure.
Our proposed spectral learning method, in contrast, is applied locally over small groups of variables, and is a guaranteed learning with sufficient number of samples~\cite{anandkumar2012tensor}. Moreover, since we integrate structure and parameter learning, we avoid recomputing the same quantities, e.g. SVD computations are required both for structure estimation (for computing distances) and parameter estimation (for whitening the tensor). Combining these operations results in huge computational savings (see Section~\ref{sec:complexity} for the exact computational complexity of our method).
\begin{algorithm}[hbtp]
\caption{\small Merging and Alignment Correction (MAC)}\label{algo:pmac}
\begin{algorithmic}[1]
\REQUIRE \emph{Latent sub-trees} $\Adj_i$ for all internal nodes $i$.
%\REQUIRE Corresponding conditional means estimated in PLRG. 
%\ENSURE Structure and parameters of the entire latent tree graphical model.
\ENSURE Global latent tree $T$ structure and parameters.
\FOR {$\Adj_i$ and $\Adj_j$ in all the sub-trees}
\IF{there are common nodes between $\Adj_i$ and $\Adj_j$ }
\STATE Find the shortest path path$(v_i,v_j;\Adj_i)$ between $v_i$ and $v_j$ on $\Adj_i$ and path$(v_i,v_j;\Adj_j)$ in $\Adj_j$;
\STATE Union the only conflicting path$(v_i,v_j;\Adj_i)$ and path$(v_i,v_j;\Adj_j)$ according to equation~\eqref{eq:unionpath} ;
\STATE Attach other nodes in $\Adj_i$ and $\Adj_j$ to the union path; 
\STATE Perform alignment correction as described in Procedure~\ref{algo:alignment}.
\ENDIF
\ENDFOR
\end{algorithmic}
\end{algorithm}

\subsection{Merging and Alignment Correction}\label{sec:align}
We have so far learnt sub-trees  and  parameters over local groups of variables, where the groups are determined by the neighborhoods of the MST. The challenge now is to combine them to obtain a globally consistent estimate.
There are non-trivial obstacles to achieving this:
first, the constructed local sub-trees span  overlapping groups of observed nodes,  and possess conflicting paths. Second, local parameters need to be re-aligned as we merge the subtrees to obtain globally consistent estimates due to the nature of unsupervised learning. To be precise,  different tensor decompositions lead to permutation of the hidden labels (i.e. columns of the transition matrices) across triplets.  Thus, we need to find the permutation matrix correcting the alignment of  hidden states of the transition matrices, so as to guarantee global consistency.

\paragraph{Structure Union:} We now describe the procedure to merge the local structures. We merge them in pairs to obtain the final global latent tree. Recall that $\Adj_i$ denotes a sub-tree constructed locally over a group, whose leader is node $v_i$. Consider a pair of subtrees
 $\Adj_i$ and $\Adj_j$, whose group leaders $v_i$ and $v_j$ are neighbors on the MST. Since $v_i$ and $v_j$ are neighbors, both the sub-trees contain them, and have different paths between them (with hidden variables added).  Moreover, note that this is the only conflicting path in the two subtrees. We now describe how we can resolve this: in $\Adj_i$, let $h_1^i$ be the neighboring hidden node for $v_i$ and $h_2^i$ be the neighbor of $v_j$. There could be more hidden nodes between $h_1^i$ and $h_2^i$. Similarly, in $\Adj_i$, let $h_1^j$ and $h_2^j$ be the corresponding nodes in $\Adj_j$.  The shortest path between $v_i$ and $v_j$ in the two sub-trees are given as follows:
\begin{align}
\text{path}(v_i,v_j;\Adj_i) & := [v_i- h_1^i- \ldots- h^i_{2}-v_j]\\
\text{path}(v_i,v_j;\Adj_j) & := [v_i- h_1^j- \ldots- h^j_{2}-v_j]
\end{align}
Then the union path is formed as follows: 
\begin{align}
\label{eq:unionpath}
\text{merge}& (\text{path}(v_i,v_j;\Adj_i) ,\text{path}(v_i,v_j;\Adj_j)) \nonumber\\
& : = [v_i- h_1^i- \ldots- h^i_{2} - h^j_{1} \ldots h_2^j - v_j]
\end{align}In other words, we retain the immediate hidden neighbor of each group leader, and break the paths on the other end. For example in Figure~\ref{Fig:StructureLearning}(d1,d2), we have the path $v_3-h_1-v_5$ in $\Adj_3$ and path $v_3-h_3-h_2-v_5$ in $\Adj_5$. The resulting path is $v_3-h_1-h_3-h_2-v_5$, as see in Figure~\ref{Fig:StructureLearning}(e).
After the union of the conflicting paths, the other nodes are attached to the resultant latent tree. We present the pseudo code in Procedure~\ref{algo:pmac} in Appendix~\ref{appen:alignment}.

\begin{algorithm}[hbtp]
\caption[Parameter Alignment Correction]{Parameter Alignment Correction \\ \textbf{(}$\mathbb{G}_r$ denotes reference group, $\mathbb{G}_o$ denotes the list of other groups, each group has a reference node denoted as $\mathcal{R}_l$, and the reference node in $\mathbb{G}_r$ is $\mathcal{R}_g$. The details on alignment at line 8 is in Appendix~\ref{appen:alignment}.\textbf{)}}\label{algo:alignment}
\begin{algorithmic}[1]
\REQUIRE Triplets and unaligned parameters estimated for these triplets, denoted as $\text{Trip}(y_i,y_j,y_k)$.
\ENSURE Aligned parameters for the entire latent tree $T$.
\STATE Select $\mathbb{G}_r$ which has \emph{sufficient children};
\STATE Select refer node $\mathcal{R}_g $ in $\mathbb{G}_r$;
\FORALL {a, b  in $\mathbb{G}_r$}
\STATE  Align $\text{Trip}_{\text{in}}(y_a,y_b, \mathcal{R}_g)$;
\ENDFOR
\FORALL {$i_g$ in $\mathbb{G}_o$}
	\STATE Select refer node  $\mathcal{R}_l $ in $\mathbb{G}_o$[$i_g$];
	\STATE Align $\text{Trip}_{\text{out}}(\mathcal{R}_g,  y_a, \mathcal{R}_l)$ and $\text{Trip}_{\text{out}}(\mathcal{R}_l, y_i, \mathcal{R}_g)$;
	\FORALL {i, j in $\mathbb{G}_o$[$i_g$]}
		\STATE Align $\text{Trip}(y_i,y_j,\mathcal{R}_l)$;
	\ENDFOR		
\ENDFOR
\end{algorithmic}
\end{algorithm}

\textbf{Parameter Alignment Correction:} As mentioned before, our parameter estimation is unsupervised, and therefore, columns of the estimated transition matrices may be permuted for different triplets over which tensor decomposition is carried out. Note that the parameter estimation within the triplet is automatically acquired through the tensor decomposition technique, so that the alignment issue only arises across triplets. We refer to this as the alignment issue and it is required at various levels.

 There are two types of triplets, namely, \emph{in-group} and \emph{out-group} triplets. A triplet of nodes $\text{Trip}(y_i,y_j,y_l)$ is said to be \emph{in-group} (denoted by  $\text{Trip}_{\text{in}}(y_i,y_j,y_l)$ ) if
 its containing nodes share a joint node $h_k$ and there are no other hidden nodes in path($y_i$, $h_k$), path($y_j$, $h_k$) or path($y_l$, $h_k$).
 Otherwise, this triplet is \emph{out-group} denoted by $\text{Trip}_{\text{out}}(y_i,y_j,y_l)$.
 We define a group as \emph{sufficient children} group if it contains at least three \emph{in-group} nodes.

Designing an \emph{in-group} alignment correction with \emph{sufficient children} is relatively simple: we achieve this by including a local reference node for all the \emph{in-group} triplets. Thus, all the triplets are aligned with the reference node.  The alignment correction is more challenging if lacking \emph{sufficient children}. We propose \emph{out-group} alignment to solve this problem. We first assign one group as a \emph{reference group}, and the \emph{local reference node} in that \emph{reference group} becomes the \emph{global reference node}. In this way, we align all recovered transition matrices in the same order of hidden states as in the reference node. Overall, we merge the local structures and align the parameters from LRG local sub-trees using Procedure~\ref{algo:pmac} and~\ref{algo:alignment}.

\section{Theoretical Gaurantees}\label{sec:complexity}

\textbf{Correctness of Proposed Parallel Algorithm: }We now provide the main result of this chapter on global consistency for our method, despite the high degree of parallelism.

\begin{theorem}\label{theorem:main_LT}
Given samples from an identifiable latent tree model, the proposed method consistently recovers the structure with $O(\log p)$ sample complexity and parameters with $O(\poly p)$ sample complexity.%, where $p$ is the number of nodes.
\end{theorem}
The proof sketch is in Appendix~\ref{appen:guarantee}.

\textbf{Computational Complexity: }
We recall some notations here: $d$ is the observable node dimension, $k$ is the hidden node dimension ($k \ll d$), $N$ is the number of samples, $p$ is the number of observable nodes, and $z$ is the number of non-zero elements in each sample.

Let $\Gamma$ denote the maximum size of the groups, over which we operate the local recursive grouping procedure.  Thus, $\Gamma$ affects the degree of parallelism for our method. Recall that it is given by the neighborhoods on MST, i.e., $\Gamma : = \max_{i} \lvert\Nb[i;\text{MST}]\rvert$.   Below, we provide a bound on $\Gamma$.
\begin{lemma}\label{lem:MST}
The maximum size of neighborhoods on MST, denoted as $\Gamma$, satisfies
\begin{equation}
\Gamma \le \Delta^{1+\frac{u_{d}}{l_{d}}\delta},
\end{equation}where $\delta := \max_{i} \{\min_{j} \{\text{path}(v_i,v_j;\mathcal{T}) \} \} $ is the effective depth, $\Delta$ is the maximum degree of $\mathcal{T}$, and the $u_{d}$ and $l_{d}$ are the upper and lower bound of information distances between neighbors on $\mathcal{T}$.
\end{lemma}
Thus, we see that for many natural cases, where the degree and the depth in the latent tree are bounded (e.g. the hidden Markov model), and the parameters are mostly homogeneous (i.e., $u_d/l_d$ is small), the group sizes are bounded, leading to a high degree of parallelism.

We summarize the computational complexity in Table~\ref{tab:computational_complexity}. Details can be found in Appendix~\ref{appen:compuComplex}. %Overall we see that our algorithm is bulk-asynchronous parallel.
\begin{table}[htbp]
   \centering
\begin{tabular}{@{} l|l|l@{}}
\hline
Algorithm Steps & Time per worker & Degree of parallelism\\
\hline
\hline
Distance Est. &  $O( N z +  d + k^3 )$ & $O(p^2)$\\ %Information Distance Estimation
MST & $O(\log p)$ & $O(p^2)$\\ %Structure: Minimum Spanning Tree
LRG & $O(\Gamma^3)$ & $O(p/ \Gamma)$\\ %Structure: Local Recursive Grouping
Tensor Decomp. & $O(\Gamma k^3 + \Gamma d k^2)$ & $O(p/ \Gamma)$\\ %Parameter: Tensor Decomposition
Merging step& $O(d k^2)$ & $O(p/ \Gamma)$\\ %Merging and Alignment Correction
\hline
\end{tabular}
\caption[Worst-case computational complexity of our algorithm]{ Worst-case computational complexity of our algorithm.  The total complexity is the product of the time per work and degree of parallelism.}%The times refer to the worst-case complexity. According to Lemma~\ref{lem:MST}, $\Gamma$ is small if $\Delta$ and $\delta$ are small. }
\label{tab:computational_complexity}
\end{table}

\section{Experiments}
\label{sec:implementation}
%See Appendix~\ref{appen:synthetic} for synthetic experiments. 
\textbf{Setup }Experiments are conducted on a server running the Red Hat Enterprise 6.6 with 64 AMD Opteron processors and 265 GBRAM. The program is written in C++, coupled with the multi-threading capabilities of the OpenMP environment~\cite{OMP} (version 1.8.1).
We use the Eigen toolkit\footnote{\scriptsize{\url{http://eigen.tuxfamily.org/index.php?title=Main_Page}}} where BLAS operations are incorporated. For SVDs
of large  matrices, we use randomized projection methods~\cite{gittens2013revisiting} as described in Appendix~\ref{apdx:svd}.

\textbf{Healthcare data analysis }The goal of our analysis is to discover a disease hierarchy based on their co-occurring relationships in the patient records. In general, longitudinal patient records store the diagnosed diseases on patients over time, where the diseases are encoded with International Classification of Diseases (ICD) code. 

\textbf{Data description } We used two large patient datasets of different sizes with respect to the number of samples, variables and dimensionality.

\emph{(1) MIMIC2:} The MIMIC2 dataset record disease history of 29,862 patients where a overall of 314,647 diagnostic events over time representing 5675 diseases are logged. We consider patients as samples and groups of diseases as variables. We analyze and compare the results by varying the group size (therefore varying $d$ and $p$).

\emph{(2) CMS:} The CMS dataset includes 1.6 million patients, for whom 15.8 million medical encounter events are logged. Across all events, 11,434 distinct diseases (represented by ICD codes) are logged.  We consider patients as samples and groups of diseases as variables. We consider specific diseases within each group as dimensions. We analyze and compare the results by varying the group size (therefore varying $d$ and $p$). 
While the MIMIC2 dataset and CMS dataset both contain logged diagnostic events, the larger volume of data in CMS provides an opportunity for testing the algorithm's scalability. We qualitatively evaluate biological implications on MIMIC2 and quantitatively evaluate algorithm performance and scalability on CMS.

To learn the disease hierarchy from data,  we also leverage some existing domain knowledge about diseases. In particular, we use an existing mapping between 
ICD codes and higher-level  Phenome-wide Association Study (PheWAS) codes~\cite{Denny:pheWAS}. We use (about 200) PheWAS codes as observed nodes and the observed node dimension is set to be binary ($d=2$) or the maximum number of ICD codes within a pheWAS code ($d=31$). 
The goal is to learn the latent nodes and the disease hierarchy and associated parameters from data.

\begin{figure*}
\tikzstyle{every node}=[font=\scriptsize]
  \begin{minipage}[t]{0.3\textwidth}
% Minimal pgfplots version: 1.3
%  http://www.mathworks.com/matlabcentral/fileexchange/22022-matlab2tikz
\begin{tikzpicture}

\begin{axis}[%
width=1.42638in,
height=1.125in,
at={(0in,0in)},
scale only axis,
separate axis lines,
every outer x axis line/.append style={black},
every x tick label/.append style={font=\color{black}},
xmin=0,
xmax=1800000,
xlabel={Number of samples},
every outer y axis line/.append style={black},
every y tick label/.append style={font=\color{black}},
ymin=0,
ymax=9000,
ylabel={Running time (seconds)},
title={(a) Running time vs Number of samples},
yticklabel style={font=\footnotesize}
]
\addplot [color=red,solid,mark=asterisk,mark options={solid},forget plot]
  table[row sep=crcr]{%
410964	3676\\
821928	4524\\
1232892	8148\\
1643857	8965\\
};
\end{axis}
\end{tikzpicture}%
  \end{minipage}%
    \hspace{0.5cm}
  \begin{minipage}[t]{0.3\textwidth}
% Minimal pgfplots version: 1.3
%  http://www.mathworks.com/matlabcentral/fileexchange/22022-matlab2tikz
\begin{tikzpicture}

\begin{axis}[%
width=1.42638in,
height=1.125in,
at={(0in,0in)},
scale only axis,
separate axis lines,
every outer x axis line/.append style={black},
every x tick label/.append style={font=\color{black}},
xmin=0,
xmax=1000,
xlabel={Number of observed nodes},
every outer y axis line/.append style={black},
every y tick label/.append style={font=\color{black}},
ymin=0,
ymax=18000,
ylabel={Running time (seconds)},
title={(b) Running time vs Number of nodes},
yticklabel style={font=\footnotesize}
]
\addplot [color=red,solid,mark=asterisk,mark options={solid},forget plot]
  table[row sep=crcr]{%
238	2491\\
476	6792\\
714	11272\\
952	16291\\
};
\end{axis}
\end{tikzpicture}%
  \end{minipage}%
    \hspace{0.5cm}
    \begin{minipage}[t]{0.3\textwidth}
% Minimal pgfplots version: 1.3
%  http://www.mathworks.com/matlabcentral/fileexchange/22022-matlab2tikz
\begin{tikzpicture}

\begin{axis}[%
width=1.42638in,
height=1.125in,
at={(0in,0in)},
scale only axis,
separate axis lines,
every outer x axis line/.append style={black},
every x tick label/.append style={font=\color{black}},
xmin=0,
xmax=70,
xlabel={Number of threads},
every outer y axis line/.append style={black},
every y tick label/.append style={font=\color{black}},
ymin=0,
ymax=70,
ylabel={Speed-up factor},
title={(c) Speed-up vs available threads},
legend style={at={(0.03,0.97)},anchor=north west,legend cell align=left,align=left,fill=none,draw=none},
yticklabel style={font=\footnotesize}
]
\addplot [color=red,solid,mark=asterisk,mark options={solid}]
  table[row sep=crcr]{%
1	1\\
8	7.9\\
16	14.7565314862185\\
32	26.4474321331067\\
64	45.2037330515936\\
};
\addlegendentry{Method speed-up};

\addplot [color=blue,dotted,mark=o,mark options={solid}]
  table[row sep=crcr]{%
1	1\\
8	8\\
16	16\\
32	32\\
64	64\\
};
\addlegendentry{Ideal speed-up};

\end{axis}
\end{tikzpicture}%
  \end{minipage}%
\caption[Running time]{\textbf{(a)} CMS dataset sub-sampling w.r.t.~varying number
of samples. \textbf{(b)} MIMIC2 dataset sub-sampling w.r.t.~varying number of observed nodes.
Each one of the observed nodes is binary ($d = 2$).
\textbf{(c)} MIMIC2 dataset: Scaling w.r.t. varying computational power, establishing the
scalability of our method even in the large $p$ regime.
The number of observed nodes is $1083$ and each
one of them is binary ($p = 1083, d = 2$).}
\label{Fig: scale}
\end{figure*}
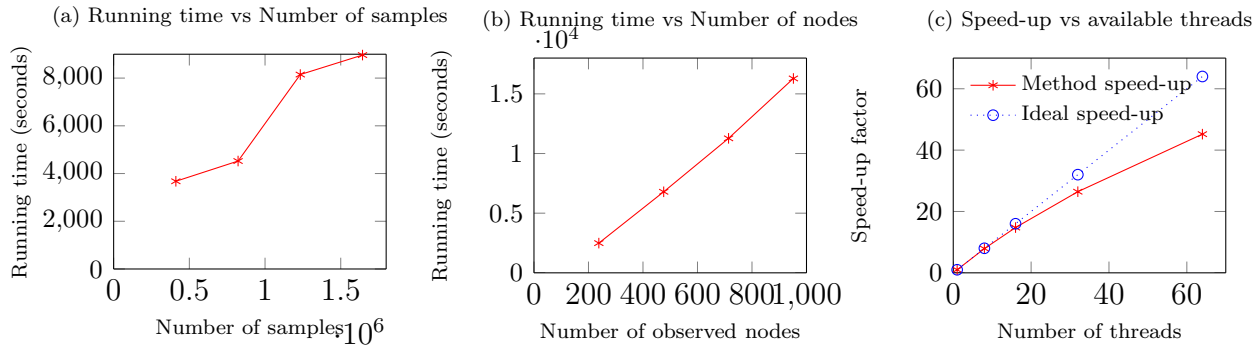

\subsection{Validation}
We conduct both quantitative and qualitative validation of the resulting disease hierarchy. 

%Qualitatively, we display our recovered tree structures for both the datasets with varying $d$. 
\textbf{Quantitative Analysis } We first compare our resulting hierarchy with a ground truth tree based on medical knowledge\footnote{The ground truth tree is the PheWAS hierarchy provided in the clinical study~\cite{Denny:pheWAS}}.
%The recovered tree structure was compared to the ground truth tree, representing memberships of diseases into disease groups as used in the PheWAS encoding, which groups similar diseases by similarity~\cite{Denny:pheWAS}.  
The standard Robinson Foulds (RF) metric~\cite{robinson1981comparison}(between our estimated latent tree and the ground truth tree) is computed to evaluate the structure recovery in Table~\ref{tab:RFmetric}. The smaller the metric is, the better the recovered tree is.  We also compare our results with a baseline: the agglomerative clustering. 
The proposed method are slightly better than the baseline and the advantage is increased with more nodes. 
However, the proposed method provides an efficient probabilistic graphical model that can support general inference which is beyond the baseline.
\begin{table}[h]
\centering
\begin{tabular}{  c | c | c |c}
\hline
Data 		& $p$ 	&    RF(agglo.) 		& RF(proposed) \\
\hline\hline
MIMIC2 	& 	163 & 0.0061	&  0.0061\\
CMS   	&	168 & 0.0060	&  0.0059\\
MIMIC2 	& 	952 & 0.0060	&  0.0011 \\
\hline
\end{tabular}
\caption[Robinson Foulds (RF) metric ]{Robinson Foulds (RF) metric compared with the ``ground-truth'' tree for both MIMIC2 and CMS dataset. Our proposed results are better as we increase the number of nodes. } \label{tab:RFmetric}
\end{table}

\textbf{Qualitative analysis } The qualitative analysis is done by a senior MD-PhD student in our team.

{(a) Case d=2: } Here we report the results from the 2-dimensional case (i.e., observed variable is binary). 
\begin{figure}[htbp]
\includegraphics[width=\textwidth]{\fighomeLT/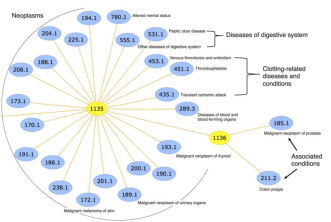}
\caption[Subtree 1 of estimated human disease hierarchy]{An example of two subtrees which represent groups of similar diseases which may commonly co-occur. Nodes colored yellow are latent nodes from learned subtrees.} \label{Fig:tree_mimic2_1}
\end{figure}
In figure \ref{Fig:tree_mimic2_1}, we show a portion of the learned tree using the MIMIC2 healthcare data. The yellow nodes are latent nodes from the learned subtrees while the blue nodes represent observed nodes(diagnosis codes) in the original dataset.  Diagnoses that are similar were generally grouped together. For example, many neoplastic diseases were grouped under the same latent node (node 1135). While some dissimilar diseases were grouped together, there usually exists a known or plausible association of the diseases in the clinical setting. For example, in figure \ref{Fig:tree_mimic2_1}, clotting-related diseases and altered mental status were grouped under the same latent node as several neoplasms. This may reflect the fact that altered mental status and clotting conditions such as thrombophlebitis can occur as complications of neoplastic diseases~\cite{Falanga:clotCancer}. The association of malignant neoplasms of prostate and colon polyps, two common cancers in males, is captured under latent node 1136~\cite{us2014united}. 

{(b) Case d =31: }
We also learn a tree from the MIMIC2 dataset, in which we grouped diseases into 163 pheWAS codes and up to 31 dimensions per variable. % \rccomment{should we put the value of K used?} \fhcomment{That's fine as we say the results are robust with respect to k}.
Figure \ref{Fig:tree_mimic2_2} shows a portion of the learned tree of four subtrees which all reflect similar diseases relating to trauma. A majority of the learned subtrees reflected clinically meaningful concepts, in that related and commonly co-occurring diseases tended to group together in the same subtrees or in nearby subtrees. 
\begin{figure}[]
\includegraphics[width=\textwidth]{\fighomeLT/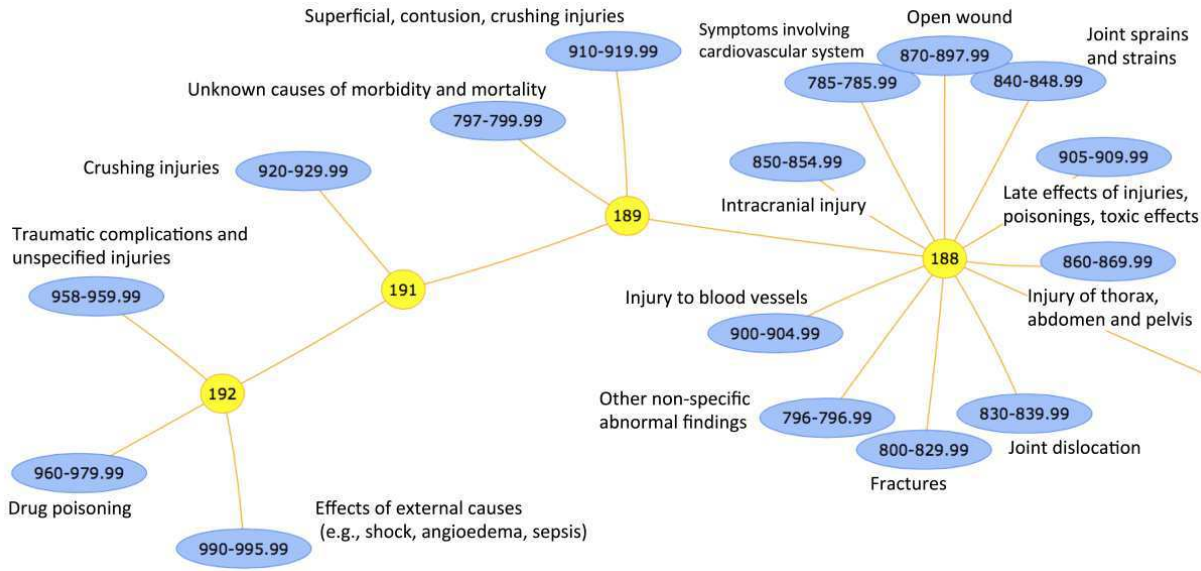}
\caption[Subtree 1 of estimated human disease hierarchy]{ An example of four subtrees which represent groups of similar diseases which may commonly co-occur. Most variables in this subtree are related to trauma.} \label{Fig:tree_mimic2_2}
\end{figure}
We also learn the disease tree from the larger CMS dataset, in which we group diseases into 168 variables and up to 31 dimensions per variable. Similar to the case from the MIMIC2 dataset, a majority of learned subtrees reflected clinically meaningful concepts. 

For both the MIMIC2 and CMS datasets, we performed a qualitative comparison of the resulting trees while varying the hidden dimension $k$ for the algorithm. The resulting trees for different values of $k$ did not exhibit significant differences. This implies that our algorithm is robust with different choices of hidden dimensions. The estimated model parameters are also robust  for different values of $k$ based on the results.
%\rccomment{we can say this (preceding) sentence regarding the parameter estimation. it shows that we at least thought about the concept and even though theres no quantitative analysis, we still did qualitative one}

\textbf{Scalability } Our algorithm is scalable w.r.t.~varying characteristics
of the input data. First, it can handle
a large number of patients efficiently,
as shown in Figure~\ref{Fig: scale}(a). % \jpcomment{To be completed}
%Our algorithm is scalable w.r.t.~varying characteristics of the input data. 
It has also a linear scaling behavior as we vary the number observed
nodes, as shown in Figure~\ref{Fig: scale}(b).
Furthermore, even in cases where the number of observed
variables is large, our method maintains an almost linear scale-up as we
vary the computational power available, as shown in Figure~\ref{Fig: scale}(c).
As such, by providing the respective resources, our algorithm is practical
under any variation of the input data characteristics.

\section{Conclusion}\label{section:Conclusion}

We present an integrated approach to structure and parameter estimation in latent tree models. 
Our method overcomes challenges such as uncertainty of location and number of hidden variables, problem of local optima with no consistency guarantees, difficulty in scalability with respect to number of variables. The proposed algorithm is ideal for parallel computing and highly scalable. We successfully applied the algorithm to a real application for disease hierarchy discovery using large patient data for 1.6m patients.

%\acks{The first author is supported by NSF BIGDATA IIS-1251267, the second author is supported in part by UCI graduate fellowship and NSF Award CCF-1219234, and the last author is supported in part by Microsoft Faculty Fellowship,  NSF Career award CCF-1254106, NSF Award CCF-1219234,  and ARO YIP Award W911NF-13-1-0084. The authors acknowledge insightful discussions with Prem Gopalan, David Mimno, David Blei, Qirong Ho, Eric Xing, Carter Butts, Blake Foster, Rui Wang, Sridhar Mahadevan, and the CULA team. Special thanks to Prem Gopalan and David Mimno for providing the variational code and answering all our questions. The authors also thank Daniel Hsu and Sham Kakade for initial discussions regarding the implementation of the tensor method. We also thank Dan Melzer for helping us with the system-related issues.}

\chapter{Discovering Cell Types with Spatial Point Process Mixture Model}\label{chapter:brain}

Cataloging the neuronal cell types that comprise circuitry of individual brain regions is a major goal of modern neuroscience and the BRAIN initiative. 
Single-cell RNA sequencing can now be used to measure the gene expression profiles of individual neurons and to categorize neurons based on their gene expression profiles. While the single-cell techniques are extremely powerful and hold great promise, they are currently still labor intensive, have a high cost per cell, and, most importantly, do not provide information on spatial distribution of cell types in specific regions of the brain. We propose a complementary approach that uses computational methods to infer the cell types and their gene expression profiles through analysis of brain-wide single-cell resolution in situ hybridization (ISH) imagery contained in the Allen Brain Atlas (ABA). We measure the spatial distribution of neurons labeled in the ISH image for each gene and model it as a spatial point process mixture, whose mixture weights are given by the cell types which express that gene. By fitting a point process mixture model jointly to the ISH images, we infer both the spatial point process distribution for each cell type and their gene expression profile. We validate our predictions of cell type-specific gene expression profiles using single cell RNA sequencing data, recently published for the mouse somatosensory cortex. Jointly with the gene expression profiles, cell features such as cell size, orientation, intensity and local density level are inferred per cell type.  This work brings together the techniques used in all previous chapters, such as image processing to extract cells and cell features from brain slices, learning a point process admixture model. 
\section{Introduction}
	\subsection{Motivations and Goals}
	The human brain comprises about one hundred billion neurons and one trillion supporting glial cells. These cells are specialized into a surprising diversity of cell types. The retina alone boasts well over 50 cell types, and it is an active area of research to perform a census of the various neuronal cell types that comprise the central nervous system. Many criteria have been used to categorize neuronal cell types, from neuronal morphology and connectivity to their functional response properties. Neurons can also be categorized based on the proteins they make. Immunohistochemistry has been used with great success for many decades to differentiate excitatory neurons from inhibitory neurons by labeling for known proteins involved in the synthesis and regulation of glutamate and GABA, the primary excitatory and inhibitory neurotransmitters respectively.
	
	More recently, there has been an effort to systematically measure the complete transcriptome of single neurons. Single-cell RNA sequencing (RNA-Seq) is an extremely powerful technique that can quantitatively determine the expression level of every gene that is expressed in individual neurons. This so-called transcriptome or gene expression / transcription profile can then be used
to define cell types by clustering. A recent study produced the most comprehensive census of cell types to date in the mouse somatosensory cortex and hippocampus by performing single-cell RNA-Seq on over 3000 neurons~\cite{zeisel2015cell}. While this study is quite exciting, tyring to replicate it for all brain regions might well require the equivalent of a thousand such experiments.  Thus, it is likely that the unprecedented insights that RNA-Seq can provide will be slow to arrive. More importantly, single cell sequencing methods are not currently able to capture the precise three-dimensional location of the individual neurons.

	Here we propose a complementary approach that uses computational strategies to identify cell types and their spatial distribution by re-analysing data published by the Allen Institute for Brain Research.  The Allen Brain Atlas (ABA) contains cellular resolution brain-wide in-situ hybridization (ISH) images for 20,000 genes\footnote{ Although the Atlas contains ISH data for approximately 20,000 distinct mouse genes, we focus on the top 1743 reliable genes whose sagittal and coronal experiments are highly correlated.}. ISH is a histological technique that labels the mRNA in all cells expressing the corresponding gene in a manner roughly proportion to the gene expression level. An example of an ISH image can be seen in figure~\ref{fig:overview}(a).

	The ABA contains genome-wide and brain-wide ISH images of the adult mouse brain. These images were generated by slicing the brain into a series of \SI{25}{\micro\metre} thin sections and performing ISH. Image series of ISH performed for different genes come from different mouse brains, since ISH can only be performed for one gene at a time. The ISH image series for different genes were then computational aligned into a common reference brain coordinate system. Such data have been productively used to infer the average transcriptomes corresponding to different brain regions.

	It is commonly thought that the ABA cannot be used to infer the transcriptomes of individual cells in a given brain region since mouse brains cannot be aligned to the precision of a single cell. This is because there is individual variation in the precise number and location of neurons from brain to brain. However, we expect that the average number and spatial distribution of neurons from each cell type to be conserved from brain to brain, for a given brain area. More concretely, we might expect that parvalbumin-expressing (PV) inhibitory interneurons in layer 2/3 of the mouse somatosensory cortex comprise approximately 7\% of all neurons and have a conserved spatial and size distribution from brain to brain. We use this fact to derive a method for simultaneously inferring the cell types in a given brain region and their gene expression profiles from the ABA.
	
	We propose to model the spatial distribution of neurons in a brain as being generated by sampling from an unknown but consistent brain-region and cell-type dependent spatial point process distribution. And since each gene might only be expressed in a subset of cell types, an ISH image for a single gene can be thought of as a mixture of spatial point processes where the mixture weights represent the individual cell types expressing that gene. We infer cell types, their gene expression profiles and their spatial distribution by unmixing the spatial point processes corresponding to the ISH images for 1743 genes. This is in notable contrast to the information provided by single-cell RNA sequencing which can only measure the gene expression profile of individual cells to high accuracy but where, due to the destructive measurement process, all information about the spatial position and distribution of cell types is lost.

\subsection{Previous Work}
Allen Brain Atlas (ABA) \cite{lein2007genome} is a landmark study which mapped the gene expression of about 20,000 genes across the entire mouse brain. The ABA dataset consists of cellular high-resolution 2d imagery of \emph{in-situ} hybridized series of brain sections, digitally aligned to a common reference atlas. However, since the \emph{in-situ} images for each gene come from different mouse brains and since there is significant variability in the individual locations of labeled cells, it is not possible to register brain-wide gene expression at a resolution higher than about $250\mu m$. Therefore, the cellular resolution detail was down-sampled to construct a coarser 3d representation of the average gene expression level in $250\mu m\times 250\mu m\times 250\mu m$ voxels.

The coarse-resolution averaged gene expression representation has been widely used and analyzed to understand differences in gene expression at the level of brain region. Hawrylycz et al \cite{hawrylycz2011multi} analyzed the correlational structure of gene expression at this scale, across the entire mouse brain. However, due to the poor resolution of the average gene expression representation, it has proven challenging to use the ABA to discover the microstructure of gene expression within a brain region.  To address this issue from a complementary perspective, Grange et al~\cite{grange2014cell} used the gene expression profiles of 64 known cell-types, combined with linear unmixing to determine the spatial distribution of these known cell-types. However, such an approach can be confounded by the presence of cell-types whose expression profiles have yet to be characterized, and limited by the resolution of the averaged gene expression representation.

In contrast to previous approaches, we aim to solve the difficult problem of automatically discovering the gene expression profiles of cell-types within a brain region by analyzing the original cellular resolution ISH imagery. 
We propose to use the spatial distributions of labeled cells, and their shapes and sizes, which are a far richer representation than simply the average expression level in $250\mu m\times 250\mu m\times 250\mu m$ voxels. This spatial point process is then un-mixed to determine the gene expression profile of cell types.

Most previous work on unmixing point process mixtures adopted parametric generative models where the point process is limited to some distribution family such as Poisson or Gaussian~\cite{ji2009spatial,kottas2007bayesian}. However, since we are not interested in building a generative model of a point process, but rather care more about inferring the mixing proportions (gene expression profile), we take a simpler parameter-free approach. This approach models only the statistics of the point process, but is not a generative model, and so cannot be use to model individual points/cells.

\input{figure_flow}

\section[Modeling Cell-types Using Spatial Point Process Features]{Modeling the Spatial Distribution of Cell-types Using Spatial Point Process Features}
Most analyses of the ABA \emph{in situ} hybridization dataset have utilized a simple measure of average expression level in relatively large $250\mu m \times 250\mu m \times 250\mu m$ voxels of brain tissue. Due to the large volume over which the expression level is averaged, such a representation cannot distinguish between  large numbers of cells expressing small amounts of RNA vs. small numbers of cells expressing large amounts of RNA. All information about the spatial organization of labeled cells, their shapes, sizes and spatial density are lost and summarized by a single scalar number. Here, we describe a more sophisticated representation of the labeled cells in an ISH image based on marked spatial point processes.
\subsection{The Marked Spatial Point Process Representation of ISH Images}
Our approach requires processing the high-resolution ISH images to detect individual labeled cells and their visual characteristics. We developed a cell detection algorithm described in the Supplementary section. Our algorithm additionally also estimates the expression level of each detected cell, its shape, size and orientation. Figure~\ref{fig:overview}(a) and Figure~\ref{fig:overview}(b) illustrate the results of our cell detection algorithm.

Since cell-types differ not only in terms of gene expression pattern, but also display a diversity of shapes, sizes and spatial densities, we sought to characterize these properties. We measured: (1) {\bf cell size} $s=[r_1,r_2]$: the radius in two principal directions of an ellipse fit to each cell; (2) {\bf cell orientation} $o$: the orientation of the first principle axis of the ellipse; (3) {\bf gene intensity level} $p$: intensity of labeling of a cell relative to the image background; (4) {\bf spatial distribution} $c$: the number of cells within a local area centered around the cell, which can be regarded as a measure of the local cell density.

The collection of detected cells within an atlas-defined brain region, along with their features, constitutes a marked spatial point process. This point process is considered ``marked'', because each point is characterized by the shape, size, expression level and local density features, in addition to just their location in space.
\subsection[Representing Spatial Point Processes Using Joint Feature Histograms]{A Model-free Approach to Representing Spatial Point Processes Using Joint Feature Histograms}
The statistical modeling of repulsive spatial point processes such as those that arise in biology is non-trivial, and many generative models such as determinantal point processes~\cite{kulesza2012determinantal}and Matern point processes have high computational complexity. But since we are not interested in directly modeling the individual labeled cells, but instead in modeling only their aggregate spatial statistics, and in inferring their gene expression profiles, we can take a simpler approach.

We use a \emph{joint histogram} simple statistics of the collection of detected cells to characterize the underlying point process from which they are drawn. This is an empirical moment approach which side-steps the need to carefully define a generative point process distribution.

As we describe in the next section, we propose to model the point process measured from the ISH image for each gene as a mixture of point processes belonging to individual cell-types. For this, we use a linear mixing model, the Latent Dirichlet Allocation model. The use of this model is greatly simplified if we carefully choose our feature representation such that the linear mixture of point processes results in a linear mixture of histogram statistics. This is clearly the case for the features we have chosen. For instance, if we sample equally from two point process distributions $P_1$ and $P_2$ with average densities of $d_1$ and $d_2$, the addition of these two point processes $P = P_1 + P_2$ results in the addition of the two densities $d = d_1 + d_2$. This is not the case for second order features, such as the distances to the nearest neighbors, which would have a more nonlinear relationship.

In figure~\ref{fig:overview}(c), we display marginal histograms corresponding to the joint histogram for two genes, Pvalb and Rasgrf2, which are well-known markers for a specific class of inhibitory and excitatory cortical neuronal cell-types respectively. %After observing the overall dynamic range of the features for every gene, we design the cell size histogram bins ranging from $s = [2,2]$, $s = [2,3]$ to $s = [8,9]$, $s = [9,9]$ \footnote{As we model each cell as a ellipse, a two dimensional vector describes the diameters in two principal directions.}, the orientation histogram bins ranging from $o = 0$ to $o= \pi$, the gene expression profiles from $p=0$ to $p=2$ and the number of cells within \SI{21}{\micro\metre} ranging from $c=0$ to $c=10$. The joint histogram feature is 4-d histogram on $ [s,o,p,c]$. 

\section{Un-mixing Spatial Point Processes to Discover Cell-types}
\subsection{Generative Model: A Variation of Latent Dirichlet Allocation }
The spatial point process histogram representation of the ABA ISH dataset results, for each brain region, is an $N_F\times N_G$ matrix $[x^m_n]$, where $N_F$ is the total number of histogram bins (henceforward called the number of histogram features) \footnote{Note that there are two types of \emph{features} -- the features characterizing each detected cell, and the features characterizing the collection of detected cells that constitute a single sample from a spatial point process}, $N_G$ is the number of genes, and $x^m_n$ is the number of cells expressing gene $n$ in histogram bin $m$.

We model the gene-spatial histogram matrix $[x^m_n]$ by assuming it is generated by a Variation of Latent Dirichlet Allocation (vLDA)~\cite{blei2003latent} model of cell types. This matrix factorization based latent variable model assumes that the ISH histograms are generated from a small number of cell-types, $K$, and each cell-type $i$ is associated with a type-dependent spatial point process histogram $h_i$ and a gene expression profile $\beta_i$.

Our generative model for each histogram bin $m$ (characterizing a particular bin in the size/ orientation/ gene profile/ spatial distribution) is as follows: Let $L^m=\sum_n^{N_G} x^m_n$ be the detected number of cells in the joint histogram bin $m$. For each cell $l$ in this bin, its cell-type $t$ is sampled from the multinomial distribution $h^m$. And given the cell-type $t$ of cell $l$, the genes $n$ expressed by this cell are sampled from a multinomial distribution given by the type-dependent gene expression profile/distribution $\beta^t$. For a given gene $n$ and histogram bin $m$, this generative process determines the number of cells that would be detected $x^m_n$.

We further place a Dirichlet prior over $h^m \sim \Dir(\alpha)$, with the concentration parameter $\alpha$ which determines the prior probability over the number of cell-types present in a given histogram bin $m$. This prior represents our prior knowledge of how many cell-types express each gene, and also how well our feature representation separates cells of different types into different histogram bins. In principle, we could generalize this to be a gene-specific prior, if we had such information available. We could also use $\alpha$ to incorporate information about our prior knowledge over the distribution of cells from each cell-type, for instance that excitatory neurons greatly outnumber inhibitory neurons in a roughly $5:1$ ratio. %For now, we use a uniform but sparse prior over cell-types and their expected number expressing each gene by using a symmetric Dirichlet prior with uniform $\alpha_1=\alpha_2=\ldots=\alpha_K$ for all cell types.

%Figure~\ref{fig:LDAmoel} shows a pictorial explanation of this LDA generative model.
We now describe how we estimate the model parameters -- the cell-type specific multinomial gene expression profile $\beta$ and the cell-type specific spatial point process histogram $h$ from the gene-specific spatial point process histograms measured from the ISH images.
\subsection{Estimating the Cell-type Dependent Gene Expression Profile $\beta$}
After testing several estimation methods for the parameters of our model, we found that non-negative matrix factorization (NMF) performed well in estimating the cell-type specific gene expression profiles $\beta$, see Figure~\ref{fig:synthetic}. We solve the following optimization problem:
\begin{equation}
\min_{\beta,h} \quad \sum^{N_F}_m \sum^{N_G}_n (x^m_n - \sum^K_t h^m_t \beta^t_n L^m)^2, \quad
s.t. \quad \beta^t_n\ge 0 , \; \sum^{N_G}_n \beta^t_n=1, \; h^m_t \ge 0, \; \sum_t^K h^m_t=1
\end{equation}
Here, the non-negativity and sum-to-one constraints on $h^m_t$ and $\beta^t_n$ ensure that $h$ and $\beta$ result in properly normalized multinomial distributions. While this estimation procedure results in joint estimates for $h$ and $\beta$, it does not enforce the Dirichlet prior over $h$. So we refine our NMF-derived estimates for $h$ using variational inference \cite{blei2003latent}.
\vspace{-0.5em}
\subsection{Estimating the Cell-type Dependent Spatial Point Process Histogram $h$}
\vspace{-0.5em}
We use a standard maximum likelihood estimation procedure for $h$ \cite{blei2003latent}. Iteratively, we refine the  inference of the cell type membership $h^m \in \Delta_{k}$ under each joint histogram feature $m$. We update $h^{m}_i$ until convergence~\cite{smola2010architecture}.
\begin{equation}\label{eq:inferrence}
h^{m}_i \leftarrow \frac{1}{L^{m} + \sum_t^K\alpha_t} \sum\limits_{n=1}^{N_G} x^m_n \frac{h^{m}_i \beta^i_n}{\sum\limits_{l=1}^{K} h^{m}_l \beta^l_n} +\alpha_i , \ \forall i\in[K], m\in[N_F]
\end{equation}
Recall that the Dirichlet prior $\alpha$ encodes the number of cell-types that we expect on average to express each gene. We set $\alpha$ to be a symmetric Dirichlet with $\alpha_1=\alpha_2=\ldots=\alpha_K$, and $\sum_t \alpha_t=0.01$ for all cell-types $t$. In practice, we observe that our estimates of $h$ are fairly insensitive to the specific choice for $\alpha$ as long as $\sum_t \alpha_t$ is small enough. The smaller $\alpha$ is, the fewer cell-types expressing a given gene we expect to observe in a single histogram bin.

\section{Results and Evaluation}
\subsection{Implementation Details}
We tested our proposed cell-type discovery algorithm using the high-resolution \emph{in situ} hybridization image series for $1743$ of the most reliably imaged and annotated genes in the ABA. Individual cells were detected in the cellular resolution ISH images using custom algorithms (detailed in Supplementary Information). 
For each detected cell, we fit ellipses and extract several local features: (a) size and shape represented as the diameters along the principle axes of the ellipse, (b) orientation of the first principle axis, (c) gene intensity level as measured by the intensity of labeling of the cell body, and (d) the number of cells detected with-in a 100 $\mu m$ radius around the cell, which is a measure of the local cell density.
We aligned the ISH images to the ABA reference atlas and, for this paper, focused our attention on cells in the somatosensory cortex, since independent RNA-Seq data exist for this region the can be used to evaluate our approach. We computed joint histograms for the collection of cells found with-in the somatosensory cortex, resulting in a spatial point process feature vector of $N_F = 10010$ histogram bins per gene.

\paragraph{Synthetic experiment: } The vLDA model we proposed is then fit to $N_G \times N_F$ gene point process histogram matrix to estimate the cell-type gene expression profile matrix $\beta$ using the non-negative matrix factorization  (NNMF) algorithm.  The reason why we choose NNMF over Variational Inference (which is a popular approach for LDA) for $\beta$ estimation is that NNMF produces more accurate $\beta$ estimation in simulated data, illustrated in Fig~\ref{fig:synthetic}. In the synthetic experiment, we simulate point process data ( with some predefined golden standard $\beta$) and use the data to  estimate $\widehat{\beta}$. The errors were computed after pairing the estimated columns of $\beta$ with a closest golden standard $\beta$ column via hypothesis testing. Note that the columns of $\beta$ are normalized to 1, so the errors are bounded. 
\begin{figure}[ht]
\subfloat[a][Validate NNMF Method]
{\begin{minipage}{0.49\textwidth}{
\begin{center}
\psfrag{Error Per Type}[Bl]{\tiny{Error Per Type}}
\psfrag{Number of Cell Types}[Bl]{\tiny{Number of Cell Types}}
\psfrag{Permuted beta error}[Bl]{\tiny{Permute $\beta$}}
\psfrag{Variational Inference beta estimation}[Bl]{\tiny{VI estimated $\beta$}}
\psfrag{NNMF beta estimation}[Bl]{\tiny{NNMF estimated $\beta$}}
\psfrag{NNMF beta robust estimation trueK 10}[Bl]{\tiny{NNMF robust}}
\includegraphics[width=\textwidth,height=1.4in]{\fighomeBrain/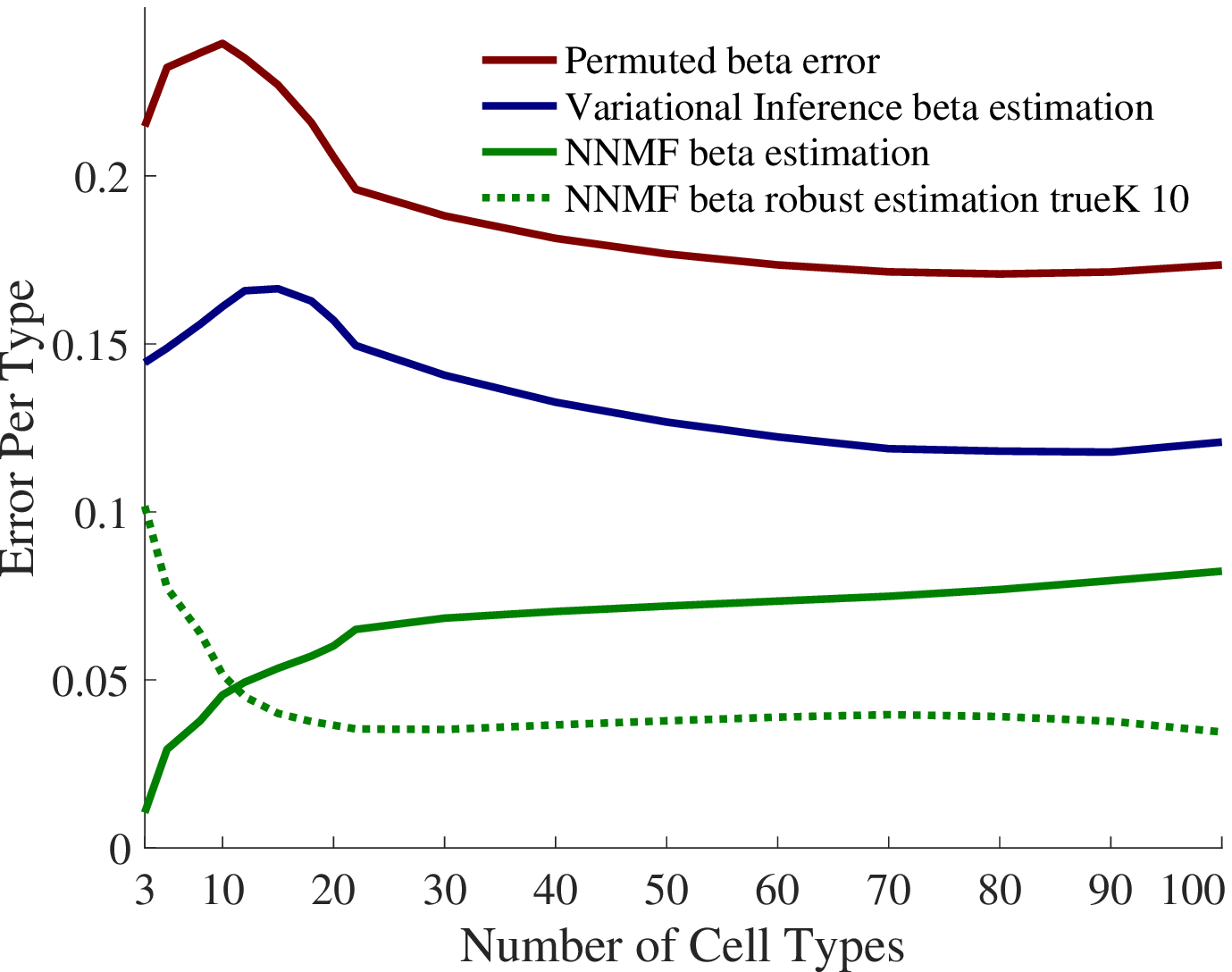}\label{fig:synthetic}
\end{center}
}
\end{minipage}}
\hfil
\subfloat[b][Validate Point Process Data]
{\begin{minipage}{0.49\textwidth}{
\begin{center}
\psfrag{Spatial Point Process Mixture ISH Data}[Bl]{\tiny{Spatial point process (ours)}}
\psfrag{Voxel Data}[Bl]{\tiny{Average expression level (baseline)}}
\psfrag{Spatial Point Process Mixture ISH Data (Permute)}[Bl]{\tiny{Spatial point process (ours, permuted)}}
\psfrag{Voxel Data (Permute)}[Bl]{\tiny{Average expression level (baseline, permuted)}}
\psfrag{Perplexity Score}[Bl]{\tiny{Perplexity Score}}
\psfrag{Number of Hidden Cell Types}[Bl]{\tiny{Number of Cell Types}}
\includegraphics[width=\textwidth]{\fighomeBrain/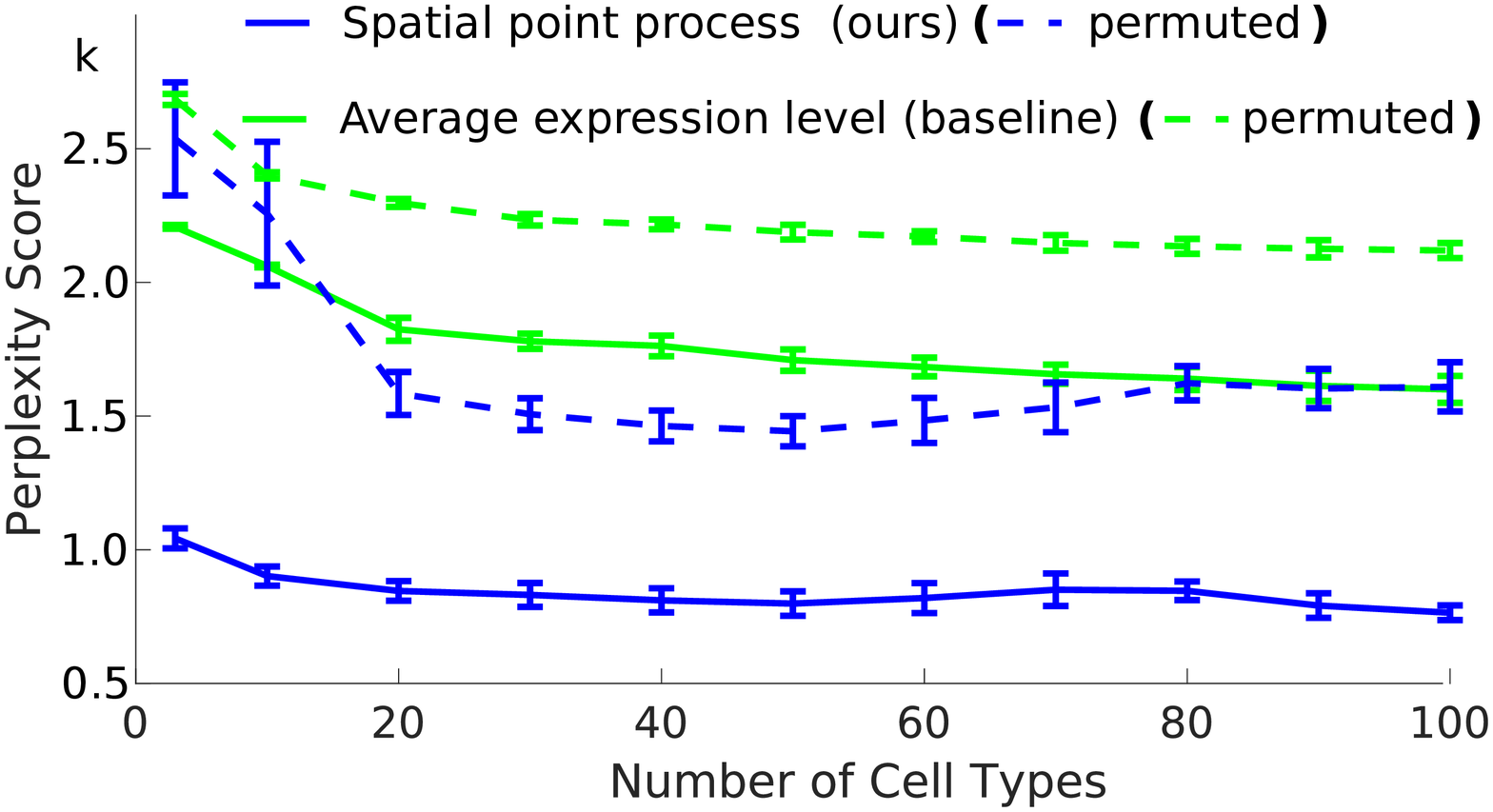}\label{fig:science}
\end{center}}
\end{minipage}}
\vspace{-0.5em}
\caption[Synthetic results and comparison with average gene expression level baseline.]{(a) {Synthetic Experiment : comparison of  Non-negative Matrix Factorization (NNMF) with Variational Inference (VI) on simulated point process cell data using known gene expression profile $\beta$.  An additional robustness test of NNMF is done to see how good the algorithm is when a wrong number of cell types $K$ is input. A permutation test (shuffling the gene expression levels between cell) is done to access statistical significance. Comparing with permute test shows that our cell-types are significantly different from chance. Error per type is computed by pairing the columns of estimated $\widehat{\beta}$ with the columns of the ground-truth $\beta$.  }
(b) {Comparison of gene expression profiles recovered for cell-types in the somatosensory cortex by fitting an LDA model using spatial point process features (ours) vs the standard average gene expression level feature (baseline). Our features provide a significantly better match, with lower perplexity, to ground truth single-cell RNA sequencing derived transcriptomes.  A permutation test is done to access statistical significance. Perplexity is computed by matching to surrogate single-cell RNA transcriptomes by shuffling the gene expression levels between cells. Comparing with permute test shows that our cell-types are significantly different from chance.  }
}
\end{figure}

\subsection{Evaluating Cell-type Gene Expression Profile Predictions}
A recent study performed single-cell RNA sequencing on $1691$ neurons isolated from mouse somatosensory cortex. We use this dataset to evaluate the quality of the cell-types we discover.

The single cell RNA-seq data, $G :=[g^1|g^2|\ldots|g^{N_C}]\in \mathbb{R}^{N_G\times N_C}$,  contains the gene expression profiles for $N_C=1691$ cells. We infer the cell types $h^i$ for these cells using equation~\eqref{eq:inferrence}, and then compute the likelihood $L^i$ of observing each for each cell under our estimated cell-type dependent gene expression profile matrix $\beta$ using equation~\eqref{eq:loglikelihood}. We can then evaluate the perplexity, a commonly used measure of goodness of fit under the vLDA model, of single cell RNA-seq data on the model we learned from our spatial point process data.

The perplexity score is a standard metric, which is defined as the geometric mean per-cell likelihood. It is a monotonically decreasing function of the log-likelihood $\mathcal{L}(G)$ of test data $G$. 
\begin{equation}
\text{perplexity}(G)= \exp(-\frac{\sum_{i=1}^{N_C} \log p({g }^i)}{\sum_{i=1}^{N_C}L^i})
\end{equation}
where the likelihood is evaluated as
\begin{equation}\label{eq:loglikelihood}
p(g^m|h^{m}, \alpha, \beta) =\frac{\Gamma\left(\sum_i\alpha_i\right)}{\prod_i \Gamma\left(\alpha_i\right)} \prod_{i=1}^{k} {(h^{m}_i)}^{\alpha_i-1} \prod\limits_{j=1}^{L^m} \left(\sum_{i=1}^{k}\sum\limits_{n=1}^{N_G} \delta_{g^{m}_j,e^n}{h^{m}_i \beta^i_n}\right).
\end{equation}
where $\delta_{i,j}$ is the Kronecker delta, $\delta_{i,j}=1$ when $i=j$ and $0$ otherwise. $e^n$ is the $n\tha$ basis vector. 

\subsection[Comparison to Standard Average Gene Expression Features]{Comparison to Standard Average Gene Expression Features Baseline and a Permutation Test for Significance}
 \begin{figure}[!htb]
\begin{center}
%\subfloat[a][ ]
\begin{minipage}{0.45\textwidth}
\bc 
\psfrag{Gad1}[rC]{\scriptsize{Gad1}}
\psfrag{Sp8}[rC]{\scriptsize{Sp8}}
\psfrag{Tox3}[rC]{\scriptsize{Tox3}}
\psfrag{Nkx2-1}[rC]{\scriptsize{Nkx2-1}}
\psfrag{Lhx6}[rC]{\scriptsize{Lhx6}}
\psfrag{Pax6}[rC]{\scriptsize{Pax6}}
\psfrag{Dlx5}[rC]{\scriptsize{Dlx5}}
\psfrag{Arx}[rC]{\scriptsize{Arx}}
\psfrag{Dlx2}[rC]{\scriptsize{Dlx2}}
\psfrag{Dlx1}[rC]{\scriptsize{Dlx1}}
\psfrag{Elavl2}[rC]{\scriptsize{Elavl2}}
\psfrag{Sp9}[rC]{\scriptsize{Sp9}}
\psfrag{Tbr1}[rC]{\scriptsize{Tbr1}}
\psfrag{Foxp2}[rC]{\scriptsize{Foxp2}}
\psfrag{Tshz2}[rC]{\scriptsize{Tshz2}}
\psfrag{Stat4}[rC]{\scriptsize{Stat4}}
\psfrag{Ascl1}[rC]{\scriptsize{Ascl1}}
\psfrag{Cux2}[rC]{\scriptsize{Cux2}}
\psfrag{Neurod1}[rC]{\scriptsize{Neurod1}}
\psfrag{Mef2c}[rC]{\scriptsize{Mef2c}}
\psfrag{Oligodendrocytes}[cc]{\scriptsize{Oligodendrocytes}}
\psfrag{Interneurons}[cc]{\scriptsize{Interneurons}}
\psfrag{S1Pyramial}[cc]{\scriptsize{S1Pyramidal}}
\psfrag{Astrocytes}[cc]{\scriptsize{Astrocytes}}
\psfrag{CA1Pyramidal}[cc]{\scriptsize{CA1Pyramidal}}
\psfrag{Ependymal}[cc]{\scriptsize{Ependymal}}
\psfrag{Microglia}[cc]{\scriptsize{Microglia}}
\psfrag{Endothelial}[cc]{\scriptsize{Endothelial}}
\psfrag{Mural}[cc]{\scriptsize{Mural}}
\includegraphics[width=\textwidth]{\fighomeBrain/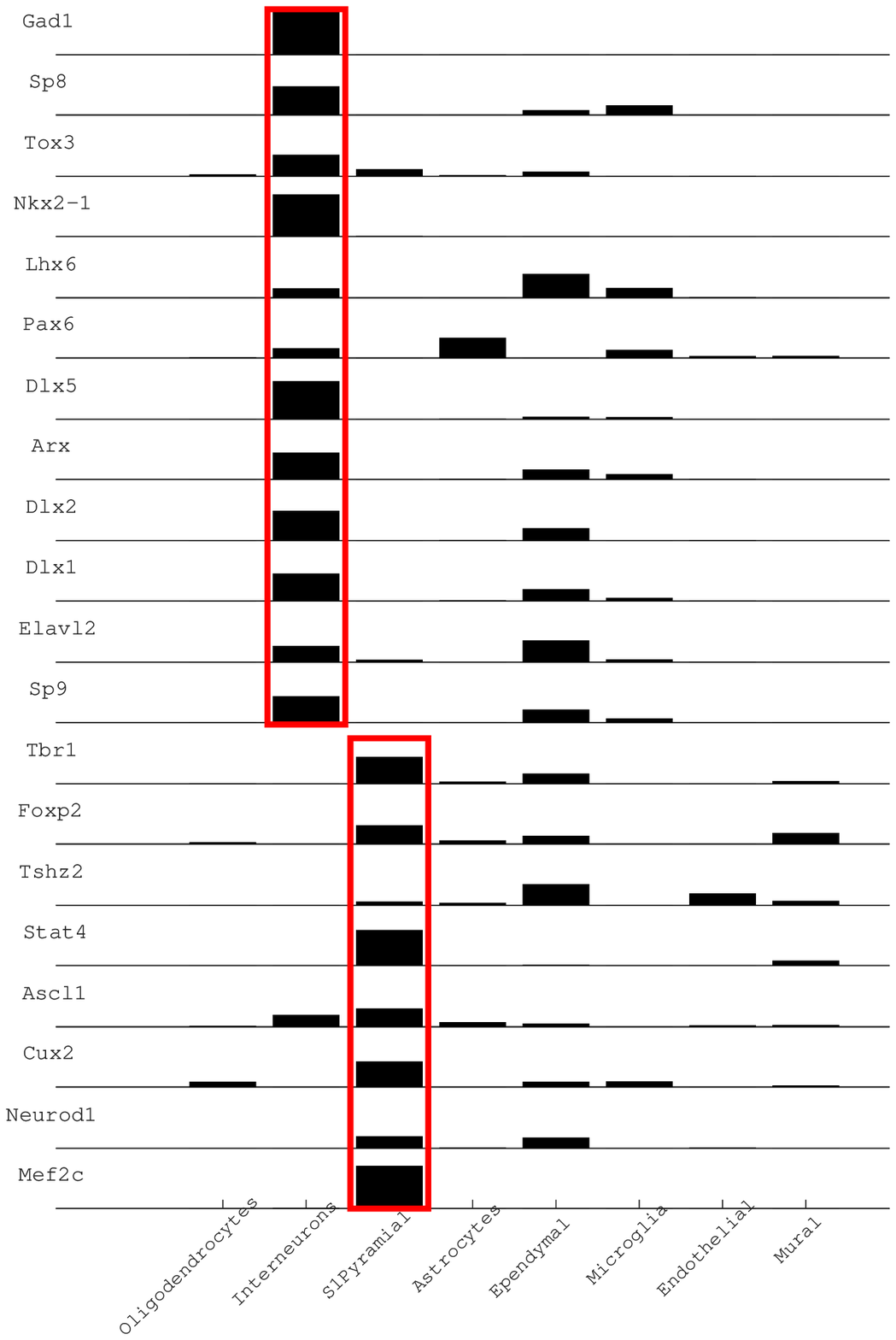}
\ec\end{minipage}
\hfil
%\subfloat[b][ ]
\begin{minipage}{0.45\textwidth}
\bc
\psfrag{Spink8}[rC]{\scriptsize{Spink8}}
\psfrag{Lhx9}[rC]{\scriptsize{Lhx9}}
\psfrag{Lmo1}[rC]{\scriptsize{Lmo1}}
\psfrag{Ptrf}[rC]{\scriptsize{Ptrf}}
\psfrag{Cldn5}[rC]{\scriptsize{Cldn5}}
\psfrag{Maf}[rC]{\scriptsize{Maf}}
\psfrag{Hcls1}[rC]{\scriptsize{Hcls1}}
\psfrag{Spi1}[rC]{\scriptsize{Spi1}}
\psfrag{Myb}[rC]{\scriptsize{Myb}}
\psfrag{Fhl1}[rC]{\scriptsize{Fhl1}}
\psfrag{Aldoc}[rC]{\scriptsize{Aldoc}}
\psfrag{Sall3}[rC]{\scriptsize{Sall3}}
\psfrag{Sox21}[rC]{\scriptsize{Sox21}}
\psfrag{Mbp}[rC]{\scriptsize{Mbp}}
\psfrag{Etv6}[rC]{\scriptsize{Etv6}}
\psfrag{Sox10}[rC]{\scriptsize{Sox10}}
\psfrag{St18}[rC]{\scriptsize{St18}}
\psfrag{Olig2}[rC]{\scriptsize{Olig2}}
\psfrag{Oligodendrocytes}[cc]{\scriptsize{Oligodendrocytes}}
\psfrag{Interneurons}[cc]{\scriptsize{Interneurons}}
\psfrag{S1Pyramial}[cc]{\scriptsize{S1 Paramidal}}
\psfrag{Astrocytes}[cc]{\scriptsize{Astrocytes}}
\psfrag{CA1Pyramidal}[cc]{\scriptsize{CA1 Pyramidal}}
\psfrag{Ependymal}[cc]{\scriptsize{Ependymal}}
\psfrag{Microglia}[cc]{\scriptsize{Microglia}}
\psfrag{Endothelial}[cc]{\scriptsize{Endothelial}}
\psfrag{Mural}[cc]{\scriptsize{Mural}}
\includegraphics[width=\textwidth]{\fighomeBrain/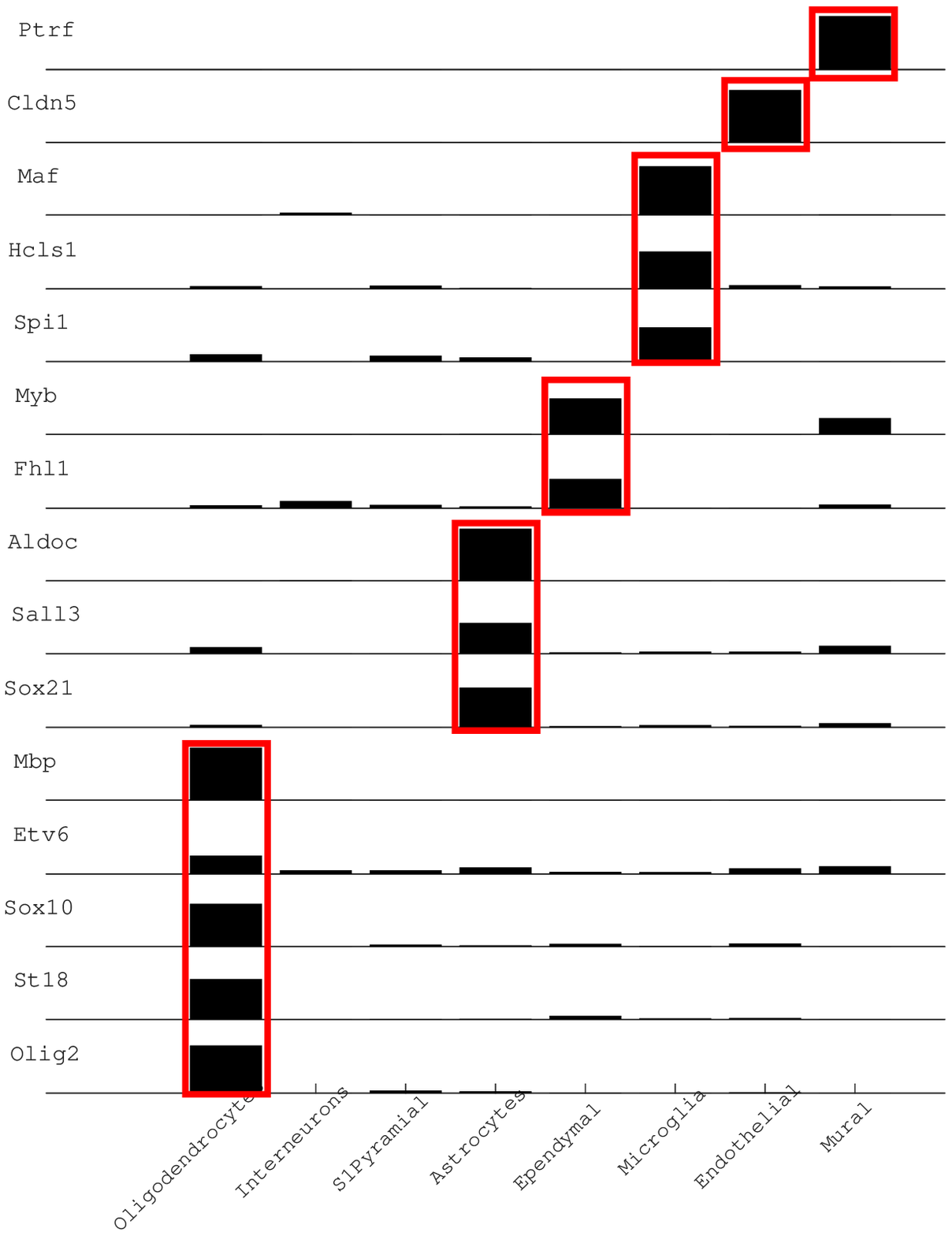}
\ec\end{minipage}\\
\end{center}
\vspace{0.5em}
\caption[Estimated $\beta$ on marker genes for 8 cell types]{Estimated memberships $\beta$ on marker genes for 8 cell types. These marker genes are used to label the columns of the membership matrix.}
% * <fraenkel@mit.edu> 2016-05-19T14:09:13.094Z:
%
% Explain what the red boxes represent
%
% ^.
\label{fig:beta}
\end{figure}

Here we demonstrate the superiority of our method and its statistical significance in two ways. First we compared the perplexity of the single-cell RNA seq dataset G under our model (figure~\ref{fig:science}, solid blue) against the perplexity of a surrogate dataset with the same marginal statistics, but whose gene-cell correlations were destroyed (figure~\ref{fig:science}, dashed blue). We generated this surrogate dataset by randomly permuting the gene expression levels for each gene across cells. This permuted dataset had a significantly higher (worse) perplexity than the true single-cell dataset. This demonstrates that our model trained to un-mix the ISH-derived spatial point processes discovered cell-types whose gene expression profiles are significantly better match to single-cells than by chance.

We also compared the predictions of cell-type gene expression profiles derived by un-mixing our spatial point process features against gene expression profiles derived by un-mixing the more standard $250\mu m \times 250\mu m \times 250\mu m$ averaged gene expression level features. We see a very large improvement in perplexity by switching from the standard simple averaging of gene expression, to extracting spatial point process features (figure~\ref{fig:science}). The single-cell RNA seq dataset analysis from figure~\ref{fig:science} shows that the perplexity of our recovered cell-types rapidly flattens after we recover approximately 10 clusters ($K=10$).

\begin{figure}[!htb]
\begin{center}
\subfloat[a][Cell diameter in principal axes]
{\begin{minipage}{0.45\textwidth}
\bc 
\psfrag{[2,2]}[cc]{\tiny{[2,2]}}
\psfrag{[2,3]}[cc]{\tiny{[2,3]}}
\psfrag{[2,4]}[cc]{\tiny{[2,4]}}
\psfrag{[3,3]}[cc]{\tiny{[3,3]}}
\psfrag{[3,4]}[cc]{\tiny{[3,4]}}
\psfrag{[3,5]}[cc]{\tiny{[3,5]}}
\psfrag{[3,6]}[cc]{\tiny{[3,6]}}
\psfrag{[3,7]}[cc]{\tiny{[3,7]}}
\psfrag{[4,4]}[cc]{\tiny{[4,4]}}
\psfrag{[4,5]}[cc]{\tiny{[4,5]}}
\psfrag{[4,6]}[cc]{\tiny{[4,6]}}
\psfrag{[4,7]}[cc]{\tiny{[4,7]}}
\psfrag{[5,7]}[cc]{\tiny{[5,7]}}
\psfrag{[6,6]}[cc]{\tiny{[6,6]}}
\psfrag{[6,7]}[cc]{\tiny{[6,7]}}
\psfrag{[7,7]}[cc]{\tiny{[7,7]}}
\psfrag{[7,8]}[cc]{\tiny{[7,8]}}
\psfrag{[7,9]}[cc]{\tiny{[7,9]}}
\psfrag{[8,8]}[cc]{\tiny{[8,8]}}
\psfrag{[8,9]}[cc]{\tiny{[8,9]}}
\psfrag{[9,9]}[cc]{\tiny{[9,9]}}
\psfrag{Oligodendrocytes}[lB]{\tiny{Oligodendrocytes}}
\psfrag{Interneurons}[lB]{\tiny{Interneurons}}
\psfrag{S1Pyramdial}[lB]{\tiny{S1Pyramdial}}
\psfrag{Astrocytes}[lB]{\tiny{Astrocytes}}
\psfrag{CA1Pyramidal}[lB]{\tiny{CA1Pyramidal}}
\psfrag{Ependymal}[lB]{\tiny{Ependymal}}
\psfrag{Microglia}[lB]{\tiny{Microglia}}
\psfrag{Endothelial}[lB]{\tiny{Endothelial}}
\psfrag{Mural}[lB]{\tiny{Mural}}
\psfrag{Cell Fraction}[cB]{\tiny{Cell Fraction}}
\psfrag{[Diameter 1, Diamter 2]}[cc]{}
\psfrag{Diameter}[cB]{\tiny{}}
\psfrag{2}[cc]{\tiny{2}}
\psfrag{3}[cc]{\tiny{3}}
\psfrag{4}[cc]{\tiny{4}}
\psfrag{5}[cc]{\tiny{5}}
\psfrag{6}[cc]{\tiny{6}}
\psfrag{7}[cc]{\tiny{7}}
\psfrag{8}[cc]{\tiny{8}}
\psfrag{9}[cc]{\tiny{9}}
\psfrag{First principal direction}[lb]{\tiny{Axis 1}}
\psfrag{Second principal direction}[lb]{\tiny{Axis 2}}
\includegraphics[width=\textwidth]
{\fighomeBrain/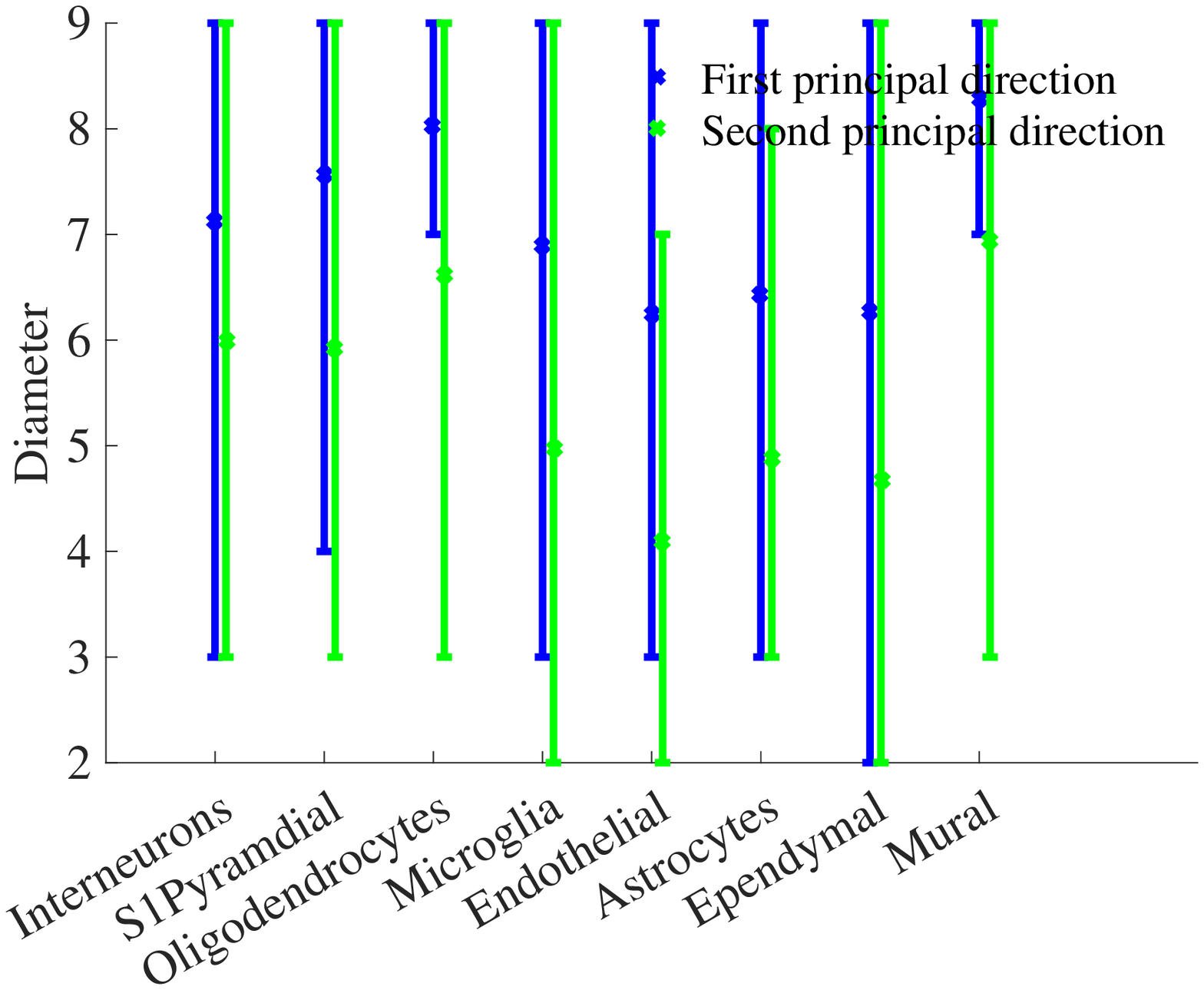}
%{\fighomeBrain/Regress4HwEstBeta_38MarkerGene_size.eps}
\ec\end{minipage}}\hfil
\subfloat[b][Orientation]
{\begin{minipage}{0.45\textwidth}\bc
\psfrag{Cell Fraction}[cB]{\tiny{Cell Fraction}}
\psfrag{Orientation}[cc]{\tiny{ }}
\psfrag{0}[cl]{\tiny{0}}
\psfrag{50}[cl]{\tiny{50}}
\psfrag{100}[cl]{\tiny{100}}
\psfrag{150}[cl]{\tiny{150}}
\psfrag{pi/6}[cc]{\tiny{$\frac{\pi}{6}$}}
\psfrag{2pi/6}[cc]{\tiny{$\frac{2\pi}{6}$}}
\psfrag{3pi/6}[cc]{\tiny{$\frac{3\pi}{6}$}}
\psfrag{4pi/6}[cc]{\tiny{$\frac{4\pi}{6}$}}
\psfrag{5pi/6}[cc]{\tiny{$\frac{5\pi}{6}$}}
\psfrag{Oligodendrocytes}[lB]{\tiny{Oligodendrocytes}}
\psfrag{Interneurons}[lB]{\tiny{Interneurons}}
\psfrag{S1Pyramdial}[lB]{\tiny{S1Pyramdial}}
\psfrag{Astrocytes}[lB]{\tiny{Astrocytes}}
\psfrag{CA1Pyramidal}[lB]{\tiny{CA1Pyramidal}}
\psfrag{Ependymal}[lB]{\tiny{Ependymal}}
\psfrag{Microglia}[lB]{\tiny{Microglia}}
\psfrag{Endothelial}[lB]{\tiny{Endothelial}}
\psfrag{Mural}[lB]{\tiny{Mural}}
\includegraphics[width = \textwidth]
{\fighomeBrain/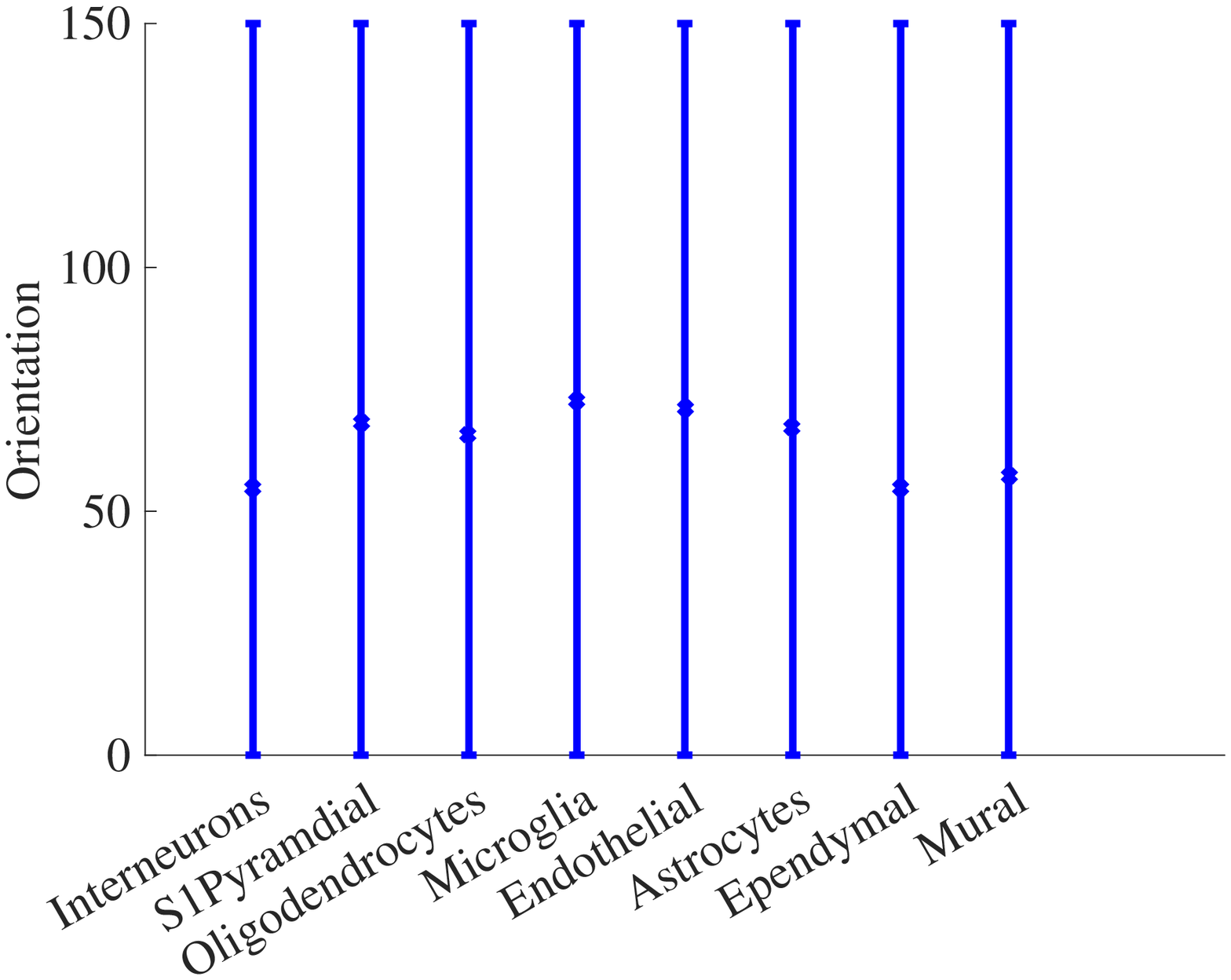}
%{\fighomeBrain/Regress4HwEstBeta_38MarkerGene_orientation.eps}
\ec\end{minipage}}\\
\subfloat[c][Intensity]
{\begin{minipage}{0.45\textwidth}\bc 
\psfrag{Cell Fraction}[cB]{\tiny{Cell Fraction}}
\psfrag{Gene Intensity}[cc]{}
\psfrag{Intensity}[cB]{\tiny{}}
\psfrag{0.2}[lB]{\tiny{0.2}}
\psfrag{0.3}[lB]{\tiny{0.3}}
\psfrag{0.4}[lB]{\tiny{0.4}}
\psfrag{0.5}[lB]{\tiny{0.5}}
\psfrag{0.6}[lB]{\tiny{0.6}}
\psfrag{0.7}[lB]{\tiny{0.7}}
\psfrag{0.8}[lB]{\tiny{0.8}}
\psfrag{0.9}[lB]{\tiny{0.9}}
\psfrag{Oligodendrocytes}[lB]{\tiny{Oligodendrocytes}}
\psfrag{Interneurons}[lB]{\tiny{Interneurons}}
\psfrag{S1Pyramdial}[lB]{\tiny{S1Pyramdial}}
\psfrag{Astrocytes}[lB]{\tiny{Astrocytes}}
\psfrag{CA1Pyramidal}[lB]{\tiny{CA1Pyramidal}}
\psfrag{Ependymal}[lB]{\tiny{Ependymal}}
\psfrag{Microglia}[lB]{\tiny{Microglia}}
\psfrag{Endothelial}[lB]{\tiny{Endothelial}}
\psfrag{Mural}[lB]{\tiny{Mural}}
\includegraphics[width=\textwidth]{\fighomeBrain/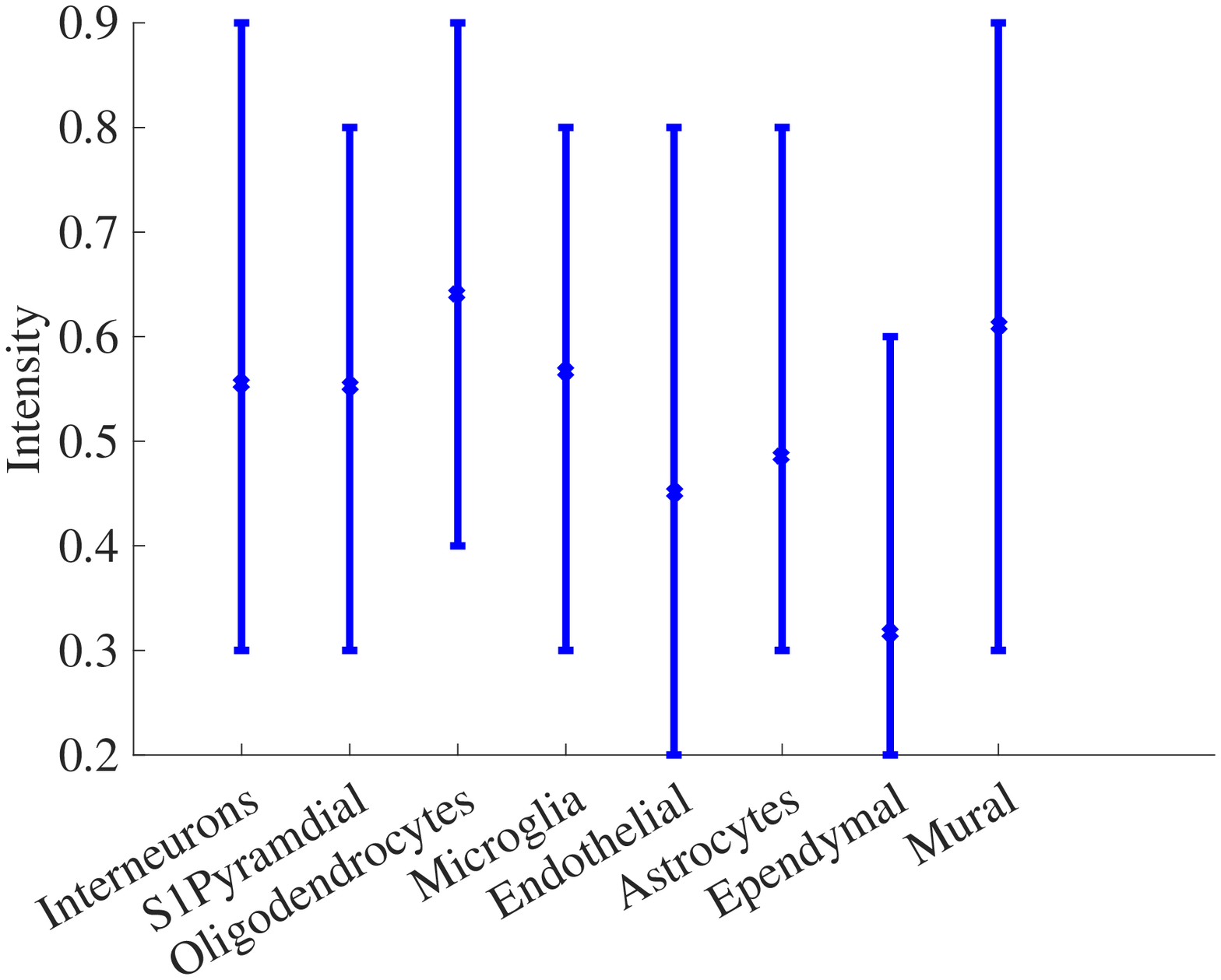}%{\fighomeBrain/Regress4HwEstBeta_38MarkerGene_intensity.eps}
\ec\end{minipage}}\hfil
\subfloat[d][Cells in \SI{100}{\micro\metre} radius]
{\begin{minipage}{0.45\textwidth}\bc 
\psfrag{Cell Fraction}[cB]{\tiny{Cell Fraction}}
\psfrag{Cells in 100 um}[cc]{}
\psfrag{0}[lB]{\tiny{0}}
\psfrag{10}[lB]{\tiny{10}}
\psfrag{20}[lB]{\tiny{20}}
\psfrag{30}[lB]{\tiny{30}}
\psfrag{40}[lB]{\tiny{40}}
\psfrag{50}[lB]{\tiny{50}}
\psfrag{Oligodendrocytes}[lB]{\tiny{Oligodendrocytes}}
\psfrag{Interneurons}[lB]{\tiny{Interneurons}}
\psfrag{S1Pyramdial}[lB]{\tiny{S1Pyramdial}}
\psfrag{Astrocytes}[lB]{\tiny{Astrocytes}}
\psfrag{CA1Pyramidal}[lB]{\tiny{CA1Pyramidal}}
\psfrag{Ependymal}[lB]{\tiny{Ependymal}}
\psfrag{Microglia}[lB]{\tiny{Microglia}}
\psfrag{Endothelial}[lB]{\tiny{Endothelial}}
\psfrag{Mural}[lB]{\tiny{Mural}}
\includegraphics[width=\textwidth]{\fighomeBrain/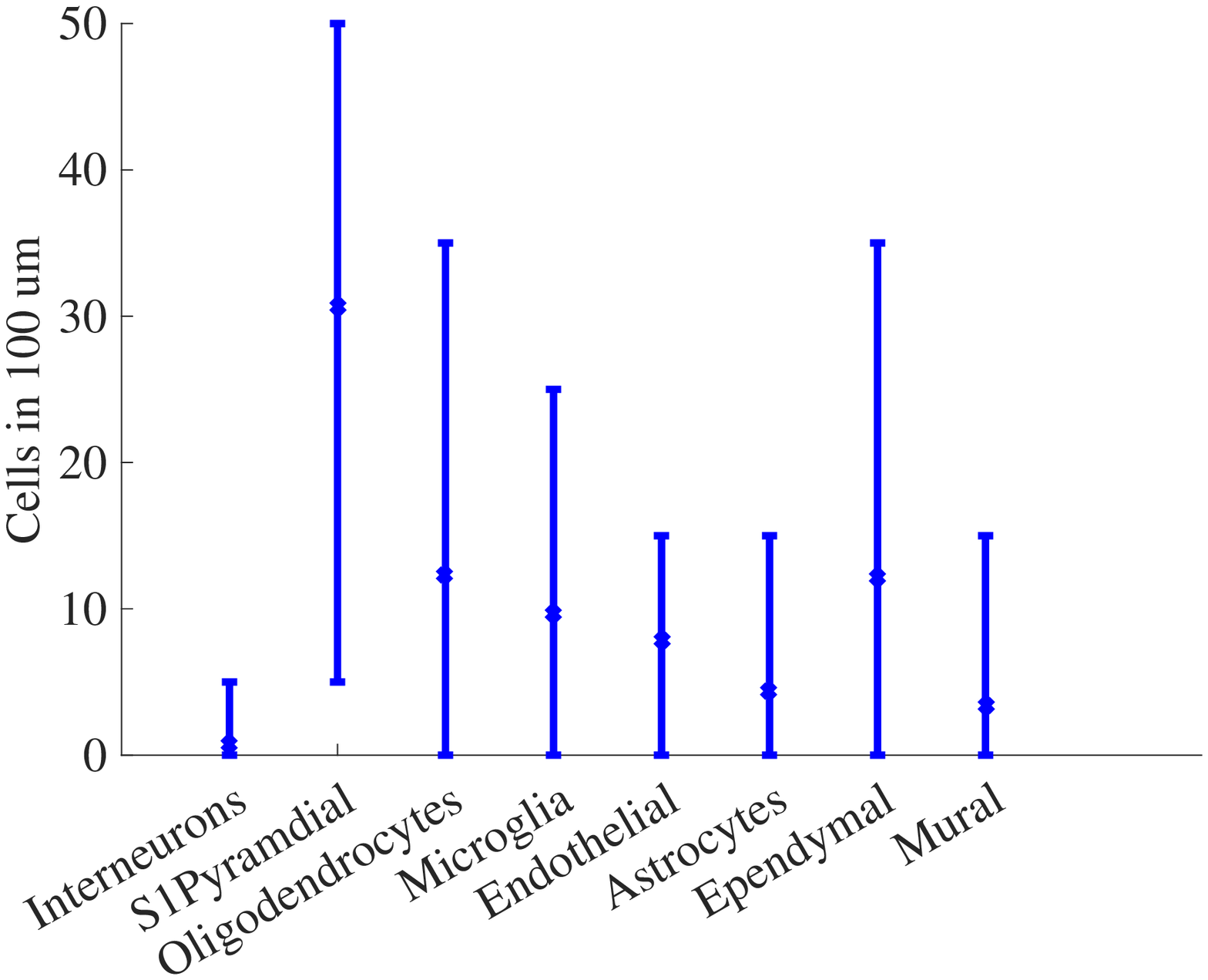}%{\fighomeBrain/Regress4HwEstBeta_38MarkerGene_density.eps}
\ec\end{minipage}}\\
\end{center}
\caption[Detected 8 cell type feature visualization]{Figure of 5\% and 95\% percentile estimated cell features for 8 cell types we detected. Inference is performed on the Spatial point process histograms data we estimated.}
\label{fig:hist_CellType}
\end{figure}

\subsection{A Brief Analysis of Recovered Cell Types in Somatosensory Cortex}

 In this section we describe the representative spatial point process statistics and gene expressions for 8 cell-types we recovered. We attempted to align our 8 clusters to cell-types defined by \cite{zeisel2015cell} in the single-cell RNA sequencing paper. We found high overlap in the gene expression profiles for all 8 clusters with known cell-types defined in~\cite{zeisel2015cell},  \emph{Interneurons}, \emph{S1 Pyramidal}, \emph{Mural}, \emph{Endothelial}, \emph{Microglia}, \emph{Ependymal}, \emph{Astrocytes} and \emph{Oligodendrocytes}, in Figure~\ref{fig:beta}.

The estimate of $\beta$ was combined with MLE to infer the cell-type specific spatial point process representation $h^m_l$. 
 In examining the spatial point process distributions that we predict for each of these cell types, we discover that while the distribution of cell body orientations is quite broad and similar across cell types, the cell count distribution, which is a measure of cell density, varies in a systematic way from one cell type to another. Fig~\ref{fig:hist_CellType}d shows that inhibitory Interneurons are less dense than S1Pyramidal neurons. This is consistent with their known prevalence, roughly 20\% of all neurons are GABAergic interneurons \cite{Markram:2004ek}, while the remaining 80\% are excitatory glutamatergic pyramidal neurons. As expected, this excitatory neuronal category of S1Pyramidal is the most common and hence most dense class of neuronal cells. They also have slightly larger cell bodies, compared to interneurons, as can be seen in Fig~\ref{fig:hist_CellType}a. The remaining 6 cell types correspond to various glial sub-types.
\section{Conclusion}
We developed a computational method for discovering cell types in a brain region by analyzing the high-resolution \emph{in situ} hybridization image series from the Allen Brain Atlas. Under the assumption that cell types have unique spatial distributions and gene expression profiles, we used a varied latent Dirichlet allocation (vLDA) based on spatial point process process mixture model to simultaneously infer the cell feature spatial distribution and gene expression profiles of cell types. By comparing our gene expression profile predictions to a single-cell RNA sequencing dataset, we demonstrated that our model improves significantly on state of the art.

The accuracy of our method relies heavily on the assumption that cell-types differ in their spatial distribution, and that our point process features perform a good job of distinguishing these differences. Thus the performance of our method can be improved by better estimates of better features. We would expect our method to perform better for large brain areas, which can be more accurately aligned, and which have more cells to estimate point process features.

There are several modifications to our vLDA model which might improve the faithfulness of our generative model to the biology. %We presently model the cell-type specific gene expression process $\beta_t$ as a multinomial distribution. A binomial distribution might be more appropriate as it does not preclude the simultaneous expression of many genes. 
We place a symmetric Dirichlet prior over cell-type multinomial distribution $h^m$ for a given histogram bin $m$. This assumes that the number of cell-types expressing each gene is the same for all genes. But since some genes are expressed more commonly and non-specifically than others, we might expect a gene-specific prior to be a better model. Further, the symmetric Dirichlet assumes that all cell-types have equal proportions of cells. But evidence suggests that excitatory neurons are more common than inhibitory neurons in cortex~\cite{harris2013cortical}, and using a non-uniform Dirichlet prior could account for this.

%Extensions? Working with voxels/moving windows rather than atlas dependent brain regions. Spatial prior over h. Can use pp features for registration, region classification.
\chapter{Conclusion and Outlook}

\section{Conclusion}
Now that we are at the end of the dissertation, we are convinced that spectral methods including tensor decomposition are good candidates for unsupervised learning. They reveal hidden structure using transformations and extract useful and clean information to characterize the complicated data. Spectral methods are proved to be potential in various application. For instance, text and image processing, social networks, healthcare analytics and neuroscience.  

Spectral methods especially matrix/tensor decomposition framework is versatile. They are straightforward to apply to flat models,  such as exchangeable model, multi-view model, and hidden Markov model, but they are also amendable to learn models with a hierarchy such as a mixture of trees and latent tree model.  Spectral methods not only perform well on traditional multiplicative sparse coding models but also outperforms the state-of-the-art on models with group invariance.  The tensor decomposition framework is efficient and is guaranteed to converge to global optima.

\section{Outlook}

Now the question is what is beyond? Could we further push the boundaries of spectral methods? Can we have a tensor library with optimal hardware support for tensor operations? In the region of high dimensional hidden space, could we develop approximated algorithms that are computational more efficient? Could we have tensor sketching where the decomposition happens in a sketching vector space, and the tensor is never explicitly formed?  Furthermore, could we use tensor decomposition to train models with other invariances  (such as rotation invariance and scaling invariance) or general invariance constraints? 
    
In the real world, we could push our framework further for more challenging tasks. In neuroscience, we would like to understand the brain; that is to systematically model and learn brain neural system and sort out its relationship to body functions. We know that deep neural network system inspired by the architecture of neural circuits have been hugely successful empirically. Could we utilize the neural network techniques to foster understanding of the brain neural circuits? Or could we use our knowledge of the brain neural circuits to understand fundamental reasons for a certain structure of a deep neural network system in machine learning? Even in healthcare analytics, simple usage of the co-occurrence of diseases is not as informative as considering other factors such as symptoms. With more information, the model gets more complicated, but we hope to achieve personalized identification of diseases or curing plans. 

Overall, there are numerous exciting open problems ahead. Graduation is not an end; rather it is a fresh start. I am looking forward to the uncertainty of the future career. Keep curious and continue exploring. May the world be more intelligent! 

% ... and so on

% These commands fix an odd problem in which the bibliography line
% of the Table of Contents shows the wrong page number.
\clearpage
\phantomsection

% "References should be formatted in style most common in discipline",
% abbrv is only a suggestion.
\bibliographystyle{abbrv}
\bibliography{thesis}

% The Thesis Manual says not to include appendix figures and tables in
% the List of Figures and Tables, respectively, so these commands from
% the caption package turn it off from this point onwards. If needed,
% it can be re-enabled later (using list=yes argument).
\captionsetup[figure]{list=no}
\captionsetup[table]{list=no}

% If you have an appendix, it should come after the references.
% The original template (from Trevor) had a custom \appendix command,
% but I found it to break figure/table counters. I'm not sure how
% reliable my fix is, so I ended up reverting back to the standard
% latex version, and renaming the custom command to \myappendix.  You
% can try both and see how things work out:
% 1) Call \appendix once, and then make each appendix a \chapter
% 2) Call \myappendix once, and then make each appendix a \section.

\appendix

\chapter{Appendix for Online Stochastic Gradient for Tensor Decomposition}
%!TEX root = saddlepoint.tex

\section{Detailed Analysis for Section~\ref{sec:sgd} in Unconstrained Case}

\label{sec:unconstrained}

In this section we give detailed analysis for noisy gradient descent, under the assumption that the unconstrained problem satisfies $(\alpha,\gamma,\epsilon,\delta)$-\name~property.

The algorithm we investigate in Algorithm~\ref{algo:sgdwn}, we can combine the randomness in the stochastic gradient oracle and the artificial noise, and rewrite the update equation in form:
\begin{equation} \label{SGD_update}
	w_t = w_{t-1} - \eta (\nabla f(w_{t-1}) + \xi_{t-1})
\end{equation}
where $\eta$ is step size, $\xi = SG(w_{t-1}) - \nabla f(w_{t-1}) + n$ (recall $n$ is a random vector on unit sphere) is the combination of two source of noise.

By assumption, we know $\xi$'s are independent and they satisfying $\E \xi = 0$, 
$\|\xi\| \le Q+1$. Due to the explicitly added noise in Algorithm \ref{algo:sgdwn}, 
we further have   $\E \xi\xi^T \succ \frac{1}{d}I$.
For simplicity, we assume $ \E \xi\xi^T = \sigma^2I$, for some constant $\sigma = \tilde{\Theta}(1)$, 
then the algorithm we are running is exactly the same as Stochastic Gradient Descent (SGD).
Our proof can be very easily extended to the case when $\frac{1}{d} I \preceq\E [\xi\xi^T] \preceq (Q+\frac{1}{d})I$ because both the upper and lower bounds are $\tilde{\Theta}(1)$.

%\jccomment{fill in ?? for the covariance of unit sphere in dimension $d$}

%Throughout the proof, we use $\|\cdot\|$ to denote the 2-norm for vectors, and operator norm for matrices.
We first restate the main theorem in the context of stochastic gradient descent.

\begin{theorem} [Main Theorem]\label{thm:sgdmain_unconstraint}
Suppose a function $f(w):\R^d\to \R$ that is $(\alpha, \gamma, \epsilon, \delta)$-\name, and has a stochastic gradient oracle where the noise satisfy $\E \xi\xi^T = \sigma^2I$. Further, suppose the function is bounded by $|f(w)| \le B$, is $\beta$-smooth and has $\rho$-Lipschitz Hessian. Then there exists a threshold $\eta_{\max} = \tilde{\Theta}(1)$, so that for any $\zeta>0$, and
for any $\eta \le \eta_{\max} / \max\{1, \log (1/\zeta)\}$,
with probability at least $1-\zeta$ in $t = \tlO(\eta^{-2}\log( 1/\zeta))$ iterations, SGD outputs a point $w_t$ that is $\tlO(\sqrt{\eta\log(1/\eta\zeta)})$-close to some local minimum $w^\star$.
\end{theorem}

Recall that $\tlO(\cdot)$ ($\tilde{\Omega},\tilde{\Theta}$) hides the factor that has
polynomial dependence on all other parameters, but is independent of $\eta$ and $\zeta$. So it focuses on the dependency on $\eta$ and $\zeta$. Throughout the proof, we interchangeably use both $\mathcal{H}(w)$ and $\Hess f(w)$ to represent the Hessian matrix of $f(w)$.

As we discussed in the proof sketch in Section~\ref{sec:sgd}, we analyze the behavior of the algorithm in three different cases. The first case is when the gradient is large.

\begin{lemma} \label{thm::case1}
Under the assumptions of Theorem~\ref{thm:sgdmain_unconstraint}, for any point with $\|\nabla f(w_0)\| \ge \sqrt{2\eta\sigma^2\beta d}$ where $\sqrt{2\eta\sigma^2\beta d} < \epsilon$, after one iteration we have:
\begin{equation}
	\E f(w_1) - f(w_{0}) \le - \tlOmega(\eta^2)
\end{equation}
\end{lemma} 

\begin{proof}
Our assumption can guarantee $\eta_{\max} < \frac{1}{\beta}$, then by update equation Eq.(\ref{SGD_update}), we have:
\begin{align}
	\E f(w_1) -  f(w_{0}) &\le \nabla f(w_{0})^T \E(w_1-w_{0}) + \frac{\beta}{2}\E\|w_1-w_{0}\|^2 \nonumber \\
	& = \nabla f(w_{0})^T \E\left (- \eta (\nabla f(w_{0}) + \xi_{0})\right )
	 + \frac{\beta}{2}\E\left \|- \eta (\nabla f(w_{0}) + \xi_{0})\right \|^2 \nonumber \\	
	& = -(\eta - \frac{\beta\eta^2}{2})\|\nabla f(w_{0})\|^2 + \frac{\eta^2 \sigma^2 \beta d}{2}\nonumber \\
	& \le -\frac{\eta}{2}\|\nabla f(w_{0})\|^2  + \frac{\eta^2\sigma^2 \beta d}{2}
	\le -\frac{\eta^2\sigma^2 \beta d}{2}
\end{align}
which finishes the proof.
\end{proof}

\begin{lemma}\label{thm::case3}
Under the assumptions of Theorem~\ref{thm:sgdmain_unconstraint}, for any initial point $w_0$ that is $\tlO(\sqrt{\eta}) < \delta$ close to a local minimum $w^\star$, 
with probability at least $1-\zeta/2$, we have following holds simultaneously:
\begin{equation}
 \forall t\le \tlO(\frac{1}{\eta^2}\log \frac{1}{\zeta}), \quad \|w_{t} - w^\star\| \le \tlO(\sqrt{\eta\log \frac{1}{\eta\zeta}})<\delta 
\end{equation}
where $w^\star$ is the locally optimal point.
\end{lemma}

\begin{proof}
We shall construct a supermartingale and use Azuma's inequality~\cite{azuma1967weighted} to prove this result.

Let filtration $\mathfrak{F}_t = \sigma\{\xi_0, \cdots \xi_{t-1}\}$, and note $\sigma\{\Delta_0, \cdots, \Delta_t \} \subset \mathfrak{F}_t$, where $\sigma\{\cdot\}$ denotes the sigma field.
Let event $\mathfrak{E}_t = \{\forall \tau \le t, \|w_{\tau} - w^\star\| \le \mu\sqrt{\eta\log \frac{1}{\eta\zeta}} < \delta \}$, where $\mu$ is independent of $(\eta, \zeta)$, and will be specified later. To ensure the correctness of proof, $\tilde{O}$ notation in this proof will never hide any dependence on $\mu$.
Clearly there's always a small enough choice of $\eta_{\max}= \tilde{\Theta}(1)$ to make $\mu\sqrt{\eta\log \frac{1}{\eta\zeta}} < \delta$ holds as long as 
$\eta \le \eta_{\max} / \max\{1, \log (1/\zeta)\}$.
Also note $\mathfrak{E}_t \subset \mathfrak{E}_{t-1}$, that is $1_{\mathfrak{E}_t} \le 1_{\mathfrak{E}_{t-1}}$.

By Definition~\ref{def:robustcondition} of 
$(\alpha, \gamma, \epsilon, \delta)$-\name, we know $f$ is locally $\alpha$-strongly convex
in the $2\delta$-neighborhood of $w^\star$.
Since $\nabla f(w^\star) = 0$, we have
\begin{align}
\nabla f(w_t)^T (w_t - w^\star)1_{\mathfrak{E}_{t}}  \ge \alpha \|w_t - w^\star\|^21_{\mathfrak{E}_{t}}
\end{align}

Furthermore, with  $\eta_{\max} < \frac{\alpha}{\beta^2}$, using $\beta$-smoothness, we have:
\begin{align}
\E[\|w_{t} - w^\star\|^21_{\mathfrak{E}_{t-1}}|\mathfrak{F}_{t-1}] 
= &\E[\|w_{t-1} - \eta (\nabla f(w_{t-1}) + \xi_{t-1}) - w^\star\|^2|\mathfrak{F}_{t-1}]1_{\mathfrak{E}_{t-1}} \nonumber \\
 = &\left[\|w_{t-1}- w^\star\|^2 - 2\eta \nabla f(w_{t-1})^T(w_{t-1}- w^\star)  \right. \nonumber\\
 &\left.+ \eta^2\|\nabla f(w_{t-1})\|^2 +\eta^2d\sigma^2\right]1_{\mathfrak{E}_{t-1}}  \nonumber \\
  \le&[(1-2\eta \alpha + \eta^2 \beta^2)\|w_{t-1}- w^\star\|^2 + \eta^2d\sigma^2]1_{\mathfrak{E}_{t-1}} \nonumber \\
 \le &[(1-\eta \alpha)\|w_{t-1}- w^\star\|^2 + \eta^2d\sigma^2]1_{\mathfrak{E}_{t-1}}
\end{align}
Therefore, we have:
\begin{equation}
\left[\E[\|w_{t} - w^\star\|^2|\mathfrak{F}_{t-1}] - \frac{\eta d\sigma^2}{ \alpha} \right]1_{\mathfrak{E}_{t-1}} \le (1-\eta \alpha) \left[\|w_{t-1} - w^\star\|^2 -\frac{\eta d\sigma^2}{ \alpha} \right]1_{\mathfrak{E}_{t-1}}
\end{equation}

Then, let $G_t = \max\{(1-\eta\alpha)^{-t}(\|w_{t} - w^\star\|^2 - \frac{\eta d\sigma^2}{ \alpha}), 0\}$, we have:
\begin{equation}
\E [G_t1_{\mathfrak{E}_{t-1}} | \mathfrak{F}_{t-1}] \le G_{t-1}1_{\mathfrak{E}_{t-1}} \le G_{t-1}1_{\mathfrak{E}_{t-2}}
\end{equation}
which means $G_t1_{\mathfrak{E}_{t-1}}$ is a supermartingale.

Therefore, with probability 1, we have:
\begin{align}
&|G_t1_{\mathfrak{E}_{t-1}} -\E[G_{t}1_{\mathfrak{E}_{t-1}}|\mathfrak{F}_{t-1}] | \nonumber \\
\le & (1-\eta\alpha)^{-t}
[~ \|w_{t-1} - \eta \nabla f(w_{t-1}) - w^\star\|\cdot \eta\|\xi_{t-1}\| + \eta^2\|\xi_{t-1}\|^2 + \eta^2d\sigma^2~]1_{\mathfrak{E}_{t-1}} \nonumber \\
\le & (1-\eta\alpha)^{-t}\cdot \tlO(\mu\eta^{1.5}\log^{\frac{1}{2}} \frac{1}{\eta\zeta}) = d_t
\end{align}

Let 
\begin{equation}
c_t = \sqrt{\sum_{\tau=1}^t d_\tau^2 } =  \tlO(\mu\eta^{1.5}\log^{\frac{1}{2}}  \frac{1}{\eta\zeta})\sqrt{\sum_{\tau=1}^t(1-\eta\alpha)^{-2\tau} } 
%\le O(\sigma)\sqrt{\sum_{\tau=1}^t(1-2\eta\alpha)^{-t} } 
%= O(\frac{\sigma}{\eta^{0.5}}) (1-2\eta \alpha)^{-\frac{t}{2}}
\end{equation}
By Azuma's inequality, with probability less than $\tlO(\eta^3\zeta)$,
we have:
\begin{align}
G_t1_{\mathfrak{E}_{t-1}}  > \tlO(1)c_t \log^{\frac{1}{2}} (\frac{1}{\eta\zeta}) + G_0
\end{align}

We know $G_t  > \tlO(1)c_t \log^{\frac{1}{2}}(\frac{1}{\eta\zeta}) + G_0$ is equivalent to:
\begin{align}
\|w_{t} - w^\star\|^2 > \tlO(\eta)
+ \tlO(1) (1-\eta\alpha)^{t}c_t \log^{\frac{1}{2}} (\frac{1}{\eta\zeta})
\end{align}
We know:
\begin{align}
&(1-\eta\alpha)^{t}c_t\log^{\frac{1}{2}} (\frac{1}{\eta\zeta}) =  \mu\cdot \tlO(\eta^{1.5}\log \frac{1}{\eta\zeta})\sqrt{\sum_{\tau=1}^t(1-\eta\alpha)^{2(t-\tau)} } \nonumber \\
= & \mu\cdot \tlO(\eta^{1.5}\log \frac{1}{\eta\zeta})\sqrt{\sum_{\tau=0}^{t-1}(1-\eta\alpha)^{2\tau} }  
\le   \mu\cdot \tlO(\eta^{1.5}\log \frac{1}{\eta\zeta})\sqrt{\frac{1}{1-(1-\eta\alpha)^2}} \nonumber\\
&= \mu\cdot \tlO(\eta\log \frac{1}{\eta\zeta})
\end{align}
This means Azuma's inequality implies, there exist some $\tilde{C} = \tlO(1)$ so that:
\begin{align}
	P\left(\mathfrak{E}_{t-1} \cap\left\{\|w_{t} - w^\star\|^2 > \mu \cdot \tilde{C}\eta\log\frac{1}{\eta\zeta}) \right\} \right)
	\le \tlO(\eta^3\zeta)
\end{align}
By choosing $\mu > \tilde{C}$, this is equivalent to:
\begin{align}
	P\left(\mathfrak{E}_{t-1} \cap\left\{\|w_{t} - w^\star\|^2 > \mu^2\eta\log\frac{1}{\eta\zeta} \right\} \right)
	\le \tlO(\eta^3\zeta)
\end{align}
Then we have:
\begin{align}
	P(\overline{\mathfrak{E}}_{t} ) = 
	P\left(\mathfrak{E}_{t-1} \cap\left\{\|w_{t} - w^\star\| > \mu\sqrt{\eta\log\frac{1}{\eta\zeta}} \right\} \right) 	+ P(\overline{\mathfrak{E}}_{t-1} ) 
	\le  \tlO(\eta^3\zeta)+ P(\overline{\mathfrak{E}}_{t-1} )
\end{align}
By initialization conditions, we know $P(\overline{\mathfrak{E}}_{0} ) = 0$, and thus
$P(\overline{\mathfrak{E}}_{t} )
\le t \tlO(\eta^3\zeta)$. Take $t=\tlO(\frac{1}{\eta^2}\log\frac{1}{\zeta})$, 
we have $P(\overline{\mathfrak{E}}_{t} ) \le \tlO(\eta \zeta \log \frac{1}{\zeta})$.
When $\eta_{\max} = \tlO(1)$ is chosen small enough, and $\eta \le \eta_{\max}/\log(1/\zeta)$, this finishes the proof.
\end{proof}

\begin{lemma} \label{thm::case2}
Under the assumptions of Theorem~\ref{thm:sgdmain_unconstraint}, for any initial point $w_0$ where $\|\nabla f(w_0)\|$ $ \le \sqrt{2\eta\sigma^2\beta d} < \epsilon$, and $\lambda_{\min}(\mathcal{H}(w_0)) \le -\gamma$, then 
%there is a fixed number of steps $T = O(\log d/\gamma\eta)$ such that:
there is a number of steps $T$ that depends on $w_0$ such that:
\begin{equation}
	\E f(w_T) - f(w_0) \le - \tlOmega(\eta)
\end{equation}
The number of steps $T$ has a fixed upper bound $T_{max}$ that is independent of $w_0$ where $T \le T_{max} = O((\log d)/\gamma\eta)$.
\end{lemma} 

\begin{remark}
In general, if we relax the assumption $\E \xi\xi^T  = \sigma^2I$ to 
$ \sigma_{\min}^2I\preceq\E \xi\xi^T \preceq \sigma_{\max}^2I$, the upper bound $T_{max}$ of number of steps required in 
Lemma \ref{thm::case2} would be increased to $T_{max} = O(\frac{1}{\gamma\eta}(\log d+ \log \frac{\sigma_{\max}}{\sigma_{\min}}))$
\end{remark}
As we described in the proof sketch, the main idea is to consider a coupled update sequence that correspond to the local second-order approximation of $f(x)$ around $w_0$. We characterize this sequence of update in the next lemma.
%
%To prove this theorem, we need to first prove two most important lemma.
%

%
%
%
%
%and $\{\tilde{w}_t\}$ be the corresponding sequence generated by running stochastic gradient on function $\tilde{f}$, with $\tilde{w}_0 = w_0$.

\begin{lemma} \label{lem::case_Gaussian}
Under the assumptions of Theorem~\ref{thm:sgdmain_unconstraint}.
Let $\tilde{f}$ defined as local second-order approximation of $f(x)$ around $w_0$:
\begin{equation}\label{def_tilde_f}
\tilde{f}(w) \doteq f(w_0) + \nabla f(w_0)^T (w-w_0) + \frac{1}{2}(w-w_0)^T\mathcal{H}(w_0)(w-w_0)
\end{equation}
$\{\tilde{w}_t\}$ be the corresponding sequence generated by running SGD on function $\tilde{f}$, with $\tilde{w}_0 = w_0$.
For simplicity, denote $\mathcal{H} = \mathcal{H}(w_0)= \Hess f(w_0)$,
then we have analytically:
\begin{align}
&\nabla \tilde{f}(\tilde{w}_t)= (1-\eta\mathcal{H})^t\nabla f(w_0) -\eta \mathcal{H}\sum_{\tau=0}^{t-1}(1-\eta\mathcal{H})^{t-\tau-1}\xi_{\tau}\\
 &\tilde{w}_{t} - w_0 = -\eta \sum_{\tau = 0}^{t-1}(1-\eta \mathcal{H})^\tau \nabla f(w_0) -\eta
	\sum_{\tau=0}^{t-1}(1-\eta\mathcal{H})^{t-\tau-1}\xi_{\tau}  \label{dif_x}
\end{align}

Furthermore, for any initial point $w_0$ where $\|\nabla f(w_0)\| \le \tlO(\eta) < \epsilon$, and $\lambda_{\min}(\mathcal{H}(w_0)) = -\gamma_0$. 
Then, there exist a $T \in \mathbb{N}$ satisfying:
\begin{equation}\label{choose_t}
 \frac{d}{\eta\gamma_0} \le \sum_{\tau=0}^{T-1}(1+\eta \gamma_0)^{2\tau} < \frac{3d}{\eta\gamma_0}
\end{equation}
with probability at least $1-\tlO(\eta^3)$, we have
following holds simultaneously for all $t\le T$:
\begin{equation}
\|\tilde{w}_t - w_0\| \le \tlO(\eta^{\frac{1}{2}}\log \frac{1}{\eta}); 
\quad\quad
\|\nabla\tilde{f}(\tilde{w}_t)\| \le \tlO(\eta^{\frac{1}{2}}\log \frac{1}{\eta})
\end{equation}
\end{lemma}

\begin{proof}
Denote $\mathcal{H} = \mathcal{H}(w_0)$, since $\tilde{f}$ is quadratic, clearly we have:
\begin{equation}\label{derivative_tilde_recursive}
\nabla \tilde{f}(\tilde{w}_t) = \nabla \tilde{f}(\tilde{w}_{t-1}) + \mathcal{H} (\tilde{w}_t - \tilde{w}_{t-1})
\end{equation}
Substitute the update equation of SGD in Eq.(\ref{derivative_tilde_recursive}), we have:
\begin{align} \label{derivative_tilde}
\nabla \tilde{f}(\tilde{w}_t) &= \nabla \tilde{f}(\tilde{w}_{t-1}) - \eta\mathcal{H} ( \nabla \tilde{f}(\tilde{w}_{t-1}) + \xi_{t-1})\nonumber \\
				&= (1-\eta\mathcal{H})\nabla \tilde{f}(\tilde{w}_{t-1}) -\eta \mathcal{H}\xi_{t-1}  \nonumber\\
				&= (1-\eta\mathcal{H})^2\nabla \tilde{f}(\tilde{w}_{t-2}) -\eta \mathcal{H}\xi_{t-1}
				-\eta  \mathcal{H}(1-\eta\mathcal{H})\xi_{t-2}=\cdots\nonumber\\
				& = (1-\eta\mathcal{H})^t\nabla f(w_0) - \eta \mathcal{H}\sum_{\tau=0}^{t-1}(1-\eta\mathcal{H})^{t-\tau-1}\xi_{\tau}
\end{align}
Therefore, we have:
\begin{align} \label{dif_tilde}
\tilde{w}_{t} - w_0 &= -\eta\sum_{\tau=0}^{t-1} (\nabla \tilde{f}(\tilde{w}_\tau) + \xi_\tau) \nonumber \\
&= -\eta\sum_{\tau=0}^{t-1}\left  (
(1-\eta\mathcal{H})^\tau \nabla f(w_0) - \eta \mathcal{H}\sum_{\tau'=0}^{\tau-1}(1-\eta\mathcal{H})^{\tau -\tau'-1}\xi_{\tau'}
 + \xi_\tau\right ) \nonumber \\
 &= -\eta \sum_{\tau = 0}^{t-1}(1-\eta \mathcal{H})^\tau \nabla f(w_0) -\eta
	\sum_{\tau=0}^{t-1}(1-\eta\mathcal{H})^{t-\tau-1}\xi_{\tau} 
\end{align}

Next, we prove the existence of $T$ in Eq.(\ref{choose_t}).
Since $\sum_{\tau=0}^{t}(1+\eta \gamma_0)^{2\tau}$ is monotonically increasing w.r.t $t$, 
and diverge to infinity as $t \rightarrow \infty$. We know there is always some 
$T \in \mathbb{N}$ gives $\frac{d}{\eta\gamma_0} \le \sum_{\tau=0}^{T-1}(1+\eta \gamma_0)^{2\tau}$.
Let $T$ be the smallest integer satisfying above equation.
By assumption, we know $\gamma \le \gamma_0 \le L$, and
\begin{equation}
\sum_{\tau=0}^{t+1}(1+\eta \gamma_0)^{2\tau}
= 1 + (1+\eta \gamma_0)^2\sum_{\tau=0}^{t}(1+\eta \gamma_0)^{2\tau}
\end{equation}
we can choose $\eta_{\max} < \min\{(\sqrt{2}-1)/L, 2d/\gamma\}$ so that
\begin{equation}
\frac{d}{\eta\gamma_0} \le \sum_{\tau=0}^{T-1}(1+\eta \gamma_0)^{2\tau} 
\le 1 + \frac{2d}{\eta\gamma_0} \le \frac{3d}{\eta\gamma_0} 
\end{equation}

Finally, by Eq.(\ref{choose_t}), we know $T = O(\log d/\gamma_0\eta)$, and $(1+\eta \gamma_0)^T \le \tlO(1)$. Also because $\E\xi = 0$ and $\|\xi\| \le Q= \tlO(1)$ with probability 1, then by Hoeffding inequality, we have for each dimension $i$ and time $t\le T$:
\begin{equation}
P\left( |\eta\sum_{\tau=0}^{t-1}(1-\eta\mathcal{H})^{t-\tau-1}\xi_{\tau, i}|
>  \tlO(\eta^{\frac{1}{2}}\log{\frac{1}{\eta}})\right) \le e^{-\tlOmega(\log^2 \frac{1}{\eta})} \le \tlO(\eta^4)
\end{equation}

then by summing over dimension $d$ and taking union bound over all $t\le T$, we directly have:
\begin{equation}
P\left( \forall t\le T, \|\eta\sum_{\tau=0}^{t-1}(1-\eta\mathcal{H})^{t-\tau-1}\xi_{\tau}\|
>  \tlO(\eta^{\frac{1}{2}}\log \frac{1}{\eta})\right) \le \tlO(\eta^3).
\end{equation}
Combine this fact with Eq.(\ref{derivative_tilde}) and Eq.(\ref{dif_tilde}), we finish the proof.

% \jccomment{expand this part of proof a bit more}

\end{proof}

Next we need to prove that the two sequences of updates are always close.

\begin{lemma} \label{lem::saddle_and_maximum}
Under the assumptions of Theorem~\ref{thm:sgdmain_unconstraint}.
and let $\{w_t\}$ be the corresponding sequence generated by running SGD on function $f$.  
Also let $\tilde{f}$ and $\{\tilde{w}_t\}$ be defined as in Lemma \ref{lem::case_Gaussian}.
Then, for any initial point $w_0$ where $\|\nabla f(w_0)\| \le \tlO(\eta) < \epsilon$, and $\lambda_{\min}(\Hess f(w_0)) = -\gamma_0$. Given the choice of $T$ as in Eq.(\ref{choose_t}), 
with probability at least $1-\tlO(\eta^2)$, we have following holds simultaneously for all $t\le T$:
\begin{align}
\|w_t - \tilde{w}_t\| \le \tlO( \eta \log^2\frac{1}{\eta});
\quad\quad
\|\nabla f(w_t) - \nabla \tilde{f}(\tilde{w}_t)\| \le \tlO( \eta \log^2\frac{1}{\eta})
\end{align}
\end{lemma} 

\begin{proof}

First, we have update function of gradient by:
\begin{align}
\nabla f(w_t) = & \nabla f(w_{t-1}) + \int_{0}^{1}\mathcal{H}(w_{t-1} + t(w_t - w_{t-1})) \mathrm{d}t \cdot (w_t - w_{t-1}) \nonumber \\
= & \nabla f(w_{t-1}) + \mathcal{H}(w_{t-1}) (w_t - w_{t-1}) + \theta_{t-1}\label{derivative_recursive}
\end{align}
where the remainder:
\begin{equation}
\theta_{t-1} \equiv \int_{0}^{1}\left[\mathcal{H}(w_{t-1} + t(w_t - w_{t-1})) - \mathcal{H}(w_{t-1})\right] \mathrm{d}t \cdot (w_t - w_{t-1})
\end{equation}
Denote $\mathcal{H} = \mathcal{H}(w_0)$, and $\mathcal{H}'_{t-1} = \mathcal{H}(w_{t-1}) - \mathcal{H}(w_0)$.
By Hessian smoothness, we immediately have:
\begin{align} 
&\|\mathcal{H}'_{t-1}\| = \|\mathcal{H}(w_{t-1}) - \mathcal{H}(w_0)\| 
\le \rho \|w_{t-1} - w_0\| \le \rho (\|w_t - \tilde{w}_t\| + \|\tilde{w}_t - w_0\|) 
\label{H'_smooth} \\
&\|\theta_{t-1}\| \le \frac{\rho}{2} \|w_t - w_{t-1}\|^2 \label{theta_smooth}
\end{align}

Substitute the update equation of SGD (Eq.(\ref{SGD_update})) into Eq.(\ref{derivative_recursive}), we have:
\begin{align}
\nabla f(w_t) &= \nabla f(w_{t-1}) -\eta(\mathcal{H}+ \mathcal{H}'_{t-1}) ( \nabla f(w_{t-1}) + \xi_{t-1}) + \theta_{t-1} \nonumber \\
				&= (1-\eta\mathcal{H})\nabla f(w_{t-1}) - \eta \mathcal{H}\xi_{t-1} -\eta \mathcal{H}'_{t-1}
				(\nabla f(w_{t-1}) + \xi_{t-1})+ \theta_{t-1}  \label{derivative}
\end{align}

Let $\Delta_t = \nabla f(w_t) - \nabla \tilde{f}(\tilde{w}_t)$ denote the difference in gradient, then
from Eq.(\ref{derivative_tilde}), Eq.(\ref{derivative}), and Eq.(\ref{SGD_update}), we have:
\begin{align} \label{Delta_recursive}
&\Delta_t =  (1-\eta \mathcal{H}) \Delta_{t-1} -\eta \mathcal{H}'_{t-1} [\Delta_{t-1} + \nabla \tilde{f}(\tilde{w}_{t-1}) + \xi_{t-1}]  + \theta_{t-1}\\
&w_t - \tilde{w}_t =  -\eta \sum_{\tau = 0}^{t-1} \Delta_\tau \label{dif}
\end{align}

Let filtration $\mathfrak{F}_t = \sigma\{\xi_0, \cdots \xi_{t-1}\}$, and note $\sigma\{\Delta_0, \cdots, \Delta_t \} \subset \mathfrak{F}_t$, where $\sigma\{\cdot\}$ denotes the sigma field. Also, let event $\mathfrak{K}_t = \{\forall \tau \le t, ~\|\nabla\tilde{f}(\tilde{w}_\tau)\| \le \tlO(\eta^{\frac{1}{2}}\log \frac{1}{\eta}), ~ 
\|\tilde{w}_\tau - w_0\| \le \tlO(\eta^{\frac{1}{2}}\log \frac{1}{\eta})\}$,
and $\mathfrak{E}_t = \{\forall \tau \le t, ~\|\Delta_{\tau}\| \le \mu \eta\log^2\frac{1}{\eta}\}$, where $\mu$ is independent of $(\eta, \zeta)$, and will be specified later. Again, $\tilde{O}$ notation in this proof will never hide any dependence on $\mu$.
Clearly, we have 
$\mathfrak{K}_t \subset \mathfrak{K}_{t-1}$
($\mathfrak{E}_t \subset \mathfrak{E}_{t-1}$), 
thus $1_{\mathfrak{K}_t} \le 1_{\mathfrak{K}_{t-1}}$
($1_{\mathfrak{E}_t} \le 1_{\mathfrak{E}_{t-1}}$), where $1_\mathfrak{K}$ is the indicator function of event $\mathfrak{K}$.

We first need to carefully bounded all terms in Eq.(\ref{Delta_recursive}), 
conditioned on event $\mathfrak{K}_{t-1}\cap \mathfrak{E}_{t-1}$, by Eq.(\ref{H'_smooth}), Eq.(\ref{theta_smooth})), and Eq.(\ref{dif}),
with probability 1, for all $t\le T\le O(\log d/\gamma_0 \eta)$, we have: 
\begin{align}
\|(1-\eta \mathcal{H}) \Delta_{t-1}\| \le \tlO(\mu\eta\log^2\frac{1}{\eta}) 
&\quad \quad 
\|\eta  \mathcal{H}'_{t-1} (\Delta_{t-1} + \nabla \tilde{f}(\tilde{w}_{t-1}))\|
\le \tlO(\eta^2 \log^2 \frac{1}{\eta}) 
\nonumber \\
\|\eta\mathcal{H}'_{t-1}\xi_{t-1}\| \le \tlO(\eta^{1.5} \log \frac{1}{\eta})
&\quad \quad \|\theta_{t-1}\| \le \tlO(\eta^2) \label{order_of_term}
\end{align}

Since event $\mathfrak{K}_{t-1}\subset \mathfrak{F}_{t-1},  \mathfrak{E}_{t-1}\subset \mathfrak{F}_{t-1}$ thus independent of $\xi_{t-1}$, we also have:
\begin{align}
&\E [((1-\eta \mathcal{H}) \Delta_{t-1})^T\eta\mathcal{H}'_{t-1}\xi_{t-1} 1_{\mathfrak{K}_{t-1}\cap \mathfrak{E}_{t-1}} ~|~ \mathfrak{F}_{t-1}] \nonumber \\
=& 1_{\mathfrak{K}_{t-1}\cap \mathfrak{E}_{t-1}}((1-\eta \mathcal{H}) \Delta_{t-1})^T\eta\mathcal{H}'_{t-1}\E [\xi_{t-1} ~|~ \mathfrak{F}_{t-1}]
= 0
\end{align}

Therefore, from Eq.(\ref{Delta_recursive}) and Eq.(\ref{order_of_term}):
\begin{align}
&\E [\|\Delta_t\|^2_21_{\mathfrak{K}_{t-1}\cap \mathfrak{E}_{t-1}} ~|~ \mathfrak{F}_{t-1}] \nonumber \\
\le &\left[(1+\eta \gamma_0)^2\|\Delta_{t-1}\|^2 
+ (1+\eta \gamma_0) \|\Delta_{t-1}\| \tlO(\eta^2 \log^2 \frac{1}{\eta}) + \tlO(\eta^{3} \log^2 \frac{1}{\eta})\right]1_{\mathfrak{K}_{t-1}\cap \mathfrak{E}_{t-1}}\nonumber \\
\le& \left[(1+\eta \gamma_0)^2\|\Delta_{t-1}\|^2  + \tlO( \mu
\eta^{3}\log^4 \frac{1}{\eta})\right]1_{\mathfrak{K}_{t-1}\cap \mathfrak{E}_{t-1}} 
\end{align}

Define
\begin{align} \label{martingle_case3}
G_t = (1+\eta \gamma_0)^{-2t} [~\|\Delta_t\|^2 + \alpha 
\eta^{2}\log^4 \frac{1}{\eta}~]
\end{align}
Then, when $\eta_{\max}$ is small enough, we have:
\begin{align}
&\E [G_t1_{\mathfrak{K}_{t-1}\cap \mathfrak{E}_{t-1}} ~|~ \mathfrak{F}_{t-1}]  
= (1+\eta \gamma_0)^{-2t}\left[\E [\|\Delta_t\|^2_21_{\mathfrak{K}_{t-1}\cap \mathfrak{E}_{t-1}} ~|~ \mathfrak{F}_{t-1}]  + \alpha \eta^{2}\log^3 \frac{1}{\eta}\right]
1_{\mathfrak{K}_{t-1}\cap \mathfrak{E}_{t-1}} \nonumber \\
&\le  (1+\eta \gamma_0)^{-2t}\left[(1+\eta \gamma_0)^2\|\Delta_{t-1}\|^2  +  \tlO(\mu\eta^{3}\log^4 \frac{1}{\eta})
+ \alpha \eta^{2}\log^4 \frac{1}{\eta}\right]
1_{\mathfrak{K}_{t-1}\cap \mathfrak{E}_{t-1}} \nonumber \\
&\le  (1+\eta \gamma_0)^{-2t}\left[(1+\eta \gamma_0)^2\|\Delta_{t-1}\|^2 
+ (1+\eta \gamma_0)^2\alpha \eta^{2}\log^4 \frac{1}{\eta}
\right]1_{\mathfrak{K}_{t-1}\cap \mathfrak{E}_{t-1}} \nonumber \\
& =  G_{t-1}1_{\mathfrak{K}_{t-1}\cap \mathfrak{E}_{t-1}} \le G_{t-1}
1_{\mathfrak{K}_{t-2}\cap \mathfrak{E}_{t-2}}
\end{align}
Therefore, we have $\E [G_t1_{\mathfrak{K}_{t-1}\cap \mathfrak{E}_{t-1}} ~|~ \mathfrak{F}_{t-1}] \le G_{t-1}1_{\mathfrak{K}_{t-2}\cap \mathfrak{E}_{t-2}}$ which means 
$G_t1_{\mathfrak{K}_{t-1}\cap \mathfrak{E}_{t-1}}$ is a supermartingale.

On the other hand, we have:
\begin{align}
\Delta_t =  (1-\eta H) \Delta_{t-1} -\eta  \mathcal{H}'_{t-1} (\Delta_{t-1} + \nabla \tilde{f}(\tilde{w}_{t-1})) -\eta \mathcal{H}'_{t-1}\xi_{t-1}  + \theta_{t-1}
\end{align}
Once conditional on filtration $\mathfrak{F}_{t-1}$, the first two terms are deterministic, and only the third and fourth term are random.
Therefore, we know, with probability 1:
\begin{equation}
|~\|\Delta_t\|^2_2 - \E[\|\Delta_t\|^2_2|\mathfrak{F}_{t-1}]~|1_{\mathfrak{K}_{t-1}\cap \mathfrak{E}_{t-1}} \le \tlO(\mu\eta^{2.5} \log^3 \frac{1}{\eta}) 
\end{equation}
Where the main contribution comes from the product of the first term and third term.
Then, with probability 1, we have:
\begin{align}
&| G_t1_{\mathfrak{K}_{t-1}\cap \mathfrak{E}_{t-1}} - \E[G_t1_{\mathfrak{K}_{t-1}\cap \mathfrak{E}_{t-1}}~|~\mathfrak{F}_{t-1}] | \nonumber \\
= &(1+2\eta \gamma_0)^{-2t}\cdot|~\|\Delta_t\|^2_2 - \E[\|\Delta_t\|^2_2|\mathfrak{F}_{t-1}]~|\cdot1_{\mathfrak{K}_{t-1}\cap \mathfrak{E}_{t-1}} 
 \le \tlO(\mu\eta^{2.5} \log^3 \frac{1}{\eta}) = c_{t-1}
\end{align}
By Azuma-Hoeffding inequality, with probability less than $\tlO(\eta^3)$, 
for $t\le T\le O(\log d/\gamma_0 \eta)$:
\begin{equation}
G_t1_{\mathfrak{K}_{t-1}\cap \mathfrak{E}_{t-1}} - G_0\cdot1 > \tlO(1)\sqrt{\sum_{\tau=0}^{t-1}{c^2_\tau}}\log (\frac{1}{\eta}) = \tlO(\mu\eta^{2}\log^4 \frac{1}{\eta})
\end{equation}
This means there exist some $\tilde{C} = \tlO(1)$ so that:
\begin{equation}
	P\left(G_t1_{\mathfrak{K}_{t-1}\cap \mathfrak{E}_{t-1}}  \ge \tilde{C}\mu\eta^{2}\log^4 \frac{1}{\eta}\right) \le \tlO(\eta^3)
\end{equation}
By choosing $\mu>\tilde{C}$, this is equivalent to:
\begin{equation}
	P\left(\mathfrak{K}_{t-1}\cap \mathfrak{E}_{t-1} \cap \left\{\|\Delta_t\|^2 \ge \mu^2 \eta^{2}\log^4 \frac{1}{\eta}\right\}\right) \le \tlO(\eta^3)
\end{equation}
Therefore, combined with Lemma \ref{lem::case_Gaussian}, we have:
\begin{align}
&P\left( \mathfrak{E}_{t-1} \cap \left\{\|\Delta_t\| \ge \mu \eta \log^2 \frac{1}{\eta}\right\}\right)  \nonumber \\
= &P\left(\mathfrak{K}_{t-1}\cap \mathfrak{E}_{t-1} \cap \left\{\|\Delta_t\| \ge \mu\eta \log^2 \frac{1}{\eta}\right\}\right) +
P\left(\overline{\mathfrak{K}}_{t-1}\cap \mathfrak{E}_{t-1} \cap \left\{\|\Delta_t\| \ge \mu\eta \log^2 \frac{1}{\eta}\right\}\right) \nonumber\\
\le &\tlO(\eta^3) + P(\overline{\mathfrak{K}}_{t-1}) \le \tlO(\eta^3)
\end{align}
Finally, we know:
\begin{align}
P(\overline{ \mathfrak{E}}_{t}) = P\left( \mathfrak{E}_{t-1} \cap \left\{\|\Delta_t\| \ge \mu\eta \log^2 \frac{1}{\eta}\right\}\right) + P(\overline{ \mathfrak{E}}_{t-1}) 
\le \tlO(\eta^3) + P(\overline{ \mathfrak{E}}_{t-1}) 
\end{align} 
Because $P(\overline{ \mathfrak{E}}_{0}) =0$, and $T\le \tlO(\frac{1}{\eta})$, we have 
$P(\overline{ \mathfrak{E}}_{T}) \le \tlO(\eta^2)$.
Due to Eq.(\ref{dif}), we have $\|w_t - \tilde{w}_t\| \le  \eta \sum_{\tau = 0}^{t-1} \|\Delta_\tau\|$, then by the definition of $\mathfrak{E}_{T}$, we finish the proof.

\end{proof}

Using the two lemmas above we are ready to prove Lemma~\ref{thm::case2}

\begin{proof}[Proof of Lemma \ref{thm::case2}]
Let $\tilde{f}$ and $\{\tilde{w}_t\}$ be defined as in Lemma \ref{lem::case_Gaussian}.
and also let $\lambda_{\min}(\mathcal{H} (w_0))$ $ =-\gamma_0$. Since $\mathcal{H}(w)$ is 
$\rho$-Lipschitz, for any $w, w_0$, we have:
\begin{equation}
f(w) \le  f(w_0) + \nabla f(w_0)^T (w-w_0) + \frac{1}{2}(w-w_0)^T \mathcal{H}(w_0) (w-w_0)
+ \frac{\rho}{6} \|w-w_0\|^3
\end{equation}

Denote $\tilde{\delta} = \tilde{w}_T - w_0$ and $\delta = w_T - \tilde{w}_T$, we have: 
\begin{align}
	f(w_T) -  f(w_0) \le& \left[\nabla f(w_0)^T (w_T-w_0) + \frac{1}{2}(w_T-w_0)^T \mathcal{H}(w_0) (w_T-w_0)+ \frac{\rho}{6} \|w_T-w_0\|^3\right] \nonumber \\
	=&\left[\nabla f(w_0)^T (\tilde{\delta} + \delta) + \frac{1}{2}(\tilde{\delta} + \delta)^T \mathcal{H}(\tilde{\delta} + \delta) + \frac{\rho}{6}\|\tilde{\delta} + \delta\|^3\right]\nonumber \\
	=&\left[\nabla f(w_0)^T \tilde{\delta} + \frac{1}{2}\tilde{\delta}^T \mathcal{H}\tilde{\delta}\right] 
	+ \left[\nabla f(w_0)^T \delta + 
	\tilde{\delta}^T \mathcal{H}\delta + 
	\frac{1}{2}\delta^T \mathcal{H}\delta + 
	\frac{\rho}{6}\|\tilde{\delta} + \delta\|^3\right]
\end{align}
Where $\mathcal{H} = \mathcal{H}(w_0)$. Denote $\tilde{\Lambda} = \nabla f(w_0)^T \tilde{\delta} + \frac{1}{2}\tilde{\delta}^T \mathcal{H}\tilde{\delta}$ be the first term, and $\Lambda = \nabla f(w_0)^T \delta + 
\tilde{\delta}^T \mathcal{H}\delta + 
\frac{1}{2}\delta^T \mathcal{H}\delta + 
\frac{\rho}{6}\|\tilde{\delta} + \delta\|^3$ be the second term.
We have $f(w_T) -  f(w_0) \le \tilde{\Lambda} + \Lambda$.

Let $\mathfrak{E}_t = \{\forall \tau \le t, 
%~\|\nabla\tilde{f}(\tilde{w}_\tau)\| \le O(\eta^{\frac{1}{2}}\log \frac{1}{\eta}), ~ 
\|\tilde{w}_\tau - w_0\| \le \tlO(\eta^{\frac{1}{2}}\log \frac{1}{\eta}), 
%~\|\nabla f(w_t) - \nabla \tilde{f}(\tilde{w}_t)\| \le O(\eta\log^2 \frac{1}{\eta}), 
~\|w_t - \tilde{w}_t\| \le \tlO(\eta\log^2 \frac{1}{\eta})\}$, by the result of Lemma \ref{lem::case_Gaussian} and Lemma \ref{lem::saddle_and_maximum}, we know $P(\mathfrak{E}_T)\ge 1-\tlO(\eta^2)$. Then, clearly, we have:

\begin{align}\label{Lambda_decomp}
	\E f(w_T) -  f(w_0) =& \E [f(w_T) -  f(w_0)]1_{\mathfrak{E}_T} + \E [f(w_T) -  f(w_0)]1_{\overline{\mathfrak{E}}_T} \nonumber \\
	\le & \E \tilde{\Lambda}1_{\mathfrak{E}_T} + \E \Lambda1_{\mathfrak{E}_T} + \E [f(w_T) -  f(w_0)]1_{\overline{\mathfrak{E}}_T} \nonumber \\
	= & \E \tilde{\Lambda} + \E \Lambda1_{\mathfrak{E}_T} + \E [f(w_T) -  f(w_0)]1_{\overline{\mathfrak{E}}_T} - \E \tilde{\Lambda}1_{\overline{\mathfrak{E}}_T}
\end{align}
We will carefully caculate $\E\tilde{\Lambda}$ term first, and then bound remaining term as ``perturbation'' to first term.

Let $\lambda_1, \cdots, \lambda_d$ be the eigenvalues of $\mathcal{H}$. By the result of lemma \ref{lem::case_Gaussian} and simple linear algebra, we have:
\begin{align} 
	\E\tilde{\Lambda} &= - \frac{\eta}{2}\sum_{i=1}^d\sum_{\tau=0}^{2T-1}(1-\eta\lambda_i)^\tau |\nabla_i  f(w_0)|^2
	+ \frac{1}{2}\sum_{i=1}^d \lambda_i \sum_{\tau=0}^{T-1}(1-\eta\lambda_i)^{2\tau}\eta^2\sigma^2
	\nonumber \\
	&\le \frac{1}{2}\sum_{i=1}^d \lambda_i \sum_{\tau=0}^{T-1}(1-\eta\lambda_i)^{2\tau}\eta^2\sigma^2
	\nonumber\\
	&\le \frac{\eta^2\sigma^2}{2}\left[\frac{d-1}{\eta} - \gamma_0\sum_{\tau=0}^{T-1}(1+\eta \gamma_0)^{2\tau}\right]
	\le -\frac{\eta\sigma^2}{2}\label{Lambda_tilde}
\end{align}
The last inequality is directly implied by the choice of $T$ as in Eq.(\ref{choose_t}).
Also, by Eq.(\ref{choose_t}), we also immediately have that 
$T = O(\log d/\gamma_0\eta) \le O(\log d/\gamma\eta)$. 
Therefore, by choose $T_{max}=O(\log d/\gamma\eta)$ with large enough constant, we have
$T \le T_{max} = O(\log d/\gamma\eta)$.

For bounding the second term, by definition of $\mathfrak{E}_t$, we have:
\begin{align}\label{Lambda_E}
	\E\Lambda 1_{\mathfrak{E}_T}=\E\left[\nabla f(w_0)^T \delta + 
	\tilde{\delta}^T \mathcal{H}\delta + 
	\frac{1}{2}\delta^T \mathcal{H}\delta + 
	\frac{\rho}{6}\|\tilde{\delta} + \delta\|^3\right]1_{\mathfrak{E}_T}
	\le \tlO(\eta^{1.5}\log^3 \frac{1}{\eta})
\end{align}

On the other hand,
since noise is bounded as $\|\xi\| \le \tlO(1)$, from the results of Lemma \ref{lem::case_Gaussian}, it's easy to show $\|\tilde{w} - w_0\| = \|\tilde{\delta}\| \le \tlO(1)$ is also bounded with probability 1. Recall the assumption that function $f$ is also bounded, then we have:
\begin{align}\label{Lambda_E_bar}
	&\E [f(w_T) -  f(w_0)]1_{\overline{\mathfrak{E}}_T} - \E \tilde{\Lambda}1_{\overline{\mathfrak{E}}_T} \nonumber \\
	= &\E [f(w_T) -  f(w_0)]1_{\overline{\mathfrak{E}}_T} - \E\left[\nabla f(w_0)^T \tilde{\delta} + \frac{1}{2}\tilde{\delta}^T \mathcal{H}\tilde{\delta}\right]1_{\overline{\mathfrak{E}}_T} 
	\le \tlO(1)P(\overline{\mathfrak{E}}_T) \le \tlO(\eta^2)
\end{align}

Finally, substitute Eq.(\ref{Lambda_tilde}), Eq.(\ref{Lambda_E}) and Eq.(\ref{Lambda_E_bar})
into Eq.(\ref{Lambda_decomp}), we finish the proof.
\end{proof}

Finally, we combine three cases to prove the main theorem.

\begin{proof} [Proof of Theorem \ref{thm:sgdmain_unconstraint}]
Let's set $\mathcal{L}_1 = \{w ~|~ \|\nabla f(w)\| \ge \sqrt{2\eta\sigma^2\beta d}\}$, $\mathcal{L}_2 = \{w ~|~ \|\nabla f(w)\| \le \sqrt{2\eta\sigma^2\beta d}$ and $\lambda_{\min} (\mathcal{H}(w) ) \le -\gamma\}$, and $\mathcal{L}_3 = \mathcal{L}^c_1 \cup \mathcal{L}^c_2$.
By choosing small enough $\eta_{\max}$, we could make $\sqrt{2\eta\sigma^2\beta d} < \min\{\epsilon, \alpha \delta\}$. Under this choice, we know from Definition~\ref{def:robustcondition} of ($\alpha, \gamma, \epsilon, \delta$)-\name that $\mathcal{L}_3$ is the locally $\alpha$-strongly convex region which is $\tlO(\sqrt{\eta})$-close to some local minimum.

We shall first prove that within $\tlO(\frac{1}{\eta^2}\log \frac{1}{\zeta})$ steps with probability at least $1-\zeta/2$ one of $w_t$ is in $\mathcal{L}_3$. Then by Lemma~\ref{thm::case3} we know with probability at most $\zeta/2$ there exists a $w_t$ that is in $\mathcal{L}_3$ but the last point is not. By union bound we will get the main result.

To prove within $\tlO(\frac{1}{\eta^2}\log \frac{1}{\zeta})$ steps with probability at least $1-\zeta/2$ one of $w_t$ is in $\mathcal{L}_3$, we first show starting from any point, in $\tlO(\frac{1}{\eta^2})$ steps with probability at least $1/2$ one of $w_t$ is in $\mathcal{L}_3$. Then we can repeat this $\log 1/\zeta$ times to get the high probability result.

Define stochastic process $\{\tau_i\}$ s.t. $\tau_0 = 0$, and 
\begin{equation}
\tau_{i+1} = 
\begin{cases}
	\tau_i + 1 &\mbox{if~} w_{\tau_i} \in \mathcal{L}_1 \cup \mathcal{L}_3\\
	\tau_i + T(w_{\tau_i}) &\mbox{if~} w_{\tau_i} \in \mathcal{L}_2
\end{cases}
\end{equation}
Where $T(w_{\tau_i})$ is defined by Eq.(\ref{choose_t}) with $\gamma_0= \lambda_{\min}(\mathcal{H}(w_{\tau_i}))$and we know $T \le T_{max} = \tlO(\frac{1}{\eta})$.

By Lemma \ref{thm::case1} and Lemma \ref{thm::case2}, we know:
\begin{align}
&\E [f(w_{\tau_{i+1}}) - f(w_{\tau_i})|w_{\tau_i} \in \mathcal{L}_1, \mathfrak{F}_{\tau_i - 1} ] 
= \E [f(w_{\tau_{i+1}}) - f(w_{\tau_i})|w_{\tau_i} \in \mathcal{L}_1] \le -\tlO(\eta^2) \\
&\E [f(w_{\tau_{i+1}}) - f(w_{\tau_i})|w_{\tau_i} \in \mathcal{L}_2, \mathfrak{F}_{\tau_i - 1} ] 
= \E [f(w_{\tau_{i+1}}) - f(w_{\tau_i})|w_{\tau_i} \in \mathcal{L}_2] \le -\tlO(\eta)
\end{align}
Therefore, combine above equation, we have:
\begin{equation}
\E [f(w_{\tau_{i+1}}) - f(w_{\tau_i})|w_{\tau_i} \not\in \mathcal{L}_3, \mathfrak{F}_{\tau_i - 1} ] 
= \E [f(w_{\tau_{i+1}}) - f(w_{\tau_i})|w_{\tau_i} \not\in \mathcal{L}_3] \le -(\tau_{i+1}-\tau_i)\tlO(\eta^2)
\end{equation}

Define event $\mathfrak{E}_i = \{\exists j \le i, ~ w_{\tau_j} \in \mathcal{L}_3\}$, clearly
$\mathfrak{E}_i \subset \mathfrak{E}_{i+1}$, thus $P(\mathfrak{E}_i) \le P(\mathfrak{E}_{i+1}) $. Finally,
consider $f(w_{\tau_{i+1}})1_{\mathfrak{E}_{i}}$, we have:
\begin{align}
	\E f(w_{\tau_{i+1}})1_{\mathfrak{E}_{i}} - \E f(w_{\tau_i})1_{\mathfrak{E}_{i-1}}
	&\le B \cdot P(\mathfrak{E}_{i} - \mathfrak{E}_{i-1}) + \E[f(w_{\tau_{i+1}}) - f(w_{\tau_{i}})|\overline{\mathfrak{E}_{i}}]\cdot P(\overline{\mathfrak{E}_{i}}) \nonumber \\
	&\le B \cdot P(\mathfrak{E}_{i}- \mathfrak{E}_{i-1}) -(\tau_{i+1}-\tau_i)\tlO(\eta^2)P(\overline{\mathfrak{E}_{i}})
\end{align}
Therefore, by summing up over $i$, we have: 
\begin{equation}
	\E f(w_{\tau_{i}})1_{\mathfrak{E}_{i}} -  f(w_{0})
	\le BP(\mathfrak{E}_{i}) -\tau_{i}\tlO(\eta^2)P(\overline{\mathfrak{E}_{i}})
	\le B-\tau_{i}\tlO(\eta^2)P(\overline{\mathfrak{E}_{i}})
\end{equation}
Since $|f(w_{\tau_{i}})1_{\mathfrak{E}_{i}}| < B$ is bounded, as $\tau_i$ grows to as large as $ \frac{6B}{\eta^2}$, we must have $P(\overline{\mathfrak{E}_{i}}) < \frac{1}{2}$.
That is, after $\tlO(\frac{1}{\eta^2})$ steps, with at least probability $1/2$, $\{w_t\}$ have at least enter $\mathcal{L}_3$ once. Since this argument holds for any starting point, we can repeat this $\log 1/\zeta$ times and we know after $\tlO(\frac{1}{\eta^2}\log 1/\zeta)$ steps, with probability at least $1-\zeta/2$, $\{w_t\}$ have at least enter $\mathcal{L}_3$ once.

%By Theorem \ref{thm::case3}, we know once enter $\mathcal{L}_3$ region, (with proper choice of $\eta$) the probability that $\{w_t\}$ will escape from $O(\eta \log \frac{1}{\eta})$ neigborhood of local minima in $O(\frac{1}{\eta^2})$ steps is less than $O(\eta)<\frac{\zeta}{2}$.
Combining with Lemma~\ref{thm::case3}, and by union bound we know
 after $\tlO(\frac{1}{\eta^2}\log 1/\zeta)$ steps, with probability at least $1-\zeta$, $w_t$
will be in the $\tlO(\sqrt{\eta \log \frac{1}{\eta\zeta}})$ neigborhood of some local minimum.
\end{proof}

%!TEX root = saddlepoint.tex

\section{Detailed Analysis for Section~\ref{sec:sgd} in Constrained Case}
\label{sec:constrained}

So far, we have been discussed all about unconstrained problem. In this section we extend our result to equality constraint problems under some mild conditions.

Consider the equality constrained optimization problem:
\begin{align}\label{eq_constraint_problem}
&\min_w \quad \quad  f(w) \\
&\text{s.t.} \quad \quad c_i(w) = 0, \quad \quad i=1, \cdots, m \nonumber
\end{align}
Define the feasible set as the set of points that satisfy all the constraints $\mathcal{W} = \{w~|~c_i(w) = 0; ~ i=1, \cdots, m \}$.

In this case, the algorithm we are running is Projected Noisy Gradient Descent.
Let function $\Pi_{\mathcal{W}}(v)$ to be the projection to the feasible set， where the projection is 
defined as the global solution of $\min_{w \in \mathcal{W}} \|v-w\|^2$.

With same argument as in the unconstrained case, we could slightly simplify and convert it to standard projected stochastic gradient descent (PSGD) with update equation:
\begin{align} \label{PSGD_update}
	&v_t = w_{t-1} - \eta \nabla f(w_{t-1}) + \xi_{t-1} \\
	&w_t = \Pi_{\mathcal{W}}(v_t)
\end{align}
As in unconstrained case, we are interested in noise $\xi$ is i.i.d satisfying $\E \xi = 0$, 
$ \E \xi\xi^T = \sigma^2I$ and $\|\xi\| \le Q$ almost surely.
Our proof can be easily extended to Algorithm \ref{algo:psgdwn} with $\frac{1}{d} I \preceq \E\xi\xi^T\preceq (Q+\frac{1}{d})I$. % slightly different noise assumption.
In this section we first introduce basic tools for handling constrained optimization problems (most these materials can be found in~\cite{wright1999numerical}), then we prove some technical lemmas that are useful for dealing with the projection step in PSGD, finally we point out how to modify the previous analysis.

\subsection{Preliminaries}

Often for constrained optimization problems we want the constraints to satisfy some regularity conditions. LICQ (linear independent constraint quantification) is a common assumption in this context.

\begin{definition}[LICQ]
In equality-constraint problem Eq.(\ref{eq_constraint_problem}), given a point $w$, 
we say that the linear independence constraint qualification (LICQ) holds if the set of constraint gradients $\{\nabla c_i(x), i=1, \cdots, m\}$ is linearly independent.
\end{definition}

In constrained optimization, we can locally transform it to an unconstrained problem by introducing Lagrangian multipliers. The Langrangian $\mathcal{L}$ can be written as
\begin{equation}
\mathcal{L}(w, \lambda) =  f(w) - \sum_{i=1}^m \lambda_i c_i(w)
\end{equation}

Then, if LICQ holds for all $w \in \mathcal{W}$, we can properly define function $\lambda^*(\cdot)$ to be:
\begin{equation}
\lambda^*(w) = \arg\min_{\lambda} \|\nabla f(w) - \sum_{i=1}^m \lambda_i \nabla c_i(w)\|
= \arg\min_{\lambda} \|\nabla_w \mathcal{L}(w, \lambda)\|
\end{equation}
where $\lambda^*(\cdot)$ can be calculated analytically:
let matrix $C(w) = (\nabla c_1(w), \cdots, \nabla c_m(w))$, then we have:
\begin{equation} \label{Lambda_star}
\lambda^*(w)  = C(w)^{\dagger}\nabla f(w) = (C(w)^TC(w))^{-1}C(w)^T\nabla f(w)
\end{equation}
where $(\cdot)^{\dagger}$ is Moore-Penrose pseudo-inverse.

In our setting we need a stronger regularity condition which we call robust LICQ (RLICQ).

\begin{definition}[\nameCQ]
In equality-constraint problem Eq.(\ref{eq_constraint_problem}), given a point $w$, 
we say that $\alpha_c$-robust linear independence constraint qualification (\nameCQ) holds if the minimum singular value of matrix $C(w) = (\nabla c_1(w), \cdots, \nabla c_m(w))$ is greater or equal to $\alpha_c$, that is $\sigma_{\min} (C(w)) \ge \alpha_c$.
\end{definition}

\begin{remark}
Given a point $w\in \mathcal{W}$, \nameCQ implies LICQ. 
While LICQ holds for all $w \in \mathcal{W}$ is a necessary condition for
$\lambda^*(w)$ to be well-defined;
it's easy to check that \nameCQ holds for all $w \in \mathcal{W}$ is a necessary condition for
$\lambda^*(w)$ to be bounded. Later, we will also see \nameCQ combined with the smoothness of $\{c_i(w)\}_{i=1}^m$ guarantee the curvature of constraint manifold to be bounded everywhere.
\end{remark}

Note that we require this condition in order to provide a quantitative bound, without this assumption there can be cases that are exponentially close to a function that does not satisfy LICQ.

We can also write down the first-order and second-order partial derivative of Lagrangian $\mathcal{L}$
at point $(w, \lambda^*(w))$:
\begin{align}
&\chi(w) = \nabla_w \mathcal{L}(w, \lambda) |_{(w, \lambda^*(w))} =\nabla f(w) - \sum_{i=1}^m \lambda^*_i(w) \nabla c_i(w)  \label{Lagrangian_1}\\
&\mathfrak{M}(w) = \nabla^2_{ww} \mathcal{L}(w, \lambda) |_{(w, \lambda^*(w))} = \Hess f(w) - \sum_{i=1}^m \lambda^*_i(w) \nabla^2 c_i(w) \label{Lagrangian_2}
\end{align}

\begin{definition}[Tangent Space and Normal Space]
Given a feasible point $w \in \mathcal{W}$, define its corresponding Tangent Space to be $\mathcal{T}(w) = \{v ~|~\nabla c_i(w)^T v = 0; ~ i=1, \cdots, m \}$, and Normal Space to be $\mathcal{T}^c(w) = \text{span}\{\nabla c_1(w) \cdots, \nabla c_m(w) \}$
\end{definition}
If $w \in \mathcal{R}^d$, and we have $m$ constraint satisfying \nameCQ, the tangent space would be
a linear subspace with dimension $d-m$; and the normal space would be a linear subspace with dimension $m$.
We also know immediately that $\chi(w)$ defined in Eq.(\ref{Lagrangian_1}) has another interpretation: it's the component of gradient $\nabla f(w)$ in tangent space.

Also, it's easy to see the normal space $\mathcal{T}^c(w)$ is the orthogonal complement of 
$\mathcal{T}$. We can also define the projection matrix of any vector onto tangent space (or normal space) to be $P_{\mathcal{T}(w)}$ (or $P_{\mathcal{T}^c(w)}$). Then, clearly, both 
$P_{\mathcal{T}(w)}$ and $P_{\mathcal{T}^c(w)}$ are orthoprojector, thus symmetric.
Also by Pythagorean theorem, we have:
\begin{equation}
\|v\|^2 = \|P_{\mathcal{T}(w)} v\|^2 + \|P_{\mathcal{T}^c(w)}v\|^2, \quad \quad \forall
v\in \mathbb{R}^d
\end{equation}

\paragraph{Taylor Expansion} Let $w, w_0 \in \mathcal{W}$, and fix $\lambda^* = \lambda^*(w_0)$ independent of $w$, assume 
$\nabla^2_{ww} \mathcal{L}(w, \lambda^*)$ is $\rho_L$-Lipschitz,
that is $\|\nabla^2_{ww} \mathcal{L}(w_1, \lambda^*) - \nabla^2_{ww} \mathcal{L}(w_2, \lambda^*) \| \le \rho_L \|w_1 - w_2\|$
By Taylor expansion, we have:
\begin{align}
\mathcal{L}(w, \lambda^*) \le &\mathcal{L}(w_0, \lambda^*) + \nabla_w \mathcal{L}(w_0, \lambda^*)^T (w-w_0)  \nonumber \\
& + \frac{1}{2}(w-w_0)^T \nabla^2_{ww} \mathcal{L}(w_0, \lambda^*) (w-w_0) + \frac{\rho_L}{6}\|w-w_0\|^3
\end{align}
Since $w, w_0$ are feasible, we know:
$\mathcal{L}(w, \lambda^*) = f(w)$ and $\mathcal{L}(w_0, \lambda^*) = f(w_0)$, 
this gives:
\begin{align} \label{Taylor_eq_constraint}
f(w)\le f(w_0) + \chi(w_0)^T (w-w_0) + \frac{1}{2}(w-w_0)^T \mathfrak{M}(w_0)(w-w_0) + \frac{\rho_L}{6}\|w-w_0\|^3
\end{align}

% Since $\lambda^*(w)$ is bounded, then exists $B_{\lambda}$ so that $\|\lambda^*(w) \|_{\infty} \le B_{\lambda}$, then clearly $\nabla^2_{ww} L(w, \lambda) |_{(w, \lambda^*(w))}$ is $\rho_L$-Lipschitz,
% where $\rho_L = \rho + B\sum_i \rho_i$. 

\paragraph{Derivative of $\chi(w)$} By taking derative of $\chi(w)$ again, we know the change of this tangent gradient can be characterized by:
\begin{align}
	\nabla \chi(w) = \mathcal{H} - \sum_{i=1}^m \lambda^*_i(w) \nabla^2 c_i(w) 
	-\sum_{i=1}^m  \nabla c_i(w) \nabla \lambda^*_i(w)^T
\end{align}
Denote 
\begin{equation}
\mathfrak{N}(w) =-\sum_{i=1}^m  \nabla c_i(w) \nabla \lambda^*_i(w)^T
\end{equation}
We immediately know that $\nabla \chi(w) = \mathfrak{M}(w) + \mathfrak{N}(w)$.

\begin{remark}\label{N_constraint}
The additional term $\mathfrak{N}(w)$ is not necessary to be even symmetric in general. This is due to the fact that $\chi(w)$ may not be the gradient of any scalar function. However, 
$\mathfrak{N}(w)$ has an important property that is: for any vector $v \in \mathbb{R}^d$, $\mathfrak{N}(w)v \in \mathcal{T}^c(w)$.
\end{remark}

Finally, for completeness, we state here the first/second-order necessary (or sufficient) conditions for optimality. Please refer to \cite{wright1999numerical} for the proof of those theorems.

\begin{theorem}[First-Order Necessary Conditions] \label{thm::first_necessary}
In equality constraint problem Eq.(\ref{eq_constraint_problem}), suppose that $w^\dagger$ is a local solution,
and that the functions $f$ and $c_i$ are continuously differentiable, and that the LICQ holds at $w^\dagger$.
Then there is a Lagrange multiplier vector $\lambda^\dagger$, such that:
\begin{align}
\nabla_w \mathcal{L}(w^\dagger, \lambda^\dagger) &= 0 \\
c_i(w^\dagger) &= 0, \quad \quad \text{for~} i=1, \cdots, m
\end{align}
These conditions are also usually referred as Karush-Kuhn-Tucker (KKT) conditions.
\end{theorem}

\begin{theorem}[Second-Order Necessary Conditions] \label{thm::second_necessary}
In equality constraint problem Eq.(\ref{eq_constraint_problem}), suppose that $w^\dagger$ is a local solution,
and that the LICQ holds at $w^\dagger$.
Let $\lambda^\dagger$ Lagrange multiplier vector for which the KKT conditions are satisfied. Then:
\begin{align}
v^T\nabla^2_{xx} \mathcal{L}(w^\dagger, \lambda^\dagger)v \ge 0 \quad \quad 
\text{for all~} v\in \mathcal{T}(w^\dagger)
\end{align}
\end{theorem}

\begin{theorem}[Second-Order Sufficient Conditions] \label{thm::second_sufficient}
In equality constraint problem Eq.(\ref{eq_constraint_problem}), suppose that for some feasible point $w^\dagger \in \mathbb{R}^d$, and there's  Lagrange multiplier vector
$\lambda^\dagger$  for which the KKT conditions are satisfied. Suppose also that:
\begin{align}
v^T\nabla^2_{xx} \mathcal{L}(w^\dagger, \lambda^\dagger)v > 0 \quad \quad 
\text{for all~} v\in \mathcal{T}(w^\dagger), v\neq 0
\end{align}
Then $w^\dagger$ is a strict local solution.
\end{theorem}

\begin{remark}
By definition Eq.(\ref{Lambda_star}), we know immediately $\lambda^*(w^\dagger)$ is one of valid Lagrange multipliers $\lambda^\dagger$ for which the KKT conditions are satisfied. This means $\chi(w^\dagger) = 
\nabla_w \mathcal{L}(w^\dagger, \lambda^\dagger)$ and $\mathfrak{M}(w^\dagger) = \mathcal{L}(w^\dagger, \lambda^\dagger)$.
\end{remark}

Therefore, Theorem \ref{thm::first_necessary}, \ref{thm::second_necessary}, \ref{thm::second_sufficient} gives strong implication that $\chi(w)$ and $\mathfrak{M}(w)$ are the right thing to look at, which are in some sense equivalent to $\nabla f(w)$ and $\Hess f(w)$ in unconstrained case.

\subsection{Geometrical Lemmas Regarding Constraint Manifold}

Since in equality constraint problem, at each step of PSGD, we are effectively considering the local manifold around feasible point $w_{t-1}$.
In this section, we provide some technical lemmas relating to the geometry of constraint manifold in preparsion for the proof of main theorem in equality constraint case.

%Recall we denote the feasible set as $\mathcal{W} = \{w~|~c_i(w) = 0; ~ i=1, \cdots, m \}$.
We first show if two points are close, then the projection in the normal space is much smaller than the projection in the tangent space.

\begin{lemma} \label{lem::normal_by_tangent}
Suppose the constraints $\{c_i\}_{i=1}^m$ are $\beta_i$-smooth, 
and \nameCQ holds for all $w\in \mathcal{W}$.
Then, let $\sum_{i=1}^m\frac{\beta_i^2}{ \alpha^2_c} = \frac{1}{R^2}$,
for any $w, w_0 \in \mathcal{W}$, let $\mathcal{T}_0 = \mathcal{T}(w_0)$, then
\begin{equation}
\|P_{\mathcal{T}^c_0} (w-w_0)\| \le \frac{1}{2R} \|w-w_0\|^2
\end{equation}
Furthermore, if $\|w-w_0\| < R$ holds, 
we additionally have:
\begin{equation}
\|P_{\mathcal{T}^c_0} (w-w_0)\| \le \frac{\|P_{\mathcal{T}_0} (w-w_0)\|^2}{R}
\end{equation}
\end{lemma}

\begin{proof}
First, since for any vector $\hat{v} \in \mathcal{T}_0$, we have 
$\|C(w_0)^T \hat{v}\| = 0$, then
by simple linear algebra, it's easy to show:
\begin{align} \label{C1_constraint}
\|C(w_0)^T (w-w_0)\|^2  =&  \|C(w_0)^T P_{\mathcal{T}^c_0} (w-w_0)\|^2 
\ge \sigma^2_{\min} \|P_{\mathcal{T}^c_0} (w-w_0)\|^2 \nonumber \\
\ge & \alpha_c^2 \|P_{\mathcal{T}^c_0} (w-w_0)\|^2
\end{align}
On the other hand, by $\beta_i$-smooth, we have:
\begin{align}
|c_i(w) - c_i(w_0) - \nabla c_i(w_0) ^T (w-w_0)| \le \frac{\beta_i}{2} \|w-w_0\|^2
\end{align}
Since $w, w_0$ are feasible points, we have $c_i(w) = c_i(w_0) = 0$, 
which gives:
\begin{equation} \label{C2_constraint}
\|C(w_0)^T (w-w_0)\|^2 = \sum_{i=1}^m (\nabla c_i(w_0) ^T (w-w_0))^2
\le \sum_{i=1}^m \frac{\beta_i^2}{4} \|w-w_0\|^4
\end{equation}
Combining Eq.(\ref{C1_constraint}) and Eq.(\ref{C2_constraint}), and the definition of $R$, 
we have:
\begin{equation}
\|P_{\mathcal{T}^c_0} (w-w_0)\|^2 \le \frac{1}{4R^2} \|w-w_0\|^4
= \frac{1}{4R^2}(\|P_{\mathcal{T}^c_0} (w-w_0)\|^2 + \|P_{\mathcal{T}_0} (w-w_0)\|^2)^2
\end{equation}
Solving this second-order inequality gives two solution
\begin{equation}
	\|P_{\mathcal{T}^c_0} (w-w_0)\| \le \frac{\|P_{\mathcal{T}_0} (w-w_0)\|^2}{R}
	\quad \text{or} \quad \|P_{\mathcal{T}^c_0} (w-w_0)\| \ge R
\end{equation}
By assumption, we know $\|w-w_0\| < R$ (so the second case cannot be true), which finishes the proof.
\end{proof}

Here, we see the $\sqrt{\sum_{i=1}^m\frac{\beta_i^2}{ \alpha^2_c}} = \frac{1}{R}$ serves as a 
upper bound of the curvatures on the constraint manifold, and equivalently, $R$ serves as a lower bound of the radius of curvature. \nameCQ and smoothness guarantee that the curvature is bounded.

Next we show the normal/tangent space of nearby points are close.

\begin{lemma} \label{lem::normal}
Suppose the constraints $\{c_i\}_{i=1}^m$ are $\beta_i$-smooth, 
and \nameCQ holds for all $w\in \mathcal{W}$.
Let $\sum_{i=1}^m\frac{\beta_i^2}{ \alpha^2_c} = \frac{1}{R^2}$,
for any $w, w_0 \in \mathcal{W}$, let $\mathcal{T}_0 = \mathcal{T}(w_0)$.
Then for all $\hat{v} \in \mathcal{T}(w)$ so that $\|\hat{v}\| = 1$, we have 
\begin{equation}
\|P_{\mathcal{T}^c_0} \cdot \hat{v}\|  \le \frac{\|w-w_0\|}{R}
\end{equation}
\end{lemma}

\begin{proof}
With similar calculation as Eq.(\ref{C1_constraint}), we immediately have:
\begin{align} \label{C3_constraint}
\|P_{\mathcal{T}^c_0} \cdot \hat{v}\|^2
\le \frac{\|C(w_0)^T \hat{v}\|^2 }{\sigma^2_{\min}(C(w))}
\le \frac{\|C(w_0)^T \hat{v}\|^2 }{\alpha_c^2}
\end{align}
Since $\hat{v} \in \mathcal{T}(w)$ , we have $C(w)^T \hat{v} = 0$, 
combined with the fact that $\hat{v}$ is a unit vector, we have:
\begin{align}
\|C(w_0)^T \hat{v}\|^2 = &\|[C(w_0) - C(w)]^T \hat{v}\|^2
= \sum_{i=1}^m ([\nabla c_i(w_0) - \nabla c_i(w)]^T \hat{v})^2 \nonumber \\
\le & \sum_{i=1}^m \|\nabla c_i(w_0) - \nabla c_i(w)\|^2 \|\hat{v}\|^2
\le \sum_{i=1}^m \beta_i^2\|w_0 - w\|^2 \label{C4_constraint}
\end{align}
Combining Eq.(\ref{C3_constraint}) and Eq.(\ref{C4_constraint}), and the definition of $R$, 
we concludes the proof.
\end{proof}

\begin{lemma} \label{lem::tangent}
Suppose the constraints $\{c_i\}_{i=1}^m$ are $\beta_i$-smooth, 
and \nameCQ holds for all $w\in \mathcal{W}$.
Let $\sum_{i=1}^m\frac{\beta_i^2}{ \alpha^2_c} = \frac{1}{R^2}$,
for any $w, w_0 \in \mathcal{W}$, let $\mathcal{T}_0 = \mathcal{T}(w_0)$.
Then for all $\hat{v} \in \mathcal{T}^c(w)$ so that $\|\hat{v}\| = 1$, we have 
\begin{equation}
\|P_{\mathcal{T}_0} \cdot \hat{v}\| \le \frac{\|w-w_0\|}{R}
\end{equation}
\end{lemma}

\begin{proof}
By definition of projection, clearly, we have
$P_{\mathcal{T}_0} \cdot \hat{v} + P_{\mathcal{T}^c_0} \cdot \hat{v} = \hat{v}$.
Since $\hat{v} \in \mathcal{T}^c(w)$, without loss of generality, assume 
$\hat{v} = \sum_{i=1}^m \lambda_i \nabla c_i(w)$.
Define $\tilde{d} = \sum_{i=1}^m \lambda_i \nabla c_i(w_0)$, 
clearly $\tilde{d} \in \mathcal{T}^c_0$. Since projection gives the closest point in subspace, 
we have:
\begin{align} \label{C5_constraint}
\|P_{\mathcal{T}_0} \cdot \hat{v}\| = &\| P_{\mathcal{T}^c_0} \cdot \hat{v} - \hat{v}\|
\le  \|\tilde{d} - \hat{v}\| \nonumber \\
\le&  \sum_{i=1}^m \lambda_i \|\nabla c_i(w_0) - \nabla c_i(w)\|
\le \sum_{i=1}^m \lambda_i \beta_i \|w_0 - w\|
\end{align}
On the other hand, let $\lambda = (\lambda_1, \cdots, \lambda_m)^T$, 
we know $C(w) \lambda = \hat{v}$, thus:
\begin{equation}
\lambda  = C(w)^{\dagger}\hat{v} = (C(w)^TC(w))^{-1}C(w)^T\hat{v}
\end{equation}
Therefore, by \nameCQ and the fact $\hat{v}$ is unit vector, we know: $\|\lambda\| \le \frac{1}{\alpha_c}$.
Combined with Eq.(\ref{C5_constraint}), we finished the proof.
\end{proof}

Using the previous lemmas, we can then prove that: starting from any point $w_0$ on constraint manifold, the result of
adding any small vector $v$ and then projected back to feasible set, is not very different from 
the result of adding $P_{\mathcal{T}(w_0)} v$.

\begin{lemma} \label{lem::projection_distance}
Suppose the constraints $\{c_i\}_{i=1}^m$ are $\beta_i$-smooth, 
and \nameCQ holds for all $w\in \mathcal{W}$.
Let $\sum_{i=1}^m\frac{\beta_i^2}{ \alpha^2_c} = \frac{1}{R^2}$,
for any $w_0 \in \mathcal{W}$, let $\mathcal{T}_0 = \mathcal{T}(w_0)$.
Then let $w_1 = w_0 + \eta \hat{v}$, and 
$w_2  = w_0 + \eta P_{\mathcal{T}_0}\cdot \hat{v}$, 
where $\hat{v} \in \mathbb{S}^{d-1}$ is a unit vector.
Then, we have:
\begin{equation}
%\|\Pi_{\mathcal{W}}(w_2) - w_2\| \le O(\eta^2) \quad \quad 
\|\Pi_{\mathcal{W}}(w_1) - w_2\| \le \frac{4\eta^2}{R}
\end{equation}
Where projection $\Pi_{\mathcal{W}}(w)$ is defined as the closet point to $w$ on feasible set $\mathcal{W}$.
\end{lemma}

\begin{proof}
First, note that $\|w_1 - w_0\| = \eta$, and 
by definition of projection, there must exist a project $\Pi_{\mathcal{W}}(w)$ inside the ball
$\mathbb{B}_\eta (w_1) = \{w~|~ \|w-w_1\| \le \eta \}$.

Denote $u_1 = \Pi_{\mathcal{W}}(w_1)$, and clearly $u_1 \in \mathcal{W}$.
we can formulate $u_1$ as the solution to following constrained optimization problems:
\begin{align}
&\min_u \quad \quad  \|w_1 - u\|^2 \\
&\text{s.t.} \quad \quad c_i(u) = 0, \quad \quad i=1, \cdots, m \nonumber
\end{align}
Since function $f(u) = \|w_1 - u\|^2$ and $c_i(u)$ are continuously differentiable by assumption, 
and the condition \nameCQ holds for all $w\in \mathcal{W}$ implies that LICQ holds for $u_1$. Therefore, by 
Karush-Kuhn-Tucker necessary conditions, 
we immediately know $(w_1 - u_1) \in \mathcal{T}^c(u_1)$.

Since $u_1 \in \mathbb{B}_\eta (w_1)$, we know $\|w_0 - u_1\| \le 2 \eta$, by Lemma \ref{lem::tangent}, 
we immediately have:
\begin{equation}
\|P_{\mathcal{T}_0} (w_1 - u_1)\| 
= \frac{\|P_{\mathcal{T}_0} (w_1 - u_1)\| }{\|w_1 - u_1\| }\|w_1 - u_1\| 
\le \frac{1}{R}\|w_0 - u_1\|\cdot \|w_1 - u_1\| \le \frac{2}{R} \eta^2
\end{equation}

Let $v_1 =  w_0 + P_{\mathcal{T}_0}(u_1 - w_0)$, 
we have:
\begin{align} \label{F1_constraint}
\|v_1 - w_2\|= &\|(v_1 - w_0) - (w_2-w_0)\| = \|P_{\mathcal{T}_0}(u_1 - w_0) - P_{\mathcal{T}_0}(w_1 - w_0)\|
\nonumber \\
= &\|P_{\mathcal{T}_0} (w_1 - u_1)\| \le \frac{2}{R} \eta^2
\end{align}

On the other hand by Lemma \ref{lem::normal_by_tangent}, we have:
\begin{equation} \label{F2_constraint}
\|u_1 - v_1\| = \|P_{\mathcal{T}^c_0} (u_1-w_0)\| \le \frac{1}{2R} \|u_1-w_0\|^2 \le \frac{2}{R} \eta^2
\end{equation}
Combining Eq.(\ref{F1_constraint}) and Eq.(\ref{F2_constraint}), we finished the proof.

\end{proof}

\subsection{Main Theorem}

Now we are ready to prove the main theorems. First we revise the definition of \name~in the constrained case.

\begin{definition}
\label{def:robustcondition_constraint}
A twice differentiable function $f(w)$ with constraints $c_i(w)$ is $(\alpha, \gamma, \epsilon, \delta)$-{\em\name}, if for any point $w$ one of the following is true
\begin{enumerate}
\item $\|\chi(w)\| \ge \epsilon$.
\item $\hat{v}^T  \mathfrak{M}(w) \hat{v}  \le -\gamma$ for some $\hat{v} \in \mathcal{T}(w)$, $\|\hat{v}\| = 1$
\item There is a local minimum $w^\star$ such that $\|w-w^\star\| \le \delta$, and for all $w'$ in the $2\delta$ neighborhood of $w^\star$, we have $\hat{v}^T  \mathfrak{M}(w') \hat{v}  \ge \alpha$ for all $\hat{v} \in \mathcal{T}(w')$, $\|\hat{v}\| = 1$
\end{enumerate}
\end{definition}

Next, we prove a equivalent formulation for PSGD.

\begin{lemma}\label{PSGD_equivalent}
Suppose the constraints $\{c_i\}_{i=1}^m$ are $\beta_i$-smooth, 
and \nameCQ holds for all $w\in \mathcal{W}$. Furthermore, 
if function $f$ is $L$-Lipschitz, and the noise $\xi$ is bounded, then
running PSGD as in Eq.(\ref{PSGD_update}) is equivalent to running:
\begin{equation}\label{PSGD_update_equivalent}
	w_t = w_{t-1} -  \eta \cdot (\chi(w_{t-1}) + P_{\mathcal{T}(w_{t-1})} \xi_{t-1}) + \iota_{t-1}
\end{equation}
where $\iota$ is the correction for projection, and $\|\iota\| \le \tlO(\eta^2)$.
\end{lemma}

\begin{proof}
Lemma~\ref{PSGD_equivalent} is a direct corollary of Lemma~\ref{lem::projection_distance}.
\end{proof}

The intuition behind this lemma is that:
when $\{c_i\}_{i=1}^m$ are smooth and \nameCQ holds for all $w\in \mathcal{W}$, 
then the constraint manifold has bounded curvature every where.
Then, if we only care about first order behavior, it's well-approximated
by the local dynamic in tangent plane, up to some second-order correction.

Therefore, by Eq.(\ref{PSGD_update_equivalent}), we see locally it's not much different from the unconstrainted case Eq.(\ref{SGD_update}) up to some negeligable correction.
In the following analysis, we will always use formula Eq.(\ref{PSGD_update_equivalent}) as the update equation for PSGD.

Since most of following proof bears a lot similarity as in unconstrained case, we only pointed out the essential steps in our following proof.

\begin{theorem} [Main Theorem for Equality-Constrained Case]
Suppose a function $f(w):\R^d\to \R$ with constraints $c_i(w):\R^d\to \R$ is $(\alpha, \gamma, \epsilon, \delta)$-\name, and has a stochastic gradient oracle with radius at most $Q$, also satisfying $\E\xi = 0$ and $\E \xi\xi^T = \sigma^2I$. Further, suppose the function 
function $f$ is $B$-bounded, $L$-Lipschitz, $\beta$-smooth, and has $\rho$-Lipschitz Hessian,
and the constraints $\{c_i\}_{i=1}^m$ is $L_i$-Lipschitz, $\beta_i$-smooth, and has $\rho_i$-Lipschitz Hessian.
Then there exists a threshold $\eta_{\max} = \tilde{\Theta}(1)$, so that for any $\zeta>0$, and
for any $\eta \le \eta_{\max} / \max\{1, \log (1/\zeta)\}$,
with probability at least $1-\zeta$ in $t = \tlO(\eta^{-2}\log (1/\zeta))$ iterations, PSGD outputs a point $w_t$ that is $\tlO(\sqrt{\eta\log(1/\eta\zeta)})$-close to some local minimum $w^\star$.
\label{thm:constrainedmain}
\end{theorem}

First, we proof the assumptions in main theorem implies the smoothness conditions for $\mathfrak{M}(w)$, $\mathfrak{N}(w)$ and $\nabla^2_{ww} \mathcal{L}(w, \lambda^*(w'))$.

\begin{lemma}\label{lem:constrainedsmooth}
Under the assumptions of Theorem~\ref{thm:constrainedmain},
there exists $\beta_M, \beta_N, \rho_M, \rho_N, \rho_L$ polynomial related to 
$B, L, \beta, \rho, \frac{1}{\alpha_c}$ and $\{L_i, \beta_i, \rho_i\}_{i=1}^m$ so that:
\begin{enumerate}
\item $\|\mathfrak{M}(w)\| \le \beta_M$ and $\|\mathfrak{N}(w)\| \le \beta_N$ for all $w \in \mathcal{W}$.
\item $\mathfrak{M}(w)$ is $\rho_M$-Lipschitz, and $\mathfrak{N}(w)$ is $\rho_N$-Lipschitz, and
$\nabla^2_{ww} \mathcal{L}(w, \lambda^*(w'))$ is $\rho_L$-Lipschitz for all $w' \in \mathcal{W}$. 
\end{enumerate}
\end{lemma}

\begin{proof}
By definition of $\mathfrak{M}(w)$, $\mathfrak{N}(w)$ and $\nabla^2_{ww} \mathcal{L}(w, \lambda^*(w'))$, 
the above conditions will holds if there exists $B_\lambda, L_\lambda, \beta_\lambda$ bounded by 
$\tlO(1)$, so that $\lambda^*(w)$ is $B_\lambda$-bounded, $L_\lambda$-Lipschitz, and $\beta_\lambda$-smooth.

By definition Eq.(\ref{Lambda_star}), we have:
\begin{equation} 
\lambda^*(w)  = C(w)^{\dagger}\nabla f(w) = (C(w)^TC(w))^{-1}C(w)^T\nabla f(w)
\end{equation}
Because $f$ is $B$-bounded, $L$-Lipschitz, $\beta$-smooth, and its Hessian is $\rho$-Lipschitz, thus, eventually, we only need to prove that there exists $B_c, L_c, \beta_c$ bounded by $\tlO(1)$, so that the pseudo-inverse $C(w)^{\dagger}$ is $B_c$-bounded, $L_c$-Lipschitz, and $\beta_c$-smooth.

Since \nameCQ holds for all feasible points, we immediately have: $\|C(w)^{\dagger}\| \le\frac{1}{\alpha_c}$, thus bounded. For simplicity, in the following context we use $C^\dagger$ to represent $C^\dagger(w)$ without ambiguity.
By some calculation of linear algebra, we have the derivative of pseudo-inverse:
\begin{align}
&\frac{\partial C(w)^{\dagger}}
{\partial w_i}=-C^{\dagger}
\frac{\partial C(w)}{\partial w_i}C^{\dagger}
+C^{\dagger}[C^{\dagger}]^T
\frac{\partial C(w)^T}{\partial w_i}(I-C C^{\dagger} )
\end{align}
Again, \nameCQ holds implies that derivative of pseudo-inverse is well-defined for every feasible point.
Let tensor $E(w), \tilde{E}(w)$ to be the derivative of $C(w), C^\dagger(w)$, which is defined as:
\begin{equation}
[E(w)]_{ijk} =  \frac{\partial [C(w)]_{ik}}{\partial w_j} \quad \quad
[\tilde{E}(w)]_{ijk} =  \frac{\partial [C(w)^\dagger]_{ik}}{\partial w_j}
\end{equation}
%If the tensor product follows $E(U,I,V)$[UEV]_{\cdot j \cdot} = \sum_{i,k} U_{\cdot i}E_{ijk}v_{k \cdot}$, 
%then we have:
Define the transpose of a 3rd order tensor $E^T_{i,j,k} = E_{k,j,i}$, then we have
\begin{equation} \label{derivative_pseudo_inverse}
\tilde{E}(w)
=-[E(w)](C^{\dagger},I,C^{\dagger})
+[E(w)^T](C^{\dagger}[C^{\dagger}]^T, I, (I-C C^{\dagger} ))
\end{equation}
where by calculation $[E(w)](I,I,e_i) = \nabla^2 c_i(w)$.

Finally, since $C(w)^\dagger$ and $\nabla^2 c_i(w)$ are bounded by $\tlO(1)$, 
by Eq.(\ref{derivative_pseudo_inverse}), we know $\tilde{E}(w)$
is bounded, that is $C(w)^{\dagger}$ is Lipschitz. 
Again, since both $C(w)^\dagger$ and $\nabla^2 c_i(w)$ are bounded, Lipschitz, 
by Eq.(\ref{derivative_pseudo_inverse}), we know $\tilde{E}(w)$ is also $\tlO(1)$-Lipschitz.
This finishes the proof.

\end{proof}

From now on,
we can use the same proof strategy as unconstraint case. Below we list the corresponding lemmas and the essential steps that require modifications.

\begin{lemma} \label{thm::case1_constraint}
Under the assumptions of Theorem~\ref{thm:constrainedmain}, with notations in Lemma~\ref{lem:constrainedsmooth}, for any point with $\|\chi(w_0)\| \ge \sqrt{2\eta\sigma^2\beta_{M} (d-m)}$ where $\sqrt{2\eta\sigma^2\beta_{M} (d-m)} < \epsilon$, after one iteration we have:
\begin{equation}
	\E f(w_1) - f(w_{0}) \le - \tlOmega(\eta^2)
\end{equation}
\end{lemma} 

\begin{proof}
Choose $\eta_{\max} < \frac{1}{\beta_{M}}$, and also small enough, then by update equation Eq.(\ref{PSGD_update_equivalent}), we have:
\begin{align}
	\E f(w_1) -  f(w_{0}) &\le \chi(w_{0})^T \E(w_1-w_{0}) + \frac{\beta_{M}}{2}\E\|w_1-w_{0}\|^2 \nonumber \\
	& \le -(\eta - \frac{\beta_{M}\eta^2}{2})\|\chi(w_{0})\|^2 + \frac{\eta^2\sigma^2 \beta_{M} (d-m)}{2}
	+ \tlO(\eta^{2})\|\chi(w_{0})\| + \tlO(\eta^3)
	\nonumber \\
	& \le -(\eta - \tlO(\eta^{1.5}) - \frac{\beta_{M}\eta^2}{2})\|\chi(w_{0})\|^2 + \frac{\eta^2\sigma^2 \beta_{M} (d-m)}{2} + \tlO(\eta^3)
	\nonumber \\
	&\le -\frac{\eta^2 \sigma^2 \beta_{M} d}{4}
\end{align}
Which finishes the proof.
\end{proof}

\begin{theorem}\label{thm::case3_constraint}
Under the assumptions of Theorem~\ref{thm:constrainedmain}, 
with notations in Lemma~\ref{lem:constrainedsmooth}, 
for any initial point $w_0$ that is $\tlO(\sqrt{\eta}) < \delta$ close to a local minimum $w^\star$, 
with probability at least $1-\zeta/2$, we have following holds simultaneously:
\begin{equation}
 \forall t\le \tlO(\frac{1}{\eta^2}\log \frac{1}{\zeta}), \quad \|w_{t} - w^\star\| \le \tlO(\sqrt{\eta\log \frac{1}{\eta\zeta}})<\delta 
\end{equation}
where $w^\star$ is the locally optimal point.
\end{theorem}

\begin{proof}
By calculus, we know 
\begin{align}
\chi(w_t) = & \chi(w^\star) + \int_{0}^{1}(\mathfrak{M} + \mathfrak{N})(w^\star + t(w_t - w^\star)) \mathrm{d}t \cdot (w_t - w^\star)
\end{align}

Let filtration $\mathfrak{F}_t = \sigma\{\xi_0, \cdots \xi_{t-1}\}$, and note $\sigma\{\Delta_0, \cdots, \Delta_t \} \subset \mathfrak{F}_t$, where $\sigma\{\cdot\}$ denotes the sigma field.
Let event $\mathfrak{E}_t = \{\forall \tau \le t, \|w_{\tau} - w^\star\| \le \mu\sqrt{\eta\log\frac{1}{\eta\zeta}} < \delta \}$, where $\mu$ is independent of $(\eta, \zeta)$, and will be specified later.

By Definition~\ref{def:robustcondition_constraint} of 
$(\alpha, \gamma, \epsilon, \delta)$-\name, we know $\mathfrak{M}(w)$ is locally $\alpha$-strongly convex restricted to its tangent space $\mathcal{T}(w)$.
in the $2\delta$-neighborhood of $w^\star$. If $\eta_{\max}$ is chosen small enough, by Remark
\ref{N_constraint} and Lemma \ref{lem::normal_by_tangent}, we have in addition:
\begin{align}
\chi(w_t)^T (w_t - w^\star)1_{\mathfrak{E}_t} &= (w_t - w^\star)^T \int_{0}^{1}(\mathfrak{M}+ \mathfrak{N})(w^\star + t(w_t - w^\star)) \mathrm{d}t \cdot (w_t - w^\star)1_{\mathfrak{E}_t} \nonumber \\ 
&\ge [\alpha \|w_t - w^\star\|^2 - \tlO(\|w_t - w^\star\|^3)]1_{\mathfrak{E}_t} \ge 0.5\alpha \|w_t - w^\star\|^21_{\mathfrak{E}_t}
\end{align}
Then, everything else follows almost the same as the proof of Lemma \ref{thm::case3}.
\end{proof}

\begin{lemma} \label{thm::case2_constraint}
Under the assumptions of Theorem~\ref{thm:constrainedmain}, with notations in Lemma~\ref{lem:constrainedsmooth},
for any initial point $w_0$ where $\|\chi(w_0)\| \le \tlO(\eta) < \epsilon$, and 
$\hat{v}^T  \mathfrak{M}(w_0) \hat{v}  \le -\gamma$ for some $\hat{v} \in \mathcal{T}(w)$, $\|\hat{v}\| = 1$, then 
%there is a fixed number of steps $T = O(\log (d-m)/\gamma\eta)$ such that:
there is a number of steps $T$ that depends on $w_0$ such that: 
\begin{equation}
	\E f(w_T) - f(w_0) \le - \tlOmega(\eta)
\end{equation}
The number of steps $T$ has a fixed upper bound $T_{max}$ that is independent of $w_0$ where $T \le T_{max} = O((\log (d-m))/\gamma\eta)$.
\end{lemma} 

Similar to the unconstrained case, we show this by a coupling sequence. Here the sequence we construct will only walk on the tangent space, by Lemmas in previous subsection, we know this is not very far from the actual sequence. We first define and characterize the coupled sequence in the following lemma:

\begin{lemma} \label{lem::case_Gaussian_constraint}
Under the assumptions of Theorem~\ref{thm:constrainedmain}, with notations in Lemma~\ref{lem:constrainedsmooth}.
Let $\tilde{f}$ defined as local second-order approximation of $f(x)$ around $w_0$
in tangent space $\mathcal{T}_0 =\mathcal{T}(w_0)$:
\begin{equation}\label{def_tilde_f_constraint}
\tilde{f}(w) \doteq f(w_0) + \chi(w_0)^T (w-w_0) + \frac{1}{2}(w-w_0)^T[P_{\mathcal{T}_0}^T\mathfrak{M}(w_0)P_{\mathcal{T}_0}](w-w_0)
\end{equation}
$\{\tilde{w}_t\}$ be the corresponding sequence generated by running SGD on function $\tilde{f}$, with $\tilde{w}_0 = w_0$, and noise projected to $\mathcal{T}_0$, (i.e. $\tilde{w}_t = \tilde{w}_{t-1} - \eta (\tilde{\chi}(\tilde{w}_{t-1}) + P_{\mathcal{T}_0}\xi_{t-1}$).
For simplicity, denote 
$\tilde{\chi}(w) = \nabla \tilde{f}(w)$, and $\widetilde{\mathfrak{M}} = P_{\mathcal{T}_0}^T\mathfrak{M}(w_0)P_{\mathcal{T}_0}$,
then we have analytically:
\begin{align}
&\tilde{\chi}(\tilde{w}_t)= (1-\eta\widetilde{\mathfrak{M}} )^t\tilde{\chi}(\tilde{w}_0) -\eta \widetilde{\mathfrak{M}} \sum_{\tau=0}^{t-1}(1-\eta\widetilde{\mathfrak{M}} )^{t-\tau-1}P_{\mathcal{T}_0}\xi_{\tau}\\
 &\tilde{w}_{t} - w_0 = -\eta \sum_{\tau = 0}^{t-1}(1-\eta \widetilde{\mathfrak{M}} )^\tau\tilde{\chi}(\tilde{w}_0) -\eta
	\sum_{\tau=0}^{t-1}(1-\eta\widetilde{\mathfrak{M}})^{t-\tau-1}P_{\mathcal{T}_0}\xi_{\tau}  \label{dif_x_constraint}
\end{align}
Further, for any initial point $w_0$ where $\|\chi(w_0)\| \le \tlO(\eta) < \epsilon$, and $\min_{\hat{v} \in \mathcal{T}(w), \|\hat{v}\| = 1} \hat{v}^T  \mathfrak{M}(w_0) \hat{v}$  $= -\gamma_0$. There exist a $T \in \mathbb{N}$ satisfying:
\begin{equation}\label{choose_t_constraint}
 \frac{d-m}{\eta\gamma_0} \le \sum_{\tau=0}^{T-1}(1+\eta \gamma_0)^{2\tau} < \frac{3(d-m)}{\eta\gamma_0}
\end{equation}
with probability at least $1-\tlO(\eta^3)$, we have
following holds simultaneously for all $t\le T$:
\begin{equation}
\|\tilde{w}_t - w_0\| \le \tlO(\eta^{\frac{1}{2}}\log \frac{1}{\eta}); 
\quad\quad
\|\tilde{\chi}(\tilde{w}_t)\| \le \tlO(\eta^{\frac{1}{2}}\log \frac{1}{\eta})
\end{equation}
\end{lemma}

\begin{proof}
Clearly we have:
\begin{equation}\label{derivative_tilde_recursive_constraint}
\tilde{\chi}(\tilde{w}_t) = \tilde{\chi}(\tilde{w}_{t-1}) + \widetilde{\mathfrak{M}} (\tilde{w}_t - \tilde{w}_{t-1})
\end{equation}
and 
\begin{equation}\label{SGD_tilde_constraint}
\tilde{w}_t = \tilde{w}_{t-1} - \eta (\tilde{\chi}(\tilde{w}_{t-1}) + P_{\mathcal{T}_0}\xi_{t-1} )
\end{equation}
This lemma is then proved by a direct application of Lemma \ref{lem::case_Gaussian}.
\end{proof}

Then we show the sequence constructed is very close to the actual sequence.

\begin{lemma} \label{lem::saddle_and_maximum_constraint}
Under the assumptions of Theorem~\ref{thm:constrainedmain}, with notations in Lemma~\ref{lem:constrainedsmooth}.
Let $\{w_t\}$ be the corresponding sequence generated by running PSGD on function $f$.  
Also let $\tilde{f}$ and $\{\tilde{w}_t\}$ be defined as in Lemma \ref{lem::case_Gaussian_constraint}.
Then, for any initial point $w_0$ where $\|\chi(w_0)\|^2 \le \tlO(\eta) < \epsilon$, and $\min_{\hat{v} \in \mathcal{T}(w), \|\hat{v}\| = 1}\hat{v}^T  \mathfrak{M}(w_0) \hat{v}  = -\gamma_0$. Given the choice of $T$ as in Eq.(\ref{choose_t_constraint}), 
with probability at least $1-\tlO(\eta^2)$, we have following holds simultaneously for all $t\le T$:
\begin{align}
\|w_t - \tilde{w}_t\| \le \tlO( \eta\log^2\frac{1}{\eta});
\end{align}
\end{lemma}

\begin{proof}
First, we have update function of tangent gradient by:
\begin{align}
\chi(w_t) = & \chi(w_{t-1}) + \int_{0}^{1}\nabla \chi(w_{t-1} + t(w_t - w_{t-1})) \mathrm{d}t \cdot (w_t - w_{t-1}) \nonumber \\
= & \chi(w_{t-1}) + \mathfrak{M}(w_{t-1}) (w_t - w_{t-1}) + \mathfrak{N}(w_{t-1}) (w_t - w_{t-1})+ \theta_{t-1}
\end{align}
where the remainder:
\begin{equation}
\theta_{t-1} \equiv \int_{0}^{1}\left[\nabla \chi(w_{t-1} + t(w_t - w_{t-1})) - \nabla \chi(w_{t-1})\right] \mathrm{d}t \cdot (w_t - w_{t-1})
\end{equation}
Project it to tangent space $\mathcal{T}_0 =\mathcal{T}(w_0)$.
Denote $\widetilde{\mathfrak{M}} = P_{\mathcal{T}_0}^T\mathfrak{M}(w_0)P_{\mathcal{T}_0}$, 
and $\widetilde{\mathfrak{M}}'_{t-1} = P_{\mathcal{T}_0}^T[~\mathfrak{M}(w_{t_1}) -\mathfrak{M}(w_0)~] P_{\mathcal{T}_0}$. Then, 
we have:
\begin{align}
P_{\mathcal{T}_0}\cdot\chi(w_t)= & P_{\mathcal{T}_0}\cdot\chi(w_{t-1}) + P_{\mathcal{T}_0}(\mathfrak{M}(w_{t-1})+ \mathfrak{N}(w_{t-1})) (w_t - w_{t-1}) + P_{\mathcal{T}_0}\theta_{t-1} \nonumber  \\
= & P_{\mathcal{T}_0}\cdot\chi(w_{t-1}) 
+ P_{\mathcal{T}_0}\mathfrak{M}(w_{t-1})P_{\mathcal{T}_0} (w_t - w_{t-1})  \nonumber \\
&+ P_{\mathcal{T}_0}\mathfrak{M}(w_{t-1}) P_{\mathcal{T}^c_0}(w_t - w_{t-1})
+ P_{\mathcal{T}_0}\mathfrak{N}(w_{t-1}) (w_t - w_{t-1})
+ P_{\mathcal{T}_0}\theta_{t-1} \nonumber \\
= & P_{\mathcal{T}_0}\cdot\chi(w_{t-1}) 
+ \widetilde{\mathfrak{M}}(w_t - w_{t-1}) + \phi_{t-1}\label{derivative_recursive_constraint}
\end{align}
Where 
\begin{equation}
\phi_{t-1} =  [~\widetilde{\mathfrak{M}}'_{t-1}  + P_{\mathcal{T}_0}\mathfrak{M}(w_{t-1}) P_{\mathcal{T}^c_0}
+ P_{\mathcal{T}_0}\mathfrak{N}(w_{t-1})~] (w_t - w_{t-1})
+ P_{\mathcal{T}_0}\theta_{t-1}
\end{equation}
By Hessian smoothness, we immediately have: 
\begin{align} 
&\|\widetilde{\mathfrak{M}}'_{t-1}\| = \|\mathfrak{M}(w_{t_1}) -\mathfrak{M}(w_0)\| 
\le \rho_M \|w_{t-1} - w_0\| \le \rho_M (\|w_t - \tilde{w}_t\| + \|\tilde{w}_t - w_0\|) 
\label{H'_smooth_constraint} \\
&\|\theta_{t-1}\| \le \frac{\rho_M+\rho_N}{2} \|w_t - w_{t-1}\|^2 \label{theta_smooth_constraint}
\end{align}

Substitute the update equation of PSGD (Eq.(\ref{PSGD_update_equivalent})) into Eq.(\ref{derivative_recursive_constraint}), we have:
\begin{align}
&P_{\mathcal{T}_0}\cdot\chi(w_t) = P_{\mathcal{T}_0}\cdot\chi(w_{t-1}) -\eta\widetilde{\mathfrak{M}} (P_{\mathcal{T}_0}\cdot \chi(w_{t-1}) + P_{\mathcal{T}_0}\cdot P_{\mathcal{T}(w_{t-1})}\xi_{t-1} ) +\widetilde{\mathfrak{M}}\cdot\iota_{t-1} + \phi_{t-1} \nonumber \\
	&= (1-\eta\widetilde{\mathfrak{M}})P_{\mathcal{T}_0}\cdot\chi(w_{t-1}) - \eta \widetilde{\mathfrak{M}} P_{\mathcal{T}_0}\xi_{t-1} + \eta \widetilde{\mathfrak{M}}P_{\mathcal{T}_0}\cdot P_{\mathcal{T}^c(w_{t-1})}\xi_{t-1}
	 +\widetilde{\mathfrak{M}}\cdot\iota_{t-1} +\phi_{t-1} 
				\label{derivative_constraint}
\end{align}

Let $\Delta_t = P_{\mathcal{T}_0}\cdot \chi(w_t) - \tilde{\chi}(\tilde{w}_t)$ denote the difference of tangent gradient in $\mathcal{T}(w_0)$, then
from Eq.(\ref{derivative_tilde_recursive_constraint}), Eq.(\ref{SGD_tilde_constraint}), 
and Eq.(\ref{derivative_constraint})
%Eq.(\ref{derivative_recursive_constraint}), and Eq.(\ref{PSGD_update_equivalent}), 
we have:
\begin{align} \label{Delta_recursive_constraint}
&\Delta_t =  (1-\eta H) \Delta_{t-1} + \eta \widetilde{\mathfrak{M}}P_{\mathcal{T}_0}\cdot P_{\mathcal{T}^c(w_{t-1})}\xi_{t-1}
	 +\widetilde{\mathfrak{M}}\cdot\iota_{t-1} +\phi_{t-1} \\
& P_{\mathcal{T}_0} \cdot (w_t-w_0) - (\tilde{w}_t-w_0) =  -\eta \sum_{\tau = 0}^{t-1} \Delta_\tau 
 + \eta \sum_{\tau = 0}^{t-1} P_{\mathcal{T}_0}\cdot P_{\mathcal{T}^c(w_{\tau})}\xi_{\tau}
	 +\sum_{\tau = 0}^{t-1}\iota_{\tau}\label{dif_constraint_tangent} 
\end{align}

By Lemma \ref{lem::normal_by_tangent}, we know if $\sum_{i=1}^m \frac{\beta_i^2}{\alpha_c^2} = \frac{1}{R^2}$, then we have:
\begin{align}
\|P_{\mathcal{T}^c_0} (w_t-w_0)\| \le \frac{\|w_t-w_0\|^2}{2R}
\label{dif_constraint_normal}
\end{align}

Let filtration $\mathfrak{F}_t = \sigma\{\xi_0, \cdots \xi_{t-1}\}$, and note $\sigma\{\Delta_0, \cdots, \Delta_t \} \subset \mathfrak{F}_t$, where $\sigma\{\cdot\}$ denotes the sigma field. Also, let event $\mathfrak{K}_t = \{\forall \tau \le t, ~\|\tilde{\chi}(\tilde{w}_\tau)\| \le \tlO(\eta^{\frac{1}{2}}\log \frac{1}{\eta}), ~ 
\|\tilde{w}_\tau - w_0\| \le \tlO(\eta^{\frac{1}{2}}\log \frac{1}{\eta})\}$,
and denote $\Gamma_t = \eta \sum_{\tau = 0}^{t-1} P_{\mathcal{T}_0}\cdot P_{\mathcal{T}^c(w_{\tau})}\xi_{\tau}$, let
 $\mathfrak{E}_t = \{\forall \tau \le t, ~\|\Delta_{\tau}\| \le \mu_1 \eta\log^2\frac{1}{\eta}, 
\|\Gamma_\tau\| \le \mu_2 \eta\log^2\frac{1}{\eta}, 
\|w_{\tau} - \tilde{w}_{\tau}\| \le \mu_3 \eta\log^2\frac{1}{\eta}\}$ where $(\mu_1, \mu_2, \mu_3)$ are is independent of $(\eta, \zeta)$, and will be determined later. To prevent ambiguity in the proof, $\tilde{O}$ notation will not hide any dependence on $\mu$.
Clearly event $\mathfrak{K}_{t-1}\subset \mathfrak{F}_{t-1},  \mathfrak{E}_{t-1}\subset \mathfrak{F}_{t-1}$ thus independent of $\xi_{t-1}$.

Then, conditioned on event $\mathfrak{K}_{t-1} \cap \mathfrak{E}_{t-1}$, 
by triangle inequality, we have $\| w_\tau - w_0\| \le \tlO(\eta^{\frac{1}{2}}\log \frac{1}{\eta})$, for all $\tau \le t-1 \le T-1$. 
We then need to carefully bound the following bound each term in Eq.(\ref{Delta_recursive_constraint}).
We know $w_t - w_{t-1} = -  \eta \cdot (\chi(w_{t-1}) + P_{\mathcal{T}(w_{t-1})} \xi_{t-1}) + \iota_{t-1}$, and then
by Lemma \ref{lem::tangent} and Lemma \ref{lem::normal}, we have:
\begin{align}
\|\eta \widetilde{\mathfrak{M}}P_{\mathcal{T}_0}\cdot P_{\mathcal{T}^c(w_{t-1})}\xi_{t-1}\| 
&\le \tlO(\eta^{1.5} \log \frac{1}{\eta}) \nonumber \\
\|\widetilde{\mathfrak{M}}\cdot\iota_{t-1} \| 
&\le \tlO(\eta^2) \nonumber \\
\|[~\widetilde{\mathfrak{M}}'_{t-1}  + P_{\mathcal{T}_0}\mathfrak{M}(w_{t-1}) P_{\mathcal{T}^c_0}
+ P_{\mathcal{T}_0}\mathfrak{N}(w_{t-1})~] ( -  \eta \cdot \chi(w_{t-1}))\|
&\le \tlO(\eta^2\log^2\frac{1}{\eta}) \nonumber \\
\|[~\widetilde{\mathfrak{M}}'_{t-1}  + P_{\mathcal{T}_0}\mathfrak{M}(w_{t-1}) P_{\mathcal{T}^c_0}
+ P_{\mathcal{T}_0}\mathfrak{N}(w_{t-1})~] ( -  \eta P_{\mathcal{T}(w_{t-1})} \xi_{t-1})\|
&\le \tlO(\eta^{1.5}\log\frac{1}{\eta}) \nonumber \\
\|[~\widetilde{\mathfrak{M}}'_{t-1}  + P_{\mathcal{T}_0}\mathfrak{M}(w_{t-1}) P_{\mathcal{T}^c_0}
+ P_{\mathcal{T}_0}\mathfrak{N}(w_{t-1})~] \iota_{t-1}\|
&\le \tlO(\eta^{2}) \nonumber \\
\|P_{\mathcal{T}_0}\theta_{t-1}\|  &\le \tlO(\eta^2)
\end{align}

%\jccomment{here we need to use projection argument}

Therefore, abstractly,  conditioned on event $\mathfrak{K}_{t-1} \cap \mathfrak{E}_{t-1}$, we could write down
the recursive equation as:
\begin{equation}
	\Delta_t =  (1-\eta H) \Delta_{t-1} + A + B
\end{equation}
where $\|A\| \le \tlO(\eta^{1.5} \log \frac{1}{\eta})$ and $\|B\| \le \tlO(\eta^{2} \log^2 \frac{1}{\eta})$, 
and in addition, by independence, easy to check we also have $\E [(1-\eta H) \Delta_{t-1} A |\mathfrak{F}_{t-1}] = 0$. This is exactly the same case as in the proof of Lemma \ref{lem::saddle_and_maximum}. By the same argument of martingale and Azuma-Hoeffding, and by choosing $\mu_1$ large enough, we can prove
\begin{align}\label{EE_1}
&P\left( \mathfrak{E}_{t-1} \cap \left\{\|\Delta_t\| \ge \mu_1 \eta\log^2\frac{1}{\eta}\right\}\right)  \le \tlO(\eta^3)
\end{align}

On the other hand, for $\Gamma_t = \eta \sum_{\tau = 0}^{t-1} P_{\mathcal{T}_0}\cdot P_{\mathcal{T}^c(w_{\tau})}\xi_{\tau}$,
we have:
\begin{align}
\E[\Gamma_t 1_{\mathfrak{K}_{t-1}\cap \mathfrak{E}_{t-1}} |\mathfrak{F}_{t-1}]
&= \left[\Gamma_{t-1}+\eta\E[P_{\mathcal{T}_0}\cdot P_{\mathcal{T}^c(w_{t-1})}\xi_{t-1}|\mathfrak{F}_{t-1}]\right]1_{\mathfrak{K}_{t-1}\cap \mathfrak{E}_{t-1}} \nonumber \\
&= \Gamma_{t-1}1_{\mathfrak{K}_{t-1}\cap \mathfrak{E}_{t-1}} \le \Gamma_{t-1}1_{\mathfrak{K}_{t-2}\cap \mathfrak{E}_{t-2}}
\end{align}

Therefore, we have $\E [\Gamma_t1_{\mathfrak{K}_{t-1}\cap \mathfrak{E}_{t-1}} ~|~ \mathfrak{F}_{t-1}] \le \Gamma_{t-1}1_{\mathfrak{K}_{t-2}\cap \mathfrak{E}_{t-2}}$ which means 
$\Gamma_t1_{\mathfrak{K}_{t-1}\cap \mathfrak{E}_{t-1}}$ is a supermartingale.

We also know by Lemma \ref{lem::tangent}, with probability 1:
\begin{align}
&| \Gamma_t1_{\mathfrak{K}_{t-1}\cap \mathfrak{E}_{t-1}} - \E[\Gamma_t1_{\mathfrak{K}_{t-1}\cap \mathfrak{E}_{t-1}}~|~\mathfrak{F}_{t-1}] |
= |\eta P_{\mathcal{T}_0}\cdot P_{\mathcal{T}^c(w_{t-1})}\xi_{t-1}|\cdot 1_{\mathfrak{K}_{t-1}\cap \mathfrak{E}_{t-1}} \nonumber \\
\le & \tlO(\eta)\|w_{t-1} - w_0\|1_{\mathfrak{K}_{t-1}\cap \mathfrak{E}_{t-1}}
\le \tlO(\eta^{1.5}\log \frac{1}{\eta}) = c_{t-1}
\end{align}
By Azuma-Hoeffding inequality, with probability less than $\tlO(\eta^3)$, 
for $t\le T\le O(\log (d-m)/\gamma_0\eta)$:
\begin{equation}
\Gamma_t1_{\mathfrak{K}_{t-1}\cap \mathfrak{E}_{t-1}} - \Gamma_0\cdot1 > \tlO(1)\sqrt{\sum_{\tau=0}^{t-1}{c^2_\tau}}\log (\frac{1}{\eta}) = \tlO(\eta\log^2 \frac{1}{\eta})
\end{equation}
This means there exists some $\tilde{C}_2 = \tlO(1)$ so that:
\begin{equation}
	P\left(\mathfrak{K}_{t-1}\cap \mathfrak{E}_{t-1} \cap \left\{\|\Gamma_t\| \ge \tilde{C}_2\eta\log^2\frac{1}{\eta}\right\}\right) \le \tlO(\eta^3)
\end{equation}
by choosing $\mu_2>\tilde{C}_2$, we have:
\begin{equation}
	P\left(\mathfrak{K}_{t-1}\cap \mathfrak{E}_{t-1} \cap \left\{\|\Gamma_t\| \ge \mu_2 \eta\log^2\frac{1}{\eta}\right\}\right) \le \tlO(\eta^3)
\end{equation}
Therefore, combined with Lemma \ref{lem::case_Gaussian_constraint}, we have:
\begin{align} \label{EE_2}
P\left( \mathfrak{E}_{t-1} \cap \left\{\|\Gamma_t\| \ge \mu_2 \eta\log^2\frac{1}{\eta}\right\}\right) 
\le \tlO(\eta^3) + P(\overline{\mathfrak{K}}_{t-1}) \le \tlO(\eta^3)
\end{align}

Finally, conditioned on event $\mathfrak{K}_{t-1} \cap \mathfrak{E}_{t-1}$, if we have $\|\Gamma_t\| \le \mu_2 \eta\log^2\frac{1}{\eta}$, then by Eq.(\ref{dif_constraint_tangent}):
\begin{equation}
\|P_{\mathcal{T}_0} \cdot (w_t-w_0) - (\tilde{w}_t-w_0)\| \le \tlO\left  ((\mu_1 + \mu_2)\eta\log^2\frac{1}{\eta}\right  )
\end{equation}
Since $\|w_{t-1} - w_0\| \le \tlO(\eta^{\frac{1}{2}}\log \frac{1}{\eta})$, 
and $\|w_{t} - w_{t-1}\| \le \tlO(\eta)$, by Eq.(\ref{dif_constraint_normal}):
\begin{equation}
\|P_{\mathcal{T}^c_0} (w_t-w_0)\| \le \frac{\|w_t-w_0\|^2}{2R}
\le \tlO(\eta\log^2 \frac{1}{\eta})
\end{equation}
Thus:
\begin{align}
\|w_t - \tilde{w}_t\|^2 = &\|P_{\mathcal{T}_0} \cdot (w_t - \tilde{w}_t)\|^2 + \|P_{\mathcal{T}^c_0} \cdot (w_t - \tilde{w}_t)\|^2 \nonumber \\
=& \|P_{\mathcal{T}_0} \cdot (w_t-w_0) - (\tilde{w}_t-w_0)\|^2 + \|P_{\mathcal{T}^c_0} (w_t-w_0)\|^2
\le \tlO((\mu_1 + \mu_2)^2\eta^{2}\log^4\frac{1}{\eta})
\end{align}
That is there exist some $\tilde{C}_3=\tlO(1)$ so that $\|w_t - \tilde{w}_t\|
\le \tilde{C}_3(\mu_1 + \mu_2)\eta\log^2\frac{1}{\eta}$
Therefore, conditioned on event $\mathfrak{K}_{t-1} \cap \mathfrak{E}_{t-1}$, we have proved that if choose
$\mu_3>\tilde{C}_3(\mu_1 + \mu_2)$, then event
$\{\|w_t - \tilde{w}_t\| \ge \mu_3 \eta\log^2\frac{1}{\eta} \} \subset \{\|\Gamma_t\| \ge \mu_2 \eta\log^2\frac{1}{\eta}\}$. Then, combined this fact with Eq.(\ref{EE_1}), Eq.(\ref{EE_2}), we have proved:
\begin{equation}
P\left( \mathfrak{E}_{t-1} \cap \overline{ \mathfrak{E}}_{t}\right) \le \tlO(\eta^3)
\end{equation}
Because $P(\overline{ \mathfrak{E}}_{0}) =0$, and $T\le \tlO(\frac{1}{\eta})$, we have 
$P(\overline{ \mathfrak{E}}_{T}) \le \tlO(\eta^2)$, which concludes the proof.

\end{proof}

These two lemmas allow us to prove the result when the initial point is very close to a saddle point.

\begin{proof}[Proof of Lemma \ref{thm::case2_constraint}]
Combine Talyor expansion Eq.\ref{Taylor_eq_constraint} with Lemma \ref{lem::case_Gaussian_constraint}, Lemma \ref{lem::saddle_and_maximum_constraint}, we prove this Lemma by the same argument as in the proof of Lemma \ref{thm::case2}.
\end{proof}

Finally the main theorem follows.

\begin{proof} [Proof of Theorem \ref{thm:constrainedmain}]
By Lemma \ref{thm::case1_constraint}, Lemma \ref{thm::case2_constraint}, and Lemma \ref{thm::case3_constraint}, with the same argument as in the proof Theorem \ref{thm:sgdmain_unconstraint}, we easily concludes this proof.
\end{proof}
%!TEX root = saddlepoint.tex

\section{Detailed Proofs for Section~\ref{sec:tensors}}

In this section we show two optimization problems (\ref{eq:findone}) and (\ref{eq:hardprob}) satisfy the $(\alpha,\gamma,\epsilon,\delta)$-\name~propery.

\subsection{Warm Up: Maximum Eigenvalue Formulation}
\label{sec:warmup}
Recall that we are trying to solve the optimization (\ref{eq:findone}), which we restate here.
\begin{align}
\max & \quad T(u,u,u,u), \\ 
\|u\|^2 &= 1. \nonumber
\end{align}
Here the tensor $T$ has orthogonal decomposition $T = \sum_{i=1}^d a_i^{\otimes 4}$. We first do a change of coordinates to work in the coordinate system specified by $(a_i)$'s (this does not change the dynamics of the algorithm). In particular, let $u = \sum_{i=1}^d x_i a_i$ (where $x\in \R^d$), then we can see $T(u,u,u,u) = \sum_{i=1}^d x_i^4$. Therefore let $f(x) = -\|x\|_4^4$, the optimization problem is equivalent to
\begin{align} \label{problem1_transformed}
\min &~~~~ f(x)\\
\text{s.t.} & ~~~~\|x\|^2_2 = 1 \nonumber
\end{align}

This is a constrained optimization, so we apply the framework developed in Section~\ref{sec:constrainedproblem}.

Let $c(x) = \|x\|_2^2 -1$. We first compute the Lagrangian
\begin{equation}
\mathcal{L}(x, \lambda) = f(x) -\lambda c(x) = -\|x\|_4^4 - \lambda (\|x\|_2^2 -1).
\end{equation}

Since there is only one constraint, and the gradient when $\|x\| = 1$ always have norm $2$, we know the set of constraints satisfy $2$-RLICQ. In particular, we can compute the correct value of Lagrangian multiplier $\lambda$, 

\begin{equation}
\lambda^*(x) = \arg\min_{\lambda} \|\nabla_x \mathcal{L}(x, \lambda)\| 
= \arg\min_{\lambda} \sum_{i=1}^d (2 x_i^3 + \lambda x_i)^2 = -2\|x\|_4^4
\end{equation}

Therefore, the gradient in the tangent space is equal to
\begin{align} \label{chi_1}
\chi(x) & = \nabla_x \mathcal{L}(x, \lambda) |_{(x, \lambda^*(x))} = \nabla f(x) -\lambda^*(x) \nabla c(x) \nonumber \\
&= -4(x_1^3, \cdots, x_d^3)^T -2 \lambda^*(x)( x_1,\cdots, x_d)^T\nonumber \\
&=4\left((x_1^2-\|x\|_4^4) x_1, \cdots, (x_d^2-\|x\|_4^4) x_d\right)
\end{align}

The second-order partial derivative of Lagrangian is equal to
\begin{align} \label{frakM_1}
\mathfrak{M}(x)
& = \nabla^2_{xx} \mathcal{L}(x, \lambda)|_{(x, \lambda^*(x))}=
\nabla^2 f(x) -\lambda^* (x)\nabla^2 c(x) \nonumber \\
&= -12 \text{diag}(x_1^2, \cdots, x_d^2)	-2 \lambda^*(x) I_d \nonumber \\
&= -12 \text{diag}(x_1^2, \cdots, x_d^2)	+ 4\|x\|_4^4 I_d
\end{align}

Since the variable $x$ has bounded norm, and the function is a polynomial, it's clear that the function itself is bounded and all its derivatives are bounded. Moreover, all the derivatives of the constraint are bounded. We summarize this in the following lemma.
\begin{lemma}
The objective function (\ref{eq:findone}) is bounded by $1$, its $p$-th order derivative is bounded by $O(\sqrt{d})$ for $p = 1,2,3$.
The constraint's $p$-th order derivative is bounded by $2$, for $p=1,2,3$. 
\end{lemma}

Therefore the function satisfy all the smoothness condition we need. Finally we show the gradient and Hessian of Lagrangian satisfy the $(\alpha,\gamma, \epsilon,\delta)$-\name~property. Note that we did not try to optimize the dependency with respect to $d$.

\begin{theorem} \label{thm:problem_1_strict_saddle}
The only local minima of 
optimization problem (\ref{eq:findone}) are $\pm a_i ~(i\in[d])$. Further it satisfy $(\alpha,\gamma, \epsilon,\delta)$-\name~for $\gamma = 7/d$, $\alpha = 3$ and $\epsilon,\delta = 1/\mbox{poly}(d)$.
\end{theorem}

In order to prove this theorem, we consider the transformed version
Eq.\ref{problem1_transformed}. We first need following two lemma for points around saddle point and local minimum respectively. We choose 
\begin{equation}\label{choice_1}
\epsilon_0=(10d)^{-4}, ~~\epsilon= 4\epsilon_0^2, ~~\delta = 2d\epsilon_0, ~~\BigC(x) = \{ i ~|  ~|x_i| >  \epsilon_0\}
\end{equation}
Where by intuition, $\BigC(x)$ is the set of coordinates whose value is relative large.

\begin{lemma}\label{lem:Problem1_case2}
Under the choice of parameters in Eq.(\ref{choice_1}),
suppose $\|\chi(x)\| \le \epsilon$, and $|\BigC(x)| \ge 2$. Then,
there exists $\hat{v} \in \mathcal{T}(x)$ and $\|\hat{v}\| = 1$, so that
$\hat{v}^T  \mathfrak{M}(x) \hat{v}  \le -7/d$.
\end{lemma}

\begin{proof}
Suppose $|\BigC(x)| = p$, and $2\le p \le d$.
Since $\|\chi(x)\| \le \epsilon = 4\epsilon_0^2$, by Eq.(\ref{chi_1}), we have for each $i \in [d]$, $|[\chi(x)]_i|  = 4|(x_i^2-\|x\|_4^4)x_i|\le 4\epsilon_0^2$.
Therefore, we have:
\begin{equation}\label{pp_1}
	\forall i \in \BigC(x), \quad \quad \quad |x_i^2-\|x\|_4^4| \le \epsilon_0
\end{equation}
and thus:
\begin{align}
&|\|x\|_4^4 - \frac{1}{p}| = |\|x\|_4^4 - \frac{1}{p}\sum_{i} x_i^2| \nonumber \\
\le &|\|x\|_4^4 - \frac{1}{p}\sum_{i \in \BigC(x)} x_i^2| + |\frac{1}{p}\sum_{i \in [d]-\BigC(x)} x_i^2|
\le \epsilon_0 + \frac{d-p}{p} \epsilon_0^2 \le 2 \epsilon_0
\end{align}
Combined with Eq.\ref{pp_1}, this means:
\begin{equation}
	\forall i \in \BigC(x), \quad \quad \quad |x_i^2-\frac{1}{p}| \le 3\epsilon_0
\end{equation}

Because of symmetry, WLOG we assume 
$\BigC(x) = \{1, \cdots, p\}$. Since $|\BigC(x)| \ge 2$, we can pick 
$\hat{v}=(a, b, 0, \cdots, 0)$. 
Here $a>0, b<0$, and $a^2+b^2=1$. 
We pick $a$ such that $a x_1+ b x_2=0$. The solution is the intersection of a radius $1$ circle and a line which passes $(0,0)$, which always exists. 
For this $\hat{v}$, we know $\|\hat{v}\| = 1$, and $\hat{v}^T x=0$ thus $\hat{v} \in \mathcal{T}(x)$.
We have:
\begin{align}
&\hat{v}^T  \mathfrak{M}(x) \hat{v}  
= -(12x_1^2+4\|x\|_4^4) a^2- (12x_2^2+4\|x\|_4^4)b^2 \nonumber \\
=& -8 x_1^2 a^2 - 8x_2^2 b^2 - 4(x_1^2-\|x\|_4^4))a^2 - 4(x_2^2-\|x\|_4^4))b^2
\nonumber \\
\le & -\frac{8}{p} + 24 \epsilon_0 + 4 \epsilon_0
\le -7/d
\end{align}
Which finishes the proof.
\end{proof}

\begin{lemma}\label{lem:Problem1_case1}
Under the choice of parameters in Eq.(\ref{choice_1}),
suppose $\|\chi(x)\| \le \epsilon$, and $|\BigC(x)| = 1$. Then,
there is a local minimum $x^\star$ such that $\|x-x^\star\| \le \delta$, and for all $x'$ in the $2\delta$ neighborhood of $x^\star$, we have $\hat{v}^T  \mathfrak{M}(x') \hat{v}  \ge 3$ for all $\hat{v} \in \mathcal{T}(x')$, $\|\hat{v}\| = 1$
\end{lemma}

\begin{proof}
WLOG, we assume $\BigC(x) = \{1\}$. Then, we immediately have for all $i>1$, 
$|x_i| \le \epsilon_0$, and thus:
\begin{equation}
1 \ge x_1^2 = 1-\sum_{i>1}x_i^2 \ge 1- d\epsilon_0^2 	
\end{equation} 
Therefore $x_1 \ge \sqrt{1-d\epsilon_0^2}$ or $x_1 \le -\sqrt{1-d\epsilon_0^2}$.
Which means $x_1$ is either close to $1$ or close to $-1$. By symmetry, we know WLOG, 
we can assume the case $x_1 \ge \sqrt{1-d\epsilon_0^2}$. Let $e_1 = (1,0,\cdots, 0)$, 
then we know:
\begin{equation}
\|x-e_1\|^2 \le (x_1-1)^2 + \sum_{i>1} x_i^2
\le 2d \epsilon_0^2 \le \delta^2
\end{equation}

Next, we show $e_1$ is a local minimum. According to Eq.\ref{frakM_1}, we know $\mathfrak{M}(e_1)$ is a diagonal matrix with $4$ on the diagonals except for the first diagonal entry (which is equal to $-8$), since $\mathcal{T}(e_1) = \text{span}\{e_2, \cdots, e_d\}$, we have:
\begin{align}
v^T\mathfrak{M}(e_1) v \ge 4 \|v\|^2 >0 \quad \quad 
\text{for all~} v\in \mathcal{T}(e_1), v\neq 0
\end{align}
Which by Theorem \ref{thm::second_sufficient} means $e_1$ is a local minimum.

Finally, 
denote $\mathcal{T}_1 = \mathcal{T}(e_1)$ be the tangent space of constraint manifold at $e_1$.
We know for all $x'$ in the $2\delta$ neighborhood of $e_1$, 
and for all $\hat{v} \in \mathcal{T}(x')$, $\|\hat{v}\| = 1$:
\begin{align}
 \hat{v}^T  \mathfrak{M}(x') \hat{v}  
 \ge &  \hat{v}^T  \mathfrak{M}(e_1) \hat{v}  - 
 |\hat{v}^T  \mathfrak{M}(e_1) \hat{v} - \hat{v}^T  \mathfrak{M}(x') \hat{v} | \nonumber\\
 = & 4\|P_{\mathcal{T}_1}\hat{v} \|^2 - 8\|P_{\mathcal{T}^c_1}\hat{v} \|^2
 - \|\mathfrak{M}(e_1) - \mathfrak{M}(x')\|\|\hat{v}\|^2 \nonumber \\
 = & 4 - 12\|P_{\mathcal{T}^c_1}\hat{v} \|^2 - \|\mathfrak{M}(e_1) - \mathfrak{M}(x')\|
\end{align}
By lemma \ref{lem::normal}, we know $\|P_{\mathcal{T}^c_1}\hat{v} \|^2 \le \|x'-e_1\|^2
\le 4\delta^2$. By Eq.(\ref{frakM_1}), we have:
\begin{align}
	&\|\mathfrak{M}(e_1) - \mathfrak{M}(x')\| \le  \|\mathfrak{M}(e_1) - \mathfrak{M}(x')\| 
	\le \sum_{(i,j)} |[\mathfrak{M}(e_1)]_{ij} - [\mathfrak{M}(x')]_{ij}| \nonumber \\
	\le&  \sum_{i} \left|-12 [e_1]^2_{i}+ 4\|e_1\|_4^4 - 12x^2_{i} + 4\|x\|_4^4\right|
	\le 64 d\delta
\end{align}
In conclusion, we have $\hat{v}^T  \mathfrak{M}(x') \hat{v} \ge 4- 48\delta^2-64 d\delta\ge 3$
which finishs the proof.
\end{proof}

Finally, we are ready to prove Theorem \ref{thm:problem_1_strict_saddle}.
\begin{proof}[Proof of Theorem \ref{thm:problem_1_strict_saddle}]

According to Lemma \ref{lem:Problem1_case2} and Lemma \ref{lem:Problem1_case1}, we immediately know the optimization problem satisfies $(\alpha,\gamma, \epsilon,\delta)$-\name.

The only thing remains to show is that the only local minima of 
optimization problem (\ref{eq:findone}) are $\pm a_i ~(i\in[d])$.
Which is equivalent to show that the only local minima of the transformed problem
is $\pm e_i ~(i\in [d])$, where $e_i = (0, \cdots, 0, 1, 0, \cdots, 0)$, where $1$ is on $i$-th coordinate.

By investigating the proof of Lemma \ref{lem:Problem1_case2} and Lemma \ref{lem:Problem1_case1}, we know these two lemmas actually hold for any small enough choice of $\epsilon_0$ satisfying $\epsilon_0 \le (10d)^{-4}$, by pushing $\epsilon_0 \rightarrow 0$, we know for any point satisfying $|\chi(x)| \le \epsilon \rightarrow 0$, 
if it is close to some local minimum, it must satisfy $1=|\BigC(x)| \rightarrow \supp(x)$. Therefore, we know the only possible local minima are $\pm e_i ~(i\in [d])$. In Lemma \ref{lem:Problem1_case1}, we proved $e_1$ is local minimum, by symmetry, we finishes the proof.
\end{proof}

\subsection{New Formulation}
\label{sec:hardcase}
In this section we consider our new formulation (\ref{eq:hardprob}). We first restate the optimization problem here:

\begin{align}
\min  \quad &\sum_{i\ne j} T(u^{(i)},u^{(i)},u^{(j)},u^{(j)}),\\
\forall i\quad & \|u^{(i)}\|^2  = 1 . \nonumber
\end{align}
Note that we changed the notation for the variables from $u_i$ to $u^{(i)}$, because in later proofs we will often refer to the particular coordinates of these vectors.

Similar to the previous section, we perform a change of basis. The effect is equivalent to making $a_i$'s equal to basis vectors $e_i$ (and hence the tensor is equal to $T = \sum_{i=1}^d e_i^{\otimes 4}$.
After the transformation the equations become
\begin{align}\label{problem2_transformed}
\min &~~~~ \sum_{(i,j):i\neq j}h(u^{(i)}, u^{(j)} )\\
\text{s.t.} & ~~~~\|u^{(i)}\|^2 = 1 \quad\quad \forall i \in [d]\nonumber
\end{align}
Here $h(u^{(i)}, u^{(j)}) = \sum_{k=1}^d (u^{(i)}_k u^{(j)}_k)^2$, $(i,j) \in [d]^2$. We divided the objective function by $2$ to simplify the calculation.

Let $U\in \R^{d^2}$ be the concatenation of $\{u^{(i)}\}$ such that $U_{ij}=u^{(i)}_j$.
Let $c_i(U) = \|u^{(i)}\|^2 -1 $ and $f(U) = \frac{1}{2}\sum_{(i,j):i\neq j}h(u^{(i)}, u^{(j)})$.
We can then compute the Lagrangian
\begin{equation}
\mathcal{L}(U, \lambda) = f(U) -\sum_{i=1}^d\lambda_i c_i(U)
=\frac{1}{2}\sum_{(i,j):i\neq j}h (u^{(i)}, u^{(j)}) - \sum_{i=1}^d\lambda_i (\|u^{(i)}\|^2 -1 )
\end{equation}

The gradients of $c_i(U)$'s are equal to $(0, \cdots, 0, 2u^{(i)}, 0, \cdots, 0)^T$, all of these vectors are orthogonal to each other (because they have disjoint supports) and have norm $2$. Therefore the set of constraints satisfy $2$-RLICQ. We can then compute the Lagrangian multipiers $\lambda^*$ as follows

\begin{equation}
\lambda^*(U) = \arg\min_{\lambda} \|\nabla_U \mathcal{L}(U, \lambda)\| 
= \arg\min_{\lambda} 
4\sum_{i}\sum_k (\sum_{j:j\neq i}U^2_{jk}U_{ik} -  \lambda_i U_{ik})^2
\end{equation}
which gives:
\begin{equation}
\lambda_i^*(U) = \arg\min_{\lambda}\sum_k (\sum_{j:j\neq i}U^2_{jk}U_{ik} -  \lambda_i U_{ik})^2
= \sum_{j:j\neq i} h (u^{(j)}, u^{(i)} )\label{eq:computelambdastar}
\end{equation}

Therefore, gradient in the tangent space is equal to
\begin{align}
\chi(U) & = \nabla_U \mathcal{L}(U, \lambda) |_{(U, \lambda^*(U))} = \nabla f(U) -\sum_{i=1}^n\lambda_i^*(U) \nabla c_i(U).
\end{align}

The gradient is a $d^2$ dimensional vector (which can be viewed as a $d\times d$ matrix corresponding to entries of $U$), and we express this in a coordinate-by-coordinate way.
For simplicity of later proof, denote:
\begin{equation}
\psi_{ik}(U) = \sum_{j:j\neq i} [U^2_{jk} -  h (u^{(j)}, u^{(i)} ) ] = 
\sum_{j: j\neq i} [U_{jk}^2-\sum_{l=1}^d U_{il}^2 U_{jl}^2] 
\end{equation}
Then we have:
\begin{align}
[\chi(U)]_{ik} & = 2 (\sum_{j:j\neq i}U^2_{jk} -  \lambda^*_i(U)  )U_{ik}  \nonumber
 \\&= 2U_{ik}\sum_{j:j\neq i} (U^2_{jk} -  h (u^{(j)}, u^{(i)} )  )\nonumber
\\& =2 U_{ik}  \psi_{ik}(U)\label{chi_2}
\end{align}

Similarly we can compute the second-order partial derivative of Lagrangian as
\begin{align}
\mathfrak{M}(U)
=\nabla^2 f(U) -\sum_{i=1}^d\lambda_i^* \nabla^2 c_i(U).
\end{align}
The Hessian is a $d^2\times d^2$ matrix, we index it by $4$ indices in $[d]$. The entries are summarized below:
\begin{align}
[\mathfrak{M}(U)]_{ik,i'k'}
= & \left.\frac{\partial}{\partial U_{i'k'}} [\nabla_U \mathcal{L}(U, \lambda)]_{ik} \right|_{(U, \lambda^*(U))}
= \left.\frac{\partial}{\partial U_{i'k'}}  [2(\sum_{j:j\neq i}U^2_{jk} -  \lambda )U_{ik}]  \right|_{(U, \lambda^*(U))}\nonumber \\
= &
\begin{cases}
	2(\sum_{j:j\neq i}U^2_{jk} -  \lambda^*_i (U)) &\mbox{~if~} k=k', i=i'\\
	4 U_{i'k} U_{ik} & \mbox{~if~} k=k', i\neq i' \\
	0 &\mbox{~if~}  k \neq k' 
\end{cases} \nonumber \\
= &
\begin{cases}
	2\psi_{ik}(U) &\mbox{~if~} k=k', i=i' \\
	4 U_{i'k} U_{ik} & \mbox{~if~} k=k', i\neq i' \\
	0 &\mbox{~if~}  k \neq k' 
\end{cases} \label{frakM_2}
\end{align}

Similar to the previous case, it is easy to bound the function value and derivatives of the function and the constraints.
\begin{lemma}
The objective function (\ref{eq:hardprob}) and  $p$-th order derivative are all bounded by $\mbox{poly}(d)$ for $p = 1,2,3$. Each constraint's $p$-th order derivative is bounded by $2$, for $p=1,2,3$. 
\end{lemma}

Therefore the function satisfy all the smoothness condition we need. Finally we show the gradient and Hessian of Lagrangian satisfy the $(\alpha,\gamma,\epsilon,\delta)$-\name~property. Again we did not try to optimize the dependency with respect to $d$.

\begin{theorem}\label{thm:problem_2_strict_saddle}
Optimization problem (\ref{eq:hardprob}) has exactly $2^d \cdot d!$ local minimum that corresponds to permutation and sign flips of $a_i$'s. Further, it satisfy $(\alpha,\gamma,\epsilon,\delta)$-\name~for $\alpha = 1$ and $\gamma,\epsilon,\delta = 1/\mbox{poly}(d)$.
\end{theorem}

Again, in order to prove this theorem, we follow the same strategy: we consider the transformed version
Eq.\ref{problem2_transformed}. and first prove the following lemmas for points around saddle point and local minimum respectively. We choose 
\begin{equation}\label{choice_2}
\epsilon_0 = (10d)^{-6}
, ~~\epsilon= 2\epsilon_0^6, ~~\delta = 2d\epsilon_0, ~~\gamma=\epsilon_0^4/4, ~~\BigC(u) = \{ k ~|  ~|u_k| >  \epsilon_0\}
\end{equation}
Where by intuition, $\BigC(u)$ is the set of coordinates whose value is relative large.

\begin{lemma}\label{lem:Problem2_case2}
Under the choice of parameters in Eq.(\ref{choice_2}),
suppose $\|\chi(U)\| \le \epsilon$, and there exists $(i,j) \in [d]^2$ so that $\BigC(u^{(i)})
\cap \BigC(u^{(j)}) \neq \emptyset$. Then,
there exists $\hat{v} \in \mathcal{T}(U)$ and $\|\hat{v}\| = 1$, so that
$\hat{v}^T  \mathfrak{M}(U) \hat{v}  \le -\gamma$.
\end{lemma}

\begin{proof}
	Again, since $\|\chi(x)\| \le \epsilon = 2\epsilon_0^6$, by Eq.(\ref{chi_2}), we have for each $i \in [d]$, $|[\chi(x)]_{ik}|  = 2| U_{ik}  \psi_{ik}(U)|\le 2\epsilon_0^6$.
Therefore, have:
\begin{equation}\label{pp_2}
	\forall k \in \BigC(u^{(i)}), \quad \quad \quad |\psi_{ik}(U)| \le \epsilon^5_0
\end{equation}

Then, we prove this lemma by dividing it into three cases. Note in order to prove that there exists $\hat{v} \in \mathcal{T}(U)$ and $\|\hat{v}\| = 1$, so that
$\hat{v}^T  \mathfrak{M}(U) \hat{v}  \le -\gamma$; it suffices to find a vector $v \in \mathcal{T}(U)$ and $\|v\| \le 1$, so that
$v^T \mathfrak{M}(U) v  \le -\gamma$.

\paragraph{Case 1}: $|\BigC(u^{(i)})| \ge 2$, $|\BigC(u^{(j)})|\ge 2$, and 
$|\BigC(u^{(i)})\cap \BigC(u^{(j)})| \ge 2$.

WLOG, assume $\{1, 2\} \in \BigC(u^{(i)})\cap \BigC(u^{(j)})$, choose $v$ to be $v_{i1} = \frac{U_{i2}}{4}$, $v_{i2} = -\frac{U_{i1}}{4}$, $v_{j1} = \frac{U_{j2}}4$ and $v_{j2} = - \frac{U_{j1}}4$. All other entries of $v$ are zero.
Clearly $v \in \mathcal{T}(U)$, and $\|v\|\le 1$. On the other hand, we know $\mathfrak{M}(U)$ restricted to these 4 coordinates $(i1, i2, j1, j2)$ is

\begin{equation}
\left(\begin{array}{cccc}
2\psi_{i1}(U) & 0 & 4U_{i1}U_{j1} & 0 \\ 
0 & 2\psi_{i2}(U) & 0 &  4U_{i2}U_{j2} \\ 
 4U_{i1}U_{j1} & 0 & 2\psi_{j1}(U) & 0 \\ 
0 &  4U_{i2}U_{j2} & 0 & 2\psi_{j2}(U)
\end{array} \right)
\end{equation}
By Eq.(\ref{pp_2}), we know all diagonal entries are $\le 2\epsilon_0^5$. 

If $U_{i1}U_{j1}U_{i2}U_{j2}$ is negative, we have the
quadratic form:
\begin{align}
	v^T\mathfrak{M}(U) v = & U_{i1}U_{j1}U_{i2}U_{j2}+\frac{1}{8}[U_{i2}^2\psi_{i1}(U)+
U_{i1}^2\psi_{i2}(U)
+U_{j2}^2\psi_{j1}(U)+
U_{j1}^2\psi_{j2}(U)] \nonumber \\
\le & -\epsilon_0^4 + \epsilon_0^5 \le -\frac{1}{4}\epsilon^4_0 = -\gamma
\end{align}
If $U_{i1}U_{j1}U_{i2}U_{j2}$ is positive we just swap the sign of the first two coordinates $v_{i1} = -\frac{U_{i2}}2$, $v_{i2} = \frac{U_{i1}}2$ and the above argument would still holds.

\paragraph{Case 2}: $|\BigC(u^{(i)})| \ge 2$, $|\BigC(u^{(j)})|\ge 2$, and 
$|\BigC(u^{(i)})\cap \BigC(u^{(j)})| = 1$.

WLOG, assume $\{1, 2\} \in \BigC(u^{(i)})$ and $ \{1, 3\}\in \BigC(u^{(j)})$, choose $v$ to be 
$v_{i1} = \frac{U_{i2}}{4}$, $v_{i2} = -\frac{U_{i1}}{4}$, $v_{j1} = \frac{U_{j3}}{4}$ and $v_{j3} = - \frac{U_{j1}}{4}$.
All other entries of $v$ are zero.
Clearly $v \in \mathcal{T}(U)$ and $\|v\|\le 1$. On the other hand, we know $\mathfrak{M}(U)$ restricted to these 4 coordinates $(i1, i2, j1, j3)$ is

\begin{equation}
\left(\begin{array}{cccc}
2\psi_{i1}(U) & 0 & 4U_{i1}U_{j1} & 0 \\ 
0 & 2\psi_{i2}(U) & 0 & 0 \\ 
4U_{i1}U_{j1} & 0 & 2\psi_{j1}(U) & 0 \\ 
0 & 0 & 0 & 2\psi_{j3}(U)
\end{array} \right)
\end{equation}
By Eq.(\ref{pp_2}), we know all diagonal entries are $\le 2\epsilon_0^5$. 
If $U_{i1}U_{j1}U_{i2}U_{j3}$ is negative, we have the
quadratic form:
\begin{align}
	v^T\mathfrak{M}(U) v = & \frac{1}{2}U_{i1}U_{j1}U_{i2}U_{j3}+\frac{1}{8}[U_{i2}^2\psi_{i1}(U)+
U_{i1}^2\psi_{i2}(U)
+U_{j3}^2\psi_{j1}(U)+
U_{j1}^2\psi_{j3}(U)] \nonumber \\
\le & -\frac{1}{2}\epsilon_0^4 + \epsilon_0^5 \le -\frac{1}{4}\epsilon^4_0 = -\gamma
\end{align}
If $U_{i1}U_{j1}U_{i2}U_{j3}$ is positive we just swap the sign of the first two coordinates $v_{i1} = -\frac{U_{i2}}2$, $v_{i2} = \frac{U_{i1}}2$ and the above argument would still holds.

\paragraph{Case 3}: Either $|\BigC(u^{(i)})| =1 $ or $|\BigC(u^{(j)})| =1 $.

WLOG, suppose $|\BigC(u^{(i)})| =1$, and $\{1\}= \BigC(u^{(i)}) $, we know:
\begin{equation}
	| (u^{(i)}_1)^2 -1| \le (d-1)\epsilon_0^2
\end{equation}
On the other hand, since $\BigC(u^{(i)})\cap \BigC(u^{(j)})\neq \emptyset$, 
we have $\BigC(u^{(i)})\cap \BigC(u^{(j)}) = \{1\}$, and thus:
\begin{equation}
|\psi_{j1}(U)| = |\sum_{i': i'\neq j} U^2_{i'1} -  \sum_{i':i'\neq j} h (u^{(i')}, u^{(j)} ) |
\le \epsilon_0^5
\end{equation}
Therefore, we have:
\begin{align}
\sum_{i':i'\neq j} h (u^{(i')}, u^{(j)} ) \ge \sum_{i': i'\neq j} U^2_{i'1} - \epsilon_0^5
\ge U^2_{i1} - \epsilon_0^5 \ge 1-d\epsilon_0^2
\end{align}
and
\begin{align}
	\sum_{k=1}^d \psi_{jk}(U) = &\sum_{i': i'\neq j}\sum^d_{k=1} U^2_{i'k} -  d\sum_{i':i'\neq j} h (u^{(i')}, u^{(j)} ) \nonumber \\
	\le & d-1 - d(1-d\epsilon_0^2)  = -1 + d^2 \epsilon_0^2
\end{align}
Thus, we know, there must exist some $ k' \in [d]$, so that $\psi_{jk'}(U) \le -\frac{1}{d}
+ d\epsilon_0^2$. 
This means we have ``large'' negative entry on the diagonal of $\mathfrak{M}$. 
Since $|\psi_{j1}(U)| \le \epsilon_0^5$, we know $k'\neq 1$. WLOG, suppose $k'=2$, we have $|\psi_{j2}(U)| > \epsilon_0^5$, thus $|U_{j2}| \le \epsilon_0$.

Choose $v$ to be 
$v_{j1} = \frac{U_{j2}}{2}$, $v_{j2} = -\frac{U_{j1}}{2}$.
All other entries of $v$ are zero.
Clearly $v \in \mathcal{T}(U)$ and $\|v\|\le 1$. On the other hand, we know $\mathfrak{M}(U)$ restricted to these 2 coordinates $(j1, j2)$ is
\begin{equation}
\left(\begin{array}{cc}
2\psi_{j1}(U) & 0 \\
0 & 2\psi_{j2}(U) \\
\end{array} \right)
\end{equation}
We know $|U_{j1}| > \epsilon_0$, $|U_{j2}| \le \epsilon_0$, $|\psi_{j1}(U)|
\le \epsilon_0^5$, and $\psi_{j2}(U) \le -\frac{1}{d} + d\epsilon_0^2$. 
Thus:
\begin{align}
	v^T\mathfrak{M}(U) v = & \frac{1}{2}\psi_{j1}(U)U_{j2}^2+ \frac{1}{2}\psi_{j2}(U)U_{j1}^2 \nonumber \\
\le & \epsilon_0^7 - (\frac{1}{d} - d\epsilon_0^2)\epsilon_0^2 
\le -\frac{1}{2d}\epsilon_0^2  \le -\gamma
\end{align}
Since by our choice of $v$, we have $\|v\|\le 1$, we can choose $\hat{v} = v/\|v\|$, and immediately have $\hat{v} \in \mathcal{T}(U)$ and $\|\hat{v}\| = 1$, and
$\hat{v}^T  \mathfrak{M}(U) \hat{v}  \le -\gamma$.
\end{proof}

\begin{lemma}\label{lem:Problem2_case1}
Under the choice of parameters in Eq.(\ref{choice_2}),
suppose $\|\chi(U)\| \le \epsilon$, and for any $(i,j) \in [d]^2$ we have $\BigC(u^{(i)})
\cap \BigC(u^{(j)}) = \emptyset$. Then,
there is a local minimum $U^\star$ such that $\|U-U^\star\| \le \delta$,
and for 
all $U'$ in the $2\delta$ neighborhood of $U^\star$, we have $\hat{v}^T  \mathfrak{M}(U') \hat{v}  \ge 1$ for all $\hat{v} \in \mathcal{T}(U')$, $\|\hat{v}\| = 1$
\end{lemma}

\begin{proof}
WLOG, we assume $\BigC(u^{(i)}) = \{i\}$ for $i=1, \cdots, d$. Then, we immediately have:
\begin{equation}
	|u^{(i)}_j| \le \epsilon_0, \quad\quad | (u^{(i)}_i)^2 -1| \le (d-1)\epsilon_0^2, \quad\quad
	\forall (i, j)\in[d]^2, j\neq i
\end{equation} 
Then $u^{(i)}_i \ge \sqrt{1-d\epsilon_0^2}$ or $u^{(i)}_i \le -\sqrt{1-d\epsilon_0^2}$.
Which means $u^{(i)}_i$ is either close to $1$ or close to $-1$. By symmetry, we know WLOG, 
we can assume the case $u^{(i)}_i \ge \sqrt{1-d\epsilon_0^2}$ for all $i\in[d]$.

Let $V\in \mathbb{R}^{d^2}$ be the concatenation of $\{e_1, e_2, \cdots, e_d\}$, then we have:
\begin{equation}
	\|U-V\|^2 = \sum_{i=1}^d \|u^{(i)} - e_i\|^2\le 2 d^2 \epsilon_0^2 \le \delta^2
\end{equation}

Next, we show $V$ is a local minimum. According to Eq.\ref{frakM_2}, we know $\mathfrak{M}(V)$ is a diagonal matrix with $d^2$ entries: 
\begin{align}
	[\mathfrak{M}(V)]_{ik,ik} = 2\psi_{ik}(V)  = 2\sum_{j: j\neq i} [V_{jk}^2-\sum_{l=1}^d V_{il}^2 V_{jl}^2] =
	\begin{cases}
		2  &\mbox{~if~} i\neq k\\
		0  &\mbox{~if~} i=k
	\end{cases}
\end{align}
We know the unit vector in the direction that corresponds to $[\mathfrak{M}(V)]_{ii,ii}$ is 
not in the tangent space $\mathcal{T}(V)$ for all $i\in[d]$. Therefore, for any $v \in \mathcal{T}(V)$, we have
\begin{align}
v^T\mathfrak{M}(e_1) v \ge 2 \|v\|^2 >0 \quad \quad 
\text{for all~} v\in \mathcal{T}(V), v\neq 0
\end{align}
Which by Theorem \ref{thm::second_sufficient} means $V$ is a local minimum.

Finally, 
denote $\mathcal{T}_V = \mathcal{T}(V)$ be the tangent space of constraint manifold at $V$.
We know for all $U'$ in the $2\delta$ neighborhood of $V$, 
and for all $\hat{v} \in \mathcal{T}(x')$, $\|\hat{v}\| = 1$:
\begin{align}
 \hat{v}^T  \mathfrak{M}(U') \hat{v}  
 \ge &  \hat{v}^T  \mathfrak{M}(V) \hat{v}  - 
 |\hat{v}^T  \mathfrak{M}(V) \hat{v} - \hat{v}^T  \mathfrak{M}(U') \hat{v} | \nonumber\\
 = & 2\|P_{\mathcal{T}_V}\hat{v} \|^2 
 - \|\mathfrak{M}(V) - \mathfrak{M}(U')\|\|\hat{v}\|^2 \nonumber \\
 = & 2 - 2\|P_{\mathcal{T}^c_V}\hat{v} \|^2 - \|\mathfrak{M}(V) - \mathfrak{M}(U')\|
\end{align}
By lemma \ref{lem::normal}, we know $\|P_{\mathcal{T}^c_V}\hat{v} \|^2
\le \|U'-V\|^2 \le 4\delta^2$. By Eq.(\ref{frakM_2}), we have:
\begin{align}
	&\|\mathfrak{M}(V) - \mathfrak{M}(U')\| \le  \|\mathfrak{M}(V) - \mathfrak{M}(U')\| 
	\le \sum_{(i,j,k)} |[\mathfrak{M}(V)]_{ik,jk} - [\mathfrak{M}(U')]_{ik,jk}| \le  100d^3\delta
\end{align}
In conclusion, we have $\hat{v}^T  \mathfrak{M}(U') \hat{v} \ge 2- 8\delta^2-100d^3\delta\ge 1$
which finishs the proof.
\end{proof}

Finally, we are ready to prove Theorem \ref{thm:problem_2_strict_saddle}.
\begin{proof}[Proof of Theorem \ref{thm:problem_2_strict_saddle}]

Similarly, $(\alpha,\gamma, \epsilon,\delta)$-\name immediately follows from Lemma \ref{lem:Problem2_case2} and Lemma \ref{lem:Problem2_case1}.

The only thing remains to show is that Optimization problem (\ref{eq:hardprob}) has exactly $2^d \cdot d!$ local minimum that corresponds to permutation and sign flips of $a_i$'s.
This can be easily proved by the same argument as in the proof of Theorem \ref{thm:problem_1_strict_saddle}.
\end{proof}

\subsection{Extending to Tensors of Different Order}

\label{app:tensorextension}

In this section we show how to generalize our algorithm to tensors of different orders. As a $8^\tha$ order tensor (and more generally, $4p^\tha$ order tensor for $p \in \mathcal{N}^+$) can always be considered to be a $4^\tha$ order tensor with components $a_i^\otimes a_i$ ( $a_i^{\otimes p}$ in general), so it is trivial to generalize our algorithm to $8^\tha$ order or any $4p^\tha$ order.

For tensors of other orders, we need to apply some transformation. As a concrete example, we show how to transform an orthogonal 3rd order tensor into an orthogonal $4^\tha$ order tensor.

We first need to define a few notations. For third order tensors $A,B\in \R^{d^3}$, we define $(A\otimes B)_{i_1,i_2,...,i_6} = A_{i_1,i_2,i_3}B_{i_4,i_5,i_6} (i_1,...,i_6\in [d])$. We also define the {\em partial trace} operation that maps a $6$-th order tensor $T\in \R^{d^6}$ to a $4$-th order tensor in $\R^{d^4}$:
$$
ptrace(T)_{i_1,i_2,i_3,i_4} = \sum_{i=1}^d T(i, i_1, i_2, i, i_3, i_4).
$$
Basically, the operation views the tensor as a $d^3\times d^3$ matrix with $d^2\times d^2$ $d\times d$ matrix blocks, then takes the trace of each matrix block. Now given a random variable $X \in \R^{d^3}$ whose expectation is an orthogonal third order tensor, we can use these operations to construct an orthogonal $4$-th order tensor:

\begin{lemma}
Suppose the expectation of random variable $X\in \R^{d^3}$ is an orthogonal 3rd order tensor:
$$
\E[X] = \sum_{i=1}^d a_i^{\otimes 3},
$$
where $a_i$'s are orthonormal vectors. Let $X'$ be an independent sample of $X$, then we know
$$
\E[ptrace(X\otimes X')] = \sum_{i=1}^d a_i^{\otimes 4}.
$$
In other words, we can construct random samples whose expectation is equal to a 4-th order orthogonal tensor.
\end{lemma}

\begin{proof}
Since $ptrace$ and $\otimes$ are all linear operations, by linearity of expectation we know 
$$
\E[ptrace(X\otimes X')] = ptrace(\E[X]\otimes \E[X']) = ptrace((\sum_{i=1}^d a_i^{\otimes 3})\otimes (\sum_{i=1}^d a_i^{\otimes 3})).
$$
We can then expand out the product:
$$
(\sum_{i=1}^d a_i^{\otimes 3})\otimes (\sum_{i=1}^d a_i^{\otimes 3})
= \sum_{i=1}^d a_i^{\otimes 6} + \sum_{i\ne j} a_i^{\otimes 3}\otimes a_j^{\otimes 3}.
$$
For the diagonal terms, we know $ptrace(a_i^\otimes 6) = \|a_i\|^2 a_i^\otimes 4 = a_i^\otimes 4$. For the $i\ne j$ terms, we know $ptrace(a_i^{\otimes 3}\otimes a_j^{\otimes 3}) = \inner{a_i,a_j} a_i^\otimes 2\otimes a_j^\otimes 2 = 0$ (since $a_i,a_j$ are orthogonal). Therefore we must have
$$
ptrace((\sum_{i=1}^d a_i^{\otimes 3})\otimes (\sum_{i=1}^d a_i^{\otimes 3}))
=\sum_{i=1}^d ptrace(a_i^{\otimes 6}) + \sum_{i\ne j} ptrace(a_i^{\otimes 3}\otimes a_j^{\otimes 3}) = \sum_{i=1}^d a_i^{\otimes 4}.
$$
This gives the result.
\end{proof}

Using similar operations we can easily convert all odd-order tensors into order $4p (p\in \N^+)$. For tensors of order $4p+2 (p\in \N^+)$, we can simply apply the partial trace and get a tensor of order $4p$ with desirable properties. Therefore our results applies for all orders of tensors.

\chapter[Appendix for Applying Online Tensor Methods for Learning LVMs]{Appendix for Applying Online Tensor Methods for Learning Latent Variable Models}
\section{Stochastic Updates}
\label{sec:apdx_update}
After obtaining the whitening matrix, we whiten the data $G^\top_{x,A}$, $G^\top_{x,B}$ and $G^\top_{x,C}$ by linear operations to get $y^t_A$, $y^t_B$ and $y^t_C\in \mathbb{R}^{k}$:
\begin{align*}
y^t_A : = \left<G^\top_{x,A}, W \right>,
\; %\\
y^t_B := \left<Z_B G^\top_{x,B},W \right>,
\; %\\
y^t_C& : = \left<Z_C G^\top_{x,C},W \right>.
\end{align*}
where $x\in X$ and $t$ denotes the index of the online data.

%According to the discuss in main paper Section~\ref{sec:sto_ten_grad_des}, our online
The stochastic gradient descent algorithm is obtained by taking the derivative of the loss function $\frac{\partial L^t(\mathbf{v})}{\partial v_i}$:
\begin{align*}
\frac{\partial L^t(\mathbf{v})}{\partial v_i}=&
\theta\sum\limits_{j=1}^{k} \left<v_j,v_i\right>^2 v_j
- \frac{(\alpha_0+1)(\alpha_0+2)}{2} \left<v_i, y_A^t\right> \left<v_i, y_B^t\right> y_C^t
- \alpha_0^2 \left<\phi_i^t,\bar{y}_A\right>\left<\phi_i^t,\bar{y}_B^t\right>\bar{y}_C \\
&+ \frac{\alpha_0(\alpha_0+1)}{2}\left<\phi_i^t, y_A^t\right>\left<\phi_i^t, y_B^t\right>\bar{y}_C
+\frac{\alpha_0(\alpha_0+1)}{2}\left<\phi_i^t, y_A^t\right>\left<\phi_i^t, \bar{y}_B\right>y_C
\\
&+\frac{\alpha_0(\alpha_0+1)}{2}\left<\phi_i^t,\bar{y}_A\right>\left<\phi_i^t,y_B^t\right>y_C
\end{align*}
for $i \in [k]$, where $y_A^t$, $y_B^t$ and $y_C^t$ are the online whitened data points as discussed in the whitening step and $\theta$ is a constant factor that we can set.

The iterative updating equation for the stochastic gradient update is given by
\begin{equation}
\phi_i^{t+1} \leftarrow \phi_i^t - \beta^t \frac{\partial L^t}{\partial v_i}\llvert_{\phi_i^t}
\end{equation}
for $i \in [k]$, where $\beta^t$ is the learning rate, $\phi^t_i$ is the last iteration eigenvector and $\phi^t_i$ is the updated eigenvector.
We update eigenvectors through
\begin{align}
\phi_i^{t+1} \leftarrow  \phi_i^t  & - \theta\beta^t \sum\limits_{j=1}^{k} \left[\left<\phi_j^t,\phi_i^t\right>^2 \phi_j^t\right]%\label{eq:term1}
%\\
%&
+ \text{shift} [ \beta^t \left<\phi_i^t, y_A^t\right> \left<\phi_i^t, y_B^t\right> y_C^t ]\label{eq:term2_1}
%\\
% &+ \beta^t \left<\phi_i^t, g_B^t\right> \left<\phi_i^t, g_C^t\right> g_A^t \label{eq:term2_2}
 %\\
%& + \beta^t \left<\phi_i^t, g_A^t\right> \left<\phi_i^t, g_C^t\right> g_B^t \label{eq:term2_3}
\end{align}

Now we shift the updating steps so that they correspond to the centered Dirichlet moment forms, i.e.,%: term~\eqref{eq:term2_1} is
\begin{align}
& \text{shift}[ \beta^t \left<\phi_i^t, y_A^t\right> \left<\phi_i^t, y_B^t\right> y_C^t ]
%\\
 :=
\beta^t \frac{(\alpha_0+1)(\alpha_0+2)}{2} \left<\phi_i^t,y_A^t\right> \left<\phi_i^t, y_B^t\right> y_C^t \nonumber %\label{eq:detail_term1}
\\
& +  \beta^t {\alpha_0^2}\left<\phi_i^t,\bar{y}_A\right> \left<\phi_i^t, \bar{y}_B\right>\bar{y}_C \nonumber %\label{eq:detail_term2}
- \beta^t \frac{\alpha_0(\alpha_0+1)}{2}\left<\phi_i^t, y_A^t\right>\left<\phi_i^t, y_B^t\right>\bar{y}_C \nonumber %\label{eq:detail_term3}
\\
&-  \beta^t \frac{\alpha_0(\alpha_0+1)}{2}\left<\phi_i^t, y_A^t\right>\left<\phi_i^t, \bar{y}_B\right>y_C %\nonumber %\label{eq:detail_term4}
%\\
-  \beta^t \frac{\alpha_0(\alpha_0+1)}{2}\left<\phi_i^t,\bar{y}_A\right>\left<\phi_i^t,y_B^t\right>y_C,%\nonumber %\label{eq:detail_term5}
\end{align}
where $\bar{y}_A:= \mathbb{E}_t [y_A^t]$ and similarly for $\bar{y}_B$ and $\bar{y}_C$.

\section{Proof of Algorithm Correctness}
%We presented our algorithm implementation details in Appendices~\ref{sec:apdx_part},~\ref{sec:apdx_white} and~\ref{sec:apdx_update}.
We now prove the correctness of our algorithm.

First, we compute $M_2$ as just $$\mathbb{E}_x\left[ \tilde{G}_{x,C}^\top \otimes \tilde{G}_{x,B}^\top| \Pi_A, \Pi_B, \Pi_C\right]$$
where we define
\begin{align*}
\tilde{G}_{x,B}^\top & := \mathbb{E}_x\left[G_{x,A}^\top \otimes G_{x,C}^\top\llvert \Pi_A, \Pi_C\right] \left(\mathbb{E}_x\left[G_{x,B}^\top \otimes G_{x,C}^\top \llvert \Pi_B, \Pi_C\right]\right)^\dag G_{x,B}^\top
\\
\tilde{G}_{x,C}^\top & :=\mathbb{E}_x\left[G_{x,A}^\top \otimes G_{x,B}^\top\llvert \Pi_A, \Pi_B \right]\left(\mathbb{E}_x\left[G_{x,C}^\top \otimes G_{x,B}^\top\llvert \Pi_B, \Pi_C \right]\right)^{\dagger} G_{x,C}^\top.
\end{align*}
Define $F_A$ as $F_A := \Pi_A^\top P^\top$, we obtain $M_2$ $=$ $\mathbb{E}\left[G^\top_{x,A} \otimes G^\top_{x,A}\right]$ $ =$ $ \Pi_A^\top P^\top \left(\mathbb{E}_x[\pi_x \pi_x^\top]\right) P \Pi_A$ $ =$ $ F_A\left(\mathbb{E}_x[\pi_x \pi_x^\top]\right)F_A^\top$. Note that $P$ is the community connectivity matrix defined as $P\in [0,1]^{k\times k}$.
Now that we know $M_2$, $\mathbb{E}\left[\pi_i^2\right]= \frac{\alpha_i(\alpha_i+1)}{\alpha_0(\alpha_0+1)}$, and $\mathbb{E}\left[\pi_i \pi_j\right]= \frac{\alpha_i \alpha_j}{\alpha_0(\alpha_0+1)} \forall i\neq j$, we can get the centered second order moments $\Pairs^{\community}$ as

\begin{align}
\Pairs^{\community} & := F_A \text{ diag}\left(\left[\frac{\alpha_1 \alpha_1+1}{\alpha_0(\alpha_0+1)},\ldots,\frac{\alpha_k \alpha_k+1}{\alpha_0(\alpha_0+1)}\right]\right)F_A^\top
\\
& = M_2 - \frac{\alpha_0}{\alpha_0+1} F_A \left(\hat{\alpha} \hat{\alpha}^\top - \text{ diag}\left(\hat{\alpha} \hat{\alpha}^\top	 \right)\right)F_A^\top
\\
& = \frac{1}{n_X} \sum\limits_{x\in X} Z_C G_{x,C}^\top G_{x,B} Z_B^\top  -\frac{\alpha_0}{\alpha_0+1} \left(\mu_{A} \mu_{A}^\top- \text{ diag}\left(\mu_{A} \mu_{X\rightarrow A}^\top\right) \right)
\end{align}
Thus, our whitening matrix is computed. Now, our whitened tensor is $\mathcal{T} $ is given by
\begin{align*}
\mathcal{T} & = \mathcal{T}^{\community} (W,W,W) = \frac{1}{n_X} \sum_{x}\left[(W^\top F_A \pi^{\alpha_0}_x) \otimes (W^\top F_A \pi^{\alpha_0}_x) \otimes (W^\top F_A \pi^{\alpha_0}_x)\right],
\end{align*}
where $\pi^{\alpha_0}_x$ is the centered vector so that $\Ebb[\pi_x^{\alpha_0}\otimes \pi_x^{\alpha_0}\otimes \pi_x^{\alpha_0}]$ is diagonal.
We then apply the stochastic gradient descent technique to decompose the third order moment.

\section{GPU Architecture}
\label{sec:apdx_arch}

The algorithm we propose is very amenable to parallelization and is scalable which makes it suitable to implement on processors with multiple cores in it. Our method consists of simple linear algebraic operations, thus enabling us to utilize \emph{Basic Linear Algebra Subprograms} (BLAS) routines such as BLAS I (vector operations), BLAS II (matrix-vector operations), BLAS III (matrix-matrix operations), Singular Value Decomposition (SVD),  and iterative operations such as stochastic gradient descent for tensor decomposition that can easily take advantage of Single Instruction Multiple Data (SIMD)  hardware units present in the GPUs. As such, our method is amenable to parallelization and is  ideal for GPU-based implementation.

\paragraph{Overview of code design: }From a higher level point of view, a typical GPU based computation is a three step process involving data transfer from CPU memory to GPU global memory, operations on the data now present in GPU memory and finally, the result transfer from the GPU memory back to the CPU memory. We use the CULA library for implementing the linear algebraic operations. %Refer appendix for details.

\paragraph{GPU compute architecture: } The GPUs achieve massive parallelism by having hundreds of homogeneous processing cores integrated on-chip. Massive replication of these cores provides the parallelism needed by the applications that run on the GPUs. These cores, for the Nvidia GPUs, are known as \emph{CUDA cores}, where each core has fully pipelined floating-point and integer arithmetic logic units. In Nvidia's Kepler architecture based GPUs% in Figure~\ref{fig:nvidiakepler}
, these CUDA cores are bunched together to form a \emph{Streaming Multiprocessor} (SMX). These SMX units act as the basic building block for Nvidia Kepler GPUs. Each GPU contains multiple SMX units where each SMX unit has 192 single-precision CUDA cores, 64 double-precision units, 32 special function units, and 32 load/store units for data movement between cores and memory.

Each SMX has L$1$, shared memory and a read-only data cache that are common to all the CUDA cores in that SMX unit. Moreover, the programmer can choose between different configurations of the shared memory and L$1$ cache. Kepler GPUs also have an L$2$ cache memory of about $1.5$MB that is common to all the on-chip SMXs. Apart from the above mentioned memories, Kepler based GPU cards come with a large DRAM memory, also known as the global memory, whose size is usually in gigabytes. This global memory is also visible to all the cores. The GPU cards usually do not exist as standalone devices. Rather they are part of a CPU based system, where the CPU and GPU interact with each other via PCI (or PCI Express) bus.

In order to program these massively parallel GPUs, Nvidia provides a framework known as \emph{CUDA} that enables the developers to write programs in languages like C, C++, and Fortran etc. A CUDA program constitutes of functions called \emph{CUDA kernels} that execute across many parallel software threads, where each thread runs on a CUDA core. Thus the GPU's performance and scalability is exploited by the simple partitioning of the algorithm into fixed sized blocks of parallel threads that run on hundreds of CUDA cores. The threads running on an SMX can synchronize and cooperate with each other via the shared memory of that SMX unit and can access the Global memory. Note that the CUDA kernels are launched by the CPU but they get executed on the GPU. Thus compute architecture of the GPU requires CPU to initiate the CUDA kernels.

CUDA enables the programming of Nvidia GPUs by exposing low level API. Apart from CUDA framework, Nvidia provides a wide variety of other tools and also supports third party libraries that can be used to program Nvidia GPUs. Since a major chunk of the scientific computing algorithms is linear algebra based, it is not surprising that the standard linear algebraic solver libraries like BLAS and \emph{Linear Algebra PACKage} (LAPACK) also have their equivalents for Nvidia GPUs in one form or another. Unlike CUDA APIs, such libraries expose APIs at a much higher-level and mask the architectural details of the underlying GPU hardware to some extent thus enabling relatively faster development time.

Considering the tradeoffs between the algorithm's computational requirements, design flexibility, execution speed and development time, we choose \emph{CULA-Dense} as our main implementation library. CULA-Dense provides GPU based implementations of the LAPACK and BLAS libraries for dense linear algebra and contains routines for systems solvers, singular value decompositions, and eigen-problems. Along with the rich set of functions that it offers, CULA provides the flexibility needed by the programmer to rapidly implement the algorithm while maintaining the performance. It hides most of the GPU architecture dependent programming details thus making it possible for rapid prototyping of GPU intensive routines.

The data transfers between the CPU memory and the GPU memory are usually explicitly initiated by CPU and are carried out via the PCI (or PCI Express) bus interconnecting the CPU and the GPU. The movement of data buffers between CPU and GPU is the most taxing in terms of time. The buffer transaction time is shown in the plot in Figure~\ref{tran_time_expt}. Newer GPUs, like Kepler based GPUs, also support useful features like GPU-GPU direct data transfers without CPU intervention.
%Our system and software specifications are given in Table~\ref{tab:sys_spec}.

\begin{figure}[H]
	\centering
	\psfrag{Logarithm of the buffer size divided by 8}[l]{\scriptsize{$\log\left(\frac{\text{buffer size}}{8}\right)$}}
	%Logarithm of the buffer size divided by 8}}%{\scriptsize{$\log \left(\frac{\text{buffer size}}{8}\right)$}}
	\psfrag{Time (in seconds)}[c]{\scriptsize{Time (s)}}
	\psfrag{CPU-GPU buffer round-trip truncation time}[c]{}
	\includegraphics[width=0.4\textwidth]{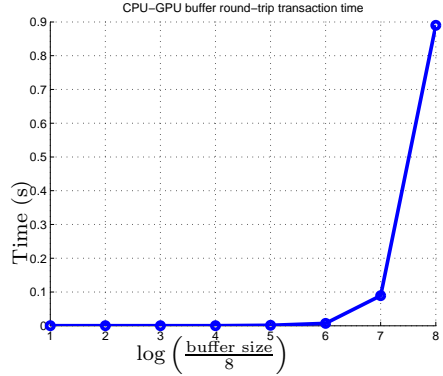}
	\caption{Experimentally measured time taken for buffer transfer between the CPU and the GPU memory in our system.}
	\label{tran_time_expt}
\end{figure}

CULA exposes two important interfaces for GPU programming namely, \emph{standard} and \emph{device}.
Using the standard interface, the developer can program without worrying about the underlying architectural details of the GPU as the standard interface takes care of all the data movements, memory allocations in the GPU and synchronization issues. This however comes at a cost.
For every standard interface function call the data is moved in and out of the GPU even if the output result of one operation is directly required by the subsequent operation.
This unnecessary movement of intermediate data can dramatically impact the performance of the program.
In order to avoid this, CULA provides the device interface.
We  use the  device interface for STGD in which the programmer is responsible for data buffer allocations in the GPU memory, the required data movements between the CPU and GPU, and operates only on the data in the GPU. Thus the subroutines of the program that are iterative in nature are good candidates for device implementation. %See Section~\ref{sec:stgd} below for a detailed description of the implementation of the STGD method.

\paragraph{Pre-processing and post-processing: }
%\fhcomment{mv this subsection to appendix, refer to it using concise sentences, maybe just two sentences. Number them. }
The pre-processing involves matrices whose leading dimension is of the order of number of nodes. These are implemented using the CULA standard interface BLAS II and BLAS III routines.

Pre-processing requires SVD computations for the Moore-Penrose pseudoinverse calculations. We use CULA SVD routines since  these SVD operations are carried out on matrices of moderate size.
%This can be easily replaced with more scalable SVD routines to handle larger datasets~\cite{Holmes_quic-svd:fast}.
We further replaced the CULA SVD routines with more scalable SVD and pseudo inverse routines using random projections~\cite{gittens2013revisiting} to handle larger datasets such as DBLP dataset in our experiment.

%We also implement the iterative method of~\cite{1966} for the pseudoinverse computations.
After STGD, the community membership matrix estimates are obtained using BLAS III routines provided by the CULA standard interface. The matrices are then used for hypothesis testing to evaluate the algorithm against the ground truth.

\section{Results on Synthetic Datasets}
\label{sec:apdx_synth}
{\em Homophily} is an important factor in social interactions~\cite{mcpherson2001birds}; the term {\em homophily} refers to the tendency that actors in the same community interact more than across different communities. Therefore, we assume diagonal dominated community connectivity matrix $P$ with diagonal elements equal to $0.9$ and off-diagonal elements equal to $0.1$.
Note that $P$ need neither be stochastic nor symmetric. Our algorithm allows for randomly generated community connectivity matrix $P$ with support $[0,1]$. In this way, we look at general directed social ties among communities.

\begin{table}%{r}{2.5in}%[htbp]
   \centering
   \begin{tabular}{@{} llllr @{}}
      \toprule
      $n$  & $k$ & $\alpha_0$ &  Error& Time (secs)\\
      \midrule
       1e2 	&10	&	0	& 0.1200	   & 0.5\\
       1e3 	&10	&	0	&  0.1010  & 1.2\\
       1e4 	&10	&	0 	&   0.0841 & 43.2\\
       1e2 	&10	&	1	&   0.1455	& 0.5\\
       1e3 	&10	&	1	&   0.1452	& 1.2\\
       1e4 	&10	&	1 	&   0.1259	& 42.2\\
      \bottomrule
   \end{tabular}
   \caption{Synthetic simulation results for different configurations. Running time is the time taken to run to convergence.}
   \label{tab:synResult}
\end{table}

We perform experiments for both the stochastic block model ($\alpha_0=0$) and the mixed membership model. For the mixed membership model, we set the concentration parameter $\alpha_0 =1$. We note that the error is around $8\% - 14\%$ and the running times are under a minute, when $n \leq 10000$ and $n \gg k$.

The results are given in Table~\ref{tab:synResult}.
We observe that more samples result in a more accurate recovery of memberships which matches intuition and theory. Overall, our learning algorithm performs better in the stochastic block model case than in the mixed membership model case although we note that the accuracy is quite high for practical purposes. Theoretically, this is expected since smaller concentration parameter $\alpha_0$ is easier for our algorithm to learn~\cite{AnandkumarEtal:community12COLT}. Also, our algorithm is scalable to an order of magnitude more in $n$ as illustrated by experiments on real-world large-scale datasets.

\section{Comparison of Error Scores}
\label{sec:otherscores} %\label{apdx:nmi_eg}
Normalized Mutual Information (NMI) score~\cite{lancichinetti2009detecting} is another popular score which is defined differently for overlapping and non-overlapping community models. For non-overlapping block model, ground truth membership for node $i$ is a discrete $k$-state categorical variable $\Pi_{\text{block}} \in [k]$ and the estimated membership is a discrete $\widehat{k}$-state categorical variable $\widehat{\Pi}_{\text{block}} \in [\widehat{k}]$. %Recall that there are $n$ nodes, t
The empirical distribution of ground truth membership categorical variable  $\Pi_{\text{block}}$ is easy to obtain. Similarly is the empirical distribution of the estimated membership categorical variable $\widehat{\Pi}_{\text{block}}$.  NMI for block model is defined as
\begin{align*}
& N_{\text{block}}(\widehat{\Pi}_{\text{block}}:\Pi_{\text{block}}) :=  \frac{H(\Pi_{\text{block}})+H(\widehat{\Pi}_{\text{block}})-H(\Pi_{\text{block}},\widehat{\Pi}_{\text{block}})}{\left(H(\Pi_{\text{block}})+H(\widehat{\Pi}_{\text{block}})\right)/2} .
\end{align*}

The NMI for overlapping communities is  a binary vector instead of a categorical variable~\cite{lancichinetti2009detecting}.  The ground truth membership for node $i$ is a binary vector of length $k$, $\mathbf{\Pi}_{{\text{mix}}}$, while the estimated membership for node $i$ is a binary vector of length $\widehat{k}$, $\mathbf{\widehat{\Pi}}_{{\text{mix}}}$. This notion coincides with one column of our membership matrices $\Pi\in\mathbb{R}^{k\times n}$ and $\widehat{\Pi}\in \mathbb{R}^{\widehat{k} \times n}$ except that our membership matrices are stochastic. In other words, we consider all the nonzero entries of $\Pi$ as 1's, then each column of our $\Pi$ is a sample for $\Pi_{{\text{mix}}}$. The $m$-th entry of this binary vector is the realization of a random variable $\Pi_{{\text{mix}}_m} = (\mathbf{\Pi}_{{\text{mix}}})_m$, whose probability distribution is
\[
P(\Pi_{{\text{mix}}_m}=1)= \frac{n_m}{n}, \quad P(\Pi_{{\text{mix}}_m}=0)= 1-\frac{n_m}{n},
\]
where $n_m$ is the number of nodes in community $m$. The same holds for $\widehat{\Pi}_{\text{mix}_m}$. The normalized conditional entropy between $\mathbf{\Pi}_{{\text{mix}}}$ and $\mathbf{\widehat{\Pi}}_{{\text{mix}}}$ is defined as
\begin{equation}\label{eqn:nmi}
H(\mathbf{\widehat{\Pi}}_{{\text{mix}}} \lvert \mathbf{{\Pi}}_{{\text{mix}}})_{\text{norm}}  :=\frac{1}{k} \sum_{j \in [k]} \min_{i \in [\widehat{k}]}
%\frac{H(\widehat{\Pi}_i\lvert{\Pi}_j)}{H(\Pi_j)}
\frac{H\left(\widehat{\Pi}_{{\text{mix}}_i}  \lvert  \Pi_{{\text{mix}}_j} \right)}{H(\Pi_{{\text{mix}}_j} )}
\end{equation}
where $\Pi_{\text{mix}_j}$ denotes the $j^{th}$ entry of $\mathbf{\Pi}_{\text{mix}}$ and similarly for $\widehat{\Pi}_{\text{mix}_i}$.
The NMI for overlapping community is
\begin{align*}
& N_{\text{mix}}(\mathbf{\widehat{\Pi}}_{\text{mix}} : \mathbf{\Pi} _{\text{mix}}):= 1-\frac{1}{2}\left[H(\mathbf{\Pi}_{\text{mix}} \lvert \mathbf{\widehat{\Pi}}_{\text{mix}})_{\text{norm}}+H(\mathbf{\widehat{\Pi}}_{\text{mix}} \lvert \mathbf{\Pi}_{\text{mix}})_{\text{norm}}\right].
\end{align*}

There are two aspects in evaluating the error.
The first aspect is the $l_1$ norm error. According to Equation~\eqref{eqn:nmi},  the error function used in NMI score is  $\frac{H\left(\widehat{\Pi}_{{\text{mix}}_i}  \lvert  \Pi_{{\text{mix}}_j} \right)}{H(\Pi_{{\text{mix}}_j} )}$.
%$\frac{H(\widehat{\Pi}_i\lvert{\Pi}_j)}{H(\Pi_j)}$.
NMI is not suitable for evaluating recovery of different sized communities. In the special case of a pair of  extremely sparse and dense membership vectors, depicted in Figure~\ref{fig:NMI}, ${H(\Pi_{\text{mix}_j})}$ is the same for both the dense and the sparse vectors since they are flipped versions of each other (0s flipped to 1s and vice versa).  However, the smaller sized community (i.e. the sparser community vector),  shown in red in Figure~\ref{fig:NMI}, is significantly more difficult to recover than the larger sized community shown in blue in Figure~\ref{fig:NMI}. Although this example is an extreme scenario that is not seen in practice, it justifies the drawbacks of the NMI. Thus, NMI is not suitable for evaluating recovery of different sized communities.
\begin{figure*}[h]
   \centering
   \psfrag{dense pi1}[r]{dense $\Pi_1$ }
   \psfrag{sparse pi2}[r]{sparse $\Pi_2$ }
   \psfrag{length n vector}[c]{length $n$ membership vector}
   \psfrag{zeros}[l]{$0$}
   \psfrag{ones}[l]{$1$}
   \psfrag{dense community}[l]{\small{large sized community}}
   \psfrag{sparse community}[l]{\small{small sized community}}
   \includegraphics[width=0.68\textwidth]{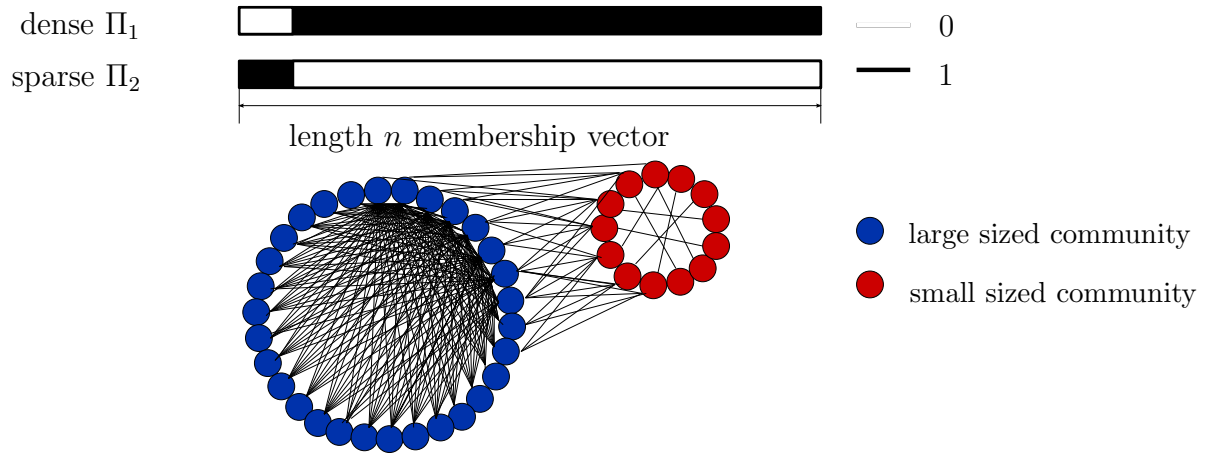}
   \caption{ A special case of a pair of extremely dense and sparse communities. Theoretically, the sparse community is more difficult to recover than the dense one. However, the NMI score penalizes both of them equally. Note that for dense $\Pi_1$, $P(\Pi_{\text{mix}_1}=0)=\frac{\text{\# of 0s in } \Pi_1}{n}$ which is equal to $P(\Pi_{\text{mix}_2}=1)=\frac{\text{\# of 1s in } \Pi_2}{n}$. Similarly, $P(\Pi_{\text{mix}_1}=1)=\frac{\text{\# of 1s in } \Pi_1}{n}$ which is equal to $P(\Pi_{\text{mix}_2}=0)=\frac{\text{\# of 0s in } \Pi_2}{n}$. Therefore, $H(\Pi_{\text{mix}_1})=H(\Pi_{\text{mix}_2})$.}
      \label{fig:NMI}
\end{figure*}
In contrast,   our error function employs a normalized $l_1$ norm error which penalizes more for larger sized communities than smaller ones.

The second aspect is the error induced by false pairings of estimated and ground-truth communities. NMI score selects only the closest estimated community through normalized conditional entropy minimization and it does not account for statistically significant dependence between an estimated community and multiple ground truth communities and vice-versa, and therefore it underestimates error.  However, our error score does not limit to a matching between the estimated and ground truth communities:  if an estimated community is found to have statistically significant correlation with multiple ground truth communities (as evaluated by the $p$-value), we penalize for the error over all such ground truth communities. Thus, our error score is a harsher measure of evaluation than NMI.
This notion of  ``soft-matching'' between ground-truth and estimated communities also enables validation of recovery of a combinatorial union of communities instead of single ones.

A number of other scores such as ``separability'', ``density'', ``cohesiveness'' and ``clustering coefficient''~\cite{yang2012defining} are non-statistical measures of faithful community recovery. % In contrast, we undertake a statistical approach for estimating the communities as well as their evaluation.
 The scores of~\cite{yang2012defining}  intrinsically aim to evaluate the level of clustering within a community.  However our goal is to measure the accuracy of recovery of the communities and not how well-clustered the communities are. %Hence, these scores are not directly relevant for our application.

Banerjee and Langford~\cite{banerjee2004objective} proposed an objective evaluation criterion for clustering which use classification performance as the evaluation measure. In contrast, we look at how well the method performs in recovering the hidden communities, and we are not evaluating predictive performance. Therefore, this measure is not used in our evaluation.

Finally, we note that cophenetic correlation is another statistical score used for evaluating clustering methods, but note that it is only valid for hierarchical clustering and it is a measure of how faithfully a dendrogram preserves the pairwise distances between the original unmodeled data points~\cite{sokal1962comparison}. Hence, it is not employed in this paper.

%\begin{figure*}
%  \centering
%  \begin{tabular}{l}
%    ``I am glad I was up so late,\\
%    \quad{}for that's the reason I was up so early.''\\
%    \em \footnotesize William Shakespeare (1564-1616), British
%    dramatist, poet.\\
%    \em \footnotesize Cloten, in Cymbeline, act 2, sc. 3, l. 33-4.
%  \end{tabular}
%  %\caption{A deep quote.}
%  %\label{fig:quote}
%\end{figure*}

\chapter{Appendix for Dictionary Learning via Convolutional Tensor Method}
\section{Cumulant Form}\label{appdx:momentForm}
 In~\cite{anandkumar2014tensor}, it is proved that in ICA model, the cumulant of observation $x$ is decomposed into multi-linear transform of a diagonal cumulant of $h$. Therefore, we aim to find the third order cumulant for input $x$. 
 
  As we know that the $r^\tha$ order moments for variable $x$ is defined as 
\begin{equation}
\mu_r : = \mathbb{E}[x^r]\in \mathbb{R}^{n\times n\times n}
\end{equation} 
Let us use $[\mu_3]_{i,j,k}$ to denote the $(i,j,k)^\tha$ entry of the third order moment. The relationship between $3^\tha$ order cumulant $\kappa_3$and $3^\tha$ order moment $\mu_3$is 
\begin{align}
[\kappa_3]_{i,j,k} = [\mu_3]_{i,j,k} -  [\mu_2]_{i,j}[\mu_1]_{k} -  [\mu_2]_{i,k}[\mu_1]_{j} -  [\mu_2]_{j,k}[\mu_1]_{i}+ 2 [\mu_1]_{i}[\mu_1]_{j}[\mu_1]_{k}
\end{align}

Therefore the shift tensor is in this format: 
We know that the shift term
\begin{align}
[Z]_{a,b,c} := \Ebb[x^i_{a}]\Ebb[x^i_{b}x^i_{c}] +  \Ebb[x_{b}]\Ebb[x_{a}x^i_{c}] +  \Ebb[x_{c}]\Ebb[x_{a}x_{b}] -2 \Ebb[x_{a}]\Ebb[x_{b}]\Ebb[x_{c}], \quad a,b,c \in [n]
 \end{align}

It is known from~\cite{anandkumar2014tensor} that cumulant decomposition in the 3 order tensor format is 
\begin{equation}
\Ebb[x\otimes x \otimes x] - Z = \sum_{j\in[nL]}\lambda_j^* \CCir_j^* \otimes \CCir_j^* \otimes \CCir_j^*
\end{equation}
%\fhcomment{Describe the ICA model: decomposition of 4-th order tensor is the tensor product of F and cumulant of (w). Therefore, its unfolded decomposition version is achieved. }

Therefore using the Khatri-Rao product property,
\begin{equation}
\flatten(\sum_{j\in[nL]}\lambda_j^* \CCir_j^* \otimes \CCir_j^* \otimes \CCir_j^*) = \sum_{j\in [nL]}\lambda_j^* \CToep^*_j (\CToep^*_j \odot\CToep^*_j)^\top ={\mathcal{F}^*} \Lambda^* \left({\mathcal{F}^*}\odot {\mathcal{F}^*}\right)^\top
\end{equation}

Therefore the unfolded third order cumulant is decomposed as 
$\Cum ={\mathcal{F}^*} \Lambda^* \left({\mathcal{F}^*}\odot {\mathcal{F}^*}\right)^\top$.

\section{Proof for Main Theorem~\ref{theorem:main}}\label{appdx:maintheorem}
{Our optimization problem is
\begin{equation}
\min\limits_{ {\modeA}} \quad  \lVert \Cum-  {\modeA} \Lambda \left({\modeC}\odot {\modeB}\right)^\top  \rVert^2_{F}\,\,
 \mbox{s.t. }   \block_l(\modeA) 
= U \cdot\Diag(\DFT(f_{l}))\cdot U\hermconj, \ \lVert f_{l} \rVert_2^2=1,  \forall  l\in[L],
\end{equation}
where we denote $D: = \Lambda \left({\modeC}\odot {\modeB}\right)^\top$ for simplicity.  Therefore the objective is to minimize $\lVert \Cum-  {\modeA} D  \rVert^2_{F}$. Let the SVD of $D$ be $D = P\Sigma Q^\top$. Since the Frobenius norm remains invariant under orthogonal transformations and full rank diagonal matrix~\cite{eberle2003finding}, it is obtained that 
\begin{equation}
\lVert \Cum - \modeA D\rVert_F^2 
=\lVert \Cum - \modeA P\Sigma Q^\top\rVert_F^2 
=\lVert \Cum Q\Sigma^\dag - \modeA P \rVert_F^2 
=\lVert \Cum Q\Sigma^\dag P^\top - \modeA \rVert_F^2  \end{equation}}

{Therefore the optimization problem in \eqref{eqn:modeAopt}
is equivalent to \begin{equation}\label{eqn:ls}
\min\limits_{ {\modeA}}    \lVert \Cum (\left({\modeC}\odot {\modeB}\right)^\top)^\dag \Lambda^\dag-  {\modeA}    \rVert^2_{F}\,\,
 \mbox{s.t. }   \block_l(\modeA) 
= U \cdot\Diag(\DFT(f_{l}))\cdot U\hermconj, \ \lVert f_{l} \rVert_2^2=1,  \forall  l\in[L]\eeq
when  $(\modeC\odot \modeB)$ and $\Lambda$ are full column rank.}

{The full rank condition requires $nL<n^2$ or $L <n$,  and it is a reasonable assumption since otherwise the filter estimates are redundant.  Since \eqref{eqn:ls} has block constraints, it can be broken down in to solving $L$ independent sub-problems
\begin{equation}\label{eqn:newCumopt}
  \min_{f_{l}}
\left\lVert
 \block_l(\newCum)\cdot \block_l(\Lambda)^\dag - U \cdot\Diag(\DFT(f_l))\cdot U\hermconj
  \right\rVert^2_{F} \\
\quad  s.t. \quad \lVert f_l \rVert_2^2=1, \forall l\in[L].
\end{equation}}

\section{Parallel Inversion of $\mathbf{\Psi}$}
\label{appdx:parallelBlockInversion}
We propose an efficient iterative algorithm to compute $\mathbf{\Psi}^\dag$ via block matrix inversion theorem\cite{golub2012matrix}.

\begin{lemma}(Parallel Inversion of row and column stacked diagonal matrix)\label{lemma:matrix_inverse}
	Let $ J^L = \mathbf{\Psi} $ be partitioned into a block form:
	\begin{equation}
	J^L = \left[
	\begin{array}{cc}
	J^{L-1} & O \\
	R & \block_L^L(\mathbf{\Psi}) \\
	\end{array}
	\right],
	\end{equation}
	where  $O : =  \left[
	\begin{array}{c}
	\block_L^1(\mathbf{\Psi}) \\
	 \vdots\\
	 \block_L^{L-1}(\mathbf{\Psi})
	\end{array}
	\right]$, and $R : = \left[ \block_{L-1}^1(\mathbf{\Psi}),\ldots,\block_{L-1}^{L}(\mathbf{\Psi})\right]$. After inverting $\block_L^L(\mathbf{\Psi})$ which takes $O(1)$ time using $O(n)$ processors, there inverse of  $\mathbf{\Psi}$ is achieved by 
	\begin{equation}\label{eq:recursiveInverseBlock}
	\mathbf{\Psi}^\dag = \left[
	\begin{array}{ll}
	(J^{L-1}-O{\block_L^L(\mathbf{\Psi})}^{-1}R)^{-1} & -{(J^{L-1})}^{-1}O({\block_L^L(\mathbf{\Psi})}-R{(J^{L-1})}^{-1}O)^{-1} \\
	-{\block_L^L(\mathbf{\Psi})}^{-1}R(J^{L-1}-O{\block_L^L(\mathbf{\Psi})}^{-1}R)^{-1} & ({\block_L^L(\mathbf{\Psi})}-R{(J^{L-1})}^{-1}O)^{-1}
	\end{array}
	\right]
	\end{equation}
	 assuming that $J^{L-1}$ and $\block_L^{L}\mathbf{\Psi}$ are invertible. 
	 \end{lemma}
	 This again requires inverting $R$, $O$ and $J^{L-1}$. Recursively applying these block matrix inversion theorem, the inversion problem is reduced to  inverting $L^2$ number of $n$ by $n$ diagonal matrices with additional matrix multiplications as indicated in equation~\eqref{eq:recursiveInverseBlock}.

Inverting a diagonal matrix results in another diagonal one, and the complexity of inverting $n\times n$ diagonal matrix is $O(1)$ with $O(n)$ processors. We can simultaneous invert all blocks. Therefore with $O(nL^2)$ processors, we invert all the diagonal matrices in $O(1)$ time. The recursion takes $L$ steps, for step $i \in[L]$ matrix multiplication  cost is O($\log nL$) with $O(n^2 L/\log (nL))$ processors. With $L$ iteration, one achieves $O(\log n + \log L)$ running time with $O(n^2L^2/(\log L + \log n))$ processors.

\chapter{Appendix for Latent Tree Learning via Hierarchical Tensor Method}
\section{Additivity of the Multivariate Information Distance}
\label{apdx:additive}
Recall that the additive information distance between nodes two categorical variables $x_i$ and $x_j$ was defined in \cite{choi2011learning}. We extend the notation of information distance to high dimensional variables via Definition~\ref{def:info_dist} and present the proof of its additivity in Lemma~\ref{lem:additive} here.
\begin{proof}
\[
\mathbb{E}[x_a x_c^\top]= \mathbb{E}[\mathbb{E}[x_a x_c^\top | x_b]] = A \mathbb{E}[x_b x_b^\top] B^\top
\]
%Here $a -> b -> c$ is the graph, and $A = P(x_a | x_b)$ and $B=P(x_c|x_b)$.
Consider three nodes $a, b, c$ such that there are edges between $a$ and $b$, and $b$ and $c$. Let the $A = \mathbb{E}(x_a | x_b)$ and $B=\mathbb{E}(x_c|x_b)$. From Definition~\ref{def:info_dist}, we have, assuming that $\mathbb{E}(x_a x_a^\top)$, $\mathbb{E}(x_b x_b^\top)$ and $\mathbb{E}(x_c x_c^\top)$ are full rank.
%fhcomment{Is this true in our yelp case? If this is true, we are in trouble......}
\begin{align*}
\dist(v_a,v_c) & = -\log \frac{  \prod\limits_{i=1}^{k}\sigma_i(\mathbb{E}(x_a x_c^\top))    }{    \sqrt{ \det (\mathbb{E}(x_a x_a^\top ) )  \det( \mathbb{E}(x_cx_c^\top) )}  }\\
e^{-\dist(v_a,v_c)} & = \det\left(  \mathbb{E}(x_ax_a^\top)^{-1/2} U^\top \mathbb{E}(x_a x_c^\top)  V \mathbb{E}(x_c x_c^\top)^{-1/2}\right)
\end{align*}
%where $k-svd(\mathbb{E}(x_ax_c^\top))= U \Sigma V^\top$.
where $k$-SVD$((\mathbb{E}(x_ax_c^\top))= U \Sigma V^\top)$.
Similarly,
\begin{align*}
e^{-\dist(v_a,v_b)} & = \det\left(  \mathbb{E}(x_ax_a^\top)^{-1/2} U^\top \mathbb{E}(x_a x_b^\top)  W \mathbb{E}(x_b x_b^\top)^{-1/2}\right)\\
e^{-\dist(v_b,v_c)} & = \det\left(  \mathbb{E}(x_bx_b^\top)^{-1/2} W^\top \mathbb{E}(x_b x_c^\top)  V \mathbb{E}(x_c x_c^\top)^{-1/2}\right)
\end{align*} where $k$-SVD$((\mathbb{E}(x_ax_b^\top))= U \Sigma W^\top)$ and $k$-SVD$((\mathbb{E}(x_bx_c^\top))= W \Sigma V^\top)$.

Therefore,
\begin{align*}
e^{-(\dist(a,b)+\dist(b,c))} & = \det(  \mathbb{E}(x_ax_a^\top)^{-1/2}  U^\top \mathbb{E}(x_a x_b^\top)   \mathbb{E}(x_b x_b^\top)^{-1/2 -1/2} \mathbb{E}(x_b x_c^\top) V   \mathbb{E}(x_c x_c^\top)^{-1/2}     )\\
& = \det (\mathbb{E}(x_a x_a^\top)^{-1/2}  U^\top A \mathbb{E}(x_b x_b^\top) B^\top V \mathbb{E}(x_c x_c^\top)^{-1/2} ) = e^{-\dist(v_a,v_c)} 
\end{align*}
%Note that $\mathbb{E}(x_a x_b^\top)= A \mathbb{E}(x_b x_b^\top)$.
We conclude that the multivariate information distance is additive. Note that $\mathbb{E}\left[x_a x_b^\top\right] = \mathbb{E}\left(\mathbb{E}\left( x_a x_b^\top \lvert x_b\right)\right)=\mathbb{E}\left(A x_b x_b^\top\right)= A \mathbb{E}(x_b x_b^\top)$.
%\fhcomment{need to add the discuss when  $\mathbb{E}(x_a x_a^\top)$ is not full rank.}
\end{proof}
We note that when the second moments are not full rank, the above distance can be extended as follows:
\[
\dist(v_a,v_c) = -\log \frac{  \prod\limits_{i=1}^{k}\sigma_i(\mathbb{E}(x_a x_c^\top))    }{    \sqrt{  \prod\limits_{i=1}^{k}\sigma_i(\mathbb{E}(x_a x_a^\top))    \prod\limits_{i=1}^{k}\sigma_i(\mathbb{E}(x_c x_c^\top))     }  }.
\]

%%%%%%%%%%%%%%%%%%%%%
\section{Local Recursive Grouping}
\label{apdx:plrg}
The Local Recursive Grouping (LRG) algorithm is a local divide and conquer procedure for learning the structure and parameter of the latent tree (Algorithm~\ref{algo:plrg}).
We perform recursive grouping simultaneously on the sub-trees of the MST. 
Each of the sub-tree consists of an internal node and its neighborhood nodes. 
We keep track of the internal nodes of the MST, and their neighbors. 
The resultant latent sub-trees after LRG can be merged easily to recover the final latent tree. 
Consider a pair of neighboring sub-trees in the MST. 
They have two common nodes (the internal nodes) which are neighbors on MST. 
Firstly we identify the path from one internal node to the other in the trees to be merged, then compute the multivariate information distances between the internal nodes and the introduced hidden nodes. 
We recover the path between the two internal nodes in the merged tree by inserting the hidden nodes closely to their surrogate node. Secondly, we merge all the leaves which are not in this path by attaching them to their parent. 
Hence, the recursive grouping can be done in parallel and we can recover the latent tree structure via this merging method.

\begin{lemma}\label{lem:surrogate}
If an observable node $v_j$ is the surrogate node of a hidden node $h_i$, then the hidden node $h_i$ can be discovered using $v_j$ and the neighbors of $v_j$ in the MST.
\end{lemma}  
This is due to the additive property of the multivariate information distance on the tree and the definition of a surrogate node. 
This observation is crucial for a completely local and parallel structure and parameter estimation.  It is also easy to see that all internal nodes in the MST are surrogate nodes. 

After the parallel construction of the MST, we look at all the internal nodes $\mathcal{X}_\text{int}$. For $v_i \in \mathcal{X}_\text{int}$, we denote the neighborhood of $v_i$ on MST as $\Nb\text{sub}(v_i;\text{MST})$ which is a small sub-tree. Note that the number of such sub-trees is equal to the number of internal nodes in MST.

For any pair of sub-trees, $\Nb_{\text{sub}}(v_i;\text{MST})$ and $\Nb_{\text{sub}}(v_j;\text{MST})$, there are two topological relationships, namely overlapping (i.e., when the sub-trees share at least one node in common) and non-overlapping (i.e., when the sub-trees do not share any nodes).

Since we define a neighborhood centered at $v_i$ as only its immediate neighbors and itself on MST, the overlapping neighborhood pair $\Nb_{\text{sub}}(v_i;\text{MST})$ and $\Nb_{\text{sub}}(v_j;\text{MST})$ can only have conflicting paths, namely path$(v_i,v_j; \Adj_i)$ and path$(v_i,v_j;\Adj_j)$, if $v_i$ and $v_j$ are neighbors in MST.

With this in mind, we locally estimate all the latent sub-trees, denoted as $\Adj_i$, by applying Recursive Grouping~\cite{choi2011learning} in a parallel manner on $\Nb\text{sub}(v_i;\text{MST}) ,\ \forall v_i \in \mathcal{X}_\text{int}$. Note that the latent nodes automatically introduced by $\RG(v_i)$ have $v_i$ as their surrogate.  We update the tree structure by joining each level in a bottom-up manner. The testing of the relationship among nodes~\cite{choi2011learning} uses the additive multivariate information distance metric (Appendix~\ref{apdx:additive}) $\Phi(v_i,v_j;k)=\dist(v_i,v_k)-\dist(v_i,v_k)$ to decide whether the nodes $v_i$ and $v_j$ are parent-child or siblings. If they are siblings, they should be joined by a hidden parent. If they are parent and child, the child node is placed as a lower level node and we add the other node as the single parent node, which is then joined in the next level.

Finally, for each internal edge of MST connecting two internal nodes $v_i$ and $v_j$, we consider merging the latent sub-trees. In the example of two local estimated latent sub-trees in Figure~\ref{Fig:StructureLearning}, we illustrate the complete local merging algorithm that we propose.  

\section{Proof Sketch for Theorem~\ref{theorem:main_LT}}\label{appen:guarantee}
We argue for the correctness of the method under exact moments. The sample complexity follows from the previous works.
In order to clarify the proof ideas, we define the notion of \emph{surrogate node}~\cite{choi2011learning} as follows.
 \begin{definition}
 \label{def:surrogate}
 Surrogate node for hidden node $h_i$ on the latent tree $\mathcal{T}=(\mathcal{V},\mathcal{E})$ is defined as
 $\text{Sg}(h_i;\mathcal{T}):= \arg \min\limits_{v_j\in \mathcal{X}} \dist(v_i,v_j)$.
 \end{definition}

 In other words, the surrogate for a hidden node is an observable node which has the minimum multivariate information distance from the hidden node. See Figure~\ref{Fig:StructureLearning}(a), the surrogate node of $h_1$, $\text{Sg}(h_1;\mathcal{T})$, is $v_3$, $\text{Sg}(h_2;\mathcal{T}) = \text{Sg}(h_3;\mathcal{T}) = v_5$. Note that the notion of the  surrogate node is only required for analysis, and our algorithm does not need to know this information.  
 
The notion of surrogacy allows us to relate the constructed MST (over observed nodes) with the underlying latent tree. It can be easily shown that contracting the hidden nodes to their surrogates on latent tree leads to MST. Local recursive grouping procedure can be viewed as reversing these contractions, and hence, we obtain consistent local sub-trees. 

We now argue the correctness of the structure union procedure, which merges the local sub-trees. In each reconstructed sub-tree $\mathcal{N}_i$, where $v_i$ is the group leader, the discovered hidden nodes $\{h^i\}$ form a  surrogate relationship with $v_i$, i.e. $\text{Sg}(h^i;\tree) = v_i$. Our merging approach maintains these surrogate relationships. For example in Figure~\ref{Fig:StructureLearning}(d1,d2), we have the path $v_3-h_1-v_5$ in $\Adj_3$ and path $v_3-h_3-h_2-v_5$ in $\Adj_5$.
The resulting path is $v_3-h_1-h_3-h_2-v_5$, as seen in Figure~\ref{Fig:StructureLearning}(e). We now argue why this is correct. As discussed before, $\text{Sg}(h_1;\tree)=v_3$ and $\text{Sg}(h_2;\tree)=\text{Sg}(h_3;\tree)=v_5$. When we merge the two subtrees, we want to preserve the paths from the group leaders to the added hidden nodes, and this ensures that the surrogate relationships are preserved in the resulting merged tree. Thus, we obtain a global consistent tree structure by merging the local structures. The correctness of parameter learning comes from the consistency of the tensor decomposition techniques and careful alignments of the hidden labels across different decompositions.
Refer to Appendix~\ref{appen:correctness},~\ref{appen:samplecomp} for proof details and the sample complexity.

 \section{Proof of Correctness for LRG}\label{appen:correctness}

\begin{definition}
A latent tree $\mathcal{T}_{\ge3}$ is defined to be a minimal (or identifiable) latent tree if it satisfies that each latent variable has at least 3 neighbors.
\end{definition}

\begin{definition}
Surrogate node for hidden node $h_i$ in  latent tree $\mathcal{T}=(\mathcal{V},\mathcal{E})$ is defined as
\[
\text{Sg}(h_i;\mathcal{T}):= \arg \min\limits_{v_j\in \mathcal{X}} \dist(v_i,v_j).
\]
\end{definition}

There are some useful observations about the MST in~\cite{choi2011learning} which we recall here. 
\begin{property}[MST $-$ surrogate neighborhood preservation]\label{prop:MST1}
The surrogate nodes of any two neighboring nodes in $\mathcal{E}$ are also neighbors in the MST. I.e.,
\[
(h_i,h_j)\in \mathcal{E} \Rightarrow (\text{Sg}(h_i),\text{Sg}(h_j)) \in \text{MST}.
\]
\end{property}

\begin{property}[MST $-$ surrogate consistency along path]\label{prop:MST2}
If $v_j\in\mathcal{X}$ and $v_h\in \text{Sg}^{-1}(v_j)$, then every node along the path connecting $v_j$ and $v_h$ belongs to the inverse surrogate set $\text{Sg}^{-1}(v_j)$, i.e.,
\[
v_i \in \text{Sg}^{-1}(v_j) ,\ \forall v_i\in \text{Path}(v_j,v_h)
\]
if 
\[
v_h\in \text{Sg}^{-1}(v_j).
\]
\end{property}
 
 The MST properties observed connect the MST over observable nodes with the original latent tree $\mathcal{T}$.  We obtain MST by contracting all the latent nodes to its surrogate node.

 Given that the correctness of CLRG algorithm is proved in~\cite{choi2011learning}, we prove the equivalence between the CLRG and PLRG.
 
 \begin{lemma}
 For any sub-tree pairs $\Nb[v_i;\text{MST}]$ and $\Nb[v_i;\text{MST}]$, there is at most one overlapping edge. The overlapping edge exists if and only if $v_i \in \Nb(v_j;\text{MST})$.
 \end{lemma}
 
 This is easy to see. 
 
 \begin{lemma}
 Denote the latent tree recovered from $\Nb[v_i;\text{MST}]$ as $\Adj_i$ and similarly for $\Nb[v_j;\text{MST}]$.  
 The inconsistency, if any, between $\Adj_i$  and  $\Adj_j$ occurs in the overlapping path$(v_i,v_j;\Adj_i)$ in and  path$(v_i,v_j;\Adj_j)$ after LRG implementation on each subtrees. 
 \end{lemma}
 
 We now prove the correctness of LRG. Let us denote the latent tree resulting from merging a subset of small latent trees as $T_{\text{LRG}}(S)$, where $S$ is the set of center of subtrees that are merged pair-wisely. CLRG algorithm in~\cite{choi2011learning} implements the RG in a serial manner. Let us denote the latent tree learned at iteration $i$ from CLRG is $T_{\text{CLRG}}(S)$, where $S$ is the set of internal nodes visited by CLRG at current iteration . We prove the correctness of LRG by induction on the iterations.
 
 At the initial step $S = \emptyset$:  $T_{\text{CLRG}}= MST$ and $T_{\text{LRG}}=MST$, thus $T_{\text{CLRG}}=T_{\text{LRG}}$.
 
 Now we assume that for the same set $S_{i-1}$, $T_{\text{CLRG}}=T_{\text{LRG}}$ is true for $r=1,\ldots,i-1$. 
 At iteration $r=i$ where CLRG employs RG on the immediate neighborhood of node $v_i$ on $T_{\text{CLRG}}(S_{i-1})$, let us assume that $H_i$ is the set of hidden nodes who are immediate neighbors of $i-1$.   The CLRG algorithm thus considers all the neighbors and implements the RG. We know that the surrogate nodes of every latent node in $H_i$ belong to previously visited nodes $S_{i-1}$. According to Property~\ref{prop:MST1} and~\ref{prop:MST2}, if we contract all the hidden node neighbors to their surrogate nodes, CLRG thus is a RG on neighborhood of $i$ on MST.  
 
 As for our LRG algorithm at this step, $T_{\text{LRG}}(S_i)$ is the merging between $T_{\text{LRG}}(S_{i-1})$and $\Adj_i$.   %(Note that the merging is parallel, but in order to prove the correctness, we only look at parts of the merging. )
 The latent nodes whose surrogate node is $j$ are introduced between the edge $(i-1,i)$.   Now that we know $\Adj_i$ is the RG output from immediate neighborhood of $i$ on MST.
 Therefore, we proved that $T_{\text{CLRG}}(S_i)= T_{\text{LRG}}(S_i)$.

\section{Cross Group Alignment Correction}\label{appen:alignment}
In order to achieve cross group alignments,  tensor decompositions on two cross group triplets have to be computed. The first triplet is formed by three nodes: reference node in group 1, $x_1$, non-reference node in group 1, $x_2$, and reference node in group 2, $x_3$.  The second triplet is formed by three nodes as well: reference node in group 2, $x_3$, non-reference node in group 2, $x_4$ and reference node in group 1, $x_1$. Let us use $h_1$ to denote the parent node in group 1, and $h_2$ the parent node in group 2. 
 
 From $\text{Trip}(x_1,x_2,x_3)$%glob_ref, ref_mem, local_ref>$ , we know
, we obtain $P(h_1|x_1) = \tilde{A}$,
%\[
%P(x_1| h_1) ==>P(h_1|x_1) = \tilde{A}
%\]
$P(x_2|h_1) = B$ and 
$P(x_3|h_1)$ $ = $ $P(x_3|h_2) P(h_2|h_1)$ $ =$ $ DE$.
From $\text{Trip}(x_3,x_4,x_1)$, we know
%<local_ref,local_mem,  glob_ref>$, we know
$P(x_3|h_2) = D\Pi$, $P(x_4|h_2) = C\Pi$ and $P(h_2|x_1) = P(h_2|h_1)P(h_1|x_1) = \Pi E\tilde{A}$, where $\Pi$ is a permutation matrix.
%\[P(x_1|h_2)==>P(h_2|x_1)=P(h_2|h_1)P(h_1|x_1) = \Pi E\tilde{A}.\]
We compute $\Pi$ as $\Pi = \sqrt{ (\Pi E\tilde{A})  (\tilde{A})^{\dag}  (DE)^{\dag} (D\Pi)}$ so that $D =(D\Pi) \Pi^{\dag}$ is aligned with group 1. Thus, when all the parameters in the two groups are aligned by permute group 2 parameters using $\Pi$, thus the alignment is completed.%, these two families are aligned.

 Similarly, the alignment correction can be done by calculating the permutation matrices while merging different threads.
 
 Overall, we merge the local structures and align the parameters from LRG locla sub-trees using Procedure~\ref{algo:pmac} and~\ref{algo:alignment}.

\section{Computational Complexity}\label{appen:compuComplex}
We recall some notations here: $d$ is the observable node dimension, $k$ is the hidden node dimension ($k \ll d$), $N$ is the number of samples, $p$ is the number of observable nodes, and $z$ is the number of non-zero elements in each sample.

Multivariate information distance estimation involves sparse matrix multiplications to compute the pairwise second moments. Each observable node has a $d \times N$ sample matrix with $z$ non-zeros per column. Computing the product $x_1 x_2^T$ from a single sample for nodes $1$ and $2$ requires $O(z)$ time and there are $N$ such sample pair products leading to $O(Nz)$ time. There are $O(p^2)$ node pairs and hence the degree of parallelism is $O(p^2)$. Next, we perform the $k$-rank SVD of each of these matrices. Each SVD takes $O(d^2 k)$ time using classical methods. Using randomized methods~\cite{gittens2013revisiting}, this can be improved to $O(d+k^3)$. %We stress that usually either $p$ is large or $d$ is large, not both. 

Next on, we construct the MST in $O(\log p)$ time per worker with $p^2$ workers. The structure learning can be done in $O(\Gamma^3)$ per sub-tree and the local neighborhood of each node can be processed completely in parallel. We assume that the group sizes $\Gamma$ are constant (the sizes are determined by the degree of nodes in the latent tree and homogeneity of parameters across different edges of the tree. The parameter estimation of each triplet of nodes consists of implicit stochastic updates involving products of $k \times k$ and $d \times k$ matrices. Note that we do not need to consider all possible triplets in groups but each node must be take care by a triplet and hence there are $O(p)$ triplets. This leads to a factor of $O(\Gamma k^3 + \Gamma d k^2)$ time per worker with $p/\Gamma$ degree of parallelism.

At last, the merging step consists of products of $k \times k$ and $d \times k$ matrices for each edge in the latent tree leading to $O(d k^2)$ time per worker with $p/\Gamma$ degree of parallelism.

\section{Sample Complexity}\label{appen:samplecomp}

From~\cite{anandkumar2011spectral}, we recall the number of samples required for the recovery of the tree structure that is consistent with the ground truth (for a precise definition of consistency, refer to Definition 2 of \cite{choi2011learning}).%faithful recovery of the latent tree structure.

\begin{lemma}
If
\begin{equation}
N > \frac{    200 k^2 B^2 t       }       {   \left(  \frac{ \gamma_{\min}^2	}  {\gamma_{\max}}       (1-\dist_{\max}) \right)^2} + \frac{7 kM^2 t} {           \frac{ \gamma_{\min}^2	}  {\gamma_{\max}}     (1-\dist_{\max})     },
%N > \frac{}{\frac{\gamma_{\text{min}}^2}{\gamma_{\text{max}}(1-\rou_{\text{max}})}}
\end{equation}
then with probability at least $1-\eta$, proposed algorithm returns $\widehat{\mathcal{T}} =\mathcal{T}$, where
$$B:=\max_{x_i,x_j\in \mathcal{X}} \left\{           \sqrt{  \max\{ \lVert \mathbb{E} [\lVert x_i\rVert^2 x_jx_j^\top]\rVert \},  \max\{ \lVert \mathbb{E} [\lVert x_j\rVert^2 x_ix_i^\top]\rVert \}          }        \right\},$$
$$M:= \max_{x_i\in \mathcal{X}} \left\{     \lVert x_i\rVert   \right\},$$
$$t:= \max_{x_i,x_j\in \mathcal{X}} \left\{              	4\ln(4                       	\frac{\mathbb{E}[\lVert x_i\rVert^2 \lVert x_j\rVert^2] -\Tr(\mathbb{E}[x_ix_j^\top]\mathbb{E}[x_jx_i^\top])}{\max\{  \lVert \mathbb{E}[\lVert x_j\rVert^2x_ix_i^\top]\rVert, \lVert \mathbb{E}[\lVert x_i\rVert^2x_jx_j^\top]\rVert \}}			  n/\eta)					   \right\}.$$
$$\gamma_{\min}:= \min\limits_{\{x_1,x_2\}} {\{\sigma\left(\mathbb{E}[x_1x_2^\top] \right)\}}
$$
$$\gamma_{\max}:= \max\limits_{\{x_1,x_2\}} {\{\sigma\left(\mathbb{E}[x_1x_2^\top] \right)\}}
$$
\end{lemma}

From~\cite{anandkumar2012two}, we recall the sample complexity for the faithful recovery of parameters via tensor decomposition methods.

We define $\epsilon_P$ to be the noise raised between empirical estimation of the second order moments and exact second order moments, and $\epsilon_T$ to be the noise raised between empirical estimation of the third order moments and the exact third order moments.
\begin{lemma}
Consider positive constants $C$, $C'$, $c$ and $c'$, the following holds. If
\begin{align*}
 \epsilon_P &\le c \frac{\frac{\lambda_k}{\lambda_1}}{k}, \quad \quad \epsilon_T\le c' \frac{\lambda_k \sigma_k^{3/2}}{k}\\
 N &\ge C\left(\log(k) + \log\left(\log\left(	\frac{\lambda_1 \sigma_k^{3/2}}{\epsilon_T}   +\frac{1}{\epsilon_P}		\right)\right)\right)\\
 L&\ge \poly(k)\log(1/\delta),
 \end{align*} then with probability at least $1-\delta$, tensor decomposition returns $(\widehat{v_i},\lambda_i): i\in[k]$ satisfying, after appropriate reordering,

 \begin{align*}
 \lVert \widehat{v_i} - v_i\rVert_2 &\le C' \left( \frac{1}{\lambda_i} \frac{1}{\sigma_k^2} \epsilon^T  + \left(\frac{\lambda_1}{\lambda_i}\frac{1}{\sqrt{\sigma_k}}+1\right)\epsilon_P		\right)\\
 \lvert \widehat{\lambda_i} - \lambda_i\rvert &\le C' \left(\frac{1}{\sigma_k^{3/2}}\epsilon_T+\lambda_1\epsilon_P\right)
 \end{align*}
for all $i\in[k]$.
\end{lemma}
We note that $\sigma_1\ge \sigma_2\ge \ldots \sigma_k >0$ are the non-zero singular values of the second order moments, %$\mathbf{P}$. 
$\lambda_1\ge\lambda_2\ge \ldots \ge \lambda_k>0$ are the ground-truth eigenvalues of the third order moments, and $v_i$ are the corresponding eigenvectors for all $i\in[k]$.

\section{Efficient SVD Using Sparsity and Dimensionality Reduction}\label{apdx:svd}
Without loss of generality, we assume that a matrix whose SVD we aim to compute has no row or column which is fully zeros, since, if it does have zero entries, such row and columns can be dropped.

Let $A \in \mathbb{R}^{n\times n}$ be the  matrix to do SVD. Let $\Phi \in R^{d \times \tilde{k}}$, where $\tilde{k} = \alpha k$ with $\alpha$ is a scalar, usually, in the range $[2,3]$. For the $i^{th}$ row of $\Phi$, if $\sum_i |\Phi|(i,:)\neq 0$ and $\sum_i |\Phi|(:,i)\neq 0$, then there is only one non-zero entry and that entry is uniformly chosen from $[\tilde{k}]$. If either $\sum_i |\Phi|(i,:)= 0$ or $\sum_i |\Phi|(:,i)=0$, we leave that row blank. Let $D\in \mathbb{R}^{d \times d}$ be a diagonal matrix with iid Rademacher entries, i.e., each non-zero entry is $1$ or $-1$ with probability $\frac{1}{2}$. Now, our embedding matrix~\cite{clarkson2013low} is $S = D \Phi $, i.e., we find $AS$ and then proceed with the Nystrom~\cite{DBLP:journals/corr/HuangNHVA13} method. Unlike the usual Nystrom method~\cite{DBLP:journals/corr/abs-1303-1849} which uses a random matrix for computing the embedding, we improve upon this by using a sparse matrix for the embedding since the sparsity improves the running time and the memory requirements of the algorithm.

\chapter{Appendix for Spatial Point Process Mixture model Learning}
\section{Morphological Basis Extraction}\label{sec:cellExtraction}
We aim to characterize the morphological basis for all cells with different size, orientation, expression profiles and spatial distribution. The traditional sparse coding introduces too many free parameters and is not suitable for compact morphological basis learning. We instead propose Gaussian prior convolutional  sparse coding (GPCSC).  The intuition for using convolution is due to the frequent replication of cells of similar shapes and the translation invariance property. 
%, see figure~\ref{fig:diagram} for an example. 
Traditional sparse coding would learn both the shape of the cell and the location of the cell. But the convolutional sparse coding would only learn the shape here. We  characterize cell spatial distribution via decoding the sparse activation map. 

To formulate the problem formally: let $\image$ be the image observed, then the convolutional sparse coding model generates observed image $\image$ using filters (resembling cell shapes)$\filter$ superposed at locations indicated by the activation map $\map$ (whose sparsity pattern indicates cell spatial distribution and activation amplitude indicates gene expression profiles. )

%\begin{figure}[htbp]
%   \centering
%   \psfrag{:cell type I}[Bl]{:cell type I}
%   \psfrag{:cell type II}[Bl]{:cell type II}
%   \includegraphics[width=0.4\textwidth]{figures/convolutional_sparse_coding.eps} % requires the graphicx package
%   \caption{Fixed Location Convolutional Gaussian Fitting. }
%   \label{fig:diagram}
%\end{figure}
Our goals of segmenting cells, extracting cell basis, and estimating gene profiles and cell locations are reduced to this optimization learning problem: 
\begin{align}\label{eq:cSparseCode}
&\min\limits_{\filter_m, \map_m^{n}}   \left\| \sum\limits_{n} \image^{n} - \sum\limits_{m=1}^{k}  \filter_m \star \map_m^{n} \right\|_\fro^2 + \sum\limits_{n} \sum\limits_{m}\lambda \left\| \map_m^{n}\right\|_0, \nonumber \\
&\text{s.t.  } \filter_m(x,y)\ge 0,  \left\|\filter_m\right\|_\filter^2 = 1, \map_m^{(n)}(x,y)\ge 0. 
\end{align}
where $\image^{n}$ is the $n^\tha$ image associated with the gene we are interested in with $D_x\times D_y$ pixels, i.e., $\image^{n}\in \mathbb{R}^{D\times D}$. 
		
We call the $F_m \in \mathbb{R}^{d\times d}$ filter, where $d$ is set to capture the local cell morphological information. The spatial coefficient for image $\image^n$ is denoted as $H_m^{(n)}\in \mathbb{R}^{(D-d+1) \times (D-d+1)}$ which represents the position of the filter $F_m $ being active on  image $\image^n$. More precisely, if $H_m^{n}(x,y) = 1$, then $F_m$ is active at $\image^{n}(x:x+d-1, y:y+d-1) $.

\subsection{Gaussian Prior Convolutional Sparse Coding}
The popular alternating approach between matching pursuit to learn activation map $\map$ and k-SVD to learn $\filter$  is general applicable to any object detection problem in image processing. However, this approach causes inexact cell number estimation as filters with multi-modality (i.e., multiple cells) are learnt.  We resolve this issue by proposing an Gaussian probability density function prior on the filters to guarantee single cell detection and achieve accurate cell number estimation. The support of $\map$ is also limited to the local maxima indicating cell centers. Note that our cell are not donut shaped, and it is reasonable to assume the darkest point being the cell center. 

Therefore, we optimize over the  objective 
$\min \left\| \sum_{n} \image^{n} - \sum_{m}  \filter_m\star \map_m^{n}\right\|_2^2+ \sum_{n} \sum_{m}\lambda \left\| \map_m^{n}\right\|_0$ such that $\filter_m$ are $2-D$ Gaussian densities with priori set top 2 principal radius and orientation. 
Alternating Minimization is used to solving the optimization problem. If we define the residual as $\sum_n\image^n -  \sum_n\sum_{m} \widehat{\filter}_m\star \widehat{\map}_m^{n}$, the gradient of the objective reduced to an iterative approach of updating filters, compute residual, optimizing activation map based on residual, compute residual and updating filters again. 
	It is easy to see that both $\frac{\partial L}{\partial F_m}(i,j)$ and $\frac{\partial L}{\partial H_m}(i,j)$ are convolution of the residual and the other variable rotated by angle $\pi$.

\subsection{Image Registration/Alignment}
A structure represents a neuronanatomical region of interest. Structures are grouped into ontologies and organized in a hierarchy or structure graph. 
We are interested in the somatosensory cortex area. So we use the affine transform from Allen Brain Institute~\cite{AMBA, lein2007genome} to align all the in-situ hybridization images with the Atlas brain to extract the correct region. 

\end{document}